\documentclass{article} % For LaTeX2e
\usepackage{iclr2027_conference,times}
\usepackage{comment}

\usepackage{amsmath}
\let\amsmathEqref\eqref
\usepackage{amsmath,amsfonts,bm}

\def\eqref#1{equation~\ref{#1}}
\def\1{\bm{1}}

\DeclareMathAlphabet{\mathsfit}{\encodingdefault}{\sfdefault}{m}{sl}
\SetMathAlphabet{\mathsfit}{bold}{\encodingdefault}{\sfdefault}{bx}{n}

\newcommand{\E}{\mathbb{E}}

\newcommand{\Var}{\mathrm{Var}}

\let\eqref\amsmathEqref  % 
\let\E\undefined         % 
\let\Var\undefined       % 

\usepackage[utf8]{inputenc} % allow utf-8 input
\usepackage[T1]{fontenc}    % use 8-bit T1 fonts
\usepackage{url}            % simple URL typesetting
\usepackage{booktabs}       % professional-quality tables
\usepackage{amsfonts}       % blackboard math symbols
\usepackage{nicefrac}       % compact symbols for 1/2, etc.
\usepackage{microtype}      % microtypography
\usepackage{xcolor}         % colors

\usepackage{algorithm}
\usepackage{algorithmic}

\usepackage{wrapfig}

\usepackage{amsmath}
\usepackage{amssymb}
\usepackage{mathtools}
\usepackage{amsthm}
\IfFileExists{bbm.sty}{\usepackage{bbm}}{}
\usepackage{amsfonts}
\usepackage{mathrsfs}

\theoremstyle{plain}
\newtheorem{theorem}{Theorem}[section]
\newtheorem{proposition}[theorem]{Proposition}
\newtheorem{lemma}[theorem]{Lemma}
\newtheorem{corollary}[theorem]{Corollary}
\theoremstyle{definition}
\newtheorem{definition}[theorem]{Definition}
\newtheorem{assumption}[theorem]{Assumption}
\theoremstyle{remark}
\newtheorem{remark}[theorem]{Remark}

\definecolor{midblue}{rgb}{0.0, 0.2, 0.6}

\usepackage[colorlinks=true,linkcolor=black,citecolor=midblue,urlcolor=blue,backref=page]{hyperref}

\usepackage{my_symbol}
\usepackage[most]{tcolorbox}
\usepackage{lipsum}
\usepackage{subcaption}

\newtcolorbox{setup}{
  colframe=cyan!30,
  colback=cyan!5,
  coltitle=black,
  fonttitle=\bfseries,
  title=Setup
}

\DeclareMathOperator{\Var}{Var}

\usepackage{multirow}
\usepackage{cleveref}

\title{Attention Graphons: A Graph Limit Perspective on Graph Transformers}

\author{%
  Caio F. Deberaldini Netto\\
  Johns Hopkins University\\
  Baltimore, MD %21218
  \\
  \texttt{cnetto1@jh.edu}\\
  \And
  Moshe Eliasof\\
  University of Cambridge\\
  Cambridge, UK\\
  \texttt{me532@cam.ac.uk}\\
  \And
  Luana Ruiz\\
  Johns Hopkins University\\
  Baltimore, MD %21218
  \\
  \texttt{lrubini1@jh.edu}\\
}

\iclrfinalcopy % Uncomment for camera-ready version, but NOT for submission.
\begin{document}

\maketitle

\begin{abstract}
  Graph Transformers produce, for each attention head, a dense $n\times n$ matrix of learned pairwise interactions. We ask a fundamental question: do these attention-induced graphs converge to a stable limit object as $n$ grows, or does the learned interaction pattern remain unstructured and size-dependent? We answer this using dense graph limit theory, treating each attention matrix as a finite sample from an underlying kernel---an \emph{attention graphon}---and studying concentration around this limit under the cut-distance. We derive a worst-case variance bound requiring no assumptions on the graphon, and a sharper regularity-aware bound based on nonparametric estimation theory. To operationalize the theory, we propose a canonicalize-then-block-average pipeline for estimating dataset-level attention graphons, and a variance-based diagnostic for testing whether attention admits a stable continuum description. Experiments across multiple graph benchmarks show that learned attention stabilizes to dataset-specific graphon structure on several datasets; that empirical cut-distance and cut-norm variance decreases with $n$ consistent with our bounds; and that attention graphons transfer to larger graph sizes with error decreasing in $n$.
\end{abstract}

\section{Introduction}
\label{sec:intro}

%\red{L: Here when we say transformers we still mean transformers applied to graphs, right? Might be worth making that clear. As we don't want someone to say we don't have experiments for sequential data.}
%Graph Transformers (GTs) \cite{dwivedi2021generalization,ying2021do,rampasek2022recipe}, based on sequence data transformers \cite{vaswani2017attention}, induce dense weighted graphs through attention scores: for an input graph with $n$ nodes, every attention head produces an
%$n\times n$ attention matrix that specifies a learned interaction pattern between all pairs of nodes.
%This raises a basic question that becomes unavoidable as graph sizes grow: \emph{how do these attention-induced graphs behave
%in the large-$n$ limit?}
%Do they converge to a stable continuum object, {potentially revealing useful low-dimensional structures}, or do they remain fundamentally finite-size and dataset-specific?
Graph Transformers (GTs) \citep{dwivedi2021generalization,ying2021do,rampasek2022recipe} 
adapt the self-attention mechanism of the Transformer architecture \citep{vaswani2017attention} to graph-structured inputs. 
For an input graph with $n$ nodes, each attention head produces an $n\times n$ matrix of 
learned pairwise interactions, which can be thought of as a dense weighted graph on the node set. This paper answers a fundamental 
question:
\vspace{-0.8em}
\begin{center}
    \textit{Do attention-induced graphs converge to a stable limit object as $n$ grows, or does the learned interaction pattern remain unstructured and size-dependent?}
    \vspace{-0.8em}
\end{center}
% \emph{Do attention-induced graphs converge to a stable limit object as $n$ grows, or does the learned interaction pattern remain unstructured and size-dependent?}

This question matters because if attention-induced graphs admit a stable limit, the interaction patterns learned by GTs can be compressed into a single, dataset-level low-dimensional object that simultaneously encodes inductive bias, supports size-agnostic comparison across graphs, and provides a theoretical basis for transferability across graph sizes.

%A natural approach for studying these questions is dense graph limit theory.
%For dense graphs, convergent sequences admit limit objects given by \emph{graphons}: measurable kernels on $[0,1]^2$, with
%convergence measured by the cut-distance $\delta_\square$ \citep{borgs2008convergent,lovasz2012large}.
%Graphons {can also be thought of as generative models for graphs of any size,}
%are invariant to node relabelings, i.e., their permutation, and control discrepancies over all node subsets
%making them well-suited for
%comparing attention patterns across graphs of different scales.
%At the same time, graphon estimation provides practical tools for recovering and comparing latent kernels from finite
%samples \citep{chan2014consistent,gao2015rate}.
We study these questions using dense graph limit theory. For dense graphs, convergent sequences admit limit objects given by \emph{graphons}--- symmetric measurable kernels on $[0,1]^2$---where convergence is characterized by the alignment of subgraph density profiles and, equivalently, by the vanishing of cut-norm discrepancies across all node bipartitions \citep{borgs2008convergent,lovasz2012large}. Graphons also serve as generative models for graphs of any size, making them well-suited for comparing attention patterns across different scales. More specifically, graphon estimation provides practical tools for recovering and comparing latent kernels from finite samples \citep{chan2014consistent,gao2015rate}.

%In this work, we link these two threads{---the interpretation of attention matrices as graph adjacencies, and the graphon toolkit---}to analyze  %Graph Transformer attention.
%attention in GTs; 
%we treat an attention matrix as a dense weighted graph and ask if it concentrates around a dataset-dependent
%``attention graphon'' as the graph size grows.
Equipped with the graphon toolkit, we formalize and study the question of whether attention matrices concentrate around a dataset-specific \emph{attention graphon} as graph size grows. Our contributions are:

\textbf{(C1) Graphon viewpoint for attention.} 
We formalize pre-softmax attention matrices as finite samples from an underlying kernel, and propose a canonicalize-then-block-average pipeline that produces a dataset-level attention graphon estimate. A paramount consequence of such perspective is \emph{transferability}. In graph machine learning (GML), graphon theory explains why GML models (e.g., graph neural networks) trained on moderate-size graphs can be transferred to a larger graph sampled from the same graphon, with a performance gap vanishing as both graphs grow \cite{ruiz20-transf}. Thus, an attention layer trained on small graph samples approximates its behavior on larger ones, with an error decaying with size. That gives a principled account of size generalization, suggesting that long-context behavior can be probed using small samples.
%We formalize attention matrices as samples from an underlying kernel up to
%relabeling, and we utilize a concrete estimator that produces a
%dataset-level attention graphon. 
%\red{L: We should be careful here as we are not the ones introducing the``sort-and-smooth'' approach, it has been around a while. M: the question is if it has been used with attention matrices? if not, we can say we link the two ?}

\textbf{(C2) Variance bounds.} We derive a worst-case variance bound for the deviation of attention graphs from their underlying graphon, and a sharper regularity-aware bound based on nonparametric graphon estimation. Together, these results ground cut-distance and cut-norm variance curves as principled diagnostics for whether attention admits a stable continuous description.

%We derive a worst-case bound for the variance of the deviation of {an attention matrix from its underlying graphon,}
%graphon samples 
%and a sharper regularity-aware proxy based on nonparametric graphon estimation.
%theory.
%These results {motivate employing}
%promote the utilization of 
%cut-distance variance curves as a diagnostic for understanding if graph attention admits a stable continuum
%description {in the form of a graphon}.

\textbf{(C3) Empirical validation.} We evaluate on multiple graph benchmarks and synthetic graphon families, showing that, for several of them, estimated attention graphons stabilize with graph size and that empirical variance decreases with $n$, consistent with the theory. For datasets admitting an attention graphon, we further show that the estimated kernel can be used to approximate attention on larger graphs, yielding transfer error curves that decrease with $n$ relative to the exactly computed attention matrices.

%We evaluate our approach on multiple graph benchmarks and synthetic graphon families. We show
%    that estimated attention graphons stabilize with size, and that
    %empirical cut distance variance 
%    {the empirical variance of the cut-distance from the sample to the graphon estimate}
    %typically 
%    decreases with
%    graph size, demonstrating the benefit of studying graph attention through the lens of graphons.

%Overall, the results presented in this paper provide a simple yet principled way to study how Graph Transformer attention behaves ``on the limit'' with respect to graph size: by turning attention matrices into estimated kernels and measuring their stability in a relabeling-invariant, size-agnostic metric.

\textbf{Notation.} 
We denote a simple undirected graph with $n = |V|$ nodes $G=(V,E)$, and its adjacency matrix
$A_G \in \{0,1\}^{n \times n}$, where $(A_G)_{ij}=1$ iff $(i,j)\in E$ and $(A_G)_{ii}=0$.
We also consider {weighted} graphs with adjacency 
%(weight) 
matrix
$A \in [0,1]^{n \times n}$.
Nodes are labeled as $[n] = \{1,\dots,n\}$, and their node features denoted $X \in \mathbb{R}^{n \times d}$. We use 
$\pi$ to express a permutation of $[n]$ with associated permutation matrix $M_\pi$.

\section{Background}\label{sec:preliminaries}
% \vskip -0.1in
% Our object of study is the attention score matrix produced by a GT layer. We view this matrix as the adjacency matrix of a dense weighted graph. 
Our object of study is the pre-softmax attention matrix produced by a GT layer. We view this matrix as the adjacency matrix of a dense weighted graph. 

\subsection{GTs and Attention Matrices as Weighted Graphs}
%As discussed in \Cref{sec:intro}, 
%We treat the attention score matrix produced by a GT head as a \emph{weighted dense graph} on the node set.
%This allows making use of dense-graph limit theory, graphons, and cut-distance to study the behavior of GTs.

%\textbf{Attention scores induce a weighted graph.}
Consider a GT operating on node representations $X^{(\ell)}\in\mathbb{R}^{n\times d}$ at layer $\ell$.
For head $h$, query, key, and value embeddings are respectively
$Q^{(\ell,h)} = X^{(\ell)}W_Q^{(\ell,h)}$, $K^{(\ell,h)} = X^{(\ell)}W_K^{(\ell,h)}$, and $V^{(\ell,h)} = X^{(\ell)}W_V^{(\ell,h)}$.
% The attention logits are
The pre and post-softmax attention matrices have entries respectively
\noindent
\begin{minipage}{0.48\textwidth}
\begin{equation}
\label{eq:attn-logits}
S^{(\ell,h)}_{ij} \;=\; \frac{\langle q^{(\ell,h)}_i,\,k^{(\ell,h)}_j\rangle}{\sqrt{d_k}} \;+\; b^{(\ell,h)}_{ij}
\end{equation}
\end{minipage}
\hfill
\begin{minipage}{0.48\textwidth}
\begin{equation}
\label{eq:attn-softmax}
P^{(\ell,h)}_{ij} \;=\; \frac{\exp(S^{(\ell,h)}_{ij})}{\sum_{t=1}^n \exp(S^{(\ell,h)}_{it})},
\end{equation}
\end{minipage}
where $b^{(\ell,h)}_{ij}$ are positional encodings (e.g., distance or spectral information), and the post-softmax attention matrix is its row-wise softmax. 

We interpret $S^{(\ell,h)}(G)$ as the adjacency matrix of a weighted directed graph on the node set $V$. In particular, for each fixed $(\ell,h)$, a Transformer layer defines a map $G \mapsto S^{(\ell,h)}(G)$ from an input graph sample to a dense weighted graph. \Cref{sec:object} motivates the choice of $S^{(\ell,h)}$ over $P^{(\ell,h)}$ and describes how its entries are mapped onto $[0,1]$.
% The pre-softmax attention matrix has entries
% \begin{equation}
% \label{eq:attn-logits}
% S^{(\ell,h)}_{ij} \;=\; \frac{\langle q^{(\ell,h)}_i,\,k^{(\ell,h)}_j\rangle}{\sqrt{d_k}} \;+\; b^{(\ell,h)}_{ij},
% \end{equation}
% % where $b^{(\ell,h)}_{ij}$ are positional encodings (e.g., distance or spectral information), and the attention scores are the row-wise softmax
% where $b^{(\ell,h)}_{ij}$ are positional encodings (e.g., distance or spectral information), and the post-softmax attention matrix is its row-wise softmax
% \begin{equation}
% \label{eq:attn-softmax}
% P^{(\ell,h)}_{ij} \;=\; \frac{\exp(S^{(\ell,h)}_{ij})}{\sum_{t=1}^n \exp(S^{(\ell,h)}_{it})}.
% \end{equation}
% We interpret $P^{(\ell,h)}(G)$ as the adjacency matrix of a weighted directed graph on the node set $V$. 
% We interpret $S^{(\ell,h)}(G)$ as the adjacency matrix of a weighted directed graph on the node set $V$. 
% %:every ordered pair $(i,j)$ has an edge of weight $P^{(\ell,h)}_{ij}$.
% % In particular, for each fixed $(\ell,h)$, a Transformer layer defines a map $G \mapsto P^{(\ell,h)}(G)$ from an input graph sample to a dense weighted graph.
% In particular, for each fixed $(\ell,h)$, a Transformer layer defines a map $G \mapsto S^{(\ell,h)}(G)$ from an input graph sample to a dense weighted graph. %\Cref{sec:object} explains why we analyze $S^{(\ell,h)}$ rather than $P^{(\ell,h)}$, and how its entries are mapped onto $[0,1]$.
% \Cref{sec:object} motivates the choice of $S^{(\ell,h)}$ over $P^{(\ell,h)}$ and describes how its entries are mapped onto $[0,1]$.

\textbf{Symmetric attention.}
%Classical graphon theory and the cut-metric are most commonly formulated for symmetric kernels, corresponding to undirected weighted graphs.
% While $P^{(\ell,h)}(G)$ is generally asymmetric and row-stochastic, symmetrization is well-motivated both theoretically---attention with tied query-key weights produces symmetric logits by construction \citep{kitaev2020reformerefficienttransformer}---and empirically, as learned symmetric attention patterns tend to be more efficient to train while keeping model's performance \citep{yang2024partsharedqk,courtois2024symmetric}. Symmetrizing attention is also a common simplification choice in theoretical analyzes of Transformers \citep{edelman2022inductive,ataee2023max,tian2023scan}. Here, we take a similar approach, working with the symmetrized matrix:
While $S^{(\ell,h)}(G)$ is generally asymmetric, symmetrization is well-motivated both theoretically---attention with tied query-key weights produces a symmetric pre-softmax attention matrix by construction \citep{kitaev2020reformerefficienttransformer}---and empirically, as learned symmetric attention patterns tend to be more efficient to train while retaining model performance \citep{yang2024partsharedqk,courtois2024symmetric}. Measured on trained models, the information discarded by symmetrization is limited, and negligible on some datasets (\Cref{tab:asym}). Symmetrizing attention is also a common simplification choice in theoretical analyses of Transformers \citep{edelman2022inductive,ataee2023max,tian2023scan}. Here, we take a similar approach, working with the symmetrized pre-softmax attention matrix:
%we map it to an undirected weighted adjacency via symmetrization:
\begin{equation}
\label{eq:sym-attn}
% A^{(\ell,h)}(G) \;\coloneqq\; \tfrac{1}{2}\bigl(P^{(\ell,h)}(G) + P^{(\ell,h)}(G)^\top\bigr)\in[0,1]^{n\times n},
\bar S^{(\ell,h)}(G) \;\coloneqq\; \tfrac{1}{2}\bigl(S^{(\ell,h)}(G) + S^{(\ell,h)}(G)^\top\bigr)\in\mathbb{R}^{n\times n}.
\end{equation}
Directional relationships become undirected ones, and pairwise affinity structure is preserved.
%This choice is a {simplifying reduction} that preserves the pairwise affinity structure while enabling direct use of the standard (symmetric) graphon framework and cut-distance.
%We note that one can work with asymmetric graphons, and we discuss this extension and its implications in Section~\ref{sec:discussion}.
This symmetrization places attention squarely in the undirected dense graph regime, where graphons are the natural limit objects.

\subsection{Graphons and Graph Distances}

To ask whether attention-induced graphs stabilize as their number of nodes grows, we require two ingredients: (1) a limiting object for sequences of dense graphs, and (2) a notion of distance between graphs under which the aforementioned sequences converge. 
%that does not depend on an arbitrary labeling (i.e., ordering) of nodes. 
Graphons, defined as symmetric measurable functions $W\colon [0,1]^2\to[0,1]$ \citep{lovasz2006limits},
%, as described below, 
are the natural limit object for dense symmetric graphs. To measure discrepancies between graphs, we use either the {cut distance}, when labelings are unknown or arbitrary, or the {cut norm}, when node correspondences are known.

%More specifically, \emph{graphon} is a symmetric measurable function $W\colon [0,1]^2\to[0,1]$ \citep{lovasz2006limits}.
%Graphons represent limits of convergent sequences of dense graphs.
%, and are identifiable up to
%measure-preserving relabelings of $[0,1]$. This non-identifiability matches the finite setting: node indices in an attention matrix are not intrinsic, so comparisons should not depend on node labeling.
%\textbf{Cut-norm and cut-distance.} 

\textbf{Cut-distance and cut-norm.} {The cut-distance is the metric under which graphs admitting graphon limits converge. For two graphs
$G_n$ and $G_m$, the cut-distance is defined via the cut-norm of their associated induced
graphons.
Given an $n$-node graph $G$ with adjacency matrix $A$, the induced graphon
$W_G : [0,1]^2 \to \mathbb{R}$ is defined as
\begin{equation}
\label{eqn:inducedgraphon}
W_G(u,v)
\;\coloneqq\;
\sum_{j=1}^n \sum_{k=1}^n
[A]_{jk}\,\mathbb{I}(u \in I_j)\,\mathbb{I}(v \in I_k),
\end{equation}
where $\{I_j\}_{j=1}^n$ is a partition of $[0,1]$ into equal-measure intervals and $\mathbb{I}$ is the indicator function.

For a kernel $K : [0,1]^2 \to [-1,1]$,e the cut-norm is defined as
\begin{equation}
\label{eq:cutnorm-graphon}
\|K\|_\square
\;\coloneqq\;
\sup_{S,T \subseteq [0,1]}
\left|
\int_{S \times T} K(u,v)\,\mathrm{d}u\,\mathrm{d}v
\right|.
\end{equation}

The cut-distance between $G_n$ and $G_m$ is then
\begin{equation}
\label{eqn:cut_dist_w}
\delta_{\square}(G_n, G_m)
\;\coloneqq\;
\inf_{\phi}
\bigl\|
W_{G_n} - W_{G_m}^{\phi}
\bigr\|_{\square},
\end{equation}
where the infimum is taken over all measure-preserving bijections
$\phi : [0,1] \to [0,1]$.
}

%\red{L: Here we can add a sentence saying something to the effect of ``with the appropriate labeling, graphs associated with the same graphon can also be shown to converge in the cut norm.'' This should help with the issue raised on Zoom earlier.}
While the cut-distance searches over all relabelings, if graphs sampled from the same graphon are compared under a labeling aligned with the graphon's latent node
ordering, their convergence may also be expressed directly in the cut-norm difference $\|W_{G_n} - W_{G_m}\|_\square$ \citep[Chapter~11]{lovasz2012large}.

\textbf{Graphons as generative models.} 
Beyond their graph limit interpretation, graphons are random graph models. This generative model perspective formalizes the idea that a finite graph is a discrete stochastic observation of an underlying kernel.
%In the weighted model $$,
To sample an $n$-node weighted graph, denoted $H(n,W)$, we first sample $n$ i.i.d. latent variables $x_1,\dots,x_n \sim \mathrm{Unif}[0,1]$ and set the edge weights as $W(x_i,x_j)$, yielding a complete graph.
To obtain undirected graphs $G(n,W)$, we instead sample edges
$(i,j) \sim \mathrm{Bernoulli}(W(x_i,x_j))$ for $i\neq j$ \citep{lovasz2012large}. From hereon, unless otherwise stated, we refer to $H(n, W)$ as $H$ and $G(n, W)$ as $G$.

\subsection{Pre-Softmax Attention Graphs}
\label{sec:object}
We build attention graphs from pre-softmax scores. Given the symmetrized pre-softmax attention matrix $\bar S^{(\ell,h)}(G)$ of \Cref{eq:sym-attn}, an increasing bijection $\rho\colon\mathbb{R}\to(0,1)$ applied entrywise, and a constant offset $c\in\mathbb{R}$, we define
\begin{equation}
\label{eq:attn-object}
A^{(\ell,h)}(G)\;\coloneqq\;\rho\bigl(\bar S^{(\ell,h)}(G)-c\bigr)\in(0,1)^{n\times n}.
\end{equation}
We take $\rho=\sigma$, the logistic sigmoid, and set $c$ to the average pre-softmax score over the graphs analyzed. This centers the entries in the nearly linear range of $\sigma$. Other admissible choices of $\rho$ are discussed in Appx~\ref{app:object}. From here on, $A^{(\ell,h)}(G)$ is the attention graph whose limit we study.

Because $\rho$ is invertible, $A^{(\ell,h)}$ loses no information about the post-softmax attention. This allows us to recompute attention from estimated kernels in later sections.
\begin{proposition}
\label{prop:invertible}
Let $\rho\colon\mathbb{R}\to(0,1)$ be a bijection, $c\in\mathbb{R}$, and $A=\rho(\bar S-c)$. Then $\mathrm{softmax}(\bar S)=\mathrm{softmax}\bigl(\rho^{-1}(A)\bigr)$, and the post-softmax attention is recovered from $A$ independently of $c$.
\end{proposition}
\begin{proof}
Invertibility gives $\bar S=\rho^{-1}(A)+c$, and the softmax is shift invariant, that is, $\mathrm{softmax}(x+t\mathbf{1})=\mathrm{softmax}(x)$ for every $t\in\mathbb{R}$.
\end{proof}

\section{Theory: Graphon Sample Concentration}\label{sec:theory}
{
%Our empirical objects are 
% We work with attention score matrices produced by a trained GT, one matrix per input graph and attention head. Since these matrices
We work with the attention graphs of \Cref{eq:attn-object} produced by a trained GT, one matrix per input graph and attention head. Since these matrices
vary in dimension with the number of nodes of the input graphs, their comparison across input samples
requires a common, size-independent representation, achieved by embedding such graphs into a common function space.
%, we rely on graphon theory 

%equipped with the cut-metric,
%the standard notion of proximity for dense graph limits.

}
%Our empirical objects are attention score matrices produced by a trained Graph Transformer, one matrix per input graph and head.
%These matrices live in different dimensions as the number of nodes varies across inputs, so to state a scaling law we need a common space in which objects of different sizes can be compared.
%Graphon theory provides this space; it embeds every finite dense weighted graph into a function space, and equips that space with the cut metric, which is the standard notion of proximity for dense-graph limits.
%The results in this Section are therefore stated for generic graphon sampling and serve as \emph{calibration}: if attention-induced graphs behave like samples from a stable kernel, then their cut-metric deviations should concentrate as the graph size $n$ increases.

\textbf{Embedding attention graphs in graphon space.}
{
Given a weighted symmetric adjacency matrix $A \in [0,1]^{n \times n}$ and its associated $n$-node weighted graph $H$, we consider its induced graphon, $W_H$, as defined in \Cref{eqn:inducedgraphon}. This construction places finite matrices of different sizes and continuous kernels $W$ in the same space, without introducing error. Without any assumptions on the node labeling in $A$, or on how it relates to the node labeling of the graphon nodes, we measure the discrepancy between them using the cut-distance $\delta_\square$ \eqref{eqn:cut_dist_w}, and write $\delta_\square(W_n,W)$ as shorthand for $\delta_\square(W_H,W)$.
% \red{This is not standard. We should write $\delta_\square(G,W)$, i.e., compare graph and graphon, not adjacency and graphon. Also, you introduced weighted graphs $H(n,W)$ in the previous section but never used this notation again. I also imagine you don't want to write the $(n,W)$ every time, so might be worth either saying you will drop it for ease of notation, or just defining weighted graphs as $H$.} 

%With this
%otation, concentration statements of the form
%$\delta_\square(A,W) \to 0$ as $n \to \infty$ are well-defined. 
}
%A weighted graph with adjacency matrix $A\in[0,1]^{n\times n}$ can be canonically associated with a step-function graphon
%$W_A\colon [0,1]^2\to[0,1]$ by partitioning $[0,1]$ into $n$ equal intervals and setting $W_A$ to be constant on each
%cell $I_i\times I_j$ with value $A_{ij}$.
%This construction is not an estimator; it is an \emph{embedding} that lets us view finite $n\times n$ matrices and continuum kernels $W$ as elements of the same space.
%We use the standard graphon cut distance $\delta_\square(\cdot,\cdot)$ \citep{lovasz2012large,borgs2008convergent}, and write
%$\delta_\square(A,W)$ as shorthand for $\delta_\square(W_A,W)$.
%With this notation, statements such as ``$\delta_\square(A,W)$ concentrates'' are well-posed even when $n$ varies, and they directly correspond to the empirical question we test for attention matrices.

%The results in this section are stated for generic graphon sampling, and serve as
%a calibration principle: if attention-induced graphs behave as samples from a
%stable kernel, then their cut-metric deviations should concentrate as the graph
%size $n$ increases.

\textbf{Testing the graphon hypothesis.} Embedding attention graphs into graphon space does not by itself imply they were generated by a graphon. To test this hypothesis, we rely on classical graphon sampling results, which characterize how far a finite sample $H(n,W)$ can deviate from W.
% its generating kernel $W$ in cut-distance. 
%% Specifically, we first derive a variance upper bound from the graphon sampling lemma \citep[Lemma~10.16]{lovasz2012large}. If attention-induced graphs at size $n$ behave as i.i.d.\ samples from a stable kernel, their empirical cut-distance variance should fall below this bound. \blue{C: This part seems to be misplaced.}

\subsection{A Universal Variance Bound}
\label{sec:worstcase}

We begin with a %worst-case 
concentration bound for graphon samples that
%, though very loose, 
requires
no structural assumptions.  Its generality comes at the cost of a slow $1/\log n$ rate, which does not explain the much faster concentration we observe for learned attention.

%This is a corollary of classical graphon sampling results
%Based on well-established results for sampling over graphons 
%\citep[Lemma~10.15]{lovasz2012large}.
%, we first devise the most general bound to the concentration of samples from a graphon.

% The bound presented in \Cref{thm:bound_1} is a direct corollary of the second sampling lemma for graphons \citep{lovasz2012large}.
% It is universal, as no assumptions\textcolor{red}{needs to be softened and verified with results and referred to them, also need to say what implication / take home message they it tells us: but too slow to explain the sharper empirical concentration we observe for learned attention.}

\begin{theorem}[Cut-distance variance]
\label{thm:bound_1}
Let $W\in\ccalW_0$ be a graphon and let $
%G \sim
% H \sim
H(n,W)$ be a weighted graph sampled from it.
Then
\begin{equation}
\Var\bigl(\delta_\square(W_n, W)\bigr) \;=\; \mathcal{O}\!\left(\tfrac{1}{\log n}\right),
\end{equation}
where $W_n$ is the graphon induced by $H(n, W)$. Proof and omitted assumptions are in Appx~\ref{app:proof_worstcase}.
\end{theorem}

The bound in Thm.~\ref{thm:bound_1} is a corollary of the second sampling lemma for graphons \citep{lovasz2012large}. It is universal, since no assumptions are made on the graphon, and it provides a natural baseline to measure the closeness of our samples to the graphon that generated them.
%, with an asymptotic convergence. 
However, 
%preliminary experiments on a sparse SBM showed us that 
this bound is not useful in practice; it decays with $(\log n)^{-1}$, which is {astronomically} loose. If we were to compare empirical sample variances to this bound, %empirical variance is always smaller than its theoretical counterpart, 
they would always be smaller even for extremely large graph samples known not to come from a graphon.

\subsection{A Regularity-Aware Variance Bound}
\label{sec:regularity}

{While the cut-distance is invariant to arbitrary relabelings, the conservativeness of Thm.~\ref{thm:bound_1} is a direct consequence of this relabeling invariance. When the underlying graphon has
additional regularity, such as H\"older or Lipschitz continuity, sharper
rates can be obtained under {aligned representations} which restrict
comparisons to graph labelings compatible with the graphon's latent ordering on
$[0,1]$. This alignment trades full permutation invariance for substantially
tighter convergence guarantees \citep{gao2015rate,ruiz2021transferability}. 
Switching to this more structured regime, we rely on nonparametric
graphon estimation rates under H\"older regularity to derive a more useful variance bound
\citep{chan2014consistent,gao2015rate}. Note that in contrast to the previous result, which controls fluctuations in cut-distance, the following bound is stated in terms of the cut-norm, reflecting
the use of aligned representations.
}
\begin{theorem}[Regularity-aware cut-{norm} variance]
\label{thm:bound_2}
Fix $\alpha>0$ and $M>0$, and assume $W\in\ccalF_\alpha(M)$ (a H\"older class; defined in Appx~\ref{app:holder}).
%Let $\theta_{ij}=
%Define $W(x_i,x_j)$ 
For latent variables $(x_1,\dots,x_n)\sim \ccalP_X$, let $H(n,W)$ be a weighted graph sample.
With block resolution $k = \lceil n^{1/(\min(\alpha,1)+1)}\rceil$, there exists $C>0$, depending on $M$, such that
\begin{equation}
\label{eq:regularity_proxy}
\Var\bigl(\|W_n-W\|_\square\bigr) = \ccalO\left(n^{-\frac{2\min(\alpha, 1)}{(\min(\alpha, 1)+1)}}\right),
\end{equation}
where $W_n$ is the graphon induced by $H(n,W)$.
Proof and omitted assumptions are in Appx~\ref{app:holder}.
\end{theorem}

%\red{L: The next 3 paragraphs are very nice but very long, they can be shrunk substantially.}

%In \Cref{thm:bound_2}, we establish a regularity-aware concentration bound for the cut-metric variance between a $k$-block approximation estimator $\hat{\theta}$ and the true graphon $W$. Unlike the worst-case bound in Theorem 4.1, which converges at the slow rate of $(\log n)^{-1}$, this result exploits the H\"older regularity of the underlying graphon to achieve substantially faster convergence, {though at the cost of requiring label alignment}. Specifically, when the graphon $W$ belongs to the H\"older class $\ccalF_\alpha(M)$ and the number of blocks scales as $k \asymp n^{1/(\min(\alpha,1)+1)}$, the variance decays as $\ccalO(n^{-2\alpha/(\alpha+1)} + \frac{\log n}{n})$. This polynomial rate represents a significant improvement over the logarithmic rate of \Cref{thm:bound_1}, demonstrating that structural assumptions on the graphon translate directly into tighter statistical guarantees.

In Thm.~\ref{thm:bound_2}, we establish a regularity-aware concentration bound for the variance of the cut norm difference between the graphon induced by $H(n,W)$, $W_n$, and the true graphon $W$. Unlike the bound in Thm.~\ref{thm:bound_1}, which decays as $(\log n)^{-1}$, 
% in this result 
faster convergence is possible due to H\"older regularity. Specifically, if $W \in \ccalF_\alpha(M)$, then the variance vanishes at a polynomial rate. %The price to pay is label alignment.

The nonparametric rate $n^{-2\min(\alpha,1)/(\min(\alpha,1)+1)}$ balances two errors in the resolution $k$. Finer blocks reduce the error of approximating a H\"older kernel by block averages, while coarser blocks average more node pairs and are less noisy. The dependence on $\min(\alpha,1)$ reflects a structural limit of piecewise-constant estimators, which cannot exploit smoothness beyond Lipschitz without local polynomial corrections \citep{olhede2014network}, so the rate stabilizes at $n^{-1}$ for $\alpha\ge 1$. The label-alignment requirement is provided by our canonicalization procedure (Sec.~\ref{sec:estimator}), which follows the SAS estimator \citep{chan2014consistent} and recovers the latent node ordering consistently under strict monotonicity of the degree function.

\begin{comment}
\begin{figure*}[t]
\centering
\includegraphics[width=0.9\linewidth]{Figures/Rebuttal/framework.png}
\caption{\textbf{Visual explanation of the cut distance $\delta_\square$ and graphon computation.} \textit{Top (cut distance)}: The cut distance roughly measures the largest discrepancy over node subsets $S, T$. \textit{Bottom (graphon computation)}: From a dense, weighted attention graph with $n = 8$, the matrix is symmetrized, canonicalized via degree sorting, block-averaged, and finally used to estimate a dataset-level graphon.}
\label{fig:pipeline}
\end{figure*}
\end{comment}

\paragraph{Remark: Choice of attention graph.}
The concentration results in this section are informative when the analyzed matrix stays bounded as $n$ grows and its limit depends on what the model learns. The attention graph of \Cref{eq:attn-object} satisfies both. Its entries lie in $(0,1)$ by construction, and pre-softmax scores carry no normalization over the nodes, so the limit is not forced to be trivial. The natural alternatives each miss one of these properties. Row-stochasticity makes the graphons induced by $P^{(\ell,h)}$ converge to the zero graphon for every model and dataset. The rescaled matrix $nP^{(\ell,h)}$ avoids this, but its entries can grow with $n$ (c.f. Appx \ref{app:object}).

\section{Practice: Attention Graphons}
\label{sec:method}

\begin{comment}
The theory in Sec.~\ref{sec:theory} is stated for generic graphon sampling: it characterizes how dense weighted graphs concentrate in cut-distance around an underlying kernel as $n$ grows.
% In this section, we connect that viewpoint to GTs by turning per-graph attention scores into (i) a \emph{comparable} dense weighted graph representation across inputs and sizes, and (ii) a \emph{dataset-level} attention-graphon estimate that we can visualize and use in our analysis.
In this section, we connect that viewpoint to GTs by turning per-graph pre-softmax attention matrices into (i) a \emph{comparable} dense weighted graph representation across inputs and sizes, and (ii) a \emph{dataset-level} attention-graphon estimate that we can visualize and use in our analysis.
Concretely, we introduce notation for attention-induced weighted graphs, describe the canonicalize-then-block-average estimator used to produce attention graphons, and define the cut-norm variance statistic that we track across 
% graph 
sizes in Sec.~\ref{sec:experiments}.
\end{comment}
We now apply the theory of \Cref{sec:theory} to GTs. We represent each input's pre-softmax attention as a graph that is comparable across inputs and sizes, and estimate a dataset-level attention graphon from these graphs. We also define the cut-norm variance statistic tracked across graph sizes in \Cref{sec:experiments}.

%\red{L: The below has already been explained earlier, no? Can we avoid the repetition either here or back there?}
%\textbf{Attention-induced weighted graphs.}
%Fix a trained Graph Transformer, a layer $\ell$, and head $h$.
%For an input graph $G$ with $n$ nodes, let $P^{(\ell,h)}(G)\in[0,1]^{n\times n}$ denote the post-softmax attention matrix, where $P^{(\ell,h)}_{ij}(G)$ is the attention mass assigned by node $i$ to node $j$.
%We will treat this matrix as the adjacency of a dense weighted graph induced by attention.

\subsection{Does GT Attention Converge to a Graphon?}
\label{sec:attn-hyp}

% The central question of this paper is whether the dense weighted graphs induced by attention scores exhibit a stable continuum structure as the number of nodes grows.
The central question of this paper is whether the dense weighted graphs induced by pre-softmax attention exhibit a stable continuum structure as the number of nodes grows.
%Rather than assuming a generative model for the \emph{input} graphs, 
We ask whether, for a fixed trained model and a fixed head, the \emph{learned attention adjacencies} can be well-approximated 
%(in a relabeling-invariant sense) 
by sampling from an underlying kernel.

%\red{L: The remainder of the subsection (below) is excellent. It'd be nice if we could get to it faster. Maybe by cutting a bit the intro to Sec. 5 and the attention-induced weighted graphs paragraph.}

\textbf{A graphon proxy model for attention.} %(for calibration, not as an assumption)
Given a trained GT, a layer $\ell$, and head $h$, we say that the attention-induced adjacencies admit a graphon description on a given dataset (or size-controlled family) if there exists a bounded symmetric kernel
$W^{(\ell,h)}\colon[0,1]^2\to[0,M]$ such that attention matrices for graphs with $n$ nodes behave like random samples from $W^{(\ell,h)}$, up to relabeling.
Concretely, the proxy model posits latent variables $x_1,\dots,x_n \overset{\mathrm{iid}}{\sim}\mbP_X$ and an underlying affinity matrix $\theta_{ij} \;=\; W^{(\ell,h)}(x_i,x_j)$,
\begin{comment}
\begin{equation}
\label{eq:attn-kernel}
\theta_{ij} \;=\; W^{(\ell,h)}(x_i,x_j),
\end{equation}
\end{comment}
with the observed attention adjacency $A^{(\ell,h)}(G)$ being a noisy observation of $\theta$.
Importantly, this proxy model makes no claim about how the \emph{input graphs} are generated; it only formalizes a hypothesis about the \emph{learned attention graphs} produced by the trained network.
This hypothesis is natural for a trained model with frozen weights, since the pre-softmax score of a pair of nodes depends only on their representations. Hence $A^{(\ell,h)}_{ij}=W^{(\ell,h)}(\omega_i,\omega_j)$ for a fixed symmetric kernel $W^{(\ell,h)}$ on node descriptors $\omega_i$, comprising features and structural role, and whether the resulting graphs converge as $n$ grows is what our diagnostic assesses.

\textbf{Practical implication.}
If attention-induced graphs admit such a stable kernel description, then the cut-distance deviations of $A^{(\ell,h)}(G)$ from a dataset-level kernel estimate should shrink with $n$.
This is precisely the behavior predicted by the cut-distance concentration results in Sec.~\ref{sec:theory}, which motivates our empirical diagnostics based on $\|\cdot\|_\square$ and its variance across graph sizes.

\begin{figure*}[t]
\centering
\begin{subfigure}[t]{0.24\textwidth}
  % \centering\includegraphics[width=\linewidth]{Figures/variance_PROTEINS_2.png}
  % \caption{\centering{\textbf{PROTEINS} %\\ ($N\in\{40,128,256,650\}$)
  \centering\includegraphics[width=\linewidth]{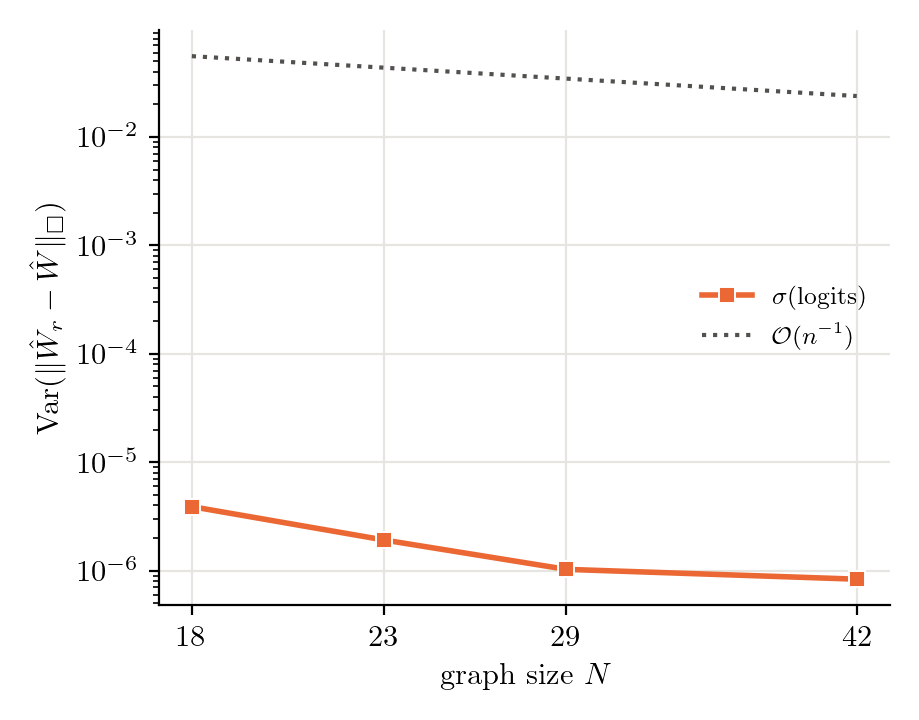}
  \caption{\centering{\textbf{NCI109} %\\ ($N\in\{16, 32, 64, 111\}$)
  }
  }
\end{subfigure}
\begin{subfigure}[t]{0.24\textwidth}
  % \centering\includegraphics[width=\linewidth]{Figures/variance_REDDIT-MULTI-5K_2.png}
  % \caption{\centering{\textbf{REDDIT}\textbf{-}\textbf{MULTI}\textbf{-}\textbf{5K}% \\ ($N\in\{128,256,512,1024\}$)
  \centering\includegraphics[width=\linewidth]{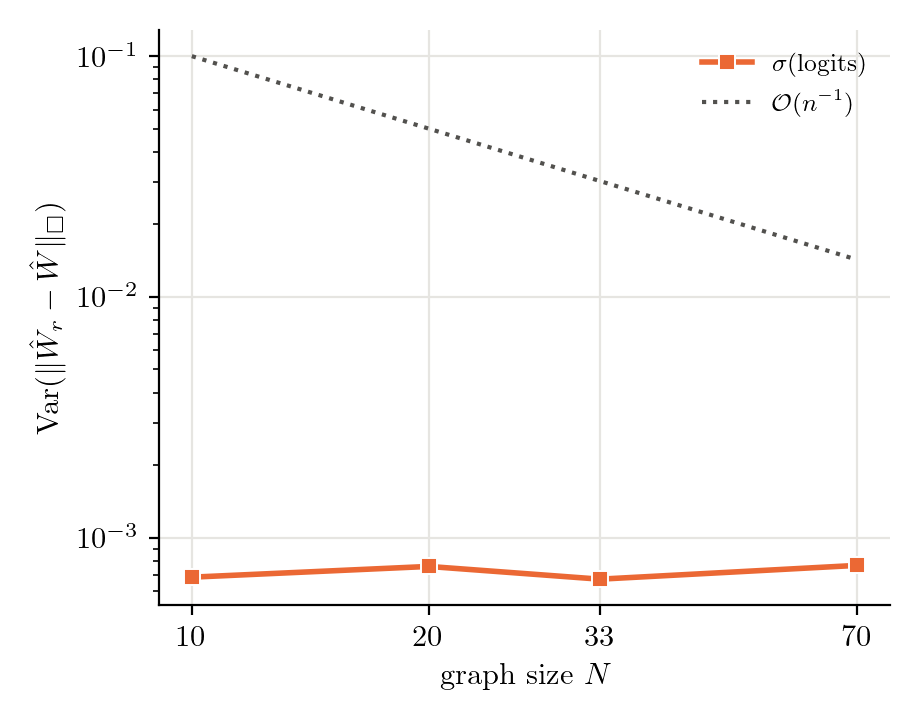}
  \caption{\centering{\textbf{PROTEINS} %\\ ($N\in\{40,128,256,650\}$)
  }}
\end{subfigure}
\begin{subfigure}[t]{0.24\textwidth}
  % \centering\includegraphics[width=\linewidth]{Figures/variance_ModelNet10_2.png}
  \centering\includegraphics[width=\linewidth]{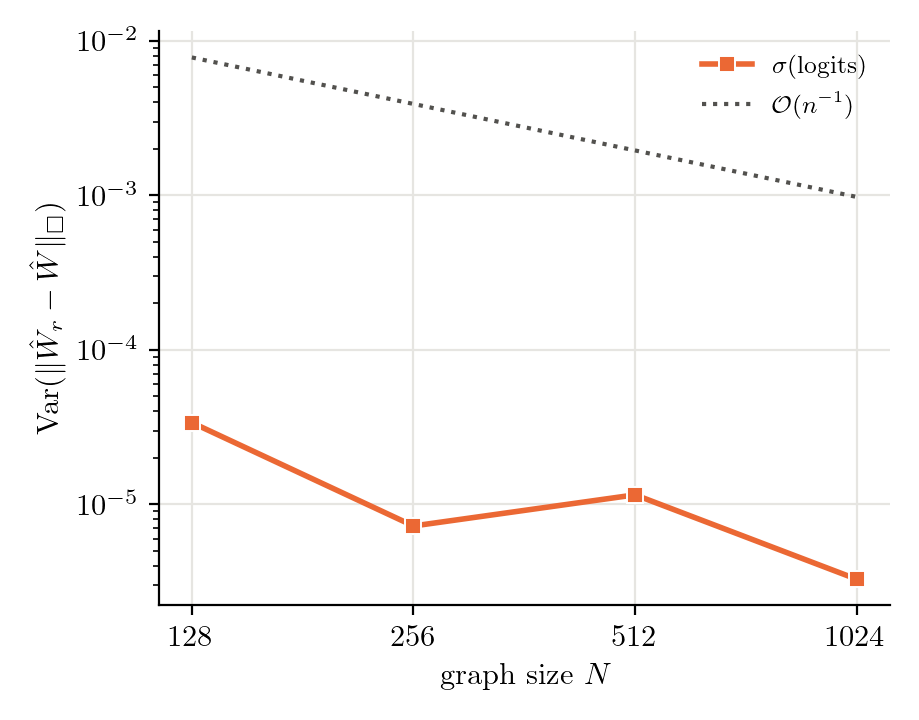}
  \caption{\centering{\textbf{ModelNet10} %\\ ($N\in\{128,256,512,1024\}$)
  }}
\end{subfigure}
\begin{subfigure}[t]{0.24\textwidth}
  % \centering\includegraphics[width=\linewidth]{Figures/variance_NoisyCSBM_2.png}
  % \centering\includegraphics[width=\linewidth]{Figures/variance_cSBM_fixed.png}
  \centering\includegraphics[width=\linewidth]{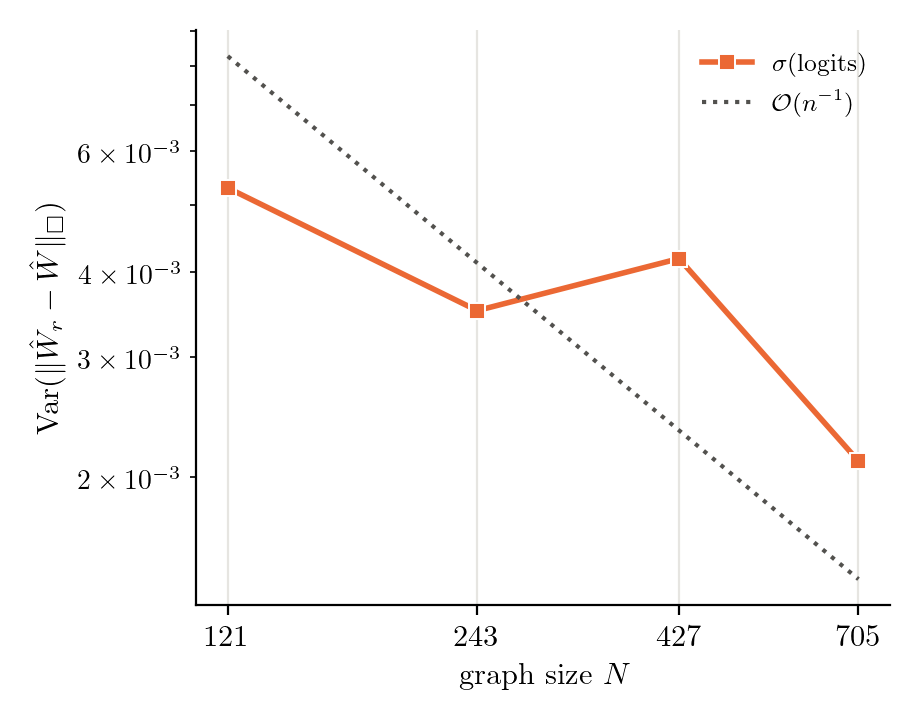}
  % \caption{\centering{\textbf{cSBM} %\\ ($N\in\{128,256,512,1024\}$)
  \caption{\centering{\textbf{REDDIT-MULTI-5K} %\\ ($N\in\{128,256,512,1024\}$)
  }}
\end{subfigure}

\caption{\small\textbf{Empirical cut-norm variance versus regularity-aware theoretical proxy.}
% Each panel compares the \blue{empirical variance} of 
Each panel compares the \orange{empirical variance} of 
% $\delta_\square(\hat W_r,\hat W)$ 
$\|\hat W_r - \hat W\|_\square$ against the 
% size-dependent 
% \green{scaling proxy} (Thm.~\ref{thm:bound_2}) 
\textcolor{darkgray}{scaling proxy} $\ccalO(n^{-1})$ (Thm.~\ref{thm:bound_2}).
% \blue{On REDDIT-MULTI-5K, the variance exceeds the proxy at the two largest sizes.}
}
\label{fig:cutvar_selected}
\vskip -0.15in
\end{figure*}

\subsection{Estimating Attention Graphons}
\label{sec:estimator}

%\red{L: I don't think in this section we need to keep repeating that the graphon is a ``dataset-level'' graphon. You can apply this estimation procedure to any set of attention graphs you want. We only need to specify we did this ``per dataset'' in the experiments. But I would leave it general here.}

% In Sec.~\ref{sec:attn-hyp} we consider each input graph and its dense weighted adjacency $A^{(\ell,h)}(G)$ induced by attention.
% We now turn a \emph{collection} of adjacencies into a single, interpretable kernel estimate.
% %This estimate serves two roles to answer our research question: (i) as a visualization of the interaction structure learned by a given head, and (ii) as a reference object for the cut distance concentration diagnostics in Section~\ref{sec:experiments}.
% Our estimator follows a common \emph{canonicalize-then-block-average} template, closely related to sort-and-smooth graphon estimators \citep{chan2014consistent}. Specifically, let $\{A_r\}_{r=1}^m$ denote attention-induced adjacencies extracted from $m$ %dataset 
% graphs, where $A_r\in[0,1]^{n_r\times n_r}$ may have varying sizes $n_r$.
% The estimator consists of three steps:

In Sec.~\ref{sec:attn-hyp} we consider each input graph and its attention-induced dense weighted adjacency $A^{(\ell,h)}(G)$. We now turn a collection of such adjacencies into a single, interpretable kernel estimate. Our estimator follows a \emph{canonicalize-then-block-average} template, closely related to the sort-and-smooth graphon estimator \citep{chan2014consistent}. Let $\{A_r\}_{r=1}^m$ denote attention-induced adjacencies from $m$ graphs, with $A_r \in [0,1]^{n_r \times n_r}$ of varying size $n_r$. The estimator has three steps:

\textbf{(i) Size normalization.}
To place all samples in a common space, we pad each $A_r$ to a shared size $N$ by repeating rows/columns,
% with zero rows/columns,
yielding $\tilde A_r\in[0,1]^{N\times N}$, 
with weights preserved.
% This operation preserves all existing weights.
%and only adds null interactions for missing nodes. 

\textbf{(ii) Canonicalization by degree sorting.}
{For each $N$-node graph, there are $N!$ possible node labelings, so considering each of these labelings for each of the $m$ graphs to estimate the graphon}
%Cut distance compares graphs up to relabeling, but explicitly optimizing over permutations 
is computationally infeasible.
We therefore use a lightweight canonicalization to reduce permutation variability: we sort nodes by their attention-induced degrees $d_i(\tilde A_r)\;=\;\sum_{j=1}^N (\tilde A_r)_{ij}$
% \begin{equation}
% d_i(\tilde A_r)\;=\;\sum_{j=1}^N (\tilde A_r)_{ij},
% \end{equation}
and{, to each attention graph,} apply the induced permutation $M^r_\pi$, producing $\bar A_r=M^r_\pi\tilde A_r{M^r_\pi}^\top$.
Intuitively, this aligns graphs using a coarse structural statistic before averaging.
% \blue{Formally, degree sorting recovers the latent node ordering when the degree function $g(x)=\int_0^1 W(x,y)\,dy$ is strictly monotone. This is the standard identifiability condition in nonparametric graphon estimation \citep{bickel2009nonparametric,yang2014nonparametric}, and the condition under which the SAS estimator is consistent \citep{chan2014consistent}, as formalized in Assumption~\ref{asm:canon}. Attention weights are real-valued, so exact ties in the degrees are not expected, and in our experiments the fraction of tied degrees among populated blocks is zero for every dataset and size. When the condition fails, degree sorting need not recover the latent ordering, and imperfect canonicalization becomes one of the ways the framework can fail (Sec.~\ref{sec:discussion}). Replacing degree sorting with spectral (Fiedler) ordering yields nearly identical variance decay and structurally consistent kernels (\Cref{fig:sensitivity_canon_hypothesis}).}

Formally, degree sorting recovers the latent node ordering when the degree function $g(x)=\int_0^1 W(x,y)\,dy$ is strictly monotone. This is the standard identifiability condition in nonparametric graphon estimation \citep{bickel2009nonparametric,yang2014nonparametric}, and the condition under which the SAS estimator is consistent \citep{chan2014consistent}, as formalized in Assumption~\ref{asm:canon}. We also tested spectral (Fiedler) ordering, which yields nearly identical variance decay (see \Cref{fig:sensitivity_canon_hypothesis}, Appx~\ref{app:add_exp}).

\textbf{(iii) Block averaging (graphon smoothing).}
% Fix a resolution $k$. 
Fix a resolution $k$. We use $k=\min(64,N)$, so the estimator refines with the graph size up to a cap. Thm.~\ref{thm:bound_2} admits $k=\lceil N^{1/(\min(\alpha,1)+1)}\rceil\in[\lceil\sqrt N\rceil,N]$, a range that contains $k$ whenever $N\le 64^2$, which covers every size we evaluate. Keeping $k\le N$ also avoids empty blocks, which would lower the measured variance for a mechanical reason (Appx~\ref{app:add_exp}). 
%(we keep $k$ fixed across datasets to enable consistent comparisons).
Partition $[N]$ into consecutive bins $B_1,\dots,B_k$ of equal size {(except for a possible remainder)}, and compute the  block mean. Then, taking the graphon induced by the matrix $\hat\Theta_r$ (cf. \Cref{eqn:inducedgraphon}) produces a step-function kernel estimate $\hat W_r$.
To obtain our final estimate for the attention graphon, we aggregate across samples. Those operations are defined as follows
\noindent
\begin{minipage}{0.48\textwidth}
\begin{equation}
\label{eq:block-avg}
(\hat\Theta_r)_{ab} \;=\; \frac{1}{|B_a|\,|B_b|}\sum_{i\in B_a}\sum_{j \in B_b} (\bar A_r)_{ij}
\end{equation}
\end{minipage}
\hfill
\begin{minipage}{0.48\textwidth}
\begin{equation}    
\hat\Theta \;=\; \frac{1}{m}\sum_{r=1}^m \hat\Theta_r,
\end{equation}
\end{minipage}
with $a, b \in [k]$.
% \begin{equation}
% \label{eq:block-avg}
% (\hat\Theta_r)_{ab} \;=\; \frac{1}{|B_a|\,|B_b|}\sum_{i\in B_a}\sum_{j\in B_b} (\bar A_r)_{ij},\  a,b\in[k].
% \end{equation}
% Taking the graphon induced by the matrix $\hat\Theta_r$ (cf. \Cref{eqn:inducedgraphon}) produces a step-function kernel estimate $\hat W_r$.
% To obtain our final estimate for the attention graphon, we aggregate across samples as
% \begin{equation}    
% \hat\Theta \;=\; \frac{1}{m}\sum_{r=1}^m \hat\Theta_r.
% \end{equation}
The graphon induced by $\hat \Theta$ defines the attention graphon estimate $\hat W$.
% for the chosen layer and head.

% \blue{\textbf{Block resolution.}
% We take the block resolution to be $k_n=\min(K,n)$ with $k=64$, so the estimator refines as graphs grow and saturates once $n$ exceeds the cap. Thm.~\ref{thm:bound_2} admits $k=\lceil n^{1/(\min(\alpha,1)+1)}\rceil$, which for $\alpha\in(0,1]$ lies between $[\lceil\sqrt n\rceil, n]$, and $k_n$ lies in this range at every size we evaluate. The cap prevents the degeneracy reported in Appendix~\ref{app:add_exp}, where a resolution exceeding the graph size leaves empty bins that lower the measured variance for a mechanical rather than statistical reason. Since $k_n$ is eventually constant while the admissible range grows with $n$, a fixed cap remains admissible only while $\lceil\sqrt n\rceil\le K$, that is, up to $n=K^2$. Extending the diagnostic to substantially larger graphs would require raising $k$ with $n$, at which point empty bins no longer arise because graphs are large relative to the resolution.}

%\textbf{Output and interpretation.}
%The result $\hat W$ is a $k\times K$ discretized kernel that summarizes the typical attention-induced interaction pattern for the dataset.
%In Section~\ref{sec:experiments}, we use $\hat W$ for head-level visualization and as the reference kernel in cut distance concentration and variance diagnostics across graph sizes.

\subsection{Hypothesis Testing}
\label{sec:ht}

\begin{figure*}[t]
    \begin{center}
        % {\includegraphics[width=0.875\linewidth]{Figures/est_graphons_2.png}}
        % {\includegraphics[width=\linewidth]{Figures/est_graphons_3.jpeg}}
        % \begin{subfigure}[t]{0.24\textwidth}\centering\includegraphics[width=\linewidth]{Figures/Rebuttal/Graphon Est/GPS/h1/modelnet_128.png}\caption*{\blue{$N=128$}}\end{subfigure}
        \raisebox{0.075\textwidth}{\rotatebox[origin=c]{90}{\small ModelNet10}}\hspace{2pt}%
        \begin{subfigure}[t]{0.22\textwidth}\centering\includegraphics[width=\linewidth,height=2.2cm]{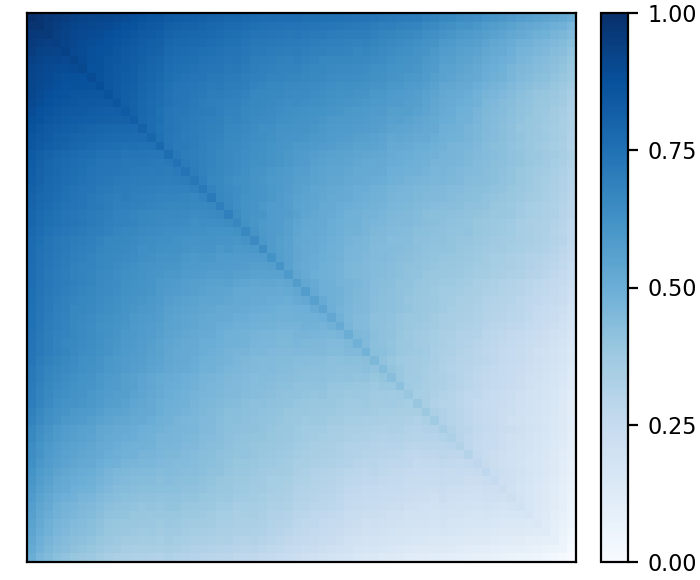}
        % \caption*{$N=128$}
        \end{subfigure}
        % \begin{subfigure}[t]{0.24\textwidth}\centering\includegraphics[width=\linewidth]{Figures/Rebuttal/Graphon Est/GPS/h1/modelnet_256.png}\caption*{\blue{$N=256$}}\end{subfigure}
        \begin{subfigure}[t]{0.22\textwidth}\centering\includegraphics[width=\linewidth,height=2.2cm]{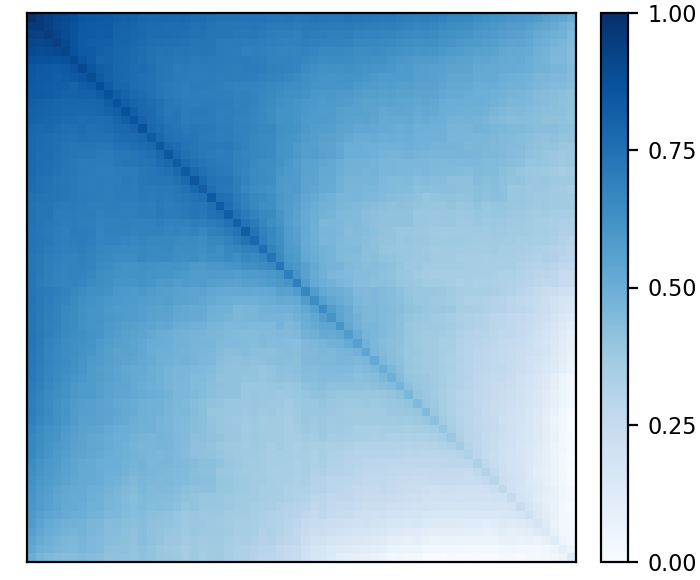}
        % \caption*{$N=256$}
        \end{subfigure}
        % \begin{subfigure}[t]{0.24\textwidth}\centering\includegraphics[width=\linewidth]{Figures/Rebuttal/Graphon Est/GPS/h1/modelnet_512.png}\caption*{\blue{$N=512$}}\end{subfigure}
        \begin{subfigure}[t]{0.22\textwidth}\centering\includegraphics[width=\linewidth,height=2.2cm]{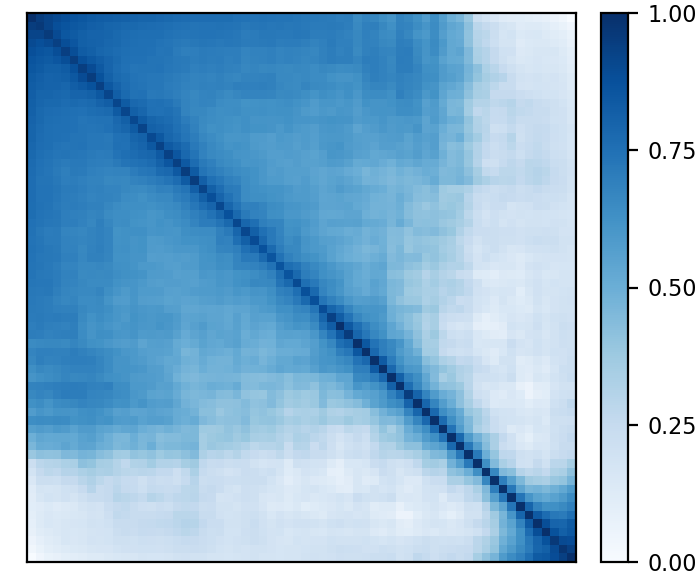}
        % \caption*{$N=512$}
        \end{subfigure}
        % \begin{subfigure}[t]{0.24\textwidth}\centering\includegraphics[width=\linewidth]{Figures/Rebuttal/Graphon Est/GPS/h1/modelnet_1024.png}\caption*{\blue{$N=1024$}}\end{subfigure}\\[2pt]
        \begin{subfigure}[t]{0.22\textwidth}\centering\includegraphics[width=\linewidth,height=2.2cm]{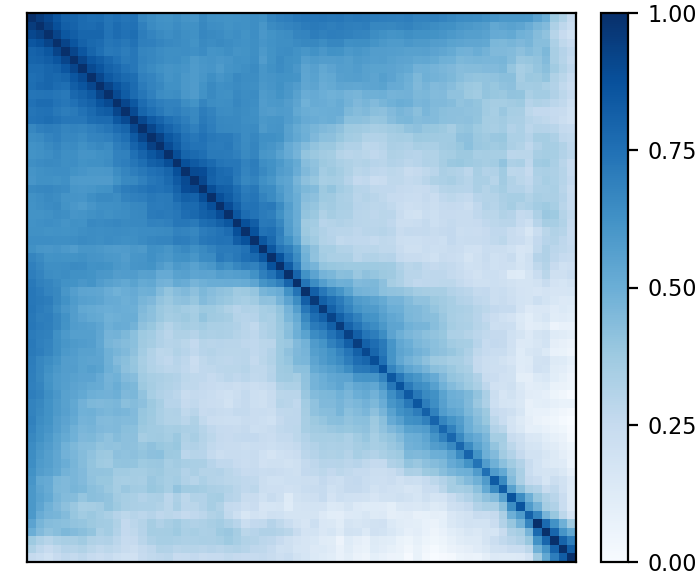}
        % \caption*{$N=1024$}
        \end{subfigure}\\[2pt]
        % \begin{subfigure}[t]{0.24\textwidth}\centering\includegraphics[width=\linewidth]{Figures/Rebuttal/Graphon Est/GPS/h1/collab_118.png}\caption*{\blue{COLLAB \\ $N=118$}}\end{subfigure}
        \begin{subfigure}[t]{0.22\textwidth}\centering\includegraphics[width=\linewidth,height=2.2cm]{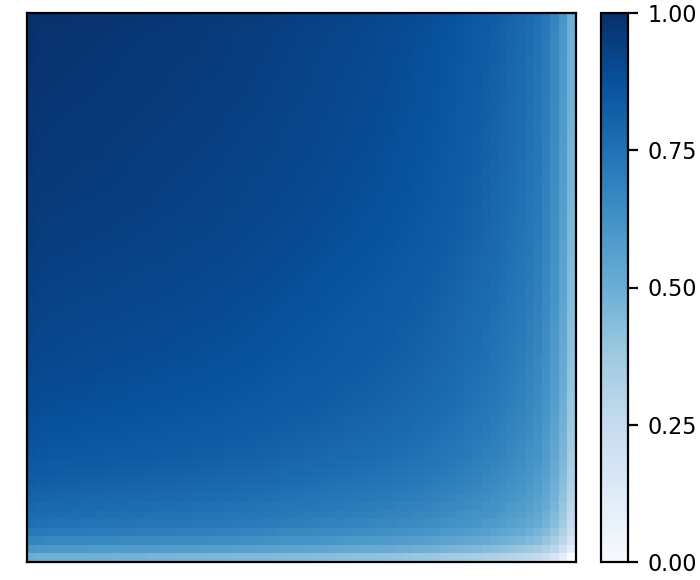}\captionsetup{justification=centering}
        % \caption*{COLLAB \\ $N=118$}
        \caption*{COLLAB}
        \end{subfigure}
        % \begin{subfigure}[t]{0.24\textwidth}\centering\includegraphics[width=\linewidth]{Figures/Rebuttal/Graphon Est/GPS/h1/proteins_70.png}\caption*{\blue{PROTEINS \\ $N=70$}}\end{subfigure}
        \begin{subfigure}[t]{0.22\textwidth}\centering\includegraphics[width=\linewidth,height=2.2cm]{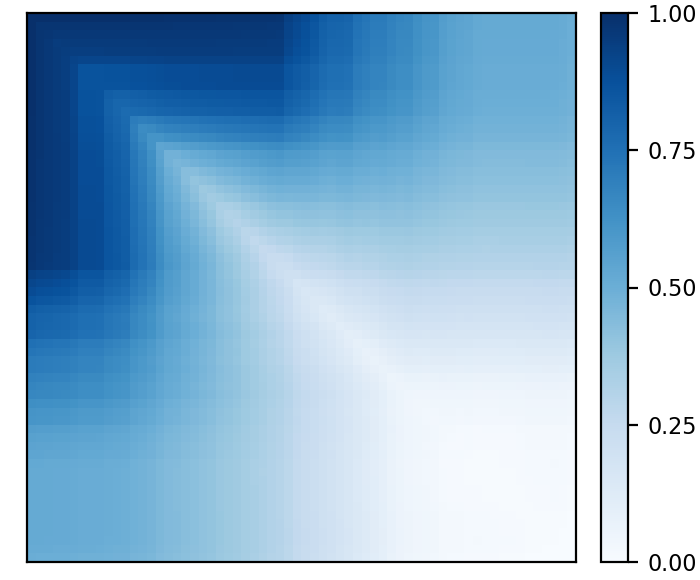}\captionsetup{justification=centering}
        % \caption*{PROTEINS \\ $N=70$}
        \caption*{PROTEINS}
        \end{subfigure}
        % \begin{subfigure}[t]{0.24\textwidth}\centering\includegraphics[width=\linewidth]{Figures/Rebuttal/Graphon Est/GPS/h1/redditmulti5k_705.png}\caption*{\blue{REDDIT-MULTI-5K \\ $N=705$}}\end{subfigure}
        \begin{subfigure}[t]{0.22\textwidth}\centering\includegraphics[width=\linewidth,height=2.2cm]{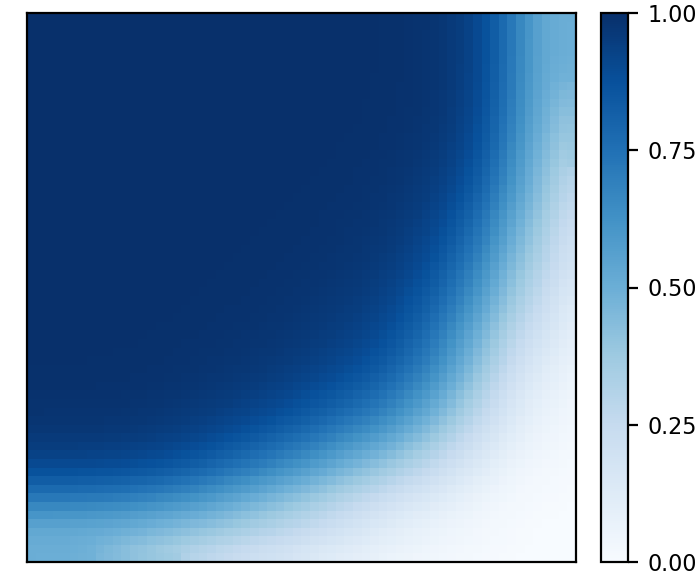}\captionsetup{justification=centering}
        % \caption*{REDDIT-MULTI-5K \\ $N=705$}
        \caption*{REDDIT-MULTI-5K}
        \end{subfigure}
        \caption{\small
        \textbf{Estimated attention graphons across datasets and graph sizes.} Each panel shows the $K{\times}K$ ($K=64$) block-step estimate of the 
        % dataset-level 
        attention kernel at target size $N$.
        % Top row: Attention graphon convergence for REDDIT-MULTI-5K, as target size $N$ grows ($N \in \{128, 256, 512, 1024\}$). As shown, kernel estimates stabilize as $N$ increases, while remaining distinct across datasets (as compared to the benchmarks below).
        \textit{Top row}: Attention graphon convergence for ModelNet10, as target size $N$ grows ($N \in \{128, 256, 512, 1024\}$). Kernel estimates stabilize as $N$ increases and remains distinct across datasets 
        % (as compared to the benchmarks below).
        \textit{Bottom row}: Attention graphon estimate 
        % at the largest target size 
        % for COLLAB, PROTEINS, and REDDIT-MULTI-5K (resp., $N = \{498, 600, 1024\}$). 
        for COLLAB, PROTEINS, and REDDIT-MULTI-5K (resp., $N = \{118, 70, 705\}$). 
        % Specifically, COLLAB ($N = 498$), PROTEINS ($N = 600$), and REDDIT-MULTI-5K ($N = 1024$).
        }
        \label{fig:est_graphons}
    \end{center}
    \vskip -0.2in
\end{figure*}

%To determine whether the attention-induced graphs come from the same graphon, we perform a hypothesis test comparing the theoretical and empirical cut distance variances. 
{We use the estimated graphon as a reference kernel to assess the consistency of
attention-induced graphs across sizes and verify our hypothesis that attention graphs for a given layer and head are samples from a common graphon. Specifically, we compute cut-norm
deviations between attention matrices and the estimated graphon, estimate their
empirical variance, and compare it against the theoretical variance bounds via
a one-sided hypothesis test.} 
Formally, given a set of attention-induced graphs $\{A_r\}_{r=1}^m$, represented by its adjacency matrices, distributed $A_i \sim W_i$ according to graphons $W_i$, we consider the following test:
\begin{align*}
    H_0\colon W_1 = \cdots = W_m && \textit{vs.} && H_a\colon \text{ {at least one} } W_i \neq W_j . %\ \text{for some} \ i\neq j,
\end{align*}
% The test statistic is the empirical variance of the cut-norm of the samples to the estimated graphon. The critical value is the upper bound in Thm.~\ref{thm:bound_2}. More practically, given that the smoothness parameter $\alpha$ is unknown, we select the value that gives the fastest rate (i.e., $\ccalO(n^{-1})$). Our decision rule is conservative: we reject $H_0$ when the empirical variance exceeds the theoretical bound, and fail to reject otherwise. We refer the reader to Appendix~\ref{app:add_exp} for additional calibration procedures (bootstrap, permutation test, and synthetic null simulation) to this test.
The test statistic is the empirical variance of the cut-norm of the samples to the estimated graphon. The critical value is the upper bound in Thm.~\ref{thm:bound_2}. More practically, given that the smoothness parameter $\alpha$ is unknown, we select the value that gives the fastest rate (i.e., $\ccalO(n^{-1})$). This is conservative, since we reject $H_0$ when the empirical variance exceeds the theoretical bound, and fail to reject otherwise. It is conservative in the sense that the fastest admissible rate yields the smallest threshold. Hence, the test is one-sided: a rejection implies that the attention graphs are not consistent with a common graphon, while a failure to reject indicates that the estimated kernels concentrate, without certifying a common limit, since low dispersion can also arise from nearly uniform attention. We refer the reader to Appx~\ref{app:add_exp} for additional calibration procedures (bootstrap, permutation test, and synthetic null simulation) to this test, and to Appx~\ref{app:method} for the constant hidden in the threshold.

\vskip -0.2in

\section{Experiments}
\label{sec:experiments}

%We validate the attention-graphon understanding developed in \Cref{sec:theory} and later put into practice in \Cref{sec:method}. 
Our experiments seek to address three questions:  
\textbf{(Q1)} Does the dispersion of attention in cut-norm decrease with size, and how does it behave compared with the regularity-aware proxy in \Cref{eq:regularity_proxy}?
\textbf{(Q2)} Do GT attention matrices stabilize to a dataset-driven kernel as graph size grows?
\textbf{(Q3)}  Can we use the attention-graphon relationship  to learn GTs more efficiently?

% estimated graphon to transfer the information learned on smaller attention to larger ones? Or can we assess the training dynamics or performance of the model on a specific dataset and downstream task?

%\subsection{Experimental Setup}
%\label{sec:exp-setup}
%Throughout our experiments, we consider the following experimental setup, with additional details provided in Appendix~\ref{app:exp_details}.

% To answer them, we use eight real-world graph classification benchmarks and three synthetic node-classification datasets: four bioinformatics datasets (MUTAG \cite{mutag1,mutag2}, PROTEINS \cite{proteins2,proteins1}, NCI1, NCI109 \cite{nci1,nci12}), three social networks (IMDB-MULTI, COLLAB, REDDIT-MULTI-5K \cite{socialnet}), one point cloud (ModelNet10 \cite{modelnet}), a sparse contextual SBM (NoisyCSBM \cite{deshpande2018contextual}), and two ground truth graphons (SmoothW and SharpW).
% detailed in Appendix~\ref{app:exp_details}). 

% To answer them, we use eight real-world graph classification benchmarks and three synthetic node-classification datasets: four bioinformatics datasets (MUTAG, PROTEINS, NCI1, NCI109), three social networks (IMDB-MULTI, COLLAB, REDDIT-MULTI-5K), one point cloud (ModelNet10), a sparse contextual SBM (NoisyCSBM), and two ground truth graphons (SmoothW and SharpW) ---
To answer them, we use eight real-world graph classification benchmarks and three synthetic node-classification datasets: MUTAG, PROTEINS, NCI1, NCI109, IMDB-MULTI, COLLAB, REDDIT-MULTI-5K, ModelNet10, a sparse contextual SBM (NoisyCSBM), and two ground truth graphons (SmoothW and SharpW) ---
% details in Appendix~\ref{app:exp_details}.
for details and extensive experiments spanning additional architectures, large-scale datasets, tasks, transferability, multi-head attention, statistical calibration (bootstrap CIs, permutation and synthetic-null tests), runtime, and sensitivity/robustness analyses, see Appxs~\ref{app:exp_details}, \ref{app:add_exp}.

\vskip -0.2in
\begin{comment}
Beyond the experiments shown in the main body, \Cref{app:add_exp} contains an extensive empirical study spanning additional architectures (Graphormer-GD, GRIT) \cite{zhang2024rethinkingexpressivepowergnns, ma2023graphinductivebiasestransformers}, large-scale graphs (OGB, LRGB) \cite{hu2020open,dwivedi2022LRGB}, tasks (CLRS algorithmic reasoning) \cite{velivckovic2022clrs}, transferability, multi-head attention, statistical calibration (bootstrap CIs, permutation and synthetic-null tests), runtime, and sensitivity/robustness analyses.
\end{comment}
% Beyond these, Appendix~\ref{app:add_exp} contains an extensive empirical study spanning additional architectures, large-scale datasets, tasks, transferability, multi-head attention, statistical calibration (bootstrap CIs, permutation and synthetic-null tests), runtime, and sensitivity/robustness analyses.

% In Appendix~\ref{app:exp_details} we provide details on our experimental settings: from the datasets utilized, to the GT used in our experiments and the evaluation on target graph sizes.

\begin{figure*}[t]
    \begin{center}
        {\includegraphics[width=0.99\linewidth]{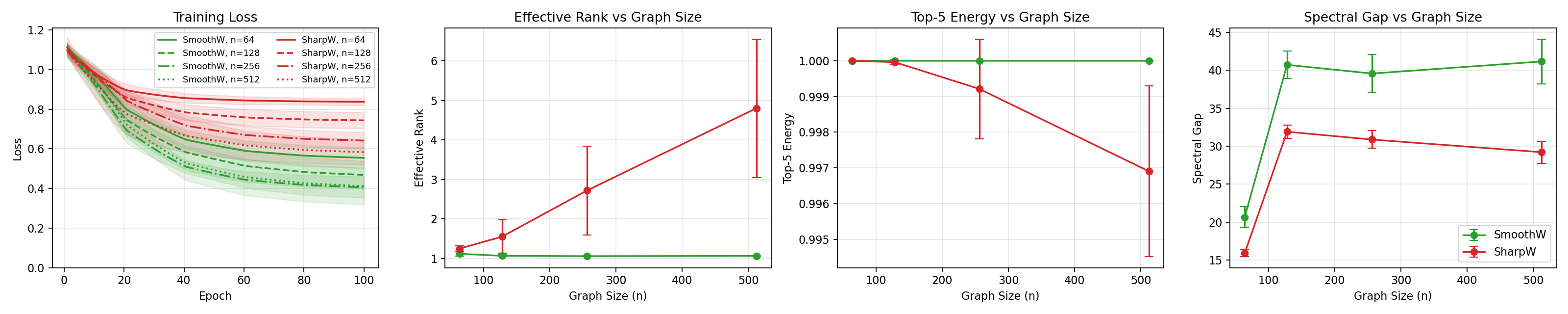}}
        \caption{\small
        \textbf{Training dynamics and spectral summaries under synthetic graphon families.}
        \textit{Left}: training loss for \green{SmoothW} and \red{SharpW} at increasing sizes.
        \textit{Middle/right}: spectral summaries of the estimated 
        % $K\times K$ 
        attention graphons (effective rank, top-5 energy, and
        spectral gap) as size grows.
        % Definitions of the metrics are given in Appendix~\ref{app:spectral}.
        Metrics are defined in Appx~\ref{app:spectral}.
        }
        \label{fig:train_dyn_vs_spectral_metrics}
    \end{center}
    \vskip -0.2in
\end{figure*}

\
\begin{table*}[t]
\small
\centering
\caption{\small Lipschitz constants and test accuracy (\%) of estimated attention graphons on synthetic and real-world datasets. For synthetic datasets, the Lipschitz constant correlates with task difficulty. For real-world datasets, it reflects the complexity of learned attention. %\textcolor{red}{lower lipschitz indiactes easier ? if so say it} \blue{C: Only true for the synthetic datasets. This is explained in the text. Also, the correlation coefficient and p-value show it, which we describe in here.}
}
\label{tab:lipschitz_results}
\resizebox{\textwidth}{!}{%
\begin{tabular}{l|ccc|cccccccc}
\toprule
\multirow{2}{*}{Metric $\downarrow$ / Dataset $\rightarrow$ } & \multicolumn{3}{c|}{\textbf{Synthetic}} & \multicolumn{8}{c}{\textbf{Real-World}} \\
\cmidrule(lr){2-4} \cmidrule(lr){5-12}
& \textsc{SmoothW} & \textsc{SharpW} & \textsc{NoisyCSBM} & \textsc{MUTAG} & \textsc{Proteins} & \textsc{NCI1} & \textsc{NCI109} & \textsc{Collab} & \textsc{IMDB-M} & \textsc{Reddit-M-5K} & \textsc{ModelNet} \\
\midrule
Graph size $n$ & 1024 & 1024 & 1024 & 28 & 600 & 111 & 111 & 498 & 89 & 1024 & 1024 \\
Lipschitz ($\hat{L}_W$) & 25.02 & 45.29 & 64.60 & 6.36 & 54.80 & 27.42 & 28.34 & 33.15 & 20.24 & 28.14 & 33.70 \\
Accuracy (\%) $\uparrow$ & 92.18 & 74.10 & 58.98 & 78.95 & 77.68 & 66.42 & 70.46 & 74.60 & 56.00 & 54.20 & 78.50 \\
\midrule
Correlation ($r$) & \multicolumn{3}{c|}{$-0.999$ ($p = 0.024$)} & \multicolumn{8}{c}{$0.207$ ($p = 0.622$)} \\
\bottomrule
\end{tabular}%
}
\vskip -0.15in
\end{table*}

\subsection{Estimated Attention Graphons}
\label{sec:exp-graphons}
\textbf{(i) Cut-norm variance \emph{vs.} theoretical scaling.}
\label{sec:exp-variance} 
%We quantify attention concentration via the empirical variance of $\Delta_r=\delta_\square(\hat W_r,\hat W)$ across samples (according to the procedure described in Appendix \ref{app:method}) and compare it to the theoretical bound from \Cref{eq:regularity_proxy}. Figure~\ref{fig:cutvar_selected} reports results for representative datasets of each domain considered (bioinformatics, social networks, point clouds, and SBM).  Across all datasets, empirical variance decreases as graph size increases (see \Cref{app:fig_cutvar_all}), and the theoretical bound is almost always bigger than the empirical counterpart, which means we fail to reject the null hypothesis, i.e., the graphs should come from the same graphon. That supports concentration of attention around a stable kernel and answers \texbf{(Q1)}.
% We quantify attention concentration via the empirical variance of $\Delta_r=\|\hat W_r - \hat W\|_\square$ across samples (Appendix \ref{app:method}) and compare it to the theoretical bound in \Cref{eq:regularity_proxy}. Fig.~\ref{fig:cutvar_selected} shows representative datasets from each domain (bioinformatics, social networks, point clouds, and cSBM). Across most of the real-world datasets, empirical variance decreases with graph size (see Fig.~\ref{app:fig_cutvar_all}), and the theoretical bound is almost always larger than the empirical counterpart, so we fail to reject the null hypothesis that the graphs come from the same graphon. This supports concentration of attention around a stable kernel and answers \textbf{(Q1)}.
We quantify attention concentration via the empirical variance of $\Delta_r=\|\hat W_r - \hat W\|_\square$ across samples (Appx \ref{app:method}) and compare it to the theoretical bound in \Cref{eq:regularity_proxy}. Fig.~\ref{fig:cutvar_selected} shows representative datasets from each domain (bioinformatics, point clouds, and social networks). On most datasets, the empirical variance decreases with graph size (see Fig.~\ref{app:fig_cutvar_all}), and on every dataset except REDDIT-MULTI-5K it stays below the theoretical bound at all sizes, so we fail to reject the null hypothesis that the graphs come from the same graphon. REDDIT-MULTI-5K exceeds the bound at its two largest sizes, where the test rejects. This supports concentration of attention around a stable kernel and answers \textbf{(Q1)}.
\textbf{(ii) Convergence to attention graphon.} In 
Fig.~\ref{fig:est_graphons}, we show the estimated attention graphons for 4 datasets with large graphs,  %PROTEINS, COLLAB, ModelNet10, REDDIT-MULTI-5K.
 % a design choice  to reduce small sample noise effects.
%, and it also represents the various domains considered in our experiments. 
Our results show that the estimated kernel becomes more consistent as $N$ increases, indicating that attention is not an arbitrary dense matrix but exhibits a stable %large-$n$ 
global structure. Across datasets, the limiting kernels are qualitatively distinct, supporting the view that GT attention induces dataset-dependent interaction kernels, providing an affirmative answer to \textbf{(Q2)}.

\subsection{Transferability, Training Dynamics, and Task-Hardness}
\label{sec:exp-synth}
%To answer \textbf{(Q3)}, we study whether the graphon structure of attention matrices can be leveraged for practical benefits. 
We  address \textbf{(Q3)} 
% --- whether the graphon structure of attention matrices can be leveraged in practice --- 
from three perspectives: transferability, training dynamics, and task-hardness.
%inference. % (i) \emph{transferability}, i.e., whether models trained on small graphs can transfer to larger graphs via graphon sampling; (ii) \emph{training dynamics}, i.e., how kernel regularity influences optimization behavior and spectral properties; and (iii) \emph{task-hardness inference}, i.e., whether the Lipschitz constant of the estimated graphon provides insight into task difficulty or model behavior. 

\textbf{Setup.} For the \emph{transferability} and \emph{training dynamics} experiments, we employ two synthetic graphon families: \emph{SmoothW}, a Hölder-smooth kernel defined as $W(x,y) = xy$, and \emph{SharpW}, a kernel with sharper transitions defined as $W(x,y) = \frac{1}{2}(e^{-xy} + |\sin(\omega(x+y))|)$ for frequency $\omega \gg 1$. 
%To generate graphs from these models, we sample $n$ latent positions $x_1, \ldots, x_n \sim \mathrm{Unif.}(0,1)$ and connect nodes $i$ and $j$ independently with probability $W(x_i, x_j)$. 
Node labels for the downstream classification task are derived from the latent positions via terciles, yielding three balanced classes. Node features are random walk positional encodings. 
For the \emph{task-hardness} experiments, our analysis includes both synthetic 
% (SmoothW, SharpW, and NoisyCSBM) 
and real-world datasets.

\begin{wrapfigure}{r}{0.43\textwidth}%0.48
    \vspace{-13pt}
    \centering
    {\includegraphics[width=0.42\textwidth,height=3.2cm]{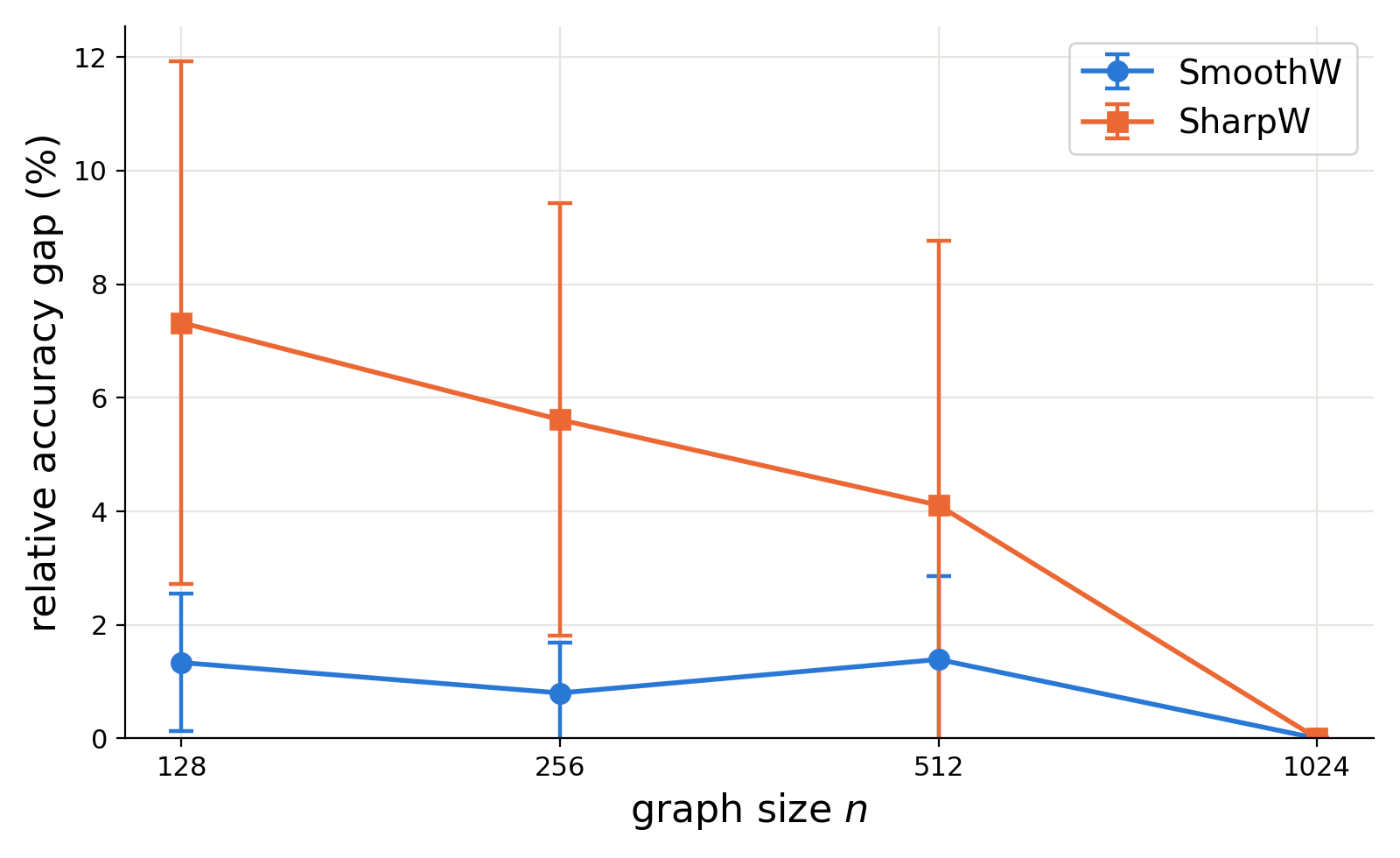}}%0.48
        \caption{\small\textbf{Attention graphon transferability to larger graphs.}
        %Accuracy comparison between a baseline model trained on larger graphs and a model trained on smaller graph sizes for datasets built from the graphon models, SmoothW and SharpW (as explained in the transferability section). Results show that one can get close to the model trained on larger datasets by training on smaller ones and transferring them, which is faster and computationally less expensive. Furthermore, the smoother graphon (SmoothW) has better transferability properties.
     Relative test-accuracy gap between a model trained and evaluated on small graphs and the same model evaluated on a large-scale test set.
     % baseline trained on larger graphs and a model trained on smaller graph sizes for datasets from the graphon models SmoothW and SharpW 
     % (as explained in the transferability section).
     Results show that training on smaller graphs and transferring can approach the performance of a model trained on larger datasets, while being faster and less computationally intensive. Furthermore, the smoother graphon 
     % (SmoothW) 
     has better transferability.
        }
    \vspace{-13pt}
        \label{fig:transfer}
\end{wrapfigure}

\textbf{(i) Transferability.} 
We evaluate whether GTs can transfer across graph sizes by training models on small graphs, estimating the attention graphon, and using them to generate attention matrices for larger graphs at inference. 
%Using the synthetic datasets obtained by sampling the SmoothW and SharpW graphon models, 
We first generate a fixed baseline dataset of graphs with $n=1024$ nodes from each graphon model.
% and train a baseline model on it
% Then, for each $n \in \{128, 256, 512, 1024\}$, we generate a separate dataset of graphs with $n$ nodes from the same graphon and train a model on this dataset. At transfer time, we bypass the query-key mechanism by sampling attention scores directly from the attention graphon estimated from the smaller graphs. 
Then, for each $n \in \{128, 256, 512, 1024\}$, we generate a separate dataset of graphs with $n$ nodes from the same graphon and train a model on this dataset. At transfer time, we bypass the query-key mechanism by sampling attention scores from the attention graphon estimated from the smaller graphs. Sampled kernel values are mapped back to attention scores by applying $\rho^{-1}$ and the row-wise softmax (see Prop.~\ref{prop:invertible}). 
%These sampled attention scores are combined with 
The value projections of the larger graphs are obtained using the frozen value matrix from the small-graph model.
%, to produce the final attention output. T
The relative accuracy gap is computed against the baseline dataset.
%trained directly on the dataset with the largest graph size. Notably, when $n=1024$, although the training and target graph sizes coincide, the graph samples differ.
Fig.~\ref{fig:transfer} shows that the accuracy gap decreases as the training graph size increases.
%, converging to zero for both datasets at $n=1024$. 
%Notably, 
SmoothW exhibits smaller transfer error throughout, consistent with its smaller Lipschitz constant. 
% (cf. Theorem \ref{thm:bound_2}).
%, enabling more reliable graphon estimation and sampling.

These results show that a model trained on smaller graphs, when evaluated at larger scales, can approach the accuracy of a baseline trained on the large-scale dataset, reducing computations without sacrificing performance. The query-key product costs $\ccalO(n^2d)$ for every input, whereas sampled attention is block-constant and can be applied to the values in $\ccalO(nd_v+K^2d_v)$ operations, linear in $n$, where $d_v$ is the value dimension (\Cref{tab:factorized} in Appx~\ref{app:add_exp}). In other words, graphon structure can be exploited for efficient inference with GTs; this answers \textbf{(Q3)} in the affirmative.
%size-agnostic deployment.

% \blue{This behavior is expected under our hypothesis. If attention matrices are samples from a graphon $W$, transferability of the induced graph filter follows from results for graphon operators \citep{ruiz2021transferability}, and the curves in \Cref{fig:transfer} are evidence that this premise holds. Conversely, transfer is not expected on datasets whose attention the variance test rejects.}
This behavior is expected under our hypothesis. If attention matrices are samples from a graphon $W$, transferability of the induced graph filter follows from results for graphon operators \citep{ruiz2021transferability}, and the curves in \Cref{fig:transfer} are evidence that this premise holds. The converse may also hold. On REDDIT-MULTI-5K, the only dataset where the test rejects, the transferability gap stays between $23\%$ and $44\%$ and does not shrink with the training size (\Cref{fig:transferability_gps_h1_2}, Appx~\ref{app:add_exp}).

\textbf{(ii) Training dynamics.} 
% Next, we empirically probe the link between kernel regularity and learnability using SmoothW and SharpW. For each dataset and graph size $n \in \{64, 128, 256, 512\}$, we train GT models for node classification and record: (1) training curves; and (2) spectral properties of the estimated attention graphon, including effective rank, top-$5$ energy concentration, and spectral gap.
Next, we probe the link between kernel regularity and learnability on SmoothW and SharpW, training GT models for node classification at $n\in\{64,128,256,512\}$ and recording training curves and spectral properties of the estimated attention graphons (effective rank, top-$5$ energy, and spectral gap). SharpW yields higher training loss than SmoothW at every size, and its estimated attention graphons have higher effective rank and lower top-$5$ energy (Fig.~\ref{fig:train_dyn_vs_spectral_metrics}), suggesting that less regular kernels are harder to learn.

\textbf{(iii) Task-hardness.}
We also examine whether the Lipschitz constant of the estimated attention graphon, computed as in Appx~\ref{app:lips_comp}, tracks task difficulty (Table~\ref{tab:lipschitz_results}). On the synthetic families, smoother graphons come with higher accuracy, in line with the training dynamics above. On real-world datasets the relation is more nuanced, and no clear correlation emerges, as these tasks also depend on node features and domain-specific factors that the graphon alone does not capture. We therefore read these results as a link between graphon regularity and learnability on synthetic data, not as a general proxy for task difficulty.

\section{Discussion and Limitations}
\label{sec:discussion}

%Our results in \Cref{sec:experiments} present a clear empirical picture: for a fixed trained GT, each attention head induces a dense weighted graphwhose large-$n$ behavior is often well described by a dataset-dependent graphon, and whose dispersion in cut distance decreases with graph size. This provides a principled way to summarize and compare attention patterns across datasets and scales. In what follows, we discuss additional aspects that stem from the understanding and experiments, providing interesting directions for future work.

Results in Sec.~\ref{sec:experiments} show that, for a fixed trained GT, each attention head induces a dense weighted graph whose large-$n$ behavior is often well described by a dataset-dependent graphon, with dispersion in cut-norm decreasing with graph size. This provides a principled way to summarize and compare attention patterns across datasets and scales. We next discuss additional aspects from this understanding and experiments, and directions for future work.

%\textbf{Can the graphon viewpoint fail?}
%The attention-graphon perspective developed in \Cref{sec:theory} is a modeling assumption; it can be violated when attention depends on size or on discrete structures that are not captured by a single size-invariant kernel.
%Empirically, non-monotonic variance curves can flag such regimes, as shown on \Cref{fig:cutvar_selected}. However, they can also arise from finite-sample effects (induced-subgraph sampling) and from imperfect canonicalization, as \Cref{tab:lipschitz_results} suggests. A promising direction for future work is to employ resolution-free and basis-invariant neural estimators, such as IGNR \cite{xia2023implicit} or MomentNet \citep{ramezanpour2025few}, which could mitigate canonicalization artifacts and yield more robust graphon estimates.
\textbf{Can the graphon viewpoint fail?}
The attention-graphon perspective in Sec.~\ref{sec:theory} is a modeling assumption, and can be violated when attention depends on size or on discrete structures not captured by a single size-invariant kernel. Empirically, non-monotonic variance curves can flag such regimes (Fig.~\ref{fig:cutvar_selected}), but they can also arise from finite-sample effects (induced-subgraph sampling) and imperfect canonicalization. 
% as \Cref{tab:lipschitz_results} suggests. 
A promising direction are resolution-free, basis-invariant neural estimators (e.g., \citep{xia2023implicit,ramezanpour2025few}) to mitigate canonicalization artifacts.
%and yield more robust graphon estimates.

% \textbf{When it might fail.}
% Collecting the failure modes discussed above, the attention-graphon description should not be relied upon when attention depends on graph size, when it encodes discrete motif structure that no single size-invariant kernel captures, when interactions are strongly asymmetric, or when the degree profile is degenerate and canonicalization cannot recover a latent ordering.
% Sparse attention is also where our test can fail. When attention is restricted to the input edges, the pre-softmax scores are $-\infty$ elsewhere, so the mass of $\smash{A^{(\ell,h)}}$ vanishes with the edge density and the induced graphons converge to the zero graphon. The variance vanishes with the mass, and the test does not reject (\Cref{fig:masked_csbm} in Appx.~\ref{app:add_exp}). Monitoring the kernel mass alongside the variance flags this regime.
\textbf{When it might fail.}
The attention-graphon view is less reliable for size-dependent or strongly asymmetric attention, for discrete motifs that no size-invariant kernel captures, and when a degenerate degree profile keeps canonicalization from recovering a latent ordering.
Our test can also fail under sparse attention: pre-softmax scores of $-\infty$ off the input edges make the mass of $\smash{A^{(\ell,h)}}$ vanish with edge density, so the induced graphons converge to the zero graphon, the variance vanishes, and the test does not reject (\Cref{fig:masked_csbm} in Appx.~\ref{app:add_exp}). Tracking the kernel mass flags this regime.

%\textbf{Canonicalization and directed attention.}
%In this work, we use degree sorting as a practical surrogate for the optimal relabeling in the cut metric.
%This is not guaranteed to recover the latent ordering of the underlying kernel, and alternative canonicalizations (e.g., based on spectral embeddings or optimal transport) are a promising direction.
%We also symmetrize attention to fit the standard (undirected) graphon formalism; extending the framework to directed kernels would allow using raw row-stochastic attention matrices.

\textbf{Canonicalization and directed attention.}
We use degree sorting as a practical surrogate for the optimal relabeling in the cut-metric, but it is not guaranteed to recover the latent ordering of the underlying kernel; alternative canonicalizations (e.g., spectral embeddings or optimal transport) are a promising direction. We also symmetrize attention to fit the standard (undirected) graphon formalism; extending the framework to directed kernels would allow using raw row-stochastic attention matrices.

%\textbf{Dense \textit{vs.} sparse limits.}
%Although attention matrices are dense by construction, the input graphs are often sparse. Our framework remains applicable since the attention mechanism produces dense pairwise interactions regardless of input sparsity, yet understanding how sparse input structure, positional encodings, and attention sparsification schemes affect the induced attention kernels is a compelling open question. Extending the theory to sparse regimes could inform the design of more interpretable attention patterns guided by input topology, while also enabling computationally efficient architectures with reduced memory overhead.
\textbf{Dense \textit{vs.} sparse limits.}
Although attention matrices are dense by construction, input graphs are often sparse. Our framework remains applicable because attention produces dense pairwise interactions regardless of input sparsity, yet understanding how sparse input structure, positional encodings, and attention sparsification schemes affect the induced attention kernels is an open question. 
%Extending the theory to sparse regimes could inform more interpretable attention patterns guided by input topology, and enable computationally efficient architectures with reduced memory overhead.

\textbf{Which heads to model.}
Since the diagnostic is applied per head, it also gives a selection rule in practice. A head is modeled by its attention graphon only when the test does not reject, and full attention is kept for the remaining heads. Because the kernels of different heads lie in a common function space, distances between them also measure head similarity, and heads with close kernels are natural candidates for merging or pruning, which we leave for future work.

\section{Conclusion}
\label{sec:conclusion}

We proposed a graphon lens that views GT attention matrices as dense weighted graphs, with cut-distance and cut-norm concentration calibrations and a pipeline that estimates dataset-level attention graphons and their stability across sizes. Across benchmarks, these graphons stabilize and cut-norm variance typically decreases with size, supporting the usefulness graph-limit tools for attention.

\bibliographystyle{iclr2027_conference}
\bibliography{refs,myIEEEabrv,bib-cumulative}

@article{cervino2021increase,
  	title="Learning by transference: {T}raining graph neural networks on growing graphs",
  	author="Cervino, J. and Ruiz, L. and Ribeiro, A.",
  	journal=IEEE_J_SP,
  	volume="71",
        pages="233--247",
        year="2023",
        publisher="IEEE"
}

@article{ruiz2021transferability,
  	author={Ruiz, Luana and Chamon, Luiz F. O. and Ribeiro, Alejandro},
  journal={IEEE Transactions on Signal Processing}, 
  title={Transferability Properties of Graph Neural Networks}, 
  year={2023},
  volume={71},
  number={},
  pages={3474-3489},
  doi={10.1109/TSP.2023.3297848}}

@inproceedings{ruiz20-transf,
	title="Graphon Neural Networks and the Transferability of Graph Neural Networks",
	author="Ruiz, L. and Chamon, L. F. O. and Ribeiro, A.",
	booktitle="34th " # NeurIPS,
	address="Vancouver, BC (Virtual)",
	year	="2020",
	month="6-12 " # DEC,
	publisher="NeurIPS Foundation"
}

@book{lovasz2012large,
  	author="Lov{\'a}sz, L.",
  	title="Large Networks and Graph Limits",
  	volume="60",
  	year="2012",
  	publisher="American Mathematical Society"
}

@article{olhede2014network,
  title={Network histograms and universality of blockmodel approximation},
  author={Olhede, Sofia C and Wolfe, Patrick J},
  journal={Proceedings of the National Academy of Sciences},
  volume={111},
  number={41},
  pages={14722--14727},
  year={2014},
  publisher={National Academy of Sciences}
}

@inproceedings{airoldi2013stochastic,
	title="Stochastic Blockmodel Approximation of a Graphon: Theory and Consistent Estimation",
  	author="Airoldi, E. M. and Costa, T. B. and Chan, S. H.",
  	booktitle="27th " # NIPS,
  	pages="692-700",
  	year="2013",
  	publisher="NIPS Foundation"
}

@article{lovasz2006limits,
 	author="Lov{\'a}sz, L. and Szegedy, B.",
  	title="Limits of Dense Graph Sequences",
  	journal="J. Comb. Theory, Series B",
 	volume="96",
  	number="6",
  	pages="933-957",
  	year="2006",
  	publisher="Elsevier"
}

@article{gao2015rate,
  	author="Gao, C. and Lu, Y. and Zhou, H. H.",
  	title="Rate-optimal Graphon Estimation",
  	journal="Ann. Stat.",
  	volume="43",
  	number="6",
  	pages="2624-2652",
  	year="2015",
  	publisher="Institute of Mathematical Statistics"
}

@article{borgs2008convergent,
  	author="Borgs, C. and Chayes, J. T. and Lov{\'a}sz, L. and S{\'o}s, V. T. and Vesztergombi, K.",
  	title="Convergent Sequences of Dense Graphs {I}: {S}ubgraph Frequencies, Metric Properties and Testing",
  	journal="Adv. Math.",
  	volume="219",
  	number="6",
  	pages="1801-1851",
  	year="2008",
  	publisher="Elsevier"
}

@article{levie2019transferability,
  	title="Transferability of Spectral Graph Convolutional Neural Networks",
  	author="Levie, R. and Huang, W. and Bucci, L. and Bronstein, M. and Kutyniok, G.",
  	journal=JMLR,
  	volume="22",
  	number="272",
  	pages="1-59",
  	year="2021"
}

@inproceedings{vaswani2017attention,
  	title="Attention is All You Need",
  	author="Vaswani, A. and Shazeer, N. and Parmar, N. and Uszkoreit, J. and Jones, L. and Gomez, A. N. and Kaiser, {\L}. and Polosukhin, I.",
  	booktitle=NIPS,
  	pages="5998-6008",
  	year="2017",
  	publisher="NIPS Foundation"
}

@inproceedings{xu2021learning,
  title={Learning graphons via structured gromov-wasserstein barycenters},
  author={Xu, Hongteng and Luo, Dixin and Carin, Lawrence and Zha, Hongyuan},
  booktitle={Proceedings of the AAAI Conference on Artificial Intelligence},
  volume={35},
  number={12},
  pages={10505--10513},
  year={2021}
}

@inproceedings{xia2023implicit,
  title={Implicit graphon neural representation},
  author={Xia, Xinyue and Mishne, Gal and Wang, Yusu},
  booktitle={International Conference on Artificial Intelligence and Statistics},
  pages={10619--10634},
  year={2023},
  organization={PMLR}
}

@article{ramezanpour2025few,
  title={A Few Moments Please: Scalable Graphon Learning via Moment Matching},
  author={Ramezanpour, Reza and Tenorio, Victor M and Marques, Antonio G and Sabharwal, Ashutosh and Segarra, Santiago},
  journal={arXiv preprint arXiv:2506.04206},
  year={2025}
}

@STRING{IEEE_J_SP         = "{IEEE} Trans. Signal Process."}

@STRING{ICML		   = "Int. Conf. Mach. Learning"}

@STRING{NIPS		   = "Neural Inform. Process. Syst."}

@STRING{ICLR		   = "Int. Conf. Learning Representations"}

@STRING{AAAI		   = "Assoc. Advancement Artificial Intell."}

@STRING{JAN		   = "Jan."}

@STRING{NOV		   = "Nov."}

@STRING{DEC		   = "Dec."}

@STRING{SIXTH		   = "6th"}

@STRING{JMLR  = "J. Mach. Learning Res."}

@article{dwivedi2021generalization,
  title={A Generalization of Transformer Networks to Graphs},
  author={Dwivedi, Vijay Prakash and Bresson, Xavier},
  journal={AAAI Workshop on Deep Learning on Graphs: Methods and Applications},
  year={2021}
}

@inproceedings{velickovic2018graph,
  title     = {Graph Attention Networks},
  author    = {Veli\v{c}kovi\'{c}, Petar and Cucurull, Guillem and Casanova, Arantxa and Romero, Adriana and Li\`{o}, Pietro and Bengio, Yoshua},
  booktitle = {International Conference on Learning Representations},
  year      = {2018},
  url       = {https://openreview.net/forum?id=rJXMpikCZ}
}

@inproceedings{xu2019how,
  title     = {How Powerful are Graph Neural Networks?},
  author    = {Xu, Keyulu and Hu, Weihua and Leskovec, Jure and Jegelka, Stefanie},
  booktitle = {International Conference on Learning Representations},
  year      = {2019},
  url       = {https://openreview.net/forum?id=ryGs6iA5Km}
}

@inproceedings{morris2019weisfeiler,
  title     = {Weisfeiler and Leman Go Neural: Higher-order Graph Neural Networks},
  author    = {Morris, Christopher and Ritzert, Martin and Fey, Matthias and Hamilton, William L. and Lenssen, Jan Eric and Rattan, Gaurav and Grohe, Martin},
  booktitle = {Proceedings of the AAAI Conference on Artificial Intelligence},
  year      = {2019}
}

@inproceedings{ying2021do,
  title     = {Do Transformers Really Perform Badly for Graph Representation?},
  author    = {Ying, Chengxuan and Cai, Tianle and Luo, Shengjie and Zheng, Shuxin and Ke, Guolin and He, Di and Shen, Yanming and Liu, Tie-Yan},
  booktitle = {Advances in Neural Information Processing Systems},
  volume    = {34},
  year      = {2021},
  url       = {https://proceedings.neurips.cc/paper/2021/hash/f1c1592588411002af340cbaedd6fc33-Abstract.html}
}

@inproceedings{rampasek2022recipe,
  title     = {Recipe for a General, Powerful, Scalable Graph Transformer},
  author    = {Ramp\'{a}\v{s}ek, Ladislav and Galkin, Michael and Dwivedi, Vijay Prakash and Luu, Anh Tuan and Wolf, Guy and Beaini, Dominique},
  booktitle = {Advances in Neural Information Processing Systems},
  volume    = {35},
  year      = {2022}
}

@article{diaconis2008graph,
  title   = {Graph limits and exchangeable random graphs},
  author  = {Diaconis, Persi and Janson, Svante},
  journal = {Rendiconti di Matematica e delle sue Applicazioni. Serie {VII}},
  volume  = {28},
  pages   = {33--61},
  year    = {2008},
  url     = {https://www.diva-portal.org/smash/record.jsf?pid=diva2:224369}
}

@inproceedings{chan2014consistent,
  title     = {A Consistent Histogram Estimator for Exchangeable Graph Models},
  author    = {Chan, Stanley H. and Airoldi, Edoardo M.},
  booktitle = {Proceedings of the 31st International Conference on Machine Learning},
  series    = {Proceedings of Machine Learning Research},
  volume    = {32},
  pages     = {208--216},
  year      = {2014},
  url       = {https://proceedings.mlr.press/v32/chan14.html}
}

@inproceedings{fey2019pyg,
  title={Fast Graph Representation Learning with {PyTorch Geometric}},
  author={Fey, Matthias and Lenssen, Jan E.},
  booktitle={ICLR Workshop on Representation Learning on Graphs and Manifolds},
  year={2019},
}

@article{kreuzer2021rethinking,
  title={Rethinking graph transformers with spectral attention},
  author={Kreuzer, Devin and Beaini, Dominique and Hamilton, Will and L{\'e}tourneau, Vincent and Tossou, Prudencio},
  journal={Advances in Neural Information Processing Systems},
  volume={34},
  pages={21618--21629},
  year={2021}
}

@inproceedings{hussain2022global,
  title={Global self-attention as a replacement for graph convolution},
  author={Hussain, Md Shamim and Zaki, Mohammed J and Subramanian, Dharmashankar},
  booktitle={Proceedings of the 28th ACM SIGKDD conference on knowledge discovery and data mining},
  pages={655--665},
  year={2022}
}

@article{muller2023attending,
  title={Attending to graph transformers},
  author={M{\"u}ller, Luis and Galkin, Mikhail and Morris, Christopher and Ramp{\'a}{\v{s}}ek, Ladislav},
  journal={arXiv preprint arXiv:2302.04181},
  year={2023}
}

@article{min2022transformer,
  title={Transformer for graphs: An overview from architecture perspective},
  author={Min, Erxue and Chen, Runfa and Bian, Yatao and Xu, Tingyang and Zhao, Kangfei and Huang, Wenbing and Zhao, Peilin and Huang, Junzhou and Ananiadou, Sophia and Rong, Yu},
  journal={arXiv preprint arXiv:2202.08455},
  year={2022}
}

@INPROCEEDINGS{olivier2007effective,
  author={Roy, Olivier and Vetterli, Martin},
  booktitle={2007 15th European Signal Processing Conference}, 
  title={The effective rank: A measure of effective dimensionality}, 
  year={2007},
  volume={},
  number={},
  pages={606-610},
  doi={}
}

@misc{cutnorm_git,
  author       = {Chiu, Ping-Ko and Cavallo, Gaston},
  title        = {Cutnorm: Approximation via Gaussian Rounding and Optimization with Orthogonality Constraints},
  year         = 2018,
  howpublished = {\url{https://github.com/pingkoc/cutnorm}},
  note         = {Accessed: 01-01-2026}
}

@inproceedings{modelnet,
  title={3d shapenets: A deep representation for volumetric shapes},
  author={Wu, Zhirong and Song, Shuran and Khosla, Aditya and Yu, Fisher and Zhang, Linguang and Tang, Xiaoou and Xiao, Jianxiong},
  booktitle={Proceedings of the IEEE conference on computer vision and pattern recognition},
  pages={1912--1920},
  year={2015}
}

@article{deshpande2018contextual,
  title={Contextual stochastic block models},
  author={Deshpande, Yash and Sen, Subhabrata and Montanari, Andrea and Mossel, Elchanan},
  journal={Advances in Neural Information Processing Systems},
  volume={31},
  year={2018}
}

@article{mutag1,
author = {Debnath, Asim Kumar and Lopez de Compadre, Rosa L. and Debnath, Gargi and Shusterman, Alan J. and Hansch, Corwin},
title = {Structure-activity relationship of mutagenic aromatic and heteroaromatic nitro compounds. Correlation with molecular orbital energies and hydrophobicity},
journal = {Journal of Medicinal Chemistry},
volume = {34},
number = {2},
pages = {786-797},
year = {1991},
doi = {10.1021/jm00106a046},
URL = { https://doi.org/10.1021/jm00106a046},
eprint = {https://doi.org/10.1021/jm00106a046}
}

@inproceedings{mutag2,
  title={Subgraph matching kernels for attributed graphs},
  author={Kriege, Nils and Mutzel, Petra},
  booktitle={Proceedings of the 29th International Coference on International Conference on Machine Learning},
  pages={291--298},
  year={2012}
}

@INPROCEEDINGS{nci1,
  author={Wale, Nikil and Karypis, George},
  booktitle={Sixth International Conference on Data Mining (ICDM'06)}, 
  title={Comparison of Descriptor Spaces for Chemical Compound Retrieval and Classification}, 
  year={2006},
  volume={},
  number={},
  pages={678-689},
  doi={10.1109/ICDM.2006.39}
}

@article{nci12,
author = {Shervashidze, Nino and Schweitzer, Pascal and van Leeuwen, Erik Jan and Mehlhorn, Kurt and Borgwardt, Karsten M.},
title = {Weisfeiler-Lehman Graph Kernels},
year = {2011},
issue_date = {2/1/2011},
publisher = {JMLR.org},
volume = {12},
issn = {1532-4435},
journal = {J. Mach. Learn. Res.},
month = nov,
pages = {2539–2561},
numpages = {23}
}

@article{proteins1,
author = {Borgwardt, Karsten M. and Ong, Cheng Soon and Sch\"{o}nauer, Stefan and Vishwanathan, S. V. N. and Smola, Alex J. and Kriegel, Hans-Peter},
title = {Protein function prediction via graph kernels},
year = {2005},
issue_date = {January 2005},
publisher = {Oxford University Press, Inc.},
address = {USA},
volume = {21},
number = {1},
issn = {1367-4803},
url = {https://doi.org/10.1093/bioinformatics/bti1007},
doi = {10.1093/bioinformatics/bti1007},
journal = {Bioinformatics},
month = jan,
pages = {47–56},
numpages = {10}
}

@article{proteins2,
title = {Distinguishing Enzyme Structures from Non-enzymes Without Alignments},
journal = {Journal of Molecular Biology},
volume = {330},
number = {4},
pages = {771-783},
year = {2003},
issn = {0022-2836},
doi = {https://doi.org/10.1016/S0022-2836(03)00628-4},
url = {https://www.sciencedirect.com/science/article/pii/S0022283603006284},
author = {Paul D. Dobson and Andrew J. Doig}
}

@inproceedings{socialnet,
author = {Yanardag, Pinar and Vishwanathan, S.V.N.},
title = {Deep Graph Kernels},
year = {2015},
isbn = {9781450336642},
publisher = {Association for Computing Machinery},
address = {New York, NY, USA},
url = {https://doi.org/10.1145/2783258.2783417},
doi = {10.1145/2783258.2783417},
booktitle = {Proceedings of the 21th ACM SIGKDD International Conference on Knowledge Discovery and Data Mining},
pages = {1365–1374},
numpages = {10},
location = {Sydney, NSW, Australia},
series = {KDD '15}
}

@inproceedings{
zhang2023rethinkingexpressivepowergnns,
title={Rethinking the Expressive Power of {GNN}s via Graph Biconnectivity},
author={Bohang Zhang and Shengjie Luo and Liwei Wang and Di He},
booktitle={International Conference on Learning Representations},
year={2023},
url={https://openreview.net/forum?id=r9hNv76KoT3}
}

@misc{ma2023graphinductivebiasestransformers,
      title={Graph Inductive Biases in Transformers without Message Passing}, 
      author={Liheng Ma and Chen Lin and Derek Lim and Adriana Romero-Soriano and Puneet K. Dokania and Mark Coates and Philip Torr and Ser-Nam Lim},
      year={2023},
      eprint={2305.17589},
      archivePrefix={arXiv},
      primaryClass={cs.LG},
      url={https://arxiv.org/abs/2305.17589}, 
}

@inproceedings{hussain2024tripletinteractionimprovesgraph,
author = {Hussain, Md Shamim and Zaki, Mohammed J. and Subramanian, Dharmashankar},
title = {Triplet interaction improves graph transformers: accurate molecular graph learning with triplet graph transformers},
year = {2024},
publisher = {JMLR.org},
booktitle = {Proceedings of the 41st International Conference on Machine Learning},
articleno = {834},
numpages = {25},
location = {Vienna, Austria},
series = {ICML'24}
}

@misc{peng2025biologicallyplausiblebraingraph,
      title={Biologically Plausible Brain Graph Transformer}, 
      author={Ciyuan Peng and Yuelong Huang and Qichao Dong and Shuo Yu and Feng Xia and Chengqi Zhang and Yaochu Jin},
      year={2025},
      eprint={2502.08958},
      archivePrefix={arXiv},
      primaryClass={cs.LG},
      url={https://arxiv.org/abs/2502.08958}, 
}

@inproceedings{kitaev2020reformerefficienttransformer,
    title       = {Reformer: The Efficient Transformer},
    author      = {Nikita Kitaev and Lukasz Kaiser and Anselm Levskaya},
    booktitle   = {International Conference on Learning Representations},
    year        = {2020},
    url         = {https://openreview.net/forum?id=rkgNKkHtvB}
}

@inproceedings{courtois2024symmetric,
  title={Symmetric dot-product attention for efficient training of bert language models},
  author={Courtois, Martin and Ostendorff, Malte and Hennig, Leonhard and Rehm, Georg},
  booktitle={Findings of the Association for Computational Linguistics: ACL 2024},
  pages={8002--8011},
  year={2024}
}

@InProceedings{yang2024partsharedqk,
  title = 	 {Partially Shared Query-Key for Lightweight Language Models},
  author =       {Yang, Kai and Partovi Nia, Vahid and Chen, Boxing and Asgharian, Masoud},
  booktitle = 	 {Proceedings of The 4th NeurIPS Efficient Natural Language and Speech Processing Workshop},
  pages = 	 {286--291},
  year = 	 {2024},
  editor = 	 {Rezagholizadeh, Mehdi and Passban, Peyman and Samiee, Soheila and Partovi Nia, Vahid and Cheng, Yu and Deng, Yue and Liu, Qun and Chen, Boxing},
  volume = 	 {262},
  series = 	 {Proceedings of Machine Learning Research},
  month = 	 {14 Dec},
  publisher =    {PMLR},
  url = 	 {https://proceedings.mlr.press/v262/yang24a.html}
}

@inproceedings{edelman2022inductive,
  title={Inductive biases and variable creation in self-attention mechanisms},
  author={Edelman, Benjamin L and Goel, Surbhi and Kakade, Sham and Zhang, Cyril},
  booktitle={International Conference on Machine Learning},
  pages={5793--5831},
  year={2022},
  organization={PMLR}
}

@article{ataee2023max,
  title={Max-margin token selection in attention mechanism},
  author={Ataee Tarzanagh, Davoud and Li, Yingcong and Zhang, Xuechen and Oymak, Samet},
  journal={Advances in neural information processing systems},
  volume={36},
  pages={48314--48362},
  year={2023}
}

@article{tian2023scan,
  title={Scan and snap: Understanding training dynamics and token composition in 1-layer transformer},
  author={Tian, Yuandong and Wang, Yiping and Chen, Beidi and Du, Simon S},
  journal={Advances in neural information processing systems},
  volume={36},
  pages={71911--71947},
  year={2023}
}

@article{hu2020open,
  title={Open graph benchmark: Datasets for machine learning on graphs},
  author={Hu, Weihua and Fey, Matthias and Zitnik, Marinka and Dong, Yuxiao and Ren, Hongyu and Liu, Bowen and Catasta, Michele and Leskovec, Jure},
  journal={Advances in neural information processing systems},
  volume={33},
  pages={22118--22133},
  year={2020}
}

@inproceedings{dwivedi2022LRGB,
  title={Long Range Graph Benchmark}, 
  author={Dwivedi, Vijay Prakash and Rampášek, Ladislav and Galkin, Mikhail and Parviz, Ali and Wolf, Guy and Luu, Anh Tuan and Beaini, Dominique},
  booktitle={Thirty-sixth Conference on Neural Information Processing Systems Datasets and Benchmarks Track},
  year={2022},
  url={https://openreview.net/forum?id=in7XC5RcjEn}
}

@inproceedings{velivckovic2022clrs,
  title={The clrs algorithmic reasoning benchmark},
  author={Veli{\v{c}}kovi{\'c}, Petar and Badia, Adri{\`a} Puigdom{\`e}nech and Budden, David and Pascanu, Razvan and Banino, Andrea and Dashevskiy, Misha and Hadsell, Raia and Blundell, Charles},
  booktitle={International Conference on Machine Learning},
  pages={22084--22102},
  year={2022},
  organization={PMLR}
}

@article{bickel2009nonparametric,
author = {Peter J. Bickel  and Aiyou Chen },
title = {A nonparametric view of network models and {N}ewman–{G}irvan and other modularities},
journal = {Proceedings of the National Academy of Sciences},
volume = {106},
number = {50},
pages = {21068-21073},
year = {2009},
doi = {10.1073/pnas.0907096106},
URL = {https://www.pnas.org/doi/abs/10.1073/pnas.0907096106},
eprint = {https://www.pnas.org/doi/pdf/10.1073/pnas.0907096106}
}

@inproceedings{yang2014nonparametric,
  title={Nonparametric estimation and testing of exchangeable graph models},
  author={Justin Yang and Christina Han and Edoardo M. Airoldi},
  booktitle={International Conference on Artificial Intelligence and Statistics},
  year={2014},
  url={https://api.semanticscholar.org/CorpusID:15784683}
}

\newpage
\appendix
\onecolumn

% \begin{comment}
\section{Related work}
\label{sec:related-extended}
%In this section, we discuss related topics to our work, ranging from attention-based models for graph machine learning problems and the importance of modeling nonlocal interactions, to limit objects for sequences of growing graphs and how to estimate them when one has only access to samples from such a model.

\textbf{Attention for graphs.}
GTs adapt the self-attention mechanism of \cite{vaswani2017attention} to graph-structured inputs by learning pairwise interactions between nodes, often coupled with structural or positional biases to preserve input graph inductive structure. Early attention-based GNNs as in \citep{velickovic2018graph} replaced fixed edge weights with feature-dependent weights, while subsequent work investigated when full global attention is beneficial and how to encode structure so that Transformers remain competitive on graph benchmarks \citep{ying2021do,kreuzer2021rethinking,hussain2022global}. Recent architectures, notably GraphGPS \citep{rampasek2022recipe}, emphasize practical approaches that combine global attention with local message-passing, yielding strong and scalable performance across tasks. Others pushed GT boundaries in expressiveness and applicability \cite{zhang2023rethinkingexpressivepowergnns,ma2023graphinductivebiasestransformers,hussain2024tripletinteractionimprovesgraph,peng2025biologicallyplausiblebraingraph}. Our work is complementary: rather than proposing a new architecture, we study the attention matrices induced by GTs, and ask whether they exhibit principled limit behavior as graph size grows.

\textbf{Nonlocal interactions on graphs.}
A central motivation for graph attention is that standard message-passing has intrinsic expressive limits tied to Weisfeiler-Leman refinements \citep{xu2019how,morris2019weisfeiler}. These results formalize that deeper local aggregation does not necessarily translate into richer discrimination of graph structures, and they motivate architectures whose interaction patterns are not constrained to a fixed local neighborhood. In this context, attention and GTs in particular offer an alternative inductive bias: they can represent long-range dependencies through learned dense interactions \cite{min2022transformer,muller2023attending}, which raises research questions about their stability and the generalization of induced interaction graphs across varying sizes. 

\textbf{Graph limits and graphons.}
Dense graph limit theory provides canonical limit objects for sequences of dense graphs in the form of \emph{graphons}, with convergence characterized (up to measure-preserving relabelings) by the cut-distance \citep{lovasz2012large,borgs2008convergent}. From a probabilistic perspective, exchangeable random graphs admit an equivalent representation via the Aldous-Hoover framework, which generates graphs from a (possibly random) graphon \citep{diaconis2008graph}. These results enable a relabeling-invariant, size-agnostic notion of proximity between large graphs and naturally suggest an operator viewpoint in which structure is encoded by kernel-induced operators. This continuum lens has also been used to study the transferability of graph neural networks across graph sizes through graphon neural networks \citep{levie2019transferability,ruiz20-transf,ruiz2021transferability,cervino2021increase}. We build on this literature by treating attention matrices as weighted dense graphs and analyzing their behavior through graphon scaling and cut-distance concentration.

\textbf{Graphon estimation.}
A practical consequence of the graphon formalism is that one can estimate and compare latent kernels from finite graphs. Consistent estimation procedures include classical histogram-based estimators such as sort-and-smooth methods \citep{airoldi2013stochastic,chan2014consistent}, optimal-transport-based methods such as GW distance \cite{xu2021learning}, and, most recently, neural methods like IGNR \cite{xia2023implicit} and MomentNet \cite{ramezanpour2025few}. Theoretically, minimax results characterize achievable convergence rates under regularity assumptions on the graphon \citep{gao2015rate}. On a different yet related note, sampling lemmas and concentration results in the cut-distance underpin the idea that a single finite graph can be a noisy sample from a stable limit object \citep{borgs2008convergent,lovasz2012large}. Our approach leverages these tools in a model-driven way: we interpret each attention matrix as a sample from an underlying kernel, quantify concentration via variance in cut-distance, and use graphon estimation to operationalize these comparisons across datasets and graph sizes. 
% \end{comment}

% \section{Framework Pipeline}

% \begin{figure*}[ht!]
% \centering
% \includegraphics[width=0.9\linewidth]{Figures/Rebuttal/framework.png}
% \caption{\textbf{Visual explanation of the cut distance $\delta_\square$ and graphon computation.} \textit{Top (cut distance)}: The cut distance roughly measures the largest discrepancy over node subsets $S, T$. \textit{Bottom (graphon computation)}: From a dense, weighted attention graph with $n = 8$, the matrix is symmetrized, canonicalized via degree sorting, block-averaged, and finally used to estimate a dataset-level graphon.}
% \label{fig:pipeline}
% \end{figure*}

% \section{Additional definitions}
% \label{app:defs}

% \paragraph{Induced graphons from matrices.}
% Given $A\in[0,1]^{n\times n}$, define the step graphon $W_A\colon[0,1]^2\to[0,1]$ by partitioning $[0,1]$ into intervals
% $I_i=[(i-1)/n,i/n)$ and setting $W_A(x,y)=A_{ij}$ when $(x,y)\in I_i\times I_j$.

% \paragraph{Sampling models.}
% For a graphon $W$, the weighted sampling model $H(n,W)$ draws latent $x_1,\dots,x_n\overset{\mathrm{iid}}{\sim}\mathrm{Unif}[0,1]$ and sets
% $H_{ij}=W(x_i,x_j)$. The simple graph model $G(n,W)$ draws edges independently as
% $A_{ij}\mid(x_i,x_j)\sim\mathrm{Bernoulli}(W(x_i,x_j))$ for $i\neq j$ \citep{lovasz2012large}.

\section{Proof of Theorem~\ref{thm:bound_1}}
\label{app:proof_worstcase}
\noindent Before laying out our main result, a few definitions and results from the literature need introduction:

\begin{definition}[Extracted \cite{lovasz2012large}]
    Let $A \in \reals^{n\times n}$. The \textbf{cut-norm} for matrices is defined as:
    \begin{align}
        \|A\|_{\square} = \frac{1}{n^2}\max_{S, T \subseteq [n]} \biggl|\sum_{i \in S, j \in T}A_{ij} \biggr|
    \end{align}
\end{definition}

\begin{definition}[Simplified \cite{lovasz2012large}]
    The above norm induces the so-called \textbf{cut-distance}. In \cite{lovasz2012large}, this metric is defined following a generalization hierarchy. For labeled graphs G, G' on the same node set $[n]$, the cut-distance is defined as:
    \begin{align}
        d_\square(G, G') = \|A_G - A_{G'}\|_\square,
    \end{align}
    % where $e_G(S, T)$ is the number of nodes of $G$ with one endnode in $S$ and the other in $T$.

    \

    For unlabeled graphs with node sets of equal length, we have:
    \begin{align}
        \hat{\delta}_{\square}(G, G') = \min_{\pi \in S_n}d_\square(G, \pi(G')),
    \end{align}
    where $\pi$ is a labeling function returning a labeled graph, and $S_n$ is the set of all possible labeling functions (which is one-to-one with the symmetric group on $n$ elements).
    %with $\hhatG, \hhatG'$ ranging over all labelings of $G, G'$ of the node set $[n]$. (In practice, we fix one of the graphs' labeling and minimize over all labelings of the other.)

    \

    Finally, for graphs with a different number of nodes, i.e., $G = (V, E)$, $G' = (V', E')$, $|V| = n$, and $|V'| = n'$:
    \begin{align}
        \delta_{\square}(G, G') = \lim_{k\to \infty}\hat{\delta}_{\square}(G(kn'), G'(kn)),
    \end{align}
    where $G(kn')$ is the `augmented' graph obtained by replacing each node of $G$ by $k$ (copy) nodes, and connecting the nodes if and only if they were connected in the original graph. 
\end{definition}

For this proof, we focus on the \textbf{latter} notation for the cut-distance, since we'll be bounding the variance of the (cut-norm) difference between graph samples and their respective graphons.

\begin{definition}
    A \textbf{graphon} is a symmetric measurable function $W\colon [0, 1]^2 \to [0, 1]$. We denote the set of all graphons $\ccalW_0$. 
\end{definition}

\begin{definition}
    Let $G$ be an $n-$node graph with associated adjacency matrix $A$. Further, let $\{I_j\}_{j=1}^n$ be an equipartition of the interval $[0, 1]$. Then, the \textbf{induced graphon} $W_G\colon [0, 1]^2 \to \reals$ is defined as
    \begin{equation}
    W_G(u, v) \;\coloneqq\; \sum_{j=1}^n\sum_{k=1}^n [A]_{jk}\mathbb{I}(u \in I_j)\mathbb{I}(v \in I_k).
    \end{equation}
    where $\mathbb{I}$ is the indicator function.
\end{definition}

\begin{definition}[Adapted \cite{lovasz2012large}]
    Analogously to the definition for matrices, the \textbf{cut-norm} of a kernel $K\colon [0, 1]^2 \to [-1, 1]$ is defined as
    \begin{equation}
        \|K\|_\square = \sup_{S, T \subseteq [0, 1]}\left|\int_{S\times T}K(u, v)dudv\right|.
    \end{equation}
\end{definition}

\begin{definition}[Adapted \cite{lovasz2012large}]
    The \textbf{cut-distance} of two kernels, $U, W\colon [0,1]^2 \to [0, 1]$, is then defined as
    \begin{equation}
        \delta_\square(U, W) = \inf_{\phi}d_\square(U, W^\phi) = \inf_\phi\|U - W^\phi\|_\square,
    \end{equation}
    for a measure-preserving bijection $\phi\colon [0, 1] \to [0, 1]$.
\end{definition}

\begin{definition}
    For $n > 0, k\in \naturals$, let $H(n, W)$ be the \textbf{node-stochastic weighted graph} sampled from the graphon $W$ by sampling an ordered $n$-tuple $(x_1, \dots, x_n)$ of independent uniform random points from $[0, 1]$, and assigning each edge $(i, j)$ the weight $W(x_i, x_j)$.
\end{definition}

\begin{definition}
    Every node-stochastic weighted graph gives rise to a \textbf{simple random graph model} $G(n, W)$: given a pair of nodes $i$ and $j$ ($i \neq j, \ i,j\in[n]$), sample the unweighted edge $(i,j)$ with probability $W(x_i, x_j)$, independently for each node pair.
\end{definition}

\begin{lemma}[{Second Sampling Lemma for Graphons, \citep[Lemma~10.16]{lovasz2012large}}]\label{lem:scnd_sampl}
     Let $n\geq 1$, and let $W \in \ccalW_0$ be a graphon. Let $H(n, W)$ and $W_n$ be its induced graphon. Then, with probability at least $1 - \text{exp}(-n/(2\log n))$:
    \begin{align}
        \delta_{\square}(W_n, W) \leq \frac{20}{\sqrt{\log n}}.
    \end{align}
\end{lemma}

Then, we can prove the following bound for the variance of the cut-distance between the graphon $W$ and a graph sample $H$:

\begin{theorem}\label{app:thm_bound_1}
    Let $n\geq 1$, and let $W \in \ccalW_0$ be a graphon. Moreover, let $H(n, W)$ be a random weighted graph sampled from $W$. Then, the variance of the cut-distance between the graphon $W$ and the sample $H$ is bounded as: 
    \begin{equation}
        \text{Var}(\delta_{\square}(W_n, W)) = \ccalO\left(\dfrac{1}{\log n}\right),
    \end{equation}
    where $W_n$ is the graphon induced by $H(n, W)$. 
\end{theorem}
\begin{proof}
    Using the definition of the variance of the cut-distance, we have
    \begin{align}
        \text{Var}(\delta_{\square}(W_n, W)) = \E{\delta_{\square}(W_n, W)^2} - (\E{\delta_{\square}(W_n, W)})^2 \nonumber,
    \end{align}
    which can then be further expanded by conditioning the first term:
    \begin{align}
        \text{Var}(\delta_{\square}(W_n, W)) &= \E{\delta_{\square}(W_n, W)^2  \ \big| \ \delta_{\square}(W_n, W)^2 > \frac{20^2}{\log n}}\mathbb{P}\bigg[{\delta_{\square}(W_n, W)^2 > \frac{20^2}{\log n}}\bigg] \nonumber\\
        &\quad + \E{\delta_{\square}(W_n, W)^2 \ \big| \ \delta_{\square}(W_n, W) ^2\leq \frac{20^2}{\log n}}\mathbb{P}\bigg[{\delta_{\square}(W_n, W)^2 \leq \frac{20^2}{\log n}}\bigg] \nonumber \\
        &\quad - (\E{\delta_{\square}(W_n, W)})^2 \nonumber.
    \end{align}
    Applying \Cref{lem:scnd_sampl} and the fact that the cut-distance is bounded above by 1, and disregarding the third term above, we have the following:
    \begin{align}
        \text{Var}(\delta_{\square}(W_n, W)) &\leq \max_{W_n}\delta_{\square}(W_n, W)^2\cdot e^{-\frac{n}{2\log n}} + \frac{20^2}{\log n } \nonumber \\
        &= \frac{20^2}{\log n } + e^{-\frac{n}{2\log n}}
        = \ccalO\left(\dfrac{1}{\log n}\right).
    \end{align}
    
    % \begin{align}
    %     &\text{Var}(\delta_{\square}(W_G, W)) = \E{\delta_{\square}(W_G, W)^2} - (\E{\delta_{\square}(W_G, W)})^2 \nonumber \\
    %     &= \E{\delta_{\square}(W_G, W)^2  \ \big| \ \delta_{\square}(W_G, W)^2 > \frac{22^2}{\log n}}\mathbb{P}\bigg[{\delta_{\square}(W_G, W)^2 > \frac{22^2}{\log n}}\bigg] \nonumber\\
    %     &\quad + \E{\delta_{\square}(W_G, W)^2 \ \big| \ \delta_{\square}(W_G, W) ^2\leq \frac{22^2}{\log n}}\mathbb{P}\bigg[{\delta_{\square}(W_G, W)^2 \leq \frac{22^2}{\log k}}\bigg] \nonumber \\
    %     &\quad - (\E{\delta_{\square}(W_G, W)})^2 \nonumber \\
    %     &\leq \max_{W_G}\delta_{\square}(W_G, W)^2\cdot e^{-\frac{n}{2\log n}} + \frac{22^2}{\log n } \label{eq:var_bound_1}\\
    %     &= \frac{22^2}{\log n } + e^{-\frac{n}{2\log n}} \label{eq:var_bound_2}\\
    %     &= \ccalO\left(\dfrac{1}{\log n}\right). \nonumber
    % \end{align}
\end{proof}

% For inequality (\ref{eq:var_bound_1}), we used the second sampling lemma from \cite{lovasz2012large}, while on (\ref{eq:var_bound_2}), we used the fact that the cut norm of the difference of any two graphs, which upper bounds the cut distance, would be bounded by 1.

% We use a standard sampling lemma for graphons. One convenient form is the \emph{second sampling lemma}
% \citep[Theorem 10.23]{lovasz2012large}: for any graphon $W\in\ccalW_0$ and $G\sim\mbG(n,W)$, with probability at least
% $1-\exp\{-n/(2\log n)\}$,
% \begin{equation}
% \label{eq:sampling_lemma}
% \delta_\square(G,W) \le \frac{22}{\sqrt{\log n}}.
% \end{equation}

% \begin{proof}
% Let $X\coloneqq\delta_\square(G,W)\in[0,1]$. Then
% \begin{align}
% \mathbb{E}[X^2]
% &=\mathbb{E}[X^2\,\mathbf{1}\{X\le 22/\sqrt{\log n}\}] + \mathbb{E}[X^2\,\mathbf{1}\{X>22/\sqrt{\log n}\}]\\
% &\le \frac{22^2}{\log n} + \mathbb{P}\{X>22/\sqrt{\log n}\}
% \le \frac{22^2}{\log n} + \exp\{-n/(2\log n)\}.
% \end{align}
% Since $\Var(X)\le \mathbb{E}[X^2]$, we conclude $\Var(\delta_\square(G,W))=\mathcal{O}(1/\log n)$.
% \end{proof}

\section{Proof of Theorem~\ref{thm:bound_2}}
\label{app:holder}
The previous result, though general, is not really useful, given the slow rate decay of the upper bound, which depends on the $\log$ of the number of nodes in the graph. Under mild assumptions, and using results from \cite{gao2015rate}, we can get a better (and more useful) bound for that estimate. Let us introduce some definitions that will be used along the proof.

\begin{definition}
    Let $\ccalZ_{n, k} = \{z: [n] \to [k]\}$ be the set of all possible mappings from $[n]$ to $[k]$, for $n, k \in \naturals$.
\end{definition}

\begin{definition}
    Let
    \[
    \nabla_{j,k}f(x,y) = \frac{\partial^{j+k}}{\partial x^j \partial y^k}f(x, y)
    \]
    be the \textbf{derivative operator} of a function $f$.
\end{definition}

\begin{definition}
    Given a symmetric function $f\colon [0, 1]^2 \to [0, 1]$ and $\ccalD = \{(x, y) \in [0, 1]^2\colon x\geq y\}$, the \textbf{H{\"o}lder norm} is defined as
    \[
    \|f\|_{\ccalH_\alpha} = \max_{j+k \leq \lfloor\alpha\rfloor} \sup_{x, y \in \ccalD} |\nabla_{j,k}f(x, y)| + \max_{j+k \leq \lfloor\alpha\rfloor}\sup_{(x, y) \neq (x', y') \in \ccalD} \frac{|\nabla_{j,k}f(x, y) - \nabla_{j,k}f(x', y')|}{(|x - x'| + |y - y'|)^{\alpha - \lfloor\alpha\rfloor}}.
    \]
\end{definition}

\begin{definition}
    The \textbf{H{\"o}lder class} is then defined as
    \[
    \ccalH_\alpha(M) = \{ f : [0,1]^2 \to [0,1]\ :\ \|f\|_{\ccalH_\alpha}\ \leq M \}.
    %,\ f(x, y) = f(y, x), \ x \geq y\},
    \]
   %where $\alpha > 0 $ is parameter that controls how smooth $f$ is, and $M > 0$ the size of the function class. 
%\footnote{Here, $M$ is assumed to be a constant.}.
\end{definition}

% If $\alpha \in (0, 1]$, then $f \in \ccalH_\alpha(M)$ is the same as $f$ being Lipschitz, i.e.,
% \begin{equation}
%     |f(x, y) - f(x', y')| \leq M(|x - x'| + |y - y'|)^\alpha,
% \end{equation}
% for any $(x, y), (x', y') \in \ccalD$. \red{This is only Lipschitz if $\alpha=1$ above.} \blue{C: I followed the definition used in \cite{gao2015rate}, which uses the more general notion of \textit{Lipschitz continuous of order $\alpha$} functions (or $\alpha$-H{\"o}lder). We can stick with the standard definition, no problem.} 

We assume that our graphon $W$ lives in the following class
\[
\ccalF_\alpha(M) = \{0\leq f \leq 1\colon f\in \ccalH_\alpha(M)\}.
\]

Let $\{x_i\}$ be a sequence of i.i.d. random variables with distribution $\mbP_{X}$ supported on $[0, 1]$. In addition, let $\theta_{ij} = W(x_i, x_j)$. For any $i \neq j$, the adjacency 
% matrices H and A 
for a weighted graph $\mbH(n, W)$ 
% and simple random graphs $G \sim \mbG(n, W)$
% , respectively, are sampled accordingly:
is sampled accordingly:
\begin{align}
(x_1, \dots, x_n) \sim \mbP_X, && 
H_{ij} = \theta_{ij} = W(x_i, x_j) 
% &&
% A_{ij}|(x_i, x_j) \sim \text{Bernoulli}(\theta_{ij}).
\end{align}

\begin{definition}
    Given a matrix $H \in \reals^{n\times n}$, let $z \in \ccalZ_{n,k}$ be a clustering map. Hence, the \textbf{graphon block estimator} is defined as
    \begin{equation}\label{eq:blockest}
        \hat{\theta}_{ab} \;=\; \frac{1}{|z^{-1}(a)|\,|z^{-1}(b)|}\sum_{i\in z^{-1}(a)}\sum_{j\in z^{-1}(b)} (H)_{ij},\  a,b\in[k],
    \end{equation}
    where the sets $\{z^{-1}(a)\colon a\in [k]\}$ form a partition of $[n]$ (clustering assignment).
\end{definition}

% \subsection{Canonicalization assumption}

\begin{assumption}[Consistent canonical ordering]\label{asm:canon}
The clustering $z$ used to form \eqref{eq:blockest} is obtained from a canonicalization procedure (e.g., degree sorting, as in \Cref{sec:estimator}) that, under the regularity conditions on $W$ assumed throughout, is consistent with a latent-sort oracle in the following
sense: $\exists \ c_1,c_2>0$ such that
\[
  \max_{a\in[k]} \operatorname{diam}\bigl(\{x_i : z(i)=a\}\bigr) \;\le\; \frac{c_1}{k}
  \qquad\text{w.p.}\ge 1-\exp(-c_2 n).
\]
\end{assumption}

\begin{remark}
Our canonicalize-then-block-average pipeline (\Cref{sec:estimator}) is a direct instantiation of the \emph{Sort-And-Smooth} (SAS)
estimator \cite{chan2014consistent}. Their Theorem~3 establishes that,
under the standing assumption that the degree function
$g(x):=\int_0^1 W(x,y)\,dy$ is strictly monotone (together with mild
regularity of $W$), the degree-sort permutation
$\hat\pi(i)=\operatorname{rank}(d_i)$, where
$d_i=\tfrac{1}{n}\sum_j H_{ij}$, converges almost surely to the true
latent sort, i.e., the empirical degrees concentrate around
$g(x_i)$ at rate $\ccalO_{\mathbb{P}}(n^{-1/2})$, and strict monotonicity of $g$
transfers this concentration into a bound on the rank permutation.
Applied to our setting, this yields
\[
  \max_{a\in[k]}\operatorname{diam}\bigl(\{x_i:z(i)=a\}\bigr) \leq \frac{1}{k}+\ccalO_{\mathbb{P}}(n^{-1/2}),
\]
which for $k=\lceil n^{1/(\min(\alpha,1)+1)}\rceil$ is dominated by the
$1/k$ term (since $k\le n^{1/2}$ for any $\alpha>0$). 
% The exponential concentration in Assumption~\ref{asm:canon} follows from a Hoeffding bound on the degree deviations, together with the monotonicity of $g$. 
This is a noiseless analog of the clustering-consistency bound from \cite{gao2015rate} for a Bernoulli model.
\end{remark}

\begin{lemma}[H\"older bias of block averages]\label{lem:holder-bias}
Let $W\in\ccalF_\alpha(M)$ and suppose Assumption~\ref{asm:canon} holds.
Under the model $H$, the block estimator in
\eqref{eq:blockest} satisfies, pointwise for every $i,j\in[n]$,
\begin{equation}\label{eq:holder-pointwise}
  \bigl|\hat\theta_{ij}-\theta_{ij}\bigr|
  \leq 2 M \left(\frac{c_1}{k}\right)^{\min(\alpha,1)}
\end{equation}
on the event $\Omega_k:=\{\operatorname{diam}(z^{-1}(a))\le c_1/k,\ \forall a\}$, which
has probability at least $1-\exp(-c_2 n)$.
\end{lemma}

\begin{proof}
Fix $i,j\in[n]$ and let $a=z(i)$, $b=z(j)$. By \Cref{eq:blockest},
\[
  \hat\theta_{ij} - \theta_{ij}
  = \frac{1}{|z^{-1}(a)||z^{-1}(b)|}
    \sum_{i'\in z^{-1}(a)}\sum_{j'\in z^{-1}(b)}
    \bigl(W(x_{i'},x_{j'}) - W(x_i,x_j)\bigr).
\]
We distinguish two regimes.

\emph{Case $\alpha\le 1$.} The H\"older norm $\|W\|_{\ccalH_\alpha}\le M$
directly controls the zeroth-order difference:
\[
  |W(x_{i'},x_{j'})-W(x_i,x_j)|
  \le M\bigl(|x_{i'}-x_i|^\alpha+|x_{j'}-x_j|^\alpha\bigr).
\]

\emph{Case $\alpha>1$.} By definition of the H\"older class,
$\|W\|_{\ccalH_\alpha}\le M$ bounds every partial derivative of order at
most $\lfloor\alpha\rfloor\ge 1$.
% ; in particular $\|\partial_x W\|_\infty,\|\partial_y W\|_\infty\le M$. 
Applying the
mean-value theorem along the segment joining $(x_i,x_j)$ to
$(x_{i'},x_{j'})$,
\[
  W(x_{i'},x_{j'})-W(x_i,x_j) = \nabla W(\xi)\cdot\bigl[(x_{i'}-x_i),\ (x_{j'}-x_j)\bigr]
\]
for some $\xi$ on that segment, so
\[
  |W(x_{i'},x_{j'})-W(x_i,x_j)|
  \le M\bigl(|x_{i'}-x_i|+|x_{j'}-x_j|\bigr).
\]
Block-constant estimators cannot exploit smoothness beyond Lipschitz, so this is the operative bound \cite{olhede2014network}. Combining both cases:
\[
  |W(x_{i'},x_{j'})-W(x_i,x_j)|
  \;\le\; M\bigl(|x_{i'}-x_i|^{\min(\alpha,1)}+|x_{j'}-x_j|^{\min(\alpha,1)}\bigr).
\]
On $\Omega_k$, both latent differences are bounded by $c_1/k$, so each
summand is bounded by $2M(c_1/k)^{\min(\alpha,1)}$, and averaging
preserves the bound. The probability statement follows from
Assumption~\ref{asm:canon}.
\end{proof}

\begin{theorem}[Noiseless nonparametric graphon estimation]\label{thm:noiseless-rate}
Let $W\in\ccalF_\alpha(M)$, suppose Assumption~\ref{asm:canon} holds, and
take $k=\lceil n^{1/(\min(\alpha,1)+1)}\rceil$. Then there exist
constants $C,C'>0$ depending only on $M,c_1,c_2$ such that
\begin{equation}\label{eq:noiseless-frob}
  \frac{1}{n^2}\sum_{i,j\in[n]}\bigl(\hat\theta_{ij}-\theta_{ij}\bigr)^2
  \;\le\; C\,n^{-\tfrac{2\min(\alpha,1)}{\min(\alpha,1)+1}}
\end{equation}
with probability at least $1-\exp(-C'n)$, uniformly over
$W\in\ccalF_\alpha(M)$ and $\mathbb{P}_X$.
\end{theorem}

\begin{proof}
On the event $\Omega_k$ of Lemma~\ref{lem:holder-bias},
\[
  \tfrac{1}{n^2}\sum_{i,j}(\hat\theta_{ij}-\theta_{ij})^2
  \le \bigl(2Mc_1^{\min(\alpha,1)}\bigr)^2 k^{-2\min(\alpha,1)}.
\]
Plugging in $k=\lceil n^{1/(\min(\alpha,1)+1)}\rceil$ gives
$k^{-2\min(\alpha,1)}\le n^{-2\min(\alpha,1)/(\min(\alpha,1)+1)}$ up to
absolute constants, and $\mathbb{P}(\Omega_k)\ge 1-\exp(-c_2 n)$. Absorbing all constants into $C,C'$ yields the result.
\end{proof}

\begin{proposition}[{Adapted, \citep[Equation~8.5]{lovasz2012large}}]\label{prop:norm_rel}
    Let
    \begin{equation*}
        \|A\|_1 = \dfrac{1}{n^2}\sum_{i,j=1}^n |A_{i, j}|, \quad \|A\|_2 = \left(\dfrac{1}{n^2}\sum_{i,j=1}^n A_{i,j}^2\right)^{1/2}, \quad \|A\|_\infty = \max_{i,j}|A_{i,j}|,
    \end{equation*}
    be the matrix $\ell_1$-norm, $\ell_2$-norm, and $\ell_\infty$-norm respectively. Then, we have the following
    \begin{equation}
        \|A\|_\square \leq \|A\|_1 \leq \|A\|_ 2 \leq \|A\|_\infty.
    \end{equation}
\end{proposition}
% \begin{proof}

%     \
    
%     \quad1) $\|A\|_\square \leq \|A\|_1$:
%     \begin{equation*}
%         \|A\|_\square = \frac{1}{n^2}\max_{S, T \subseteq [n]} \biggl|\sum_{i \in S, j \in T}A_{i, j} \biggr| \leq \frac{1}{n^2}\max_{S, T \subseteq [n]} \sum_{i \in S, j \in T}|A_{i, j}| = \frac{1}{n^2}\sum_{i,j=1}^n |A_{i,j}| = \|A\|_1.
%     \end{equation*}

%     \quad2) $\|A\|_1 \leq \|A\|_2$:
%     \begin{equation*}
%         \|A\|_1 = \frac{1}{n^2}\sum_{i,j=1}^n |A_{i,j}| \leq \frac{1}{n^2}\left(\sum_{i,j=1}^n A_{i,j}^2\right)^{1/2}\left(\sum_{i,j=1}^n 1\right)^{1/2} = \frac{n}{n^2}\left(\sum_{i,j=1}^n A_{i,j}^2\right)^{1/2} = \left(\frac{1}{n^2}\sum_{i,j=1}^n A_{i,j}^2\right)^{1/2} = \|A\|_2.
%     \end{equation*}

%     \quad3) $\|A\|_2 \leq \|A\|_\infty$:
%     \begin{equation*}
%         \|A\|_2 = \left(\frac{1}{n^2}\sum_{i,j=1}^n A_{i,j}^2\right)^{1/2} \leq \left(\frac{1}{n^2}\max_{i,j}A_{i,j}^2\sum_{i,j}^n 1\right)^{1/2} = \left(\frac{n^2}{n^2}\max_{i,j}A_{i,j}^2\right)^{1/2} = \max_{i,j}|A_{i, j}| = \|A\|_\infty.
%     \end{equation*}
% \end{proof}

\begin{corollary}[Cut-norm bound under noiseless observations]\label{cor:noiseless-cut}
    Under the hypotheses of \Cref{thm:noiseless-rate}, there exist constants $C, C' > 0$ such that
    \[
    \|\hat{\theta} - \theta\|_\square^2 \leq Cn^{-\tfrac{2\min(\alpha,1)}{\min(\alpha,1)+1}},
    \]
    with probability at least $1 - \exp(-C'n)$.
\end{corollary}
\begin{proof}
    By Proposition~\ref{prop:norm_rel}, together with \Cref{thm:noiseless-rate}
    \[
    \|\hat{\theta} - \theta\|_\square^2 \leq \|\hat{\theta} - \theta\|_2^2 = \frac{1}{n^2}\sum_{i, j \in [n]}(\hat{\theta}_{i,j} - \theta_{i, j})^2 \leq C\,n^{-\tfrac{2\min(\alpha,1)}{\min(\alpha,1)+1}}.
    \]
\end{proof}
% That corollary yields a powerful result that bounds the cut-distance of the graphon and its estimate.
%, here represented as $\hat{\theta}$. 

Equipped with these tools, we can formalize a tighter bound to the cut-norm variance: 

\begin{theorem}[A tighter bound]\label{app:thm_bound_2}
    Let $n \geq 1$ and $W \in \ccalF_\alpha(M)$ be a graphon. Furthermore, given $\theta_{i,j} = W(x_i, x_j)$, for latent random variables $(x_1, \dots x_n) \sim \ccalP_X$, let $\hat{\theta}$ be the graphon's block approximation, constructed from $H(n, W)$. Then, for $k = \lceil n^{1/(\min(\alpha, 1) + 1)}\rceil$, there exists a constant $C > 0$, depending only on $M, c_1, c_2$, such that
    \[
    \text{Var}(\|\hat{\theta} - \theta\|_\square) = \ccalO\left(n^{-\tfrac{2\min(\alpha,1)}{\min(\alpha,1)+1}}\right).
    \]
\end{theorem}
\begin{proof}
    Analogously, let us write out the variance of the cut-norm
    \begin{align}
    \text{Var}(\|\hat{\theta} - \theta\|_\square) = \E{\|\hat{\theta} - \theta\|_\square^2} - \left(\E{\|\hat{\theta} - \theta\|_\square}\right)^2 \leq \E{\|\hat{\theta} - \theta\|_\square^2}. \nonumber
    \end{align}
    Now, let $f_{\|\hat{\theta} - \theta\|_\square}$ be the probability distribution of the r.v. $\|\hat{\theta} - \theta\|_\square$. Then, we have:
    \begin{align}
    \text{Var}(\|\hat{\theta} - \theta\|_\square) &\leq  \int_{0}^1 wf_{\|\hat{\theta} - \theta\|_\square^2}(w)dw \nonumber\\
    &= \int_{0}^{Ch(n)}wf_{\|\hat{\theta} - \theta\|_\square^2}(w)dw + \int_{Ch(n)}^1wf_{\|\hat{\theta} - \theta\|_\square^2}(w)dw, && h(n) = n^{-\frac{2\min(\alpha,1)}{(\min(\alpha,1)+1)}}\nonumber
    \end{align}
    Both terms are weighted sums of the values assumed by the cut-norm in the interval of integration. Hence, each term can be bounded by the sum of the weights multiplied by the maximum value, i.e., the upper limit of integration, attained by the cut-norm:
    \begin{align}
    \text{Var}(\|\hat{\theta} - \theta\|_\square)
    &\leq Ch(n)\int_{0}^{Ch(n)}f_{\|\hat{\theta} - \theta\|_\square^2}(w)dw + \int_{Ch(n)}^1f_{\|\hat{\theta} - \theta\|_\square^2}(w)dw \nonumber \\
    &= Ch(n)\mbP\left(\|\hat{\theta} - \theta\|_\square^2 \leq Ch(n)\right) + \left(1-\mbP\left(\|\hat{\theta} - \theta\|_\square^2 \leq Ch(n)\right)\right) \nonumber\\
    &\leq Ch(n)
    %\left(1 - \exp(-C'n)\right) 
    + \exp(-C'n) = \ccalO\left(n^{-\tfrac{2\min(\alpha,1)}{\min(\alpha,1)+1}}\right).
    \end{align}
\end{proof}

A key distinction between \Cref{app:thm_bound_1,app:thm_bound_2} lies in the metric employed. The first result uses the cut-distance $\delta_\square$, which involves an infimum over all measure-preserving bijections and thus operates on \emph{unlabeled} graphons. This generality comes at the cost of a loose $(\log n)^{-1}$ rate. In contrast, the tighter bound in \Cref{app:thm_bound_2} uses the cut-norm $\|\cdot\|_\square$, which compares graphons under a \emph{fixed labeling}. This requires establishing a canonical node ordering (in our case, induced by sorting nodes according to degree). Whilst this labeling assumption is necessary to achieve the faster polynomial rate, it reflects a practical requirement: to leverage regularity for tighter concentration, one must first align the graph to a consistent reference frame.

\section{Choice of the attention object}
\label{app:object}
% This section compares three candidate attention graphs for a head, namely the symmetrized post-softmax matrix $\bar P=\tfrac12(P+P^\top)$, its rescaled version $n\bar P$, and the pre-softmax object $A=\rho(\bar S-c)$ of \Cref{eq:attn-object}, and proves Proposition~\ref{prop:invertible}.
This section compares three candidate attention graphs for a head, namely the symmetrized post-softmax matrix $\bar P=\tfrac12(P+P^\top)$, its rescaled version $n\bar P$, and the pre-softmax object $A=\rho(\bar S-c)$ of \Cref{eq:attn-object}, and complements Proposition~\ref{prop:invertible}.

\paragraph{Post-softmax attention.}
\begin{lemma}
\label{lem:post-softmax}
Let $P\in[0,1]^{n\times n}$ be row-stochastic and $\bar P=\tfrac12(P+P^\top)$. Then $\delta_\square(W_{\bar P},0)=\|W_{\bar P}\|_\square=1/n$. Moreover, if some row of $P$ places all its mass on a single entry, then $\max_{i,j}n\bar P_{ij}\ge n/2$.
\end{lemma}
\begin{proof}
Since $\bar P\ge0$, the supremum in \Cref{eq:cutnorm-graphon} is attained at $S=T=[0,1]$, so $\|W_{\bar P}\|_\square=n^{-2}\sum_{i,j}\bar P_{ij}=1/n$, because the entries of $P$ and of $P^\top$ each sum to $n$. The zero graphon is invariant under relabeling, so $\delta_\square(W_{\bar P},0)=\|W_{\bar P}\|_\square$. Finally, $P_{ij}=1$ implies $\bar P_{ij}\ge1/2$.
\end{proof}
The first statement holds for every model and dataset, so post-softmax attention graphs converge to the zero graphon regardless of what the network has learned, and their convergence carries no information about it. Rescaling by $n$ restores a mean entry of one, but the second statement shows that $n\bar P$ can leave every bounded class of graphons, which the results of \Cref{sec:theory} require.

\paragraph{Empirical comparison.}
\Cref{fig:objects} measures both effects. On the left, a single GPS model trained on BFS at $n=32$ is evaluated with frozen weights up to $n=256$. The mean entry of $\bar P$ equals $1/n$, falling from $0.031$ to $0.0039$, and the spread of its kernel, the standard deviation of the kernel entries, falls from $1.3\times10^{-3}$ to $8.9\times10^{-5}$, so both vanish. The rescaled matrix has mean entry exactly one. The mean entry of $A$ stays between $0.500$ and $0.514$, with a spread between $0.005$ and $0.008$. Without the offset, the mean entry of $\sigma(\bar S)$ is between $0.78$ and $0.79$, so $c$ only moves the entries to the center of the range of $\sigma$. On the right, on NoisyCSBM, the largest entry of $n\bar P$, averaged over graphs, grows from $90$ to $681$ as $n$ goes from $128$ to $1024$, close to linearly, while the largest logit magnitude grows only from $42$ to $53$. 
% On the other eight datasets of this comparison (BFS, Dijkstra, COLLAB, MUTAG, NCI1, PROTEINS, SmoothW, and SharpW), the largest entry of $n\bar P$ stays between $1.0$ and $1.7$.}

\begin{figure}[ht]
\centering
\includegraphics[width=\linewidth,trim=0 0 0 46,clip]{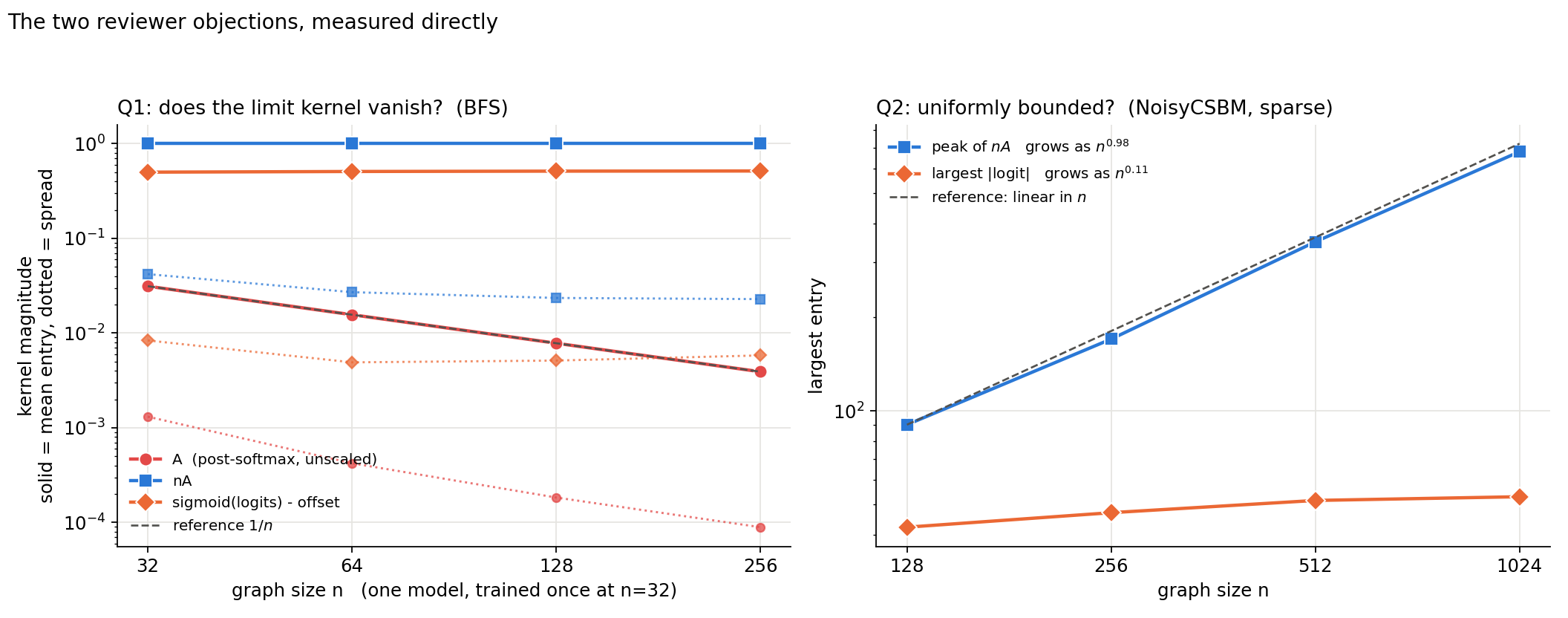}
\caption{Comparison of candidate attention objects. \textit{Left}: mean entry (solid) and spread (dotted) of the kernels of $\bar P$, $n\bar P$, and $A$ on BFS, for one GPS model trained at $n=32$ and evaluated with frozen weights. \textit{Right}: largest entry of $n\bar P$ and largest logit magnitude on NoisyCSBM. In the legend, $A$ and $nA$ denote $\bar P$ and $n\bar P$, and ``sigmoid(logits) - offset'' denotes $A=\sigma(\bar S-c)$.}
\label{fig:objects}
\end{figure}

\paragraph{Why $\rho$ must be a bijection.}
Proposition~\ref{prop:invertible} uses both properties of a bijection. Injectivity is what makes the post-softmax attention a function of $A$. If $\rho(x)=\rho(y)$ with $x\neq y$, let $\bar S$ have every entry equal to $c$ except $\bar S_{11}=x+c$, and let $\bar S'$ equal $\bar S$ except $\bar S'_{11}=y+c$. Both matrices are symmetric and $\rho(\bar S-c)=\rho(\bar S'-c)$, but the first entries of their first softmax rows, $e^{x}/(e^{x}+n-1)$ and $e^{y}/(e^{y}+n-1)$, differ. Surjectivity onto $(0,1)$ makes $\rho^{-1}$ defined at every value in $(0,1)$, so attention can be generated from any estimated kernel with values in $(0,1)$, including values sampled at new sizes, and not only from entries of observed attention graphs.

\paragraph{Bounded pre-softmax scores.}
\begin{lemma}
\label{lem:bounded-scores}
% Let $\rho\colon\mathbb{R}\to(0,1)$ be increasing, and let $(\bar S_n)$ be symmetric pre-softmax matrices with $|(\bar S_n)_{ij}-c|\le B$ for all $n$, $i$, and $j$. Then $A_n=\rho(\bar S_n-c)$ has entries in $[\rho(-B),\rho(B)]$, and every cut-distance limit $W$ of $(W_{A_n})$ satisfies $\rho(-B)\le W\le\rho(B)$ almost everywhere.
Let $\rho$ be an increasing bijection as in Proposition~\ref{prop:invertible}, and let $(\bar S_n)$ be symmetric pre-softmax matrices with $|(\bar S_n)_{ij}-c|\le B$ for all $n$, $i$, and $j$. Then the attention graphs $A_n=\rho(\bar S_n-c)$ of \Cref{eq:attn-object} have entries in $[\rho(-B),\rho(B)]$, and every cut-distance limit $W$ of $(W_{A_n})$ satisfies $\rho(-B)\le W\le\rho(B)$ almost everywhere.
\end{lemma}
\begin{proof}
Monotonicity gives $a\le(A_n)_{ij}\le b$, with $a=\rho(-B)$ and $b=\rho(B)$. Let $W$ be a cut-distance limit of $(W_{A_n})$, and choose measure-preserving bijections $\phi_n$ with $\varepsilon_n\coloneqq\|W_{A_n}-W^{\phi_n}\|_\square\to0$. For measurable $E,F\subseteq[0,1]$,
\[
\int_{\phi_n(E)\times\phi_n(F)}W\;=\;\int_{E\times F}W^{\phi_n}\;\ge\;\int_{E\times F}W_{A_n}-\varepsilon_n\;\ge\;a\,|E|\,|F|-\varepsilon_n.
\]
Every pair of measurable sets arises as $(\phi_n(E),\phi_n(F))$, with $|\phi_n(E)|=|E|$ and $|\phi_n(F)|=|F|$, so $\int_{E\times F}(W-a)\ge-\varepsilon_n$ for all measurable $E,F$. Letting $n\to\infty$, the signed measure $U\mapsto\int_U(W-a)$ is nonnegative on measurable rectangles, hence on finite disjoint unions of them, and by the monotone class theorem on every measurable $U\subseteq[0,1]^2$. Thus $W\ge a$ almost everywhere, and $W\le b$ follows in the same way.
\end{proof}
% With bounded scores, the mass of $A$ therefore stays away from zero at every size, and any limit is not the zero graphon.
With bounded scores, the attention graphs of \Cref{eq:attn-object} therefore stay in the bounded setting of \Cref{sec:theory} at every size, with mass bounded away from zero, and by Proposition~\ref{prop:invertible} they still determine the post-softmax attention.

\paragraph{Admissible maps.}
% By Proposition~\ref{prop:invertible}, any bijection $\rho$ yields an object from which attention can be recovered, and a continuous injective map on $\mathbb{R}$ is strictly monotone. Increasing examples are the cumulative distribution functions of continuous distributions with full support, such as the logistic sigmoid $\sigma(x)=1/(1+e^{-x})$, the rescaled hyperbolic tangent $(1+\tanh x)/2=\sigma(2x)$, the standard Gaussian cumulative distribution function, and $\tfrac12+\tfrac1\pi\arctan x$. Maps that merge distinct scores, such as clipping or thresholding, cannot be inverted. We use $\sigma$, whose inverse is the log-odds function, so the pre-softmax scores are recovered as $\bar S_{ij}=\log\bigl(A_{ij}/(1-A_{ij})\bigr)+c$. Block averages of values in $[\rho(-B),\rho(B)]$ remain in that interval, so $\rho^{-1}$ is finite on every estimated kernel and sampled attention is well defined.}
Continuous bijections $\rho\colon\mathbb{R}\to(0,1)$ are strictly monotone, and the increasing ones are the cumulative distribution functions of continuous distributions with full support, such as the logistic sigmoid $\sigma(x)=1/(1+e^{-x})$, the rescaled hyperbolic tangent $(1+\tanh x)/2=\sigma(2x)$, the standard Gaussian cumulative distribution function, and $\tfrac12+\tfrac1\pi\arctan x$. Maps that merge distinct scores, such as clipping or thresholding, cannot be inverted. We use $\sigma$, whose inverse is the log-odds function, so the pre-softmax scores are recovered as $\bar S_{ij}=\log\bigl(A_{ij}/(1-A_{ij})\bigr)+c$. Block averages of values in $[\rho(-B),\rho(B)]$ remain in that interval, so $\rho^{-1}$ is finite on every estimated kernel and sampled attention is well defined.

\section{Additional details of \Cref{sec:method}}
\label{app:method}

%\red{L: I don't think we need this next paragraph. Calculating empirical variances is trivial and I added a sentence above explaining that the critical value of the hypothesis test is the bound. Would cut the whole thing or leave very little.}

\paragraph{Variance-based concentration analysis.}
The theory in Section~\ref{sec:theory} predicts that, under a stable kernel view, attention-induced graphs should concentrate on cut-distance as $n$ grows.
To test this prediction empirically, we quantify how tightly per-graph kernel estimates cluster around the dataset-level estimate.
For each sample $r$, define the cut-norm residual
\begin{equation}
\Delta_r \;\coloneqq\; \|\hat W_r - \hat W\|_\square,
\end{equation}
and summarize dispersion by the empirical variance
\begin{equation}
\label{eq:emp-var}
\widehat{\Var}_{\mathrm{emp}} \;\coloneqq\; \frac{1}{m-1}\sum_{r=1}^m (\Delta_r-\bar\Delta)^2,
\qquad
\bar\Delta\coloneqq\frac{1}{m}\sum_{r=1}^m\Delta_r.
\end{equation}
We compute $\widehat{\Var}_{\mathrm{emp}}$ within size bins (fixed $n$ range) to obtain a function of graph size, and compare its decay to the proxy scaling suggested by Theorem~\ref{thm:bound_2}:
\begin{equation}
\label{eq:theory-proxy}
\Var_{\mathrm{th}}(n) \;\propto\; \left(n^{-\tfrac{2\min(\alpha,1)}{\min(\alpha,1)+1}}\right).
\end{equation}
Specifically, given that the smoothness parameter $\alpha$ is latent, we took a conservative approach and use the fastest rate possible, i.e., $\Var_{\mathrm{th}}(n) = n^{-1}$ as the theoretical threshold. In the plots we report $\widehat{\Var}_{\mathrm{emp}}(n)$ and the proxy $\Var_{\mathrm{th}}(n)$.
The resulting decision rule rejects $H_0$ at size $n$ when $\widehat{\Var}_{\mathrm{emp}}(n)>Cn^{-1}$, with $C=1$ (Sec.~\ref{sec:ht}).
The rate in Thm.~\ref{thm:bound_2} holds up to a constant $C$ that depends on the regularity of $W$ and on the canonicalization, and may thus differ across datasets. Taking $C=1$ fixes the units of the comparison without estimating $C$, and since $C$ enters multiplicatively, it shifts the threshold without changing its dependence on $n$.

% \red{L: Just to double check - the below is for computing the cut norm right? Not the cut distance. Because once we are in the structured setting (which is essentially what sort-and-smooth is doing), there is no need to loop over all labelings - everything is labeled the same way. In any case I think the paragraph below has too much detail, can be moved to the appendices perhaps.} \blue{C: Yes. It's to compute the cut norm.}

\paragraph{Approximating cut-norm for block graphons.}
All of our kernels are represented as $K\times K$ block graphons (step functions), so evaluating distances requires computing the cut-norm of a $K\times K$ matrix.
Exact computation of the cut-norm is NP-hard, so we use an approximate algorithm that uses an SDP relaxation combined with a rounding technique, and a fast optimization method with orthogonality constraints \cite{cutnorm_git}. In our experiments, the resulting distance estimates are stable with respect to the number of restarts, and we fix this parameter throughout for fair cross-dataset comparisons.

\begin{algorithm}[t]
\caption{Estimating attention graphons and measuring concentration (fixed $\ell,h$)}
\label{alg:method}
\begin{algorithmic}[1]
\REQUIRE Trained Graph Transformer; dataset graphs $\{G_r\}_{r=1}^m$; target size $N$; block resolution $k$.
\FOR{$r=1$ to $m$}
    % \STATE Extract attention $P^{(\ell,h)}(G_r)$ and symmetrize $A_r \leftarrow \tfrac{1}{2}(P^{(\ell,h)}+P^{(\ell,h)\top})$.
    \STATE Extract logits $S^{(\ell,h)}(G_r)$ and form $A_r \leftarrow \rho\bigl(\tfrac{1}{2}(S^{(\ell,h)}+S^{(\ell,h)\top})-c\bigr)$ via \eqref{eq:attn-object}.
    \STATE Size-normalize: $\tilde A_r \leftarrow \mathrm{Pad}(A_r,N)$.
    \STATE Canonicalize: $\bar A_r \leftarrow M_\pi^r\tilde A_r{M_\pi^r}^\top$ by sorting nodes by $d_i(\tilde A_r)$.
    \STATE Block-average: compute $\hat\Theta_r\in\mathbb{R}^{K\times K}$ via \eqref{eq:block-avg}.
\ENDFOR
\STATE Aggregate template: $\hat\Theta\leftarrow \frac{1}{m}\sum_{r=1}^m \hat\Theta_r$.
\STATE Compute discrepancies $\Delta_r\leftarrow\|\hat\Theta_r - \hat\Theta\|_\square$ and $\widehat{\Var}_{\mathrm{emp}}$ via \eqref{eq:emp-var}.
\STATE Report concentration curves and check the hypothesis test decision rule.
% and the ratio $\widehat{\Var}_{\mathrm{emp}}/\Var_{\mathrm{th}}(n;\alpha)$.
\end{algorithmic}
\end{algorithm}

\section{Experimental details}\label{app:exp_details}
Experiments were conducted on a server with 2x NVIDIA RTX 6000 Ada Generation (48GB) GPUs, 500GB of RAM, and an AMD EPYC 7453 28-Core Processor. Both servers used Ubuntu 22.04.4 LTS as a Linux distro.

We evaluate attention-graphon diagnostics on molecular graphs, social networks, and 3D geometry. We use 8 real-world graph classification benchmarks: 4 bioinformatics datasets (MUTAG \cite{mutag1,mutag2}, PROTEINS \cite{proteins2,proteins1}, NCI1, NCI109 \cite{nci1,nci12}), 3 social interaction datasets (IMDB-MULTI, COLLAB, REDDIT-MULTI-5K \cite{socialnet}), and a 3D point-cloud benchmark represented as graphs (ModelNet10 \cite{modelnet}), and  3 synthetic node-classification settings: a sparse contextual SBM (cSBM) \cite{deshpande2018contextual}, and two graphon datasets (SmoothW and SharpW).

We used a 90/10 split for training and test sets for each dataset. For datasets where node features are absent, we computed them using random-walk positional encodings with a walk length of 16. We performed 10 runs for each experiment and reported the mean and standard deviation, except for the attention graphons, for which we reported only the mean for better visualization.

For fairer comparison across datasets, we set up a one-layer GPS with a hidden dimension size of 32, and no local message passing. Wherever noted otherwise, we used a one-head attention in each experiment. We used PyTorch Geometric (PyG) \cite{fey2019pyg} as our framework\footnote{Within PyG, you can build a GPS model without message passing by setting the convolution in \texttt{GPSConv} to \texttt{None}.}. We trained this model for both node and graph classification, depending on the dataset, employing cross-entropy as the loss function. Finally, we used Adam as optimizer, with a half-life learning rate decay every 20 epochs, starting at $1e-3$.

\begin{table*}[ht]
    \centering
    \caption{{GT training hyperparameters for each dataset.}}
    \resizebox{\columnwidth}{!}{%
    \begin{tabular}{lccccccccccc}
    \toprule
    \textbf{Parameter} & \textbf{MUTAG} & \textbf{PROTEINS} & \textbf{NCI1} & \textbf{NCI109} & \textbf{COLLAB} & \textbf{IMDB-MULTI} & \textbf{REDDIT-MULTI-5K} & \textbf{ModelNet10} & \textbf{NoisyCSBM} & \textbf{SmoothW} & \textbf{SharpW} \\
    \midrule
    Batch size   & 128   & 128   & 128   & 128   & 128 & 128 & 128 & 128 & 128 & 128 & 128\\
    Learning rate   & 0.001     & 0.001     & 0.001     & 0.001 & 0.001 & 0.001 & 0.001     & 0.001     & 0.001     & 0.001 & 0.001\\
    Hidden dimension & 32 & 32 & 32 & 32 & 32 & 32 & 32 & 32 & 32 & 32 & 32\\
    Num. of layers & 1 & 1 & 1 & 1 & 1 & 1 & 1 & 1 & 1 & 1 & 1\\
    \bottomrule
    \end{tabular}
    }
    \label{tab:gt_hyperparams}
\end{table*}

\subsection{Synthetic graphons}
\begin{table}[ht]
    \centering
    \caption{Graphons used for the synthetic dataset experiments. The first is a smooth graphon, whereas the second is a high-frequency injected smooth graphon. For all experiments, the frequency used was $\omega=16$.}
    \begin{tabular}{c|c}
         &  $W(x, y)$\\
         \bottomrule
         \emph{SmoothW} & $xy$\\
         \emph{SharpW} & $\dfrac{\exp(-xy) + |\sin (\omega(x + y))|}{2}$
    \end{tabular}
    \label{tab:synth_graphons}
\end{table}

Table~\ref{tab:synth_graphons} presents the graphons for the synthetic dataset experiments. \emph{SmoothW} stands for a smooth graphon, while \emph{SharpW} stands for a high-frequency injected smooth graphon, which has sharp transitions.

To build the datasets, we generate graphs sampled from these graphons by sampling $n$ latent positions $x_1, \ldots, x_n \sim \mathrm{Unif}[0,1]$ and connecting nodes $i$ and $j$ independently with probability $W(x_i, x_j)$. Node labels for the downstream classification task are derived from the latent positions via terciles, yielding three balanced classes. Node features are computed using random walk positional encodings with a walk length of 16.

\subsection{Spectral summaries of estimated attention graphons}
\label{app:spectral}

Let $\hat\Theta\in\mathbb{R}^{K\times K}$ be a block matrix estimate of an attention graphon.
We use three simple summaries:
\begin{itemize}
\item \textbf{Effective rank \cite{olivier2007effective}:} let $\sigma_1\ge\cdots\ge\sigma_K\ge 0$ be the singular values of $\hat\Theta$ and
$p_i=\sigma_i/\sum_j\sigma_j$. The effective rank is $\mathrm{erank}(\hat\Theta)=\exp\bigl(-\sum_i p_i\log p_i\bigr)$.
\item \textbf{Top-$5$ energy:} the fraction of Frobenius energy captured by the top singular values,
$\mathrm{E}_5(\hat\Theta)=\frac{\sum_{i=1}^5\sigma_i^2}{\sum_{j=1}^K\sigma_j^2}$.
\item \textbf{Spectral gap:} 
% for the symmetric part $\hat{\Theta}_{\text{sym.}} = \tfrac{1}{2}(\hat\Theta+\hat\Theta^\top)$, 
let $\lambda_1\ge\lambda_2$ be the
top eigenvalues of $\hat{\Theta}$ and define $\mathrm{gap}(\hat\Theta)=\lambda_1-\lambda_2$.
\end{itemize}

\subsection{Lipschitz constant computation}
\label{app:lips_comp}
To estimate the Lipschitz constant of an attention graphon, we proceed as follows: given a $K \times K$ histogram estimate of the graphon, we first construct an interpolant over the unit square $[0,1]^2$ by mapping the discrete grid indices to equally spaced points in $[0,1]$. We then evaluate this interpolant on a finer grid of resolution $R \times R$ (with $R = 200$ in our experiments) to obtain a smooth approximation of the graphon surface. The partial derivatives with respect to both coordinates are computed via finite differences on this dense grid. The Lipschitz constant is then estimated as the maximum gradient magnitude over all grid points, i.e., $\hat{L}_W = \max_{ij} \sqrt{(\partial_x W_{ij})^2 + (\partial_y W_{ij})^2}$.
% which provides an upper bound on the rate of change of the graphon across the domain.

\section{Framework pipeline}

\begin{figure*}[ht!]
\centering
\includegraphics[width=0.9\linewidth]{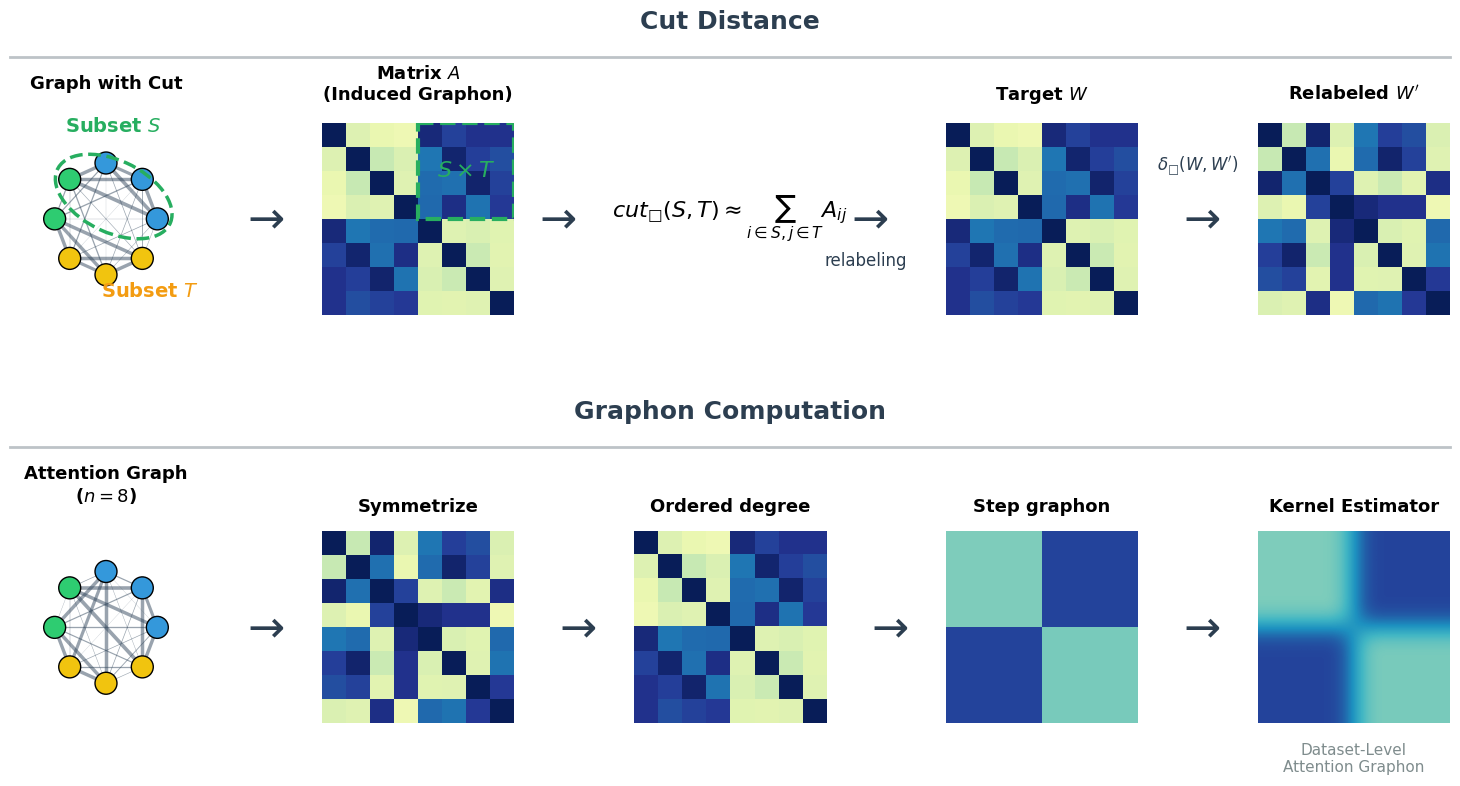}
\caption{\textbf{Visual explanation of the cut distance $\delta_\square$ and graphon computation.} \textit{Top (cut distance)}: The cut distance roughly measures the largest discrepancy over node subsets $S, T$. \textit{Bottom (graphon computation)}: From a dense, weighted attention graph with $n = 8$, the matrix is symmetrized, canonicalized via degree sorting, block-averaged, and finally used to estimate a dataset-level graphon.}
\label{fig:pipeline}
\end{figure*}

\section{Additional results}
\label{app:add_exp}

We report an extensive set of additional experiments conducted to further
stress-test our claims across architectures, datasets, and design choices.
These cover (i) additional GT architectures (Graphormer-GD, GRIT) \cite{zhang2023rethinkingexpressivepowergnns,ma2023graphinductivebiasestransformers}, (ii)
larger-scale benchmarks (OGBmolhiv, LRGBPeptides) \cite{hu2020open,dwivedi2022LRGB}, (iii) algorithmic
reasoning tasks from CLRS (BFS, Dijkstra) \cite{velivckovic2022clrs}, (iv) transferability, (v)
multi-head attention, (vi) runtime, and (vii) sensitivity to the two main
hyperparameters of our estimation pipeline (block size $k$ and
canonicalization method). 
% Throughout, we address the three questions (Q1)--(Q3) from the main text.

\paragraph{Cut-norm variance: empirical vs.\ theoretical scaling.}
\Cref{app:fig_cutvar_all} reports the comparison between the empirical and
theoretical cut-norm across all datasets for GPS. \Cref{fig:ht_gps_h4}
extends this to the 4-head setting and additionally covers the algorithmic
reasoning datasets (BFS, Dijkstra) and the larger-scale benchmarks
(LRGBPeptides, OGBmolhiv). \Cref{fig:ht_graphormer_h1,fig:ht_grit_h1}
repeat the analysis for Graphormer-GD and GRIT respectively. Across all
architectures and benchmarks, the empirical variance tracks---and stays
below---the regularity-aware proxy from \Cref{thm:bound_2}, confirming that the
observed decay is a genuine property of learned attention rather than an
architecture-specific artifact. \Cref{fig:ht_gps_stronger} reports the stronger
hypothesis test, which remains informative on all datasets considered.

\paragraph{Statistical calibration: bootstrap, permutation test, and synthetic null.}
\Cref{fig:ht_gps_stronger} reports a calibrated version of the cut-norm
variance test on GPS across BFS, COLLAB, Dijkstra, LRGBPeptides, MUTAG,
NCI1, PROTEINS, and REDDIT-MULTI-5K. Each panel overlays three
complementary diagnostics on top of the empirical variance curve:
$95\%$ bootstrap confidence intervals (shaded band), a synthetic null
distribution (mean $\pm$ 1 std), and a permutation-test $p$-value
reported in-panel.

\emph{Bootstrap confidence intervals.}
For each graph size $n$ and dataset, we construct $95\%$ bootstrap CIs
by resampling per-graph cut-norm distances ($500$ resamples). Across all
datasets and models, intervals are narrow and lie well below the
theoretical bound, confirming that the variance estimates are statistically
stable and not driven by finite-sample noise.

\emph{Permutation test for monotone decay.}
We test whether the log--log slope of variance vs.\ $n$ is significantly
negative under a permutation null. With four graph sizes, the minimum
achievable $p$-value is $1/24\approx 0.042$. On several datasets (e.g.,
BFS $p=0.036$, NCI1 $p=0.043$), observed $p$-values are at or near this
minimum, meaning the decay ordering is among the most extreme outcomes
under the null. On datasets where theory predicts weaker convergence
(e.g., MUTAG, REDDIT-MULTI-5K), $p$-values are correspondingly larger,
consistent with reduced statistical power rather than absence of the
effect.

\emph{Synthetic null simulation.}
We pass random symmetric matrices through the same estimation pipeline
to characterize the expected variance under a no-structure baseline. On
multiple datasets, the experimental variance at larger $n$ falls below
the null mean, providing direct evidence that trained attention matrices
are more concentrated than random ones; where variance remains above the
null, this diagnostic makes weaker convergence visible rather than being
masked.

Taken together, the three calibrations rule out sampling noise, spurious
orderings, and unstructured attention as explanations, and support the
interpretation that the observed decay reflects a genuine size-dependent
convergence of the learned kernel.

\paragraph{Single-head attention graphons.}
\Cref{app:fig_est_graphons} shows the estimated attention graphons for GPS across
all datasets and target sizes, complementing \Cref{fig:est_graphons} of the main text.
\Cref{fig:graphon_est_graphormer_h1_1,fig:graphon_est_graphormer_h1_2}, and
\Cref{fig:graphon_est_grit_h1_1,fig:graphon_est_grit_h1_2} report the analogous
estimates for Graphormer-GD and GRIT. In every case, the estimates stabilize
within each dataset as $N$ grows while remaining clearly distinct across
datasets.
% indicating that the limiting kernel is a property of the data distribution rather than of the specific GT architecture, though some differences can be seen across models.

\paragraph{Multi-head attention graphons.}
% \Cref{fig:est_graphons_multihead} shows per-head graphon estimates at $N=128$,
% and 
\Cref{fig:graphon_est_gps_h4} shows per-head graphon estimates at $N=128$ for 
% extends this to 
LRGBPeptides,
ModelNet10, OGBmolhiv, and REDDIT-MULTI-5K at dataset-appropriate sizes.
Heads converge to visibly different kernels, supporting the view that each
head implements a distinct interaction mechanism while the per-head limit is
still dataset-determined.

\paragraph{Transferability.}
\Cref{fig:transferability_gps_h1_1,fig:transferability_gps_h1_2} plot the transferability gap as
a function of graph size for GPS (single-head) across all eleven benchmarks.
\Cref{fig:transferability_gps_h4_1,fig:transferability_gps_h4_2} repeat the analysis in
the 4-head setting, and \Cref{fig:transferability_graphormer_h1,fig:transferability_grit_h1}
cover Graphormer-GD and GRIT. The gap decreases consistently with $N$ across
architectures and datasets, indicating that GTs trained on small graphs
extrapolate to larger ones along the predicted scaling.

\paragraph{Runtime.}
\Cref{fig:runtime} reports per-graph runtime (log scale) of our estimation
% pipeline, decomposed into degree-based canonicalization and block averaging
% (\emph{Sort+smooth}), the approximate cut-norm computation (\emph{Cutnorm}),
% and the total. The sort-and-smooth step is nearly constant in $n$, while
% the cut-norm computation dominates; in absolute terms the pipeline processes
% each graph in milliseconds, so the procedure remains tractable on the
pipeline, decomposed into estimation, which computes the attention graph, orders its nodes, and computes the block averages, the approximate cut-norm computation (\emph{Cutnorm}), and the total. Estimation is nearly constant in $n$, while the cut-norm computation dominates, and the pipeline processes each graph in tens of milliseconds, so the procedure remains tractable on the
larger-scale benchmarks (LRGBPeptides, OGBmolhiv).

\paragraph{Factorized attention.}
% Sampled attention is block-constant, $A=S\hat\Theta S^\top$, where $S\in\{0,1\}^{n\times K}$ assigns nodes to blocks. The attention output can therefore be computed as $AV=S\bigl(\hat\Theta(S^\top V)\bigr)$, which costs $\ccalO(nd_v)$ to aggregate $V$ by block, $\ccalO(K^2d_v)$ to apply $\hat\Theta$, and $\ccalO(nd_v)$ to scatter the result back to the nodes. The factorization is exact because the kernel is block-constant by construction. \Cref{tab:factorized} compares its runtime with that of the dense product for $d_v=32$. The maximum relative error between the two outputs is of order $10^{-15}$.}
Sampled attention is block-constant. The sampled logit between a node in block $a$ and a node in block $b$ is $\rho^{-1}(\hat\Theta_{ab})+c$, so the row-wise softmax gives $\tilde P=Z\tilde\Theta Z^\top$, where $Z\in\{0,1\}^{n\times K}$ assigns nodes to blocks and $\tilde\Theta_{ab}=\exp\bigl(\rho^{-1}(\hat\Theta_{ab})\bigr)/\sum_{b'}|B_{b'}|\exp\bigl(\rho^{-1}(\hat\Theta_{ab'})\bigr)$. The attention output can therefore be computed as $\tilde PV=Z\bigl(\tilde\Theta(Z^\top V)\bigr)$, which costs $\ccalO(K^2)$ to form $\tilde\Theta$, $\ccalO(nd_v)$ to aggregate $V$ by block, $\ccalO(K^2d_v)$ to apply $\tilde\Theta$, and $\ccalO(nd_v)$ to scatter the result back to the nodes. The factorization is exact because the kernel is block-constant by construction. \Cref{tab:factorized} compares its runtime with that of the dense product for $d_v=32$. The maximum relative error between the two outputs is of order $10^{-15}$.

\paragraph{Sensitivity to block size $k$.}
\Cref{fig:k_sensitivity_graphon} shows the estimated graphon for
$K\in\{16,32,64,128\}$ on BFS, COLLAB, Dijkstra, LRGBPeptides, and PROTEINS:
larger $k$ yields finer-grained estimates but the kernels remain
qualitatively stable. \Cref{fig:k_sensitivity_hypothesis} examines the effect of
% $k$ on the hypothesis test: larger $k$ generally yields lower empirical
% variance, with the caveat that on datasets with many small graphs (e.g.,
% MUTAG) choosing $k$ larger than typical graph sizes introduces zero-valued
% regions; even so, at $K=16$ the variance remains below the theoretical
% bound. Our default $K=64$ balances resolution and robustness.
$k$ on the hypothesis test. For fixed $k$, the variance curves nearly coincide and stay below the theoretical bound, and the growing resolution $K_n=\lceil\sqrt n\rceil$ differs only at the smallest sizes, so the default $K=64$ does not drive our conclusions.

\paragraph{Sensitivity to canonicalization.}
\Cref{fig:sensitivity_canon_hypothesis} compares degree sorting against spectral
(Fiedler) sorting for the hypothesis test on MUTAG, PROTEINS, COLLAB, and
LRGBPeptides: despite differences in absolute values, the two methods
exhibit nearly identical decay patterns, with curves tightly aligned in
log-log scale.
Figures~\ref{fig:sensitivity_canon_graphon_mutag}-\ref{fig:sensitivity_canon_graphon_lrbbpep} show the actual
per-size graphon estimates under both canonicalizations. The observed
variance decay therefore reflects genuine properties of the learned
attention rather than an artifact of the node-ordering strategy.

\paragraph{Node features versus graph structure.}
The descriptor of each node combines its features and its structural role (Sec.~\ref{sec:attn-hyp}), so an estimated kernel depends on both. To separate their contributions, we replace the node features with i.i.d.\ noise and rerun the pipeline. On SmoothW and SharpW, where the features are random-walk encodings and hence functions of the topology, randomizing them destroys concentration (\Cref{tab:featabl}). On PROTEINS, whose node attributes are informative in their own right, the variance is nearly unchanged, while the estimated kernel shifts by about $0.16$ in relative $L_1$ distance. Features thus enter the kernel in both cases. On the synthetic families they carry the structural signal that stabilizes attention, whereas on PROTEINS the stability survives their removal.

\paragraph{Effect of symmetrization.}
% \Cref{tab:asym} reports the relative residual $\|P-A\|_F/\|P\|_F$ between the attention matrix $P$ and its symmetrization $A$ in \Cref{eq:sym-attn}, as a range over the sizes of each sweep.}
\Cref{tab:asym} reports the relative residual $\|P-\bar P\|_F/\|P\|_F$ between the attention matrix $P$ and its symmetrization $\bar P=\tfrac12(P+P^\top)$, as a range over the sizes of each sweep.

\paragraph{Attention restricted to input edges.}
We train the same GPS architecture on NoisyCSBM with attention masked to the input edges and self-loops. Masked pre-softmax scores are $-\infty$, so $A$ vanishes off the edges, and its mass follows the attended density, falling from $0.035$ at $n=128$ to $0.0044$ at $n=1024$ (\Cref{fig:masked_csbm}). The induced graphons therefore converge to the zero graphon, yet the variance stays four to six orders of magnitude below $n^{-1}$, and the test does not reject. With full attention, which ignores the edges, the kernel mass stays near $1/2$.

\begin{figure*}[ht!]
\centering
\begin{subfigure}[t]{0.32\textwidth}
  \centering\includegraphics[width=\linewidth]{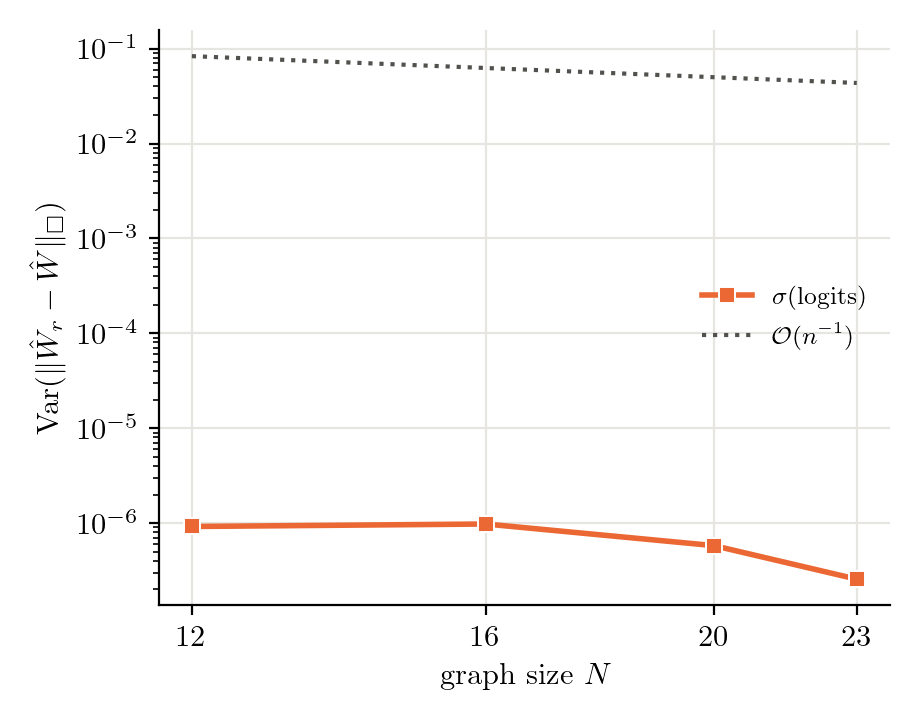}
  % \centering\includegraphics[width=\linewidth]{Figures/exp1_mutag_min_max_6_variance.png}
  \caption{\centering{\textbf{MUTAG}}}
\end{subfigure}
\begin{subfigure}[t]{0.32\textwidth}
  \centering\includegraphics[width=\linewidth]{Figures/variance_PROTEINS_fixed.png}
  % \centering\includegraphics[width=\linewidth]{Figures/exp1_proteins_min_max_6_variance.png}
  \caption{\centering{\textbf{PROTEINS}}}
\end{subfigure}
\begin{subfigure}[t]{0.32\textwidth}
  \centering\includegraphics[width=\linewidth]{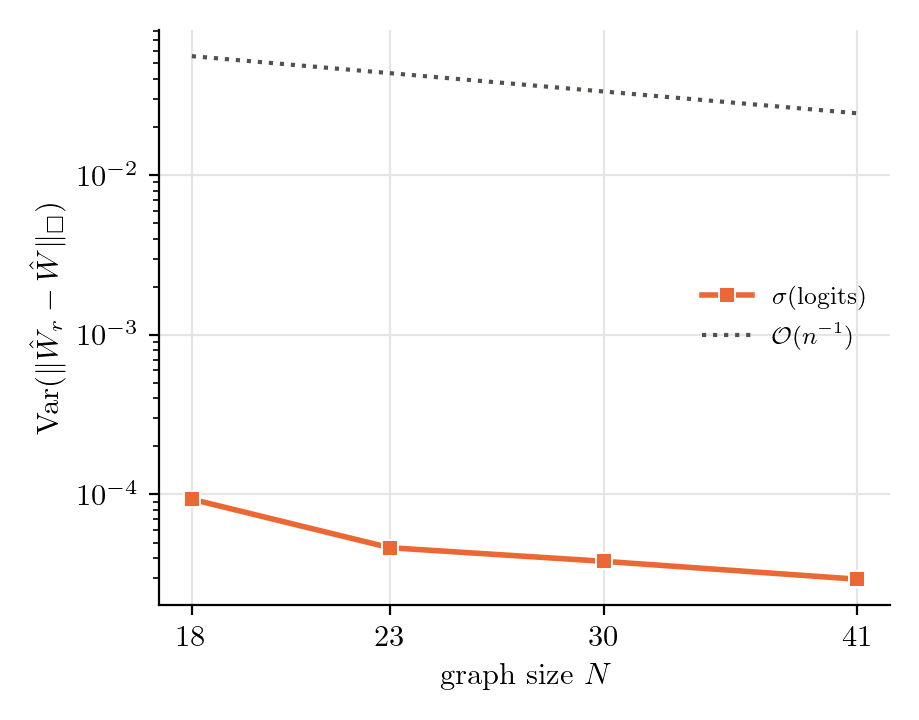}
  \caption{\centering{\textbf{NCI1}}}
\end{subfigure}

\vspace{0.3em}

\begin{subfigure}[t]{0.32\textwidth}
  \centering\includegraphics[width=\linewidth]{Figures/variance_NCI109_fixed.png}
  \caption{\centering{\textbf{NCI109}}}
\end{subfigure}
\begin{subfigure}[t]{0.32\textwidth}
  \centering\includegraphics[width=\linewidth]{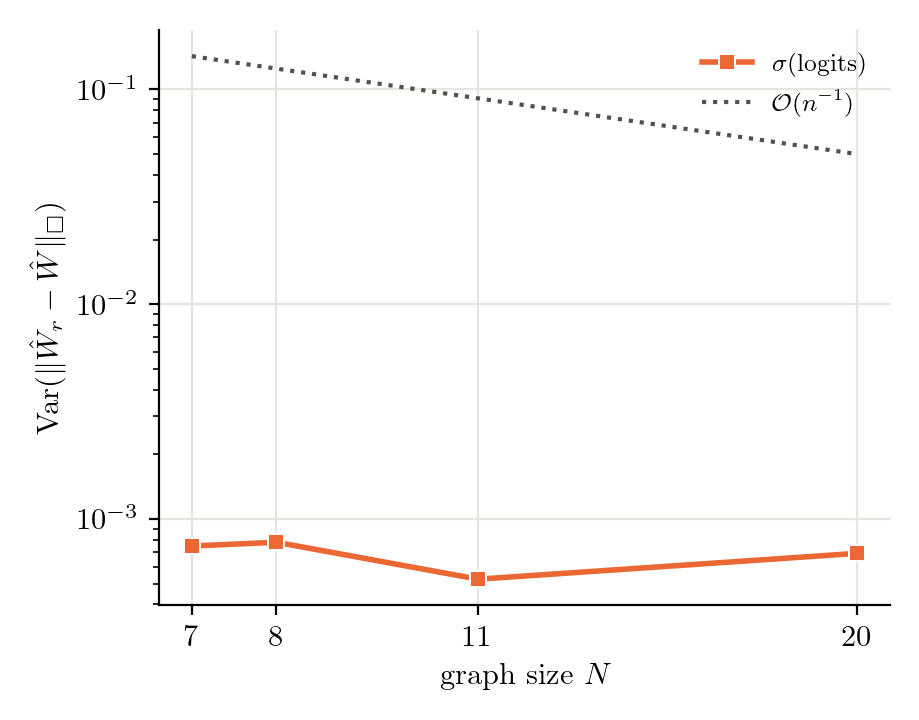}
  \caption{\centering{\textbf{IMDB}\textbf{-}\textbf{MULTI}}}
\end{subfigure}
\begin{subfigure}[t]{0.32\textwidth}
  \centering\includegraphics[width=\linewidth]{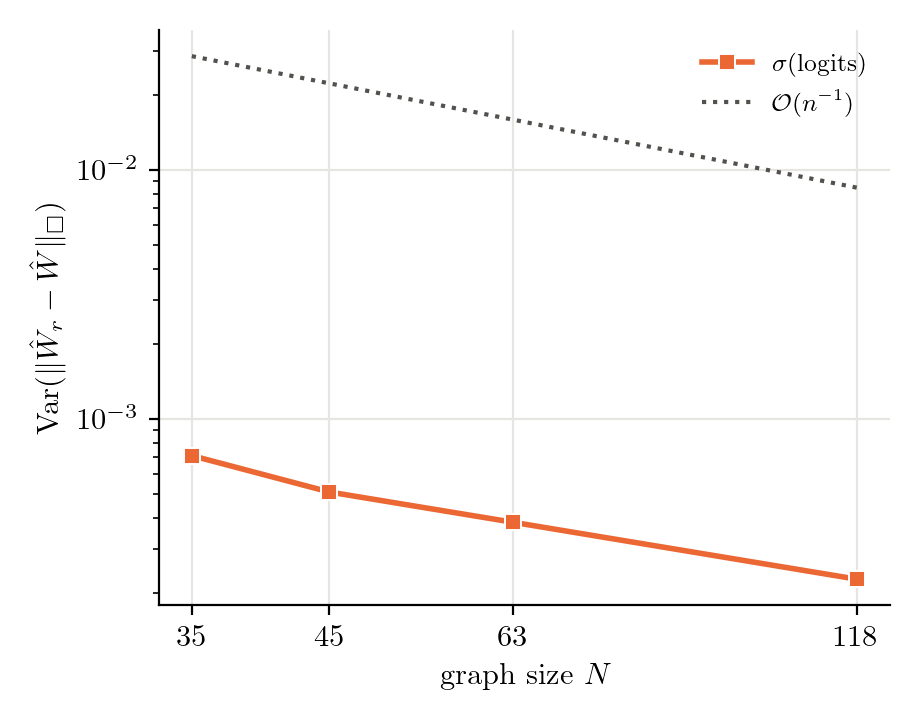}
  \caption{\centering{\textbf{COLLAB}}}
\end{subfigure}

\vspace{0.3em}

\begin{subfigure}[t]{0.32\textwidth}
  \centering\includegraphics[width=\linewidth]{Figures/variance_REDDIT-MULTI_fixed.png}
  \caption{\centering{\textbf{REDDIT}\textbf{-}\textbf{MULTI}\textbf{-}\textbf{5K}}}
\end{subfigure}
\begin{subfigure}[t]{0.32\textwidth}
  \centering\includegraphics[width=\linewidth]{Figures/variance_ModelNet10_fixed.png}
  \caption{\centering{\textbf{ModelNet10}}}
\end{subfigure}
\begin{subfigure}[t]{0.32\textwidth}
  \centering\includegraphics[width=\linewidth]{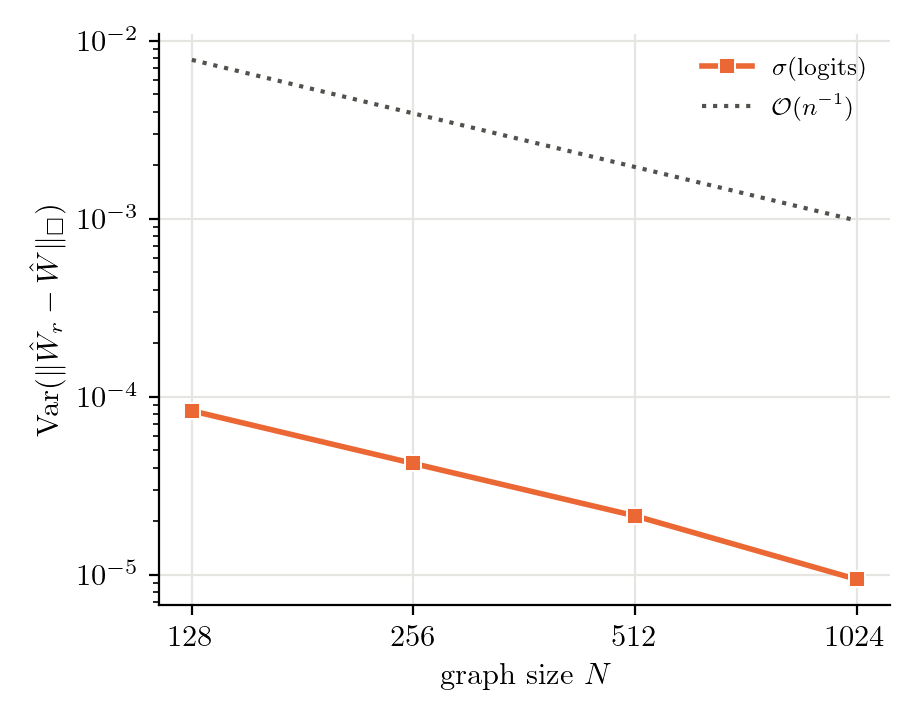}
  \caption{\centering{\textbf{cSBM}}}
\end{subfigure}

\caption{\textbf{Empirical cut-norm variance versus a regularity-aware theoretical proxy (GPS, single-head).}
% Each panel compares the empirical variance of $\|\hat W_r - \hat W\|_\square$ \blue{(blue)} against the size-dependent scaling
Each panel compares the \orange{empirical variance} of $\|\hat W_r - \hat W\|_\square$ against the size-dependent \textcolor{darkgray}{scaling proxy} $\ccalO(n^{-1})$
% proxy from Theorem~\ref{thm:bound_2} \green{(green)}.
from Theorem~\ref{thm:bound_2}.
% The target sizes $N$ are dataset-dependent and are listed in each caption.
}
\label{app:fig_cutvar_all}
\vskip -0.1in
\end{figure*}

\begin{comment}
\paragraph{Single-head attention graphons}
In \Cref{app:fig_est_graphons}, we show the attention graphon convergence for all the datasets presented in \Cref{fig:est_graphons}. It is worth noting the different patterns that each benchmark converges to as the graph sizes grow.
\end{comment}
\begin{figure*}[ht!]
    \vskip 0.1in
    \begin{center}
        % {\includegraphics[width=\linewidth]{Figures/est_graphons.png}}
        {\includegraphics[width=\linewidth]{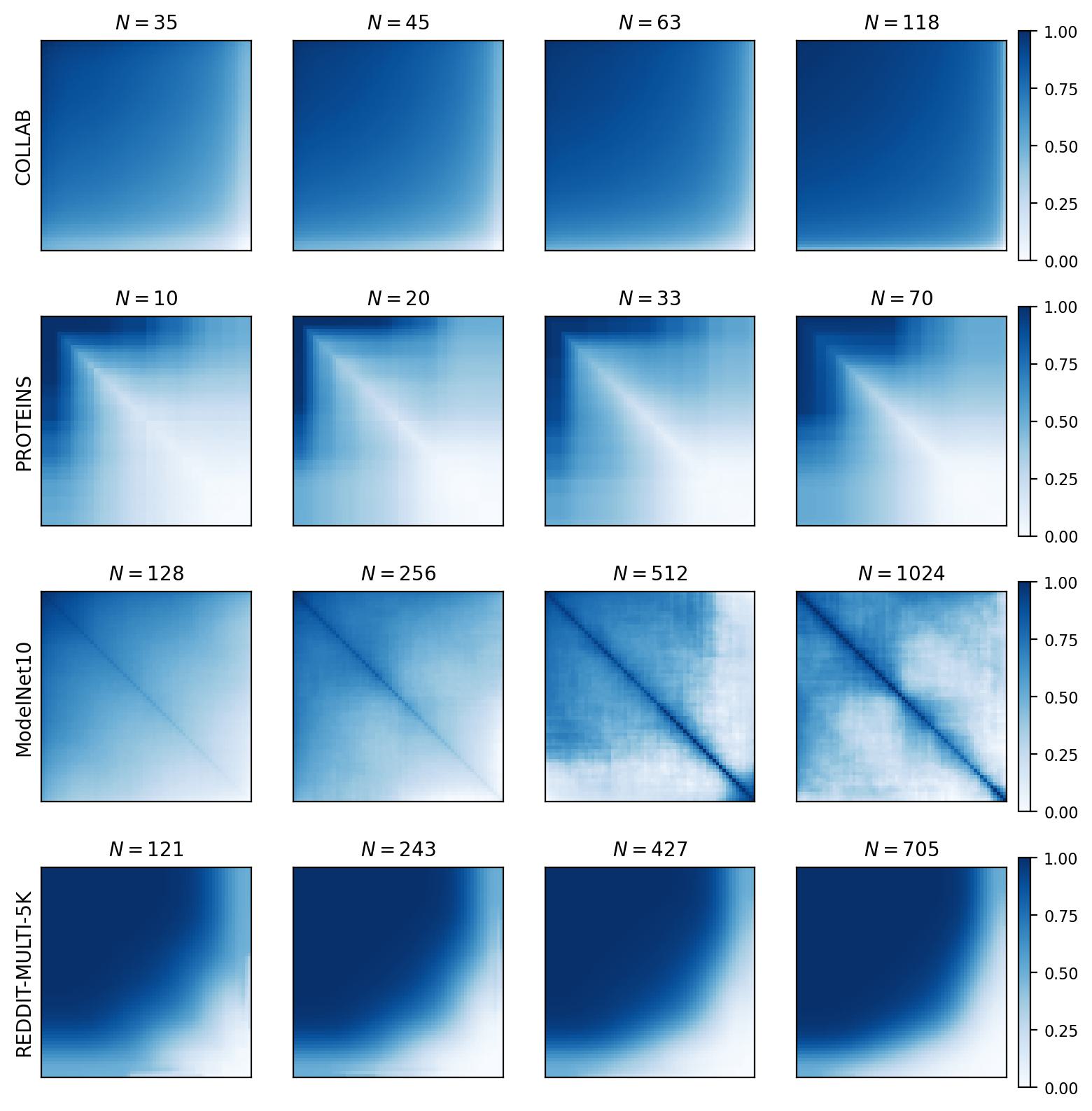}}
        \caption{
        \textbf{Estimated attention graphons across datasets and graph sizes (GPS, single-head).}
        Each row corresponds to a dataset, and each panel shows the $K{\times}K$ (here $K=64$) block-step estimate of the
        dataset-level attention kernel at a target size $N$.
        % The target sizes in this figure are: COLLAB ($N\in\{64,128,256,498\}$), PROTEINS ($N\in\{100,200,400,600\}$),
        The target sizes in this figure are: COLLAB ($N\in\{35,45,63,118\}$), PROTEINS ($N\in\{10,20,33,70\}$),
        % ModelNet10 ($N\in\{128,256,512,1024\}$), and REDDIT-MULTI-5K ($N\in\{128,256,512,1024\}$).
        ModelNet10 ($N\in\{128,256,512,1024\}$), and REDDIT-MULTI-5K ($N\in\{121,243,427,705\}$).
        Kernel estimates stabilize within each dataset as $N$ grows while remaining distinct across datasets.
        }
        \label{app:fig_est_graphons}
    \end{center}
    \vskip -0.1in
\end{figure*}

\begin{comment}
\paragraph{Multi-head attention graphons}
Figure~\ref{fig:est_graphons_multihead} shows that different heads converge to different dataset-dependent kernels,
supporting the view that each head implements a distinct interaction mechanism.
\end{comment}

\begin{comment}
\begin{figure*}[ht!]
    \centering
    \includegraphics[width=\linewidth]{Figures/est_graphons_multihead_2.png}
    \caption{\textbf{Multi-head attention graphons. (GPS, 4 heads).} Dataset-level graphon estimates for multiple heads ($h \in \{1, 2, 3, 4\}$) at a fixed size
    ($N=128$). Different heads exhibit distinct kernel structures.}
    \label{fig:est_graphons_multihead}
\end{figure*}
\end{comment}

\begin{figure*}[ht!]
\centering
% Row 1: lrgbpeptides
\begin{subfigure}[t]{0.19\textwidth}
  \centering\includegraphics[width=\linewidth]{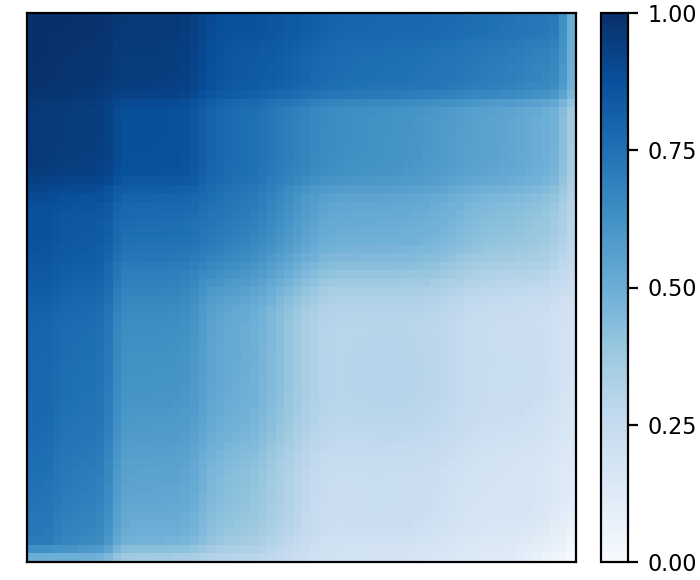}
\end{subfigure}
\begin{subfigure}[t]{0.19\textwidth}
  \centering\includegraphics[width=\linewidth]{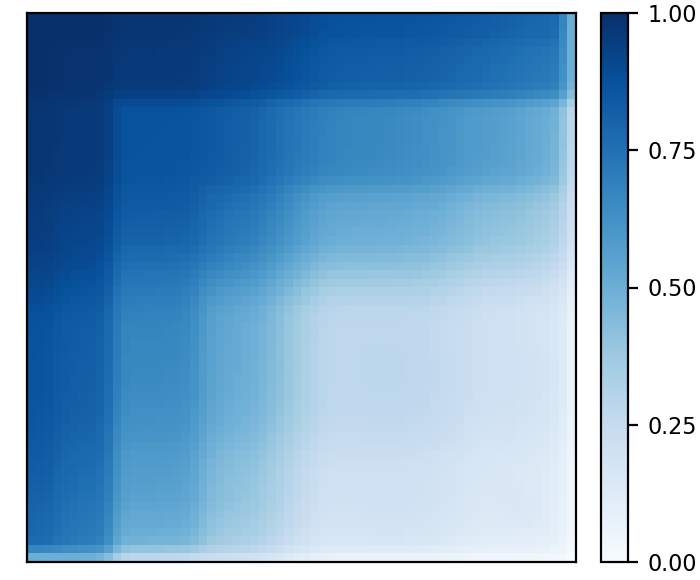}
\end{subfigure}
\begin{subfigure}[t]{0.19\textwidth}
  \centering\includegraphics[width=\linewidth]{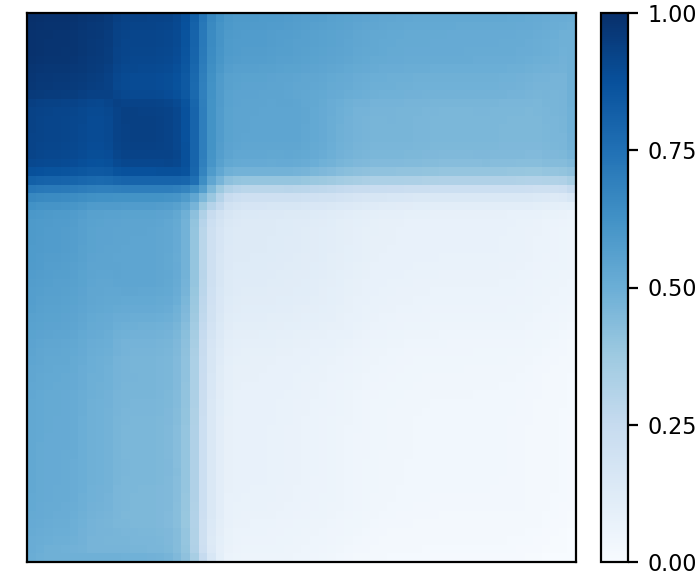}
\end{subfigure}
\begin{subfigure}[t]{0.19\textwidth}
  \centering\includegraphics[width=\linewidth]{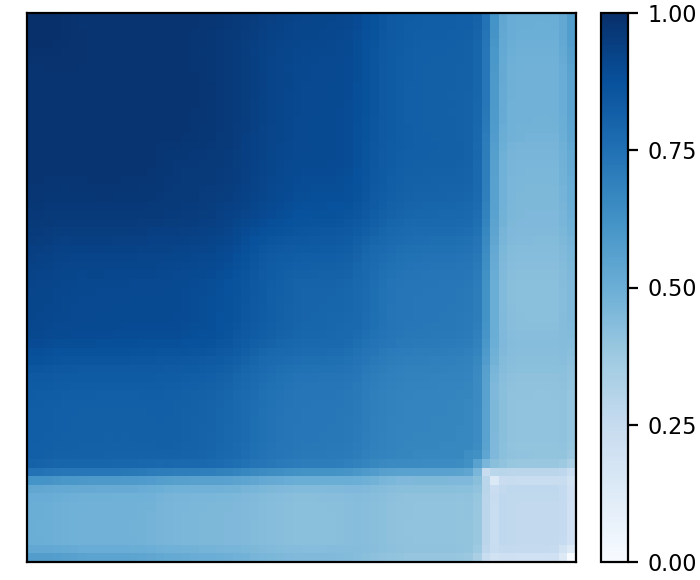}
\end{subfigure}
\begin{subfigure}[t]{0.19\textwidth}
  \centering\includegraphics[width=\linewidth]{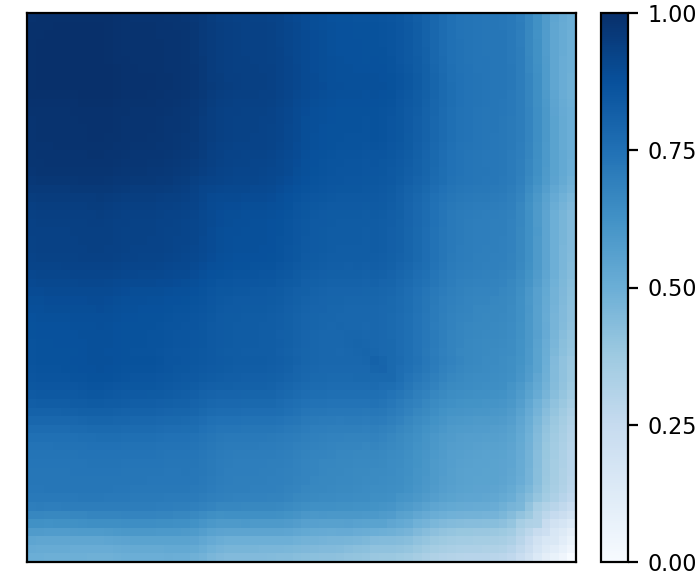}
\end{subfigure}

\vspace{0.3em}

% Row 2: modelnet
\begin{subfigure}[t]{0.19\textwidth}
  \centering\includegraphics[width=\linewidth]{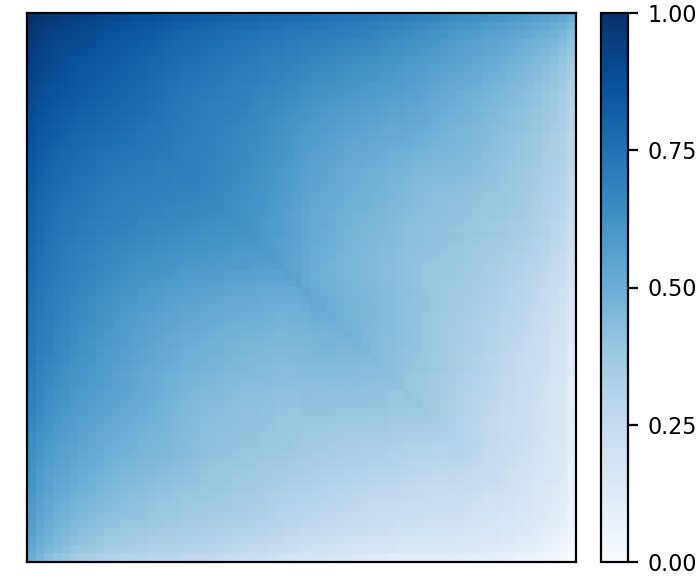}
\end{subfigure}
\begin{subfigure}[t]{0.19\textwidth}
  \centering\includegraphics[width=\linewidth]{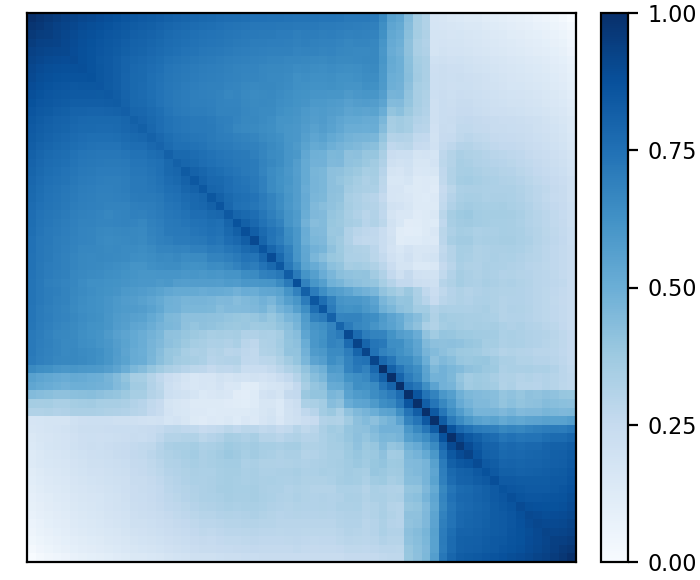}
\end{subfigure}
\begin{subfigure}[t]{0.19\textwidth}
  \centering\includegraphics[width=\linewidth]{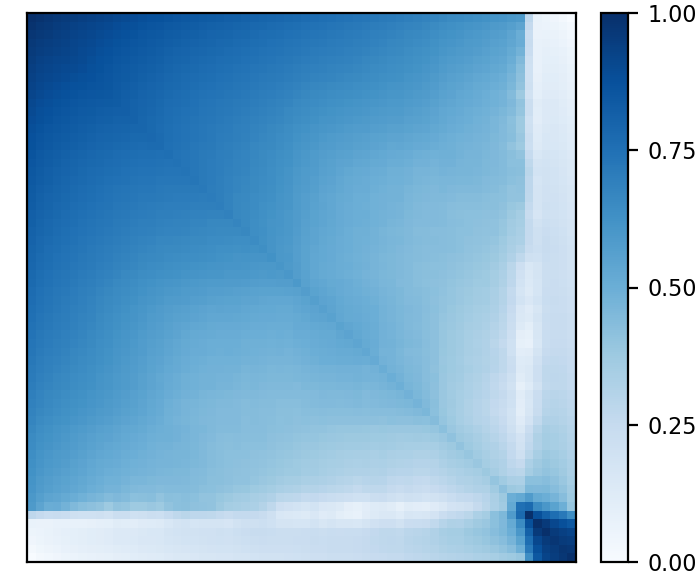}
\end{subfigure}
\begin{subfigure}[t]{0.19\textwidth}
  \centering\includegraphics[width=\linewidth]{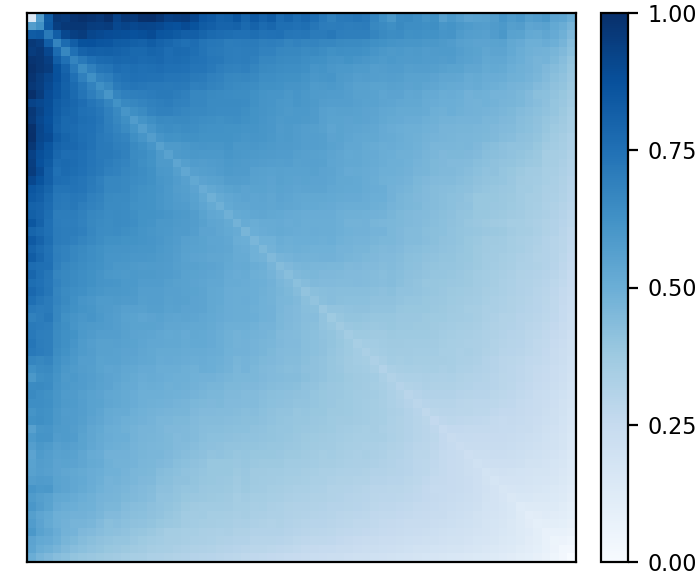}
\end{subfigure}
\begin{subfigure}[t]{0.19\textwidth}
  \centering\includegraphics[width=\linewidth]{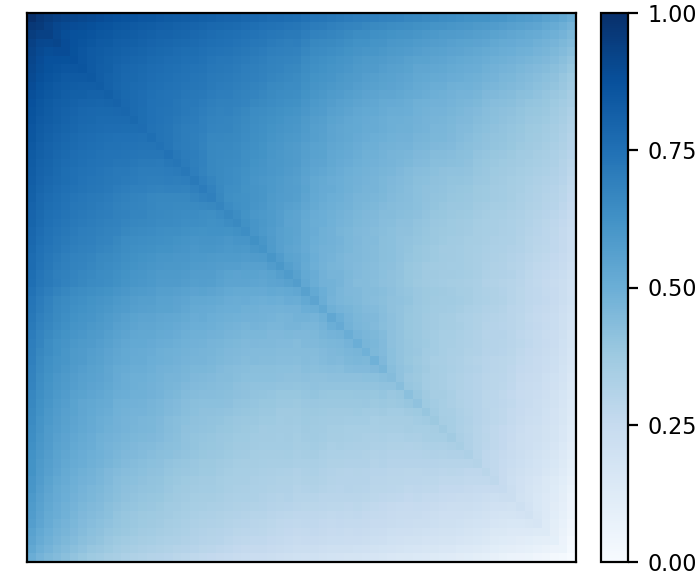}
\end{subfigure}

\vspace{0.3em}

% Row 3: ogbmolhiv
\begin{subfigure}[t]{0.19\textwidth}
  \centering\includegraphics[width=\linewidth]{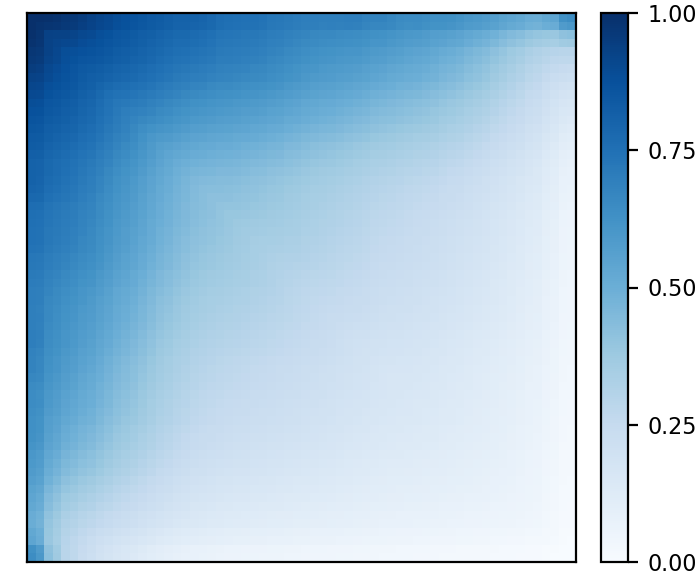}
\end{subfigure}
\begin{subfigure}[t]{0.19\textwidth}
  \centering\includegraphics[width=\linewidth]{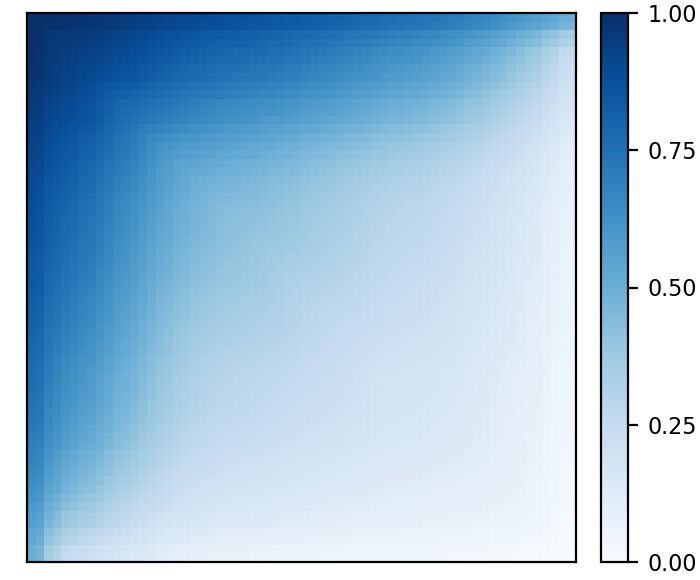}
\end{subfigure}
\begin{subfigure}[t]{0.19\textwidth}
  \centering\includegraphics[width=\linewidth]{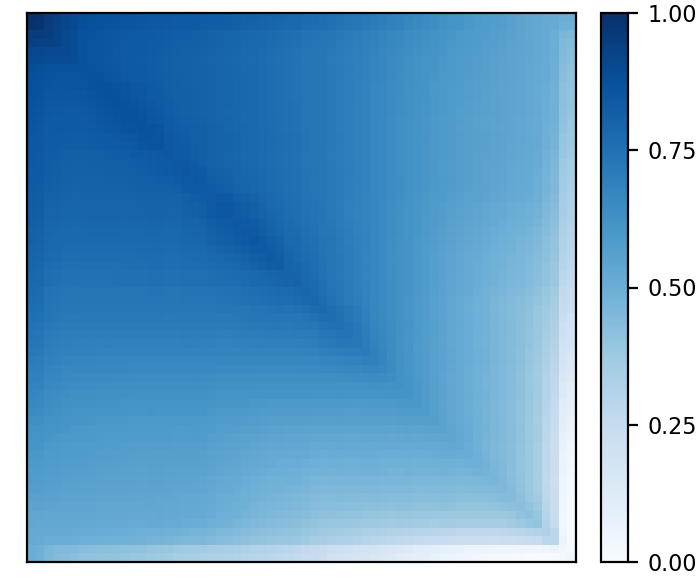}
\end{subfigure}
\begin{subfigure}[t]{0.19\textwidth}
  \centering\includegraphics[width=\linewidth]{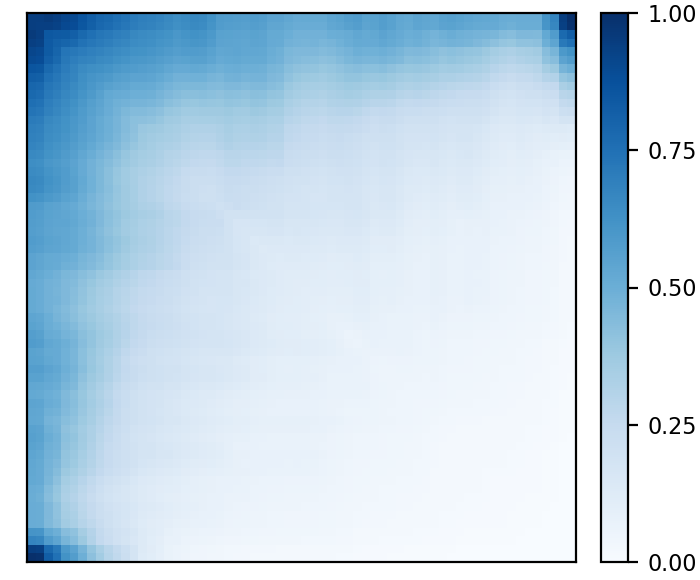}
\end{subfigure}
\begin{subfigure}[t]{0.19\textwidth}
  \centering\includegraphics[width=\linewidth]{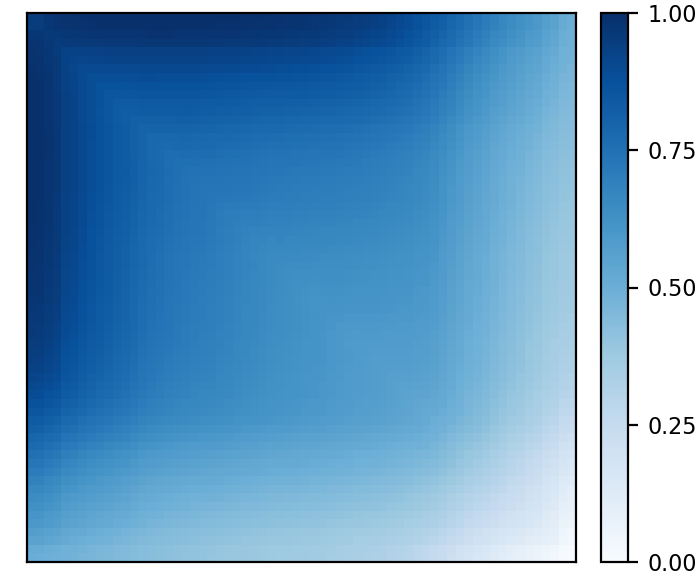}
\end{subfigure}

\vspace{0.3em}

% Row 4: redditmulti5k
\begin{subfigure}[t]{0.19\textwidth}
  \centering\includegraphics[width=\linewidth]{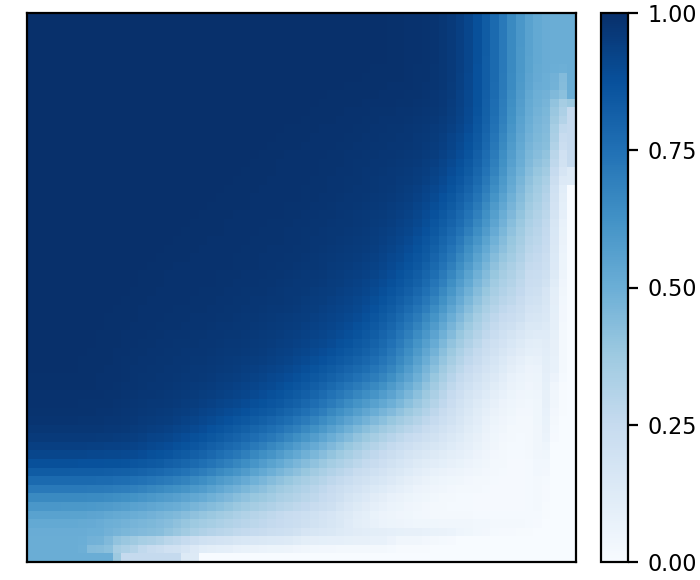}
\end{subfigure}
\begin{subfigure}[t]{0.19\textwidth}
  \centering\includegraphics[width=\linewidth]{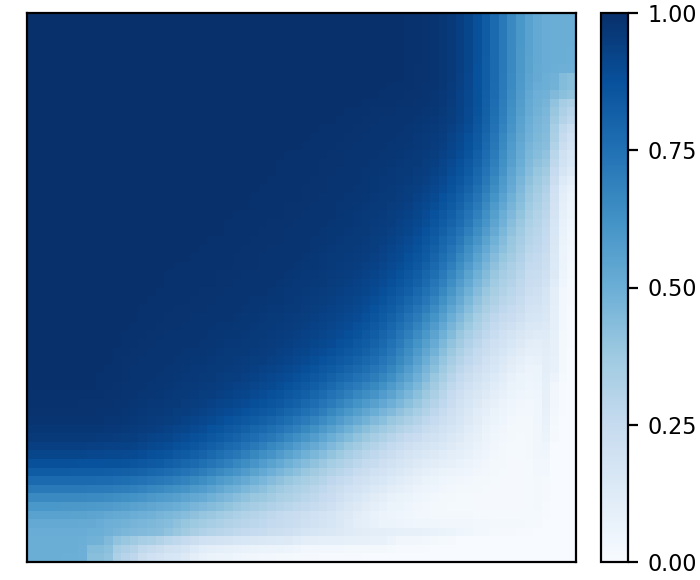}
\end{subfigure}
\begin{subfigure}[t]{0.19\textwidth}
  \centering\includegraphics[width=\linewidth]{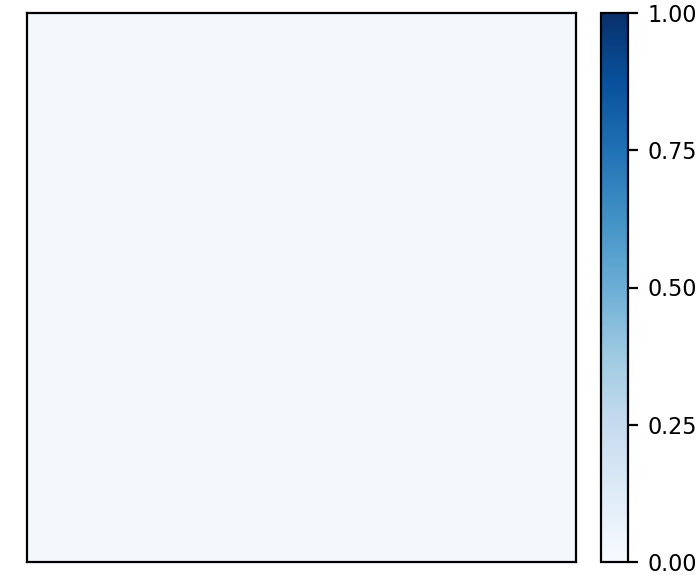}
\end{subfigure}
\begin{subfigure}[t]{0.19\textwidth}
  \centering\includegraphics[width=\linewidth]{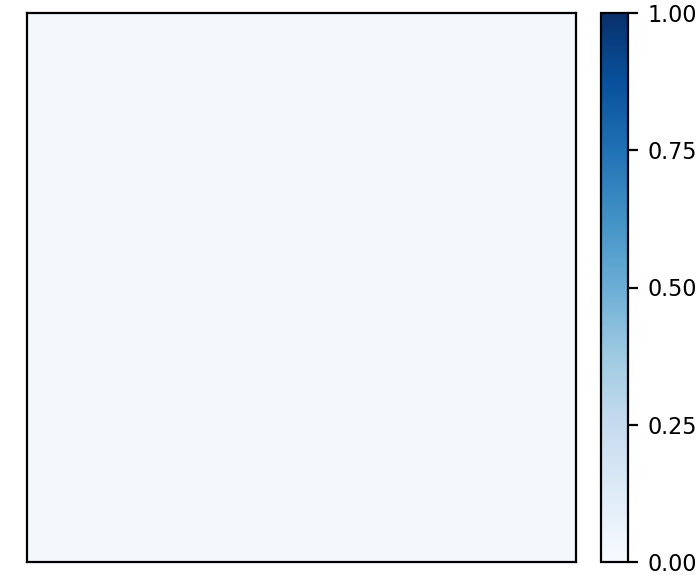}
\end{subfigure}
\begin{subfigure}[t]{0.19\textwidth}
  \centering\includegraphics[width=\linewidth]{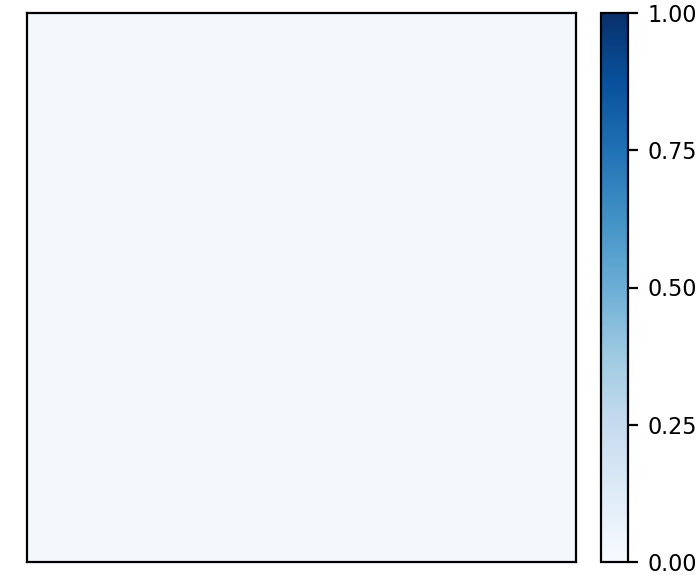}
\end{subfigure}

% \caption{\textbf{Multi-head attention graphons (GPS, 4 heads).} Each row corresponds to a dataset (LRGBPeptides, ModelNet10, OGBmolhiv, REDDIT-MULTI-5K), and each column shows the average graphon followed by head-specific graphon estimates ($h \in \{1,2,3,4\}$). Each dataset has a different fixed size ($N \in \{64, 256, 512\}$), as stated in the figures' titles. Different heads exhibit distinct kernel structures.}
\caption{\textbf{Multi-head attention graphons (GPS, 4 heads).} Each row corresponds to a dataset (LRGBPeptides, ModelNet10, OGBmolhiv, REDDIT-MULTI-5K), and each column shows the average graphon followed by head-specific graphon estimates ($h \in \{1,2,3,4\}$). Each dataset has a different fixed size ($N \in \{36, 259, 427, 512\}$), as stated in the figures' titles. Different heads exhibit distinct kernel structures.}
\label{fig:graphon_est_gps_h4}
\end{figure*}

\begin{figure*}[ht!]
\centering
% Row 1: collab
\begin{subfigure}[t]{0.24\textwidth}
  \centering\includegraphics[width=\linewidth]{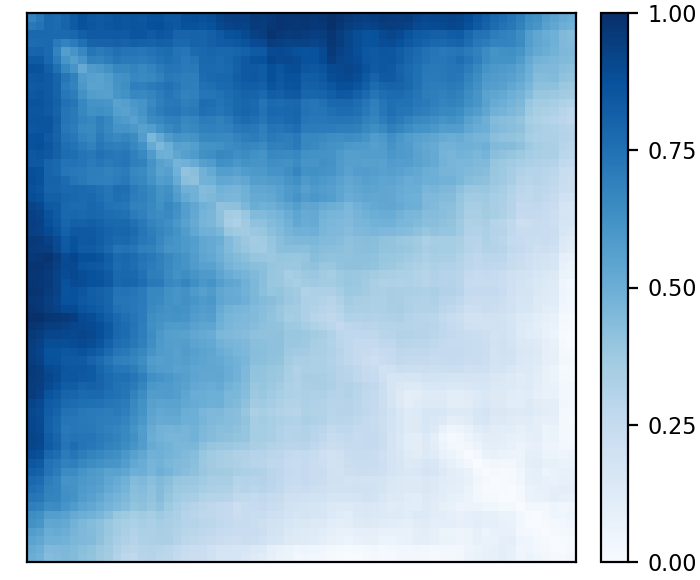}
\end{subfigure}
\begin{subfigure}[t]{0.24\textwidth}
  \centering\includegraphics[width=\linewidth]{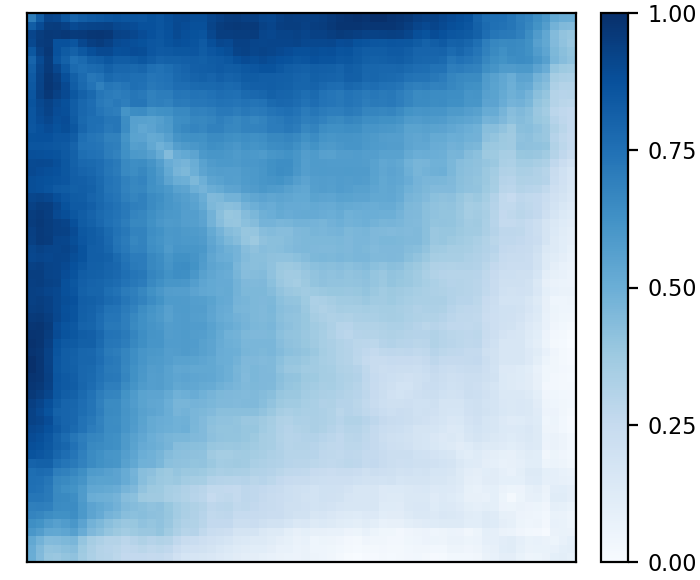}
\end{subfigure}
\begin{subfigure}[t]{0.24\textwidth}
  \centering\includegraphics[width=\linewidth]{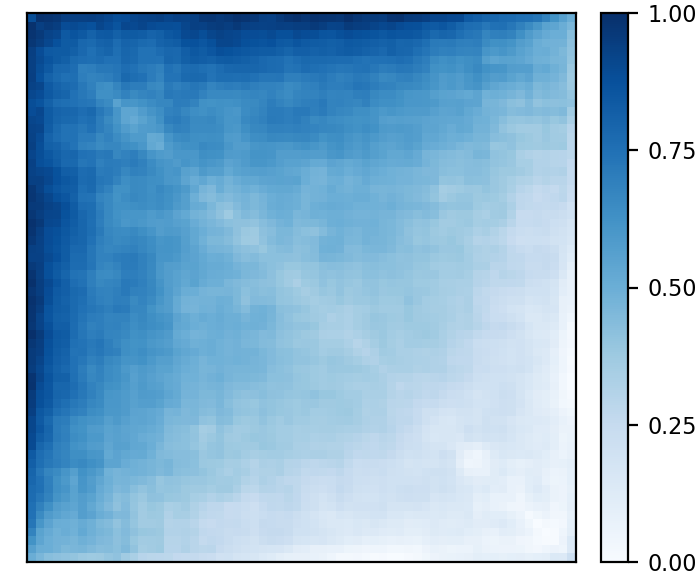}
\end{subfigure}
\begin{subfigure}[t]{0.24\textwidth}
  \centering\includegraphics[width=\linewidth]{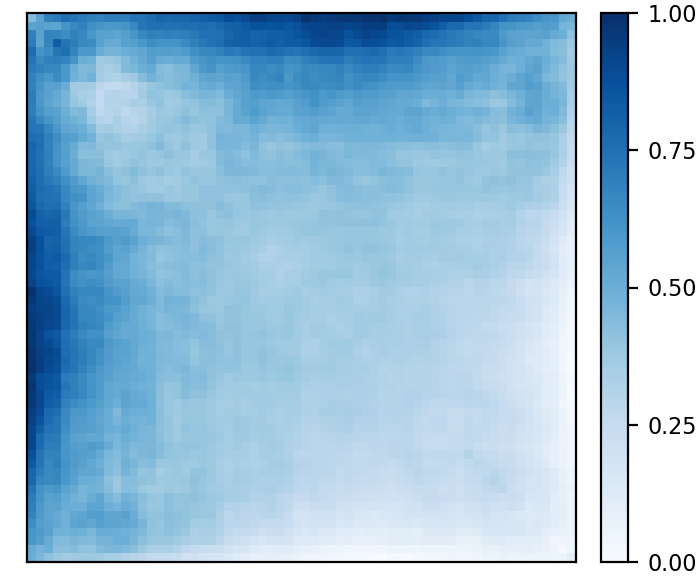}
\end{subfigure}

\vspace{0.3em}

% Row 2: imdbmulti
\begin{subfigure}[t]{0.24\textwidth}
  \centering\includegraphics[width=\linewidth]{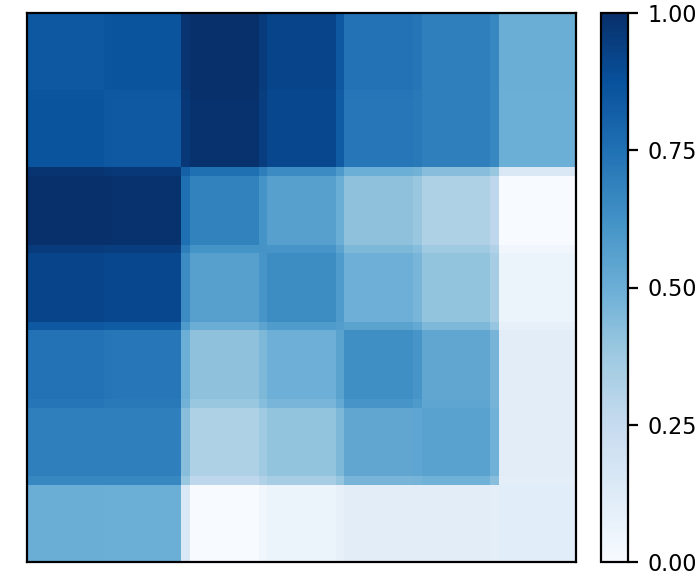}
\end{subfigure}
\begin{subfigure}[t]{0.24\textwidth}
  \centering\includegraphics[width=\linewidth]{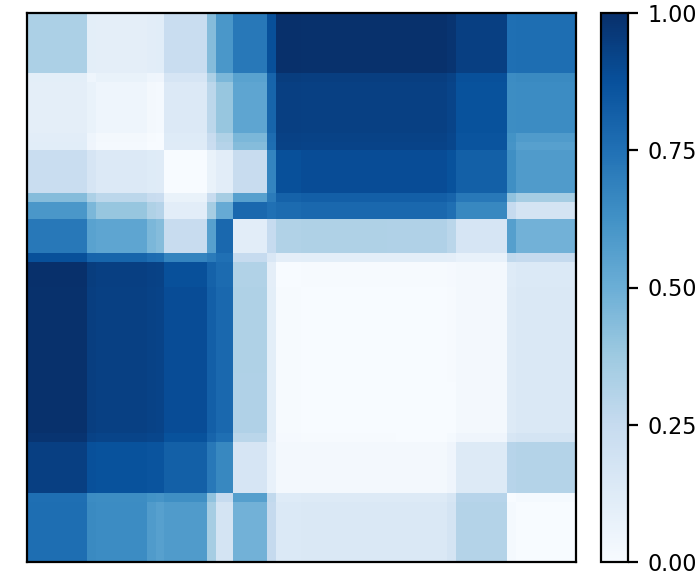}
\end{subfigure}
\begin{subfigure}[t]{0.24\textwidth}
  \centering\includegraphics[width=\linewidth]{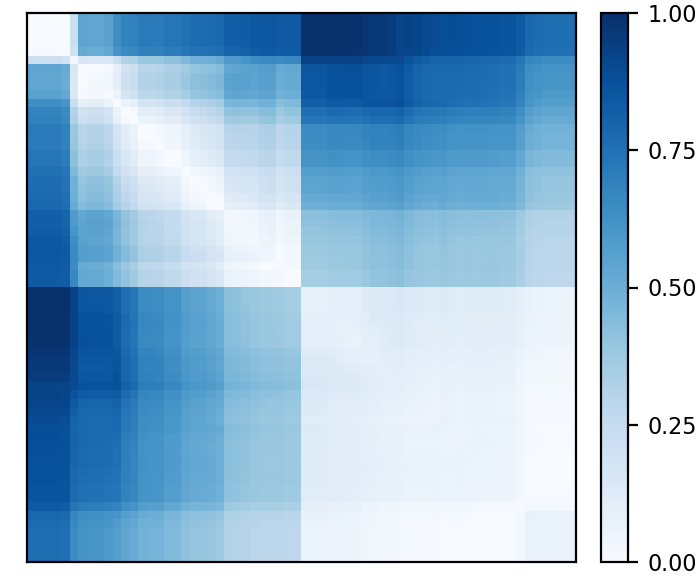}
\end{subfigure}
\begin{subfigure}[t]{0.24\textwidth}
  \centering\includegraphics[width=\linewidth]{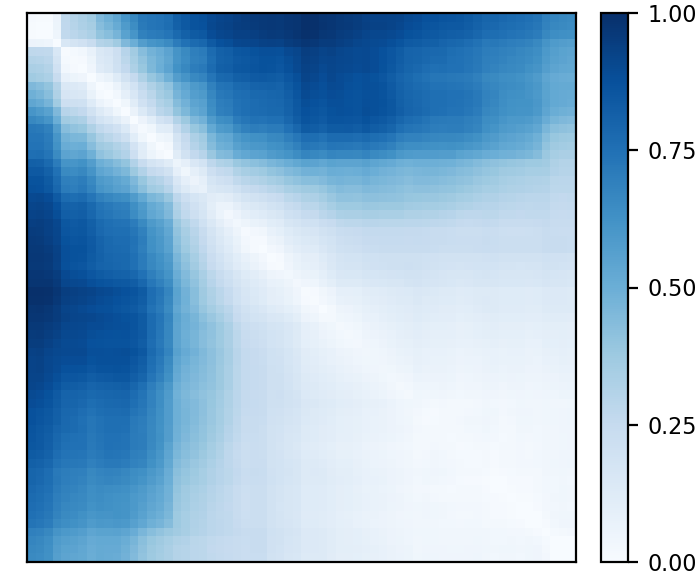}
\end{subfigure}

\vspace{0.3em}

% Row 3: lrgbpeptides
\begin{subfigure}[t]{0.24\textwidth}
  \centering\includegraphics[width=\linewidth]{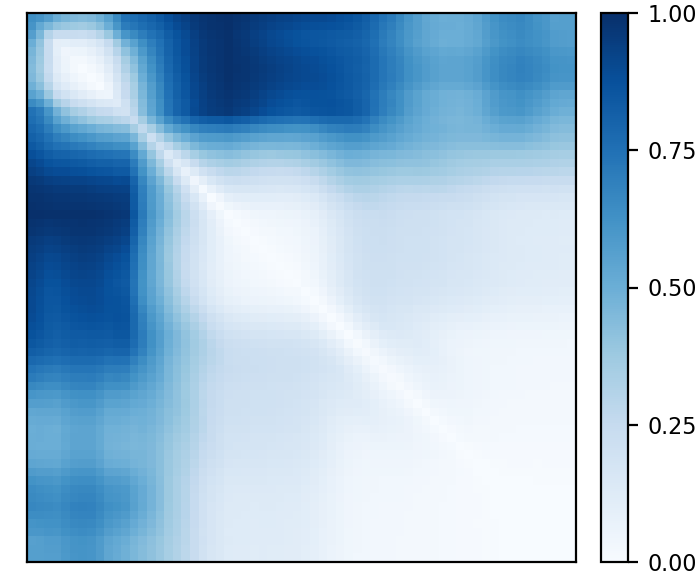}
\end{subfigure}
\begin{subfigure}[t]{0.24\textwidth}
  \centering\includegraphics[width=\linewidth]{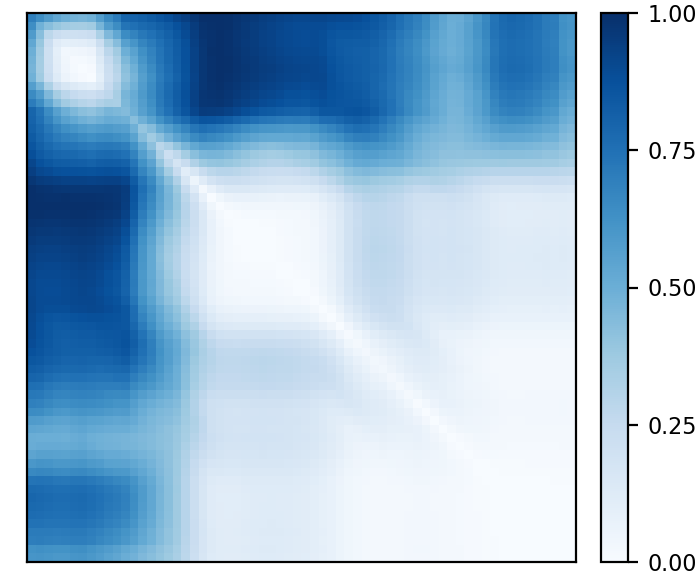}
\end{subfigure}
\begin{subfigure}[t]{0.24\textwidth}
  \centering\includegraphics[width=\linewidth]{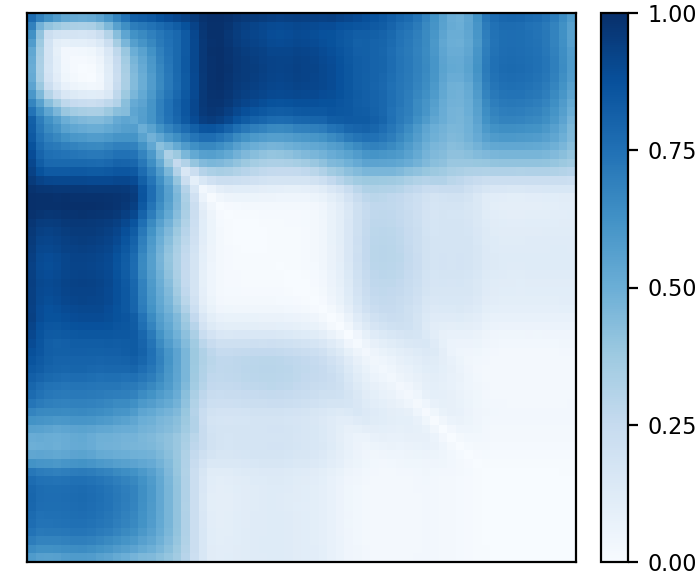}
\end{subfigure}
\begin{subfigure}[t]{0.24\textwidth}
  \centering\includegraphics[width=\linewidth]{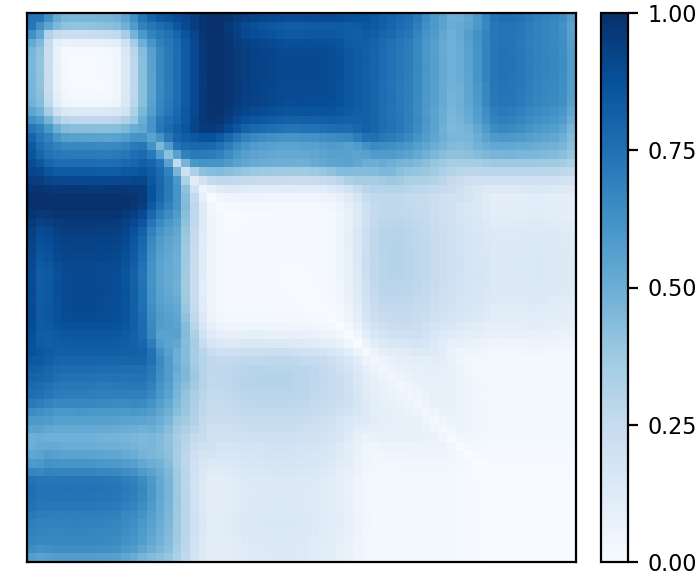}
\end{subfigure}

\vspace{0.3em}

% Row 4: modelnet
\begin{subfigure}[t]{0.24\textwidth}
  \centering\includegraphics[width=\linewidth]{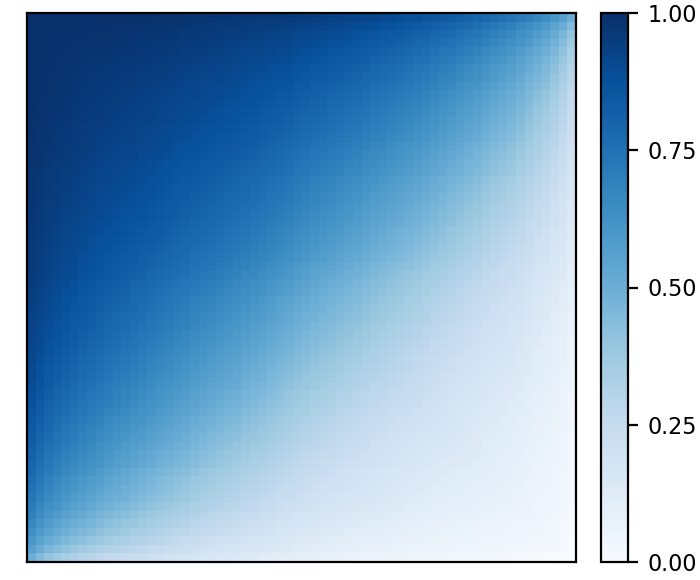}
\end{subfigure}
\begin{subfigure}[t]{0.24\textwidth}
  \centering\includegraphics[width=\linewidth]{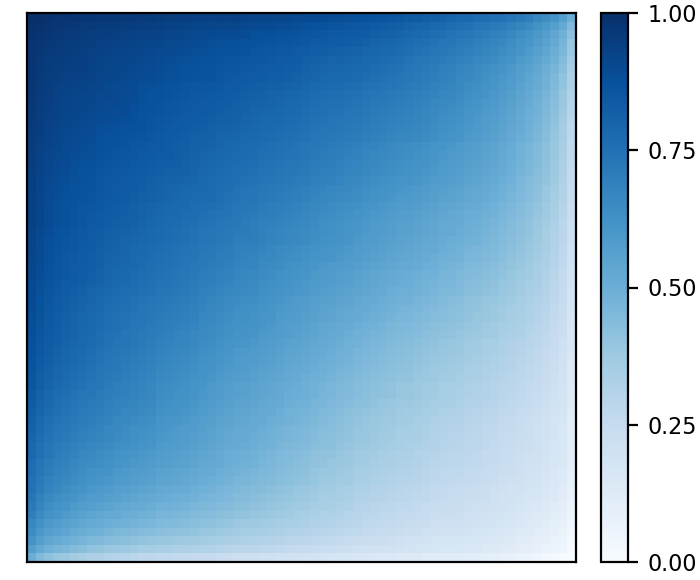}
\end{subfigure}
\begin{subfigure}[t]{0.24\textwidth}
  \centering\includegraphics[width=\linewidth]{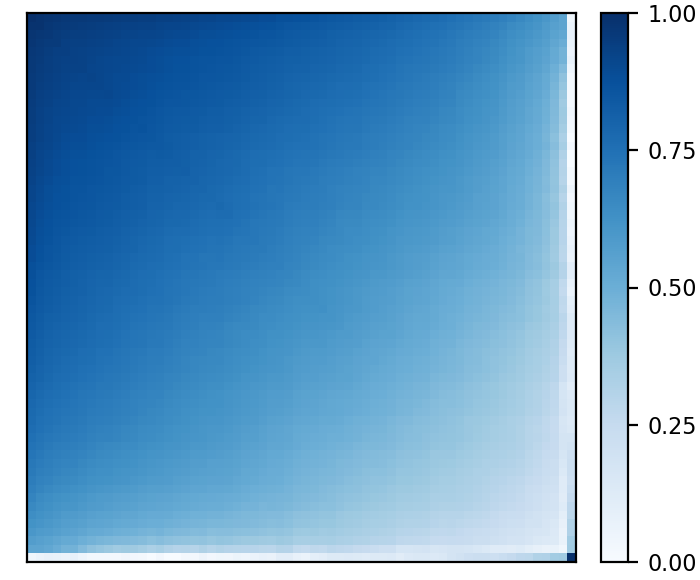}
\end{subfigure}
\begin{subfigure}[t]{0.24\textwidth}
  \centering\includegraphics[width=\linewidth]{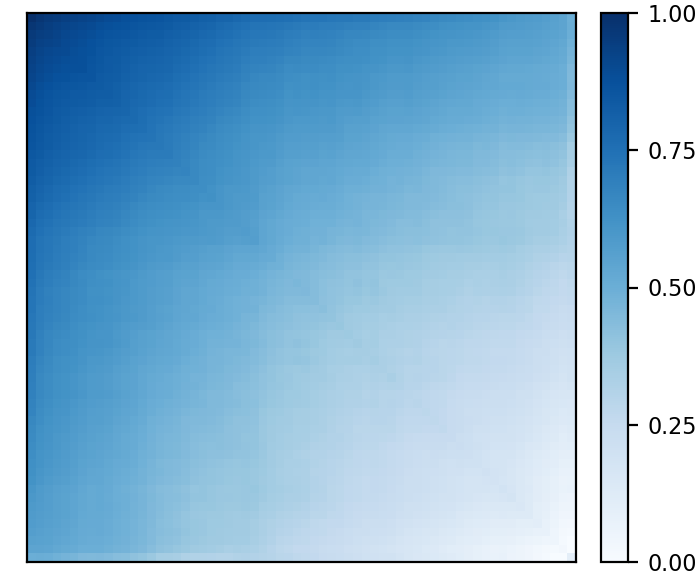}
\end{subfigure}

\vspace{0.3em}

% Row 5: mutag
\begin{subfigure}[t]{0.24\textwidth}
  \centering\includegraphics[width=\linewidth]{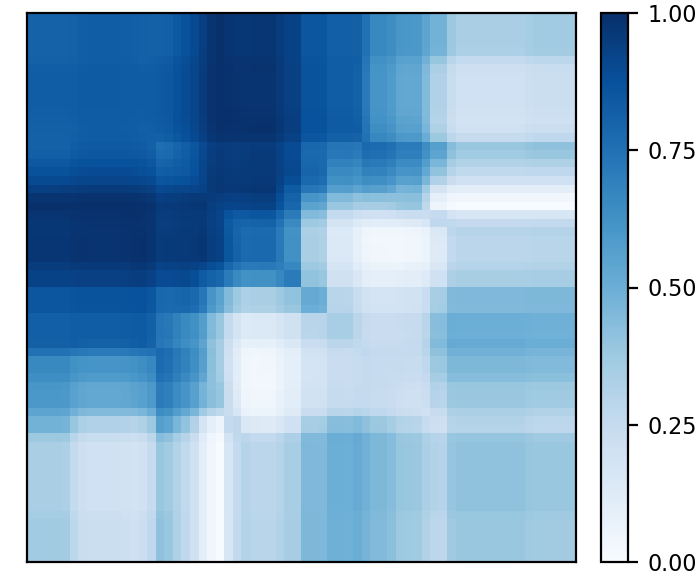}
\end{subfigure}
\begin{subfigure}[t]{0.24\textwidth}
  \centering\includegraphics[width=\linewidth]{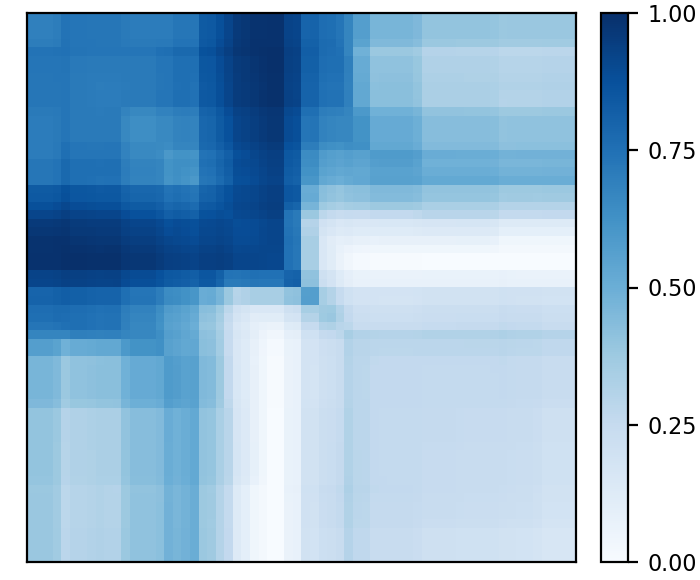}
\end{subfigure}
\begin{subfigure}[t]{0.24\textwidth}
  \centering\includegraphics[width=\linewidth]{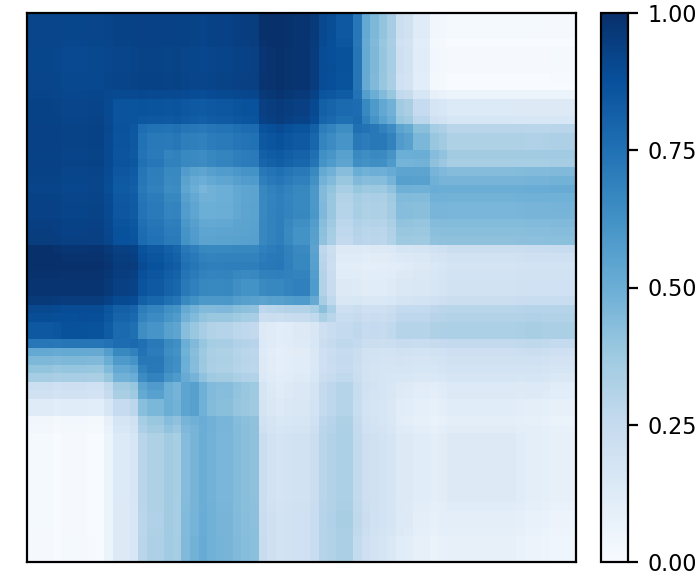}
\end{subfigure}
\begin{subfigure}[t]{0.24\textwidth}
  \centering\includegraphics[width=\linewidth]{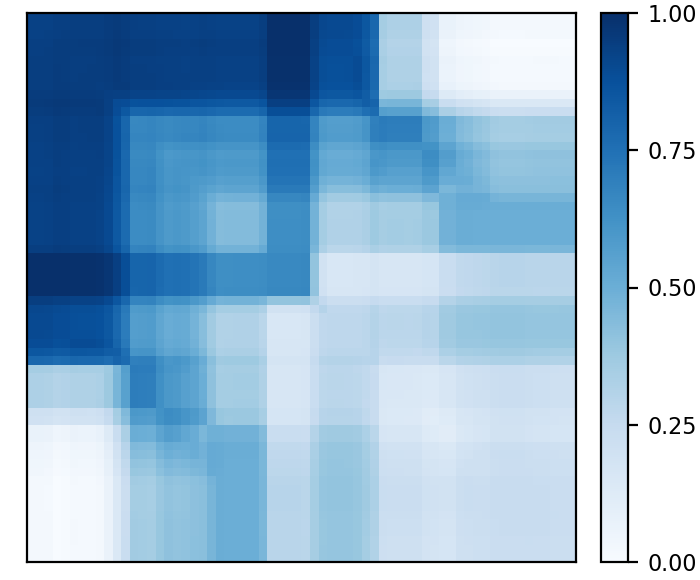}
\end{subfigure}

\caption{\textbf{Estimated attention graphons across datasets and graph sizes (Graphormer-GD, single-head).} Each row corresponds to a dataset (COLLAB, IMDB-MULTI, LRGBPeptides, ModelNet10, MUTAG), and each column shows the block-step estimate of the dataset-level attention kernel at increasing target sizes $N$. Kernel estimates stabilize within each dataset as $N$ grows.}
\label{fig:graphon_est_graphormer_h1_1}
\end{figure*}

\begin{figure*}[ht!]
\centering
% Row 1: nci1
\begin{subfigure}[t]{0.24\textwidth}
  \centering\includegraphics[width=\linewidth]{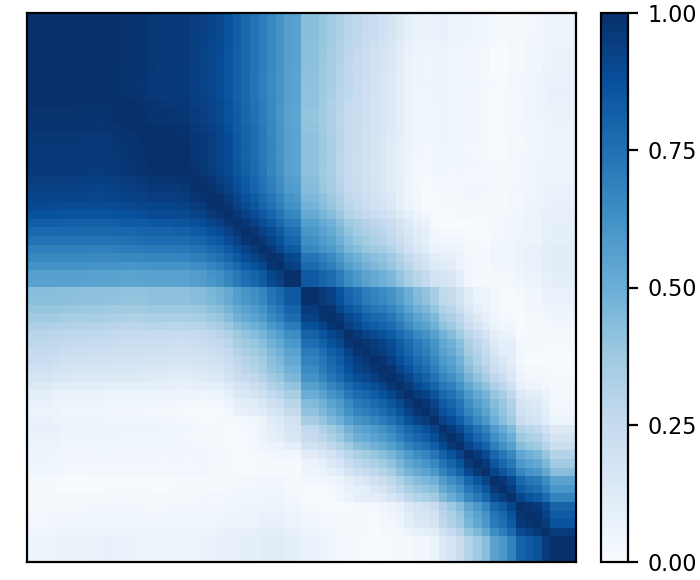}
\end{subfigure}
\begin{subfigure}[t]{0.24\textwidth}
  \centering\includegraphics[width=\linewidth]{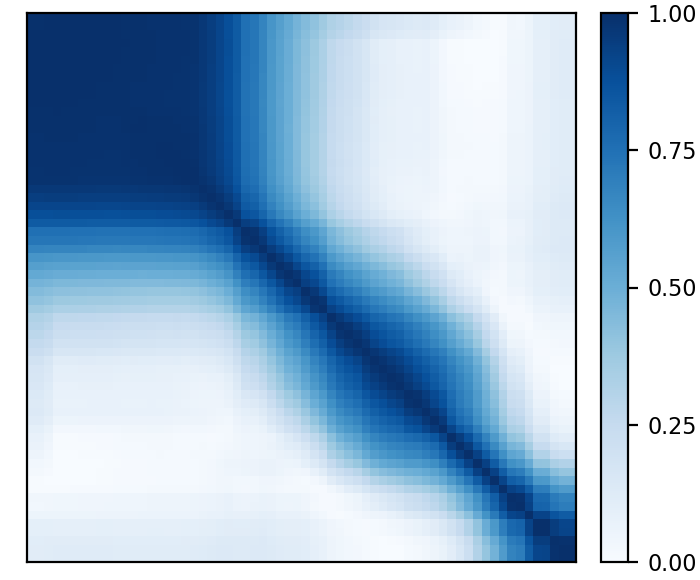}
\end{subfigure}
\begin{subfigure}[t]{0.24\textwidth}
  \centering\includegraphics[width=\linewidth]{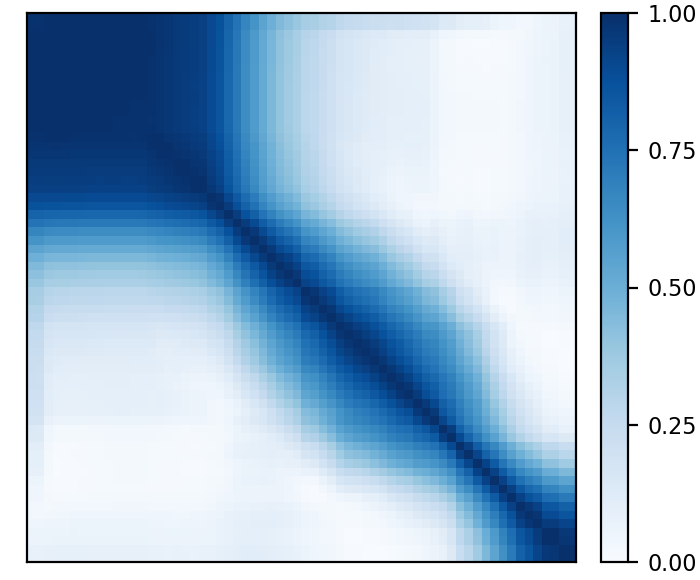}
\end{subfigure}
\begin{subfigure}[t]{0.24\textwidth}
  \centering\includegraphics[width=\linewidth]{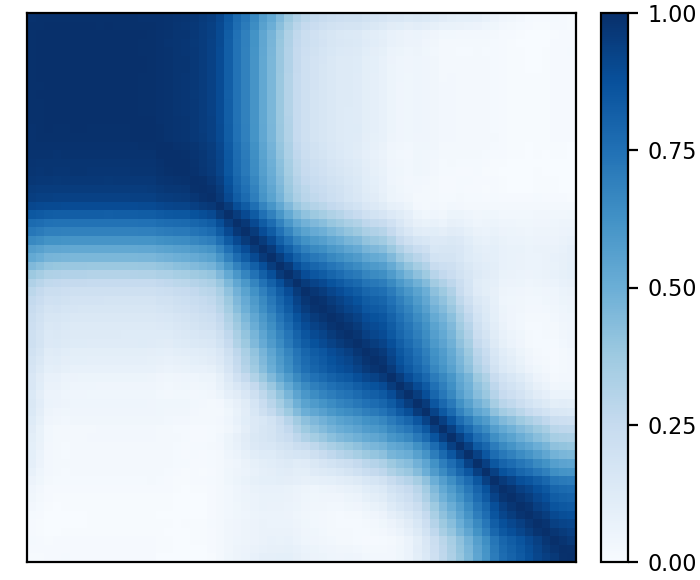}
\end{subfigure}

\vspace{0.3em}

% Row 2: nci109
\begin{subfigure}[t]{0.24\textwidth}
  \centering\includegraphics[width=\linewidth]{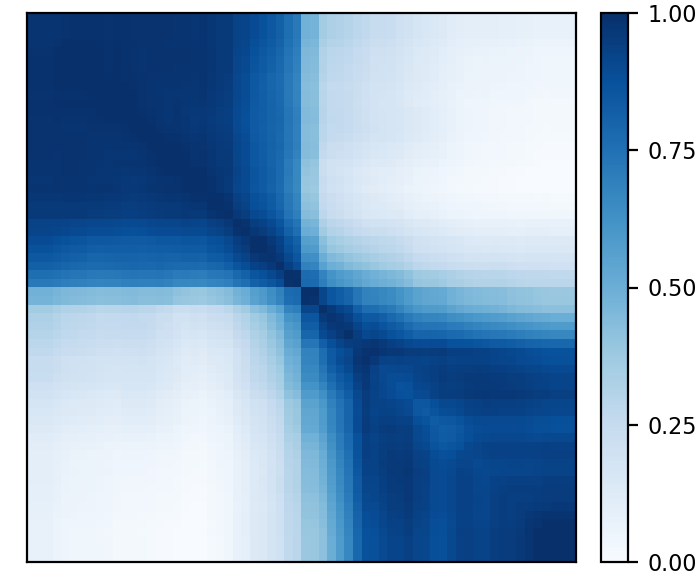}
\end{subfigure}
\begin{subfigure}[t]{0.24\textwidth}
  \centering\includegraphics[width=\linewidth]{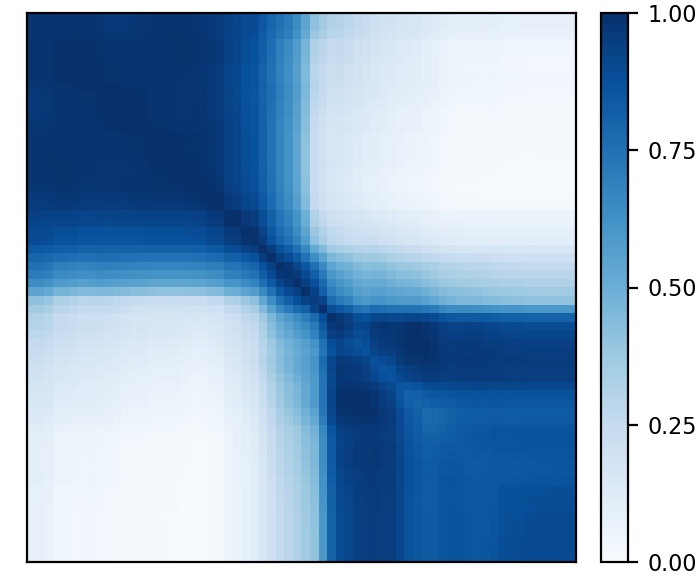}
\end{subfigure}
\begin{subfigure}[t]{0.24\textwidth}
  \centering\includegraphics[width=\linewidth]{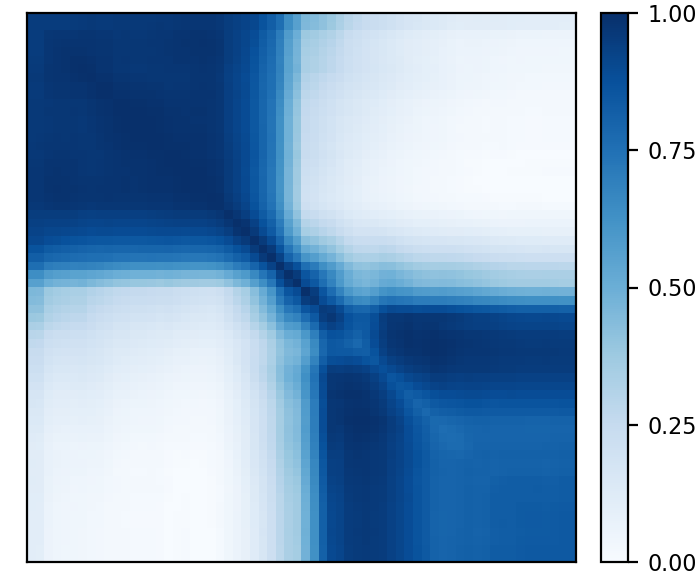}
\end{subfigure}
\begin{subfigure}[t]{0.24\textwidth}
  \centering\includegraphics[width=\linewidth]{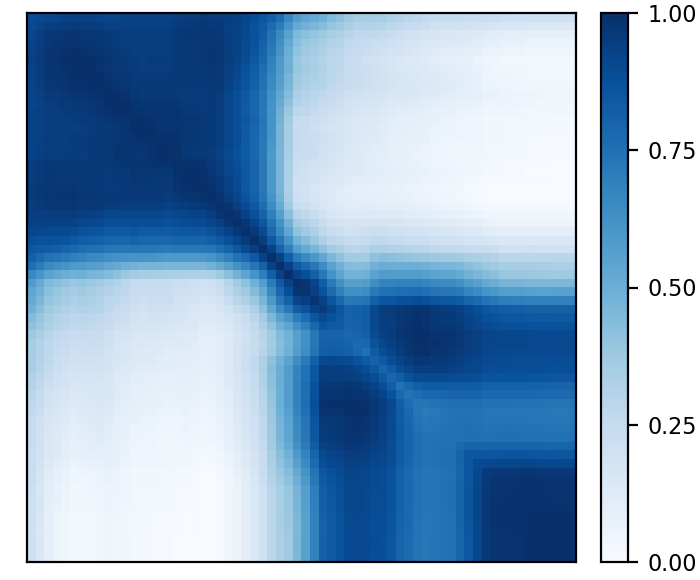}
\end{subfigure}

\vspace{0.3em}

% Row 3: ogbmolhiv
\begin{subfigure}[t]{0.24\textwidth}
  \centering\includegraphics[width=\linewidth]{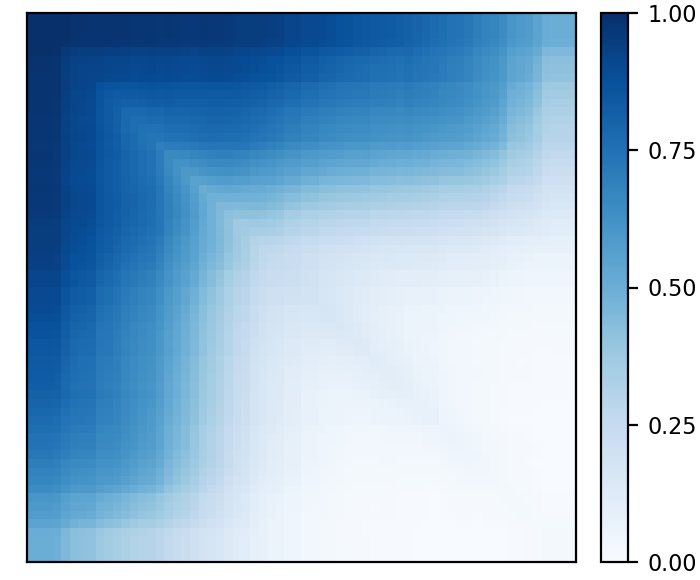}
\end{subfigure}
\begin{subfigure}[t]{0.24\textwidth}
  \centering\includegraphics[width=\linewidth]{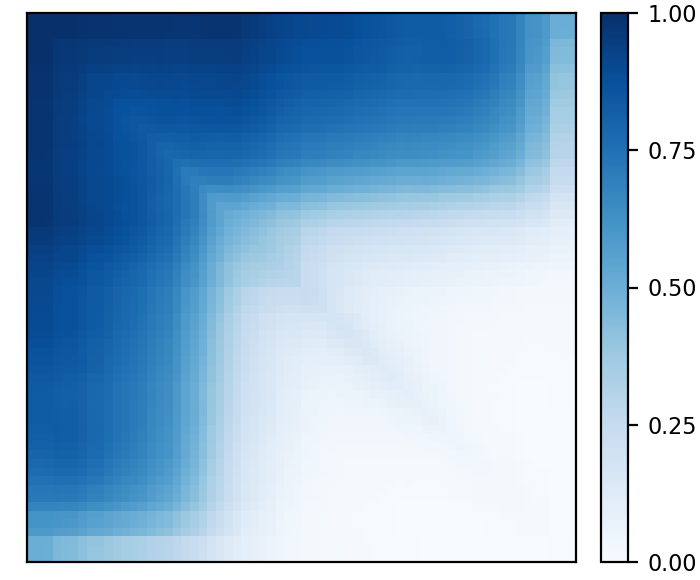}
\end{subfigure}
\begin{subfigure}[t]{0.24\textwidth}
  \centering\includegraphics[width=\linewidth]{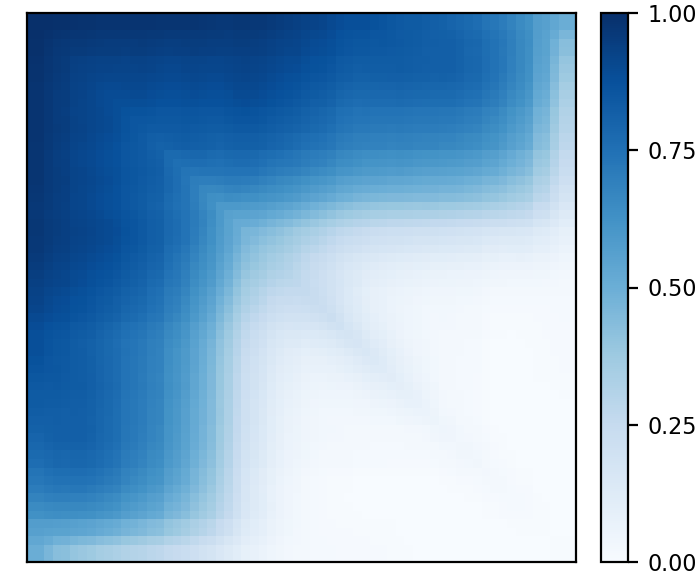}
\end{subfigure}
\begin{subfigure}[t]{0.24\textwidth}
  \centering\includegraphics[width=\linewidth]{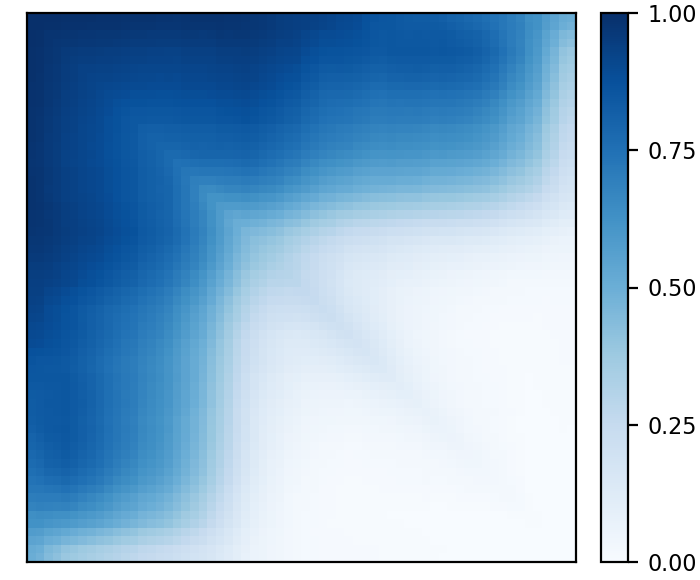}
\end{subfigure}

\vspace{0.3em}

% Row 4: proteins
\begin{subfigure}[t]{0.24\textwidth}
  \centering\includegraphics[width=\linewidth]{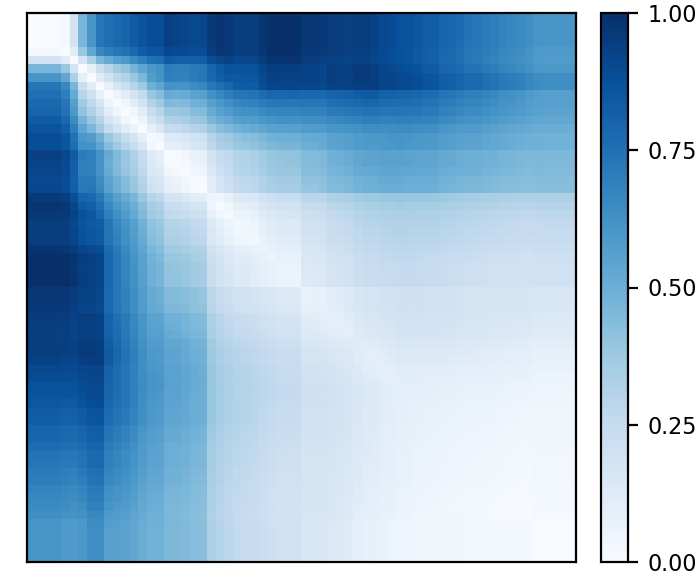}
\end{subfigure}
\begin{subfigure}[t]{0.24\textwidth}
  \centering\includegraphics[width=\linewidth]{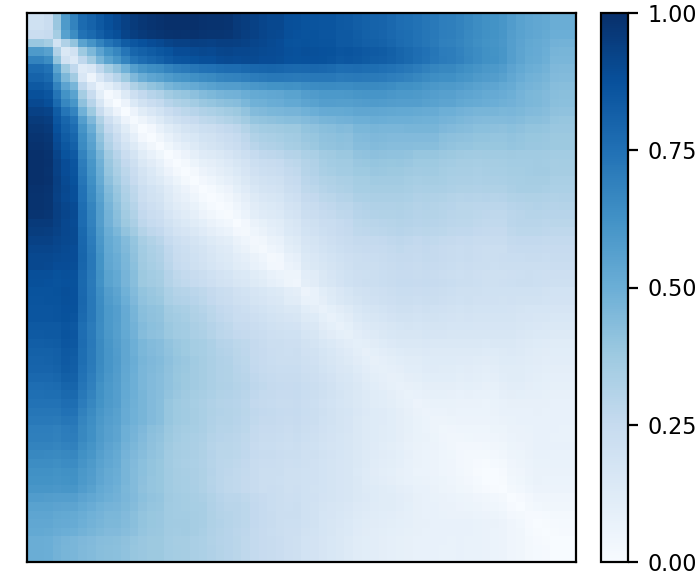}
\end{subfigure}
\begin{subfigure}[t]{0.24\textwidth}
  \centering\includegraphics[width=\linewidth]{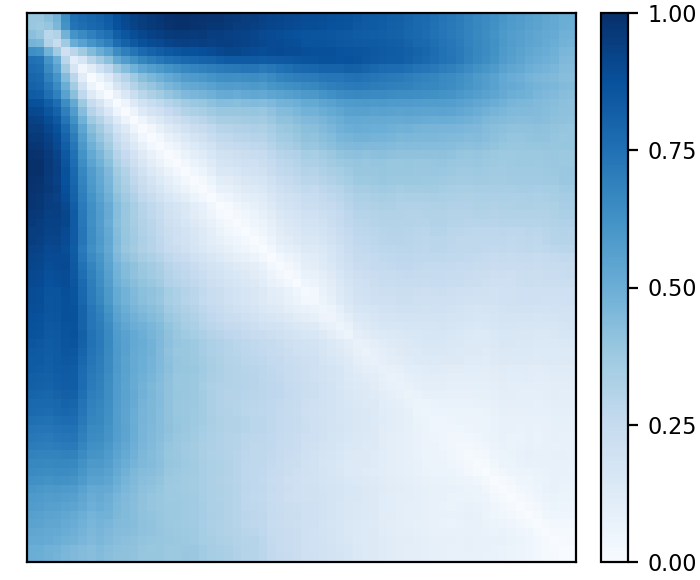}
\end{subfigure}
\begin{subfigure}[t]{0.24\textwidth}
  \centering\includegraphics[width=\linewidth]{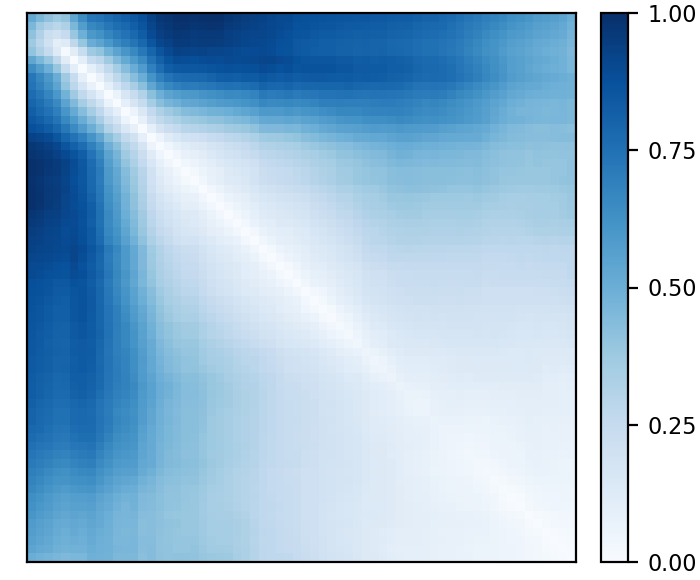}
\end{subfigure}

\vspace{0.3em}

% Row 5: redditmulti5k
\begin{subfigure}[t]{0.24\textwidth}
  \centering\includegraphics[width=\linewidth]{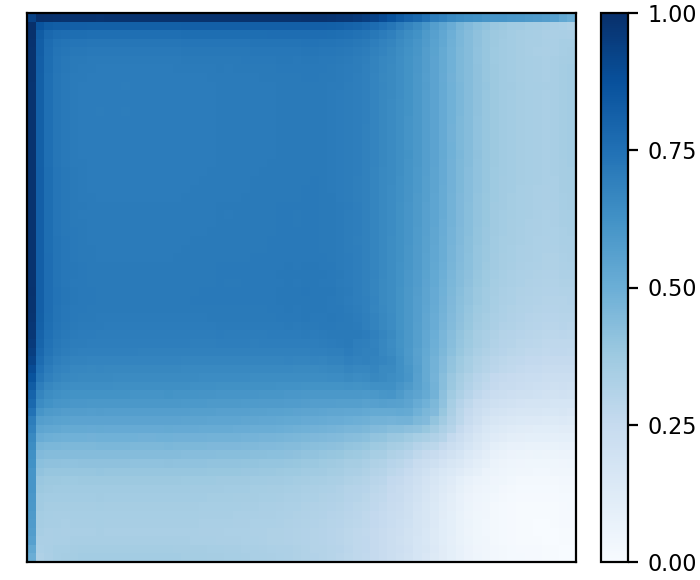}
\end{subfigure}
\begin{subfigure}[t]{0.24\textwidth}
  \centering\includegraphics[width=\linewidth]{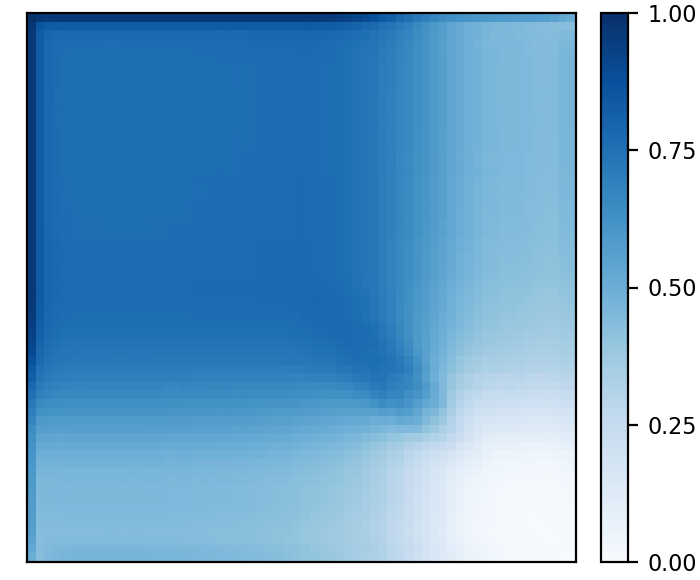}
\end{subfigure}
\begin{subfigure}[t]{0.24\textwidth}
  \centering\includegraphics[width=\linewidth]{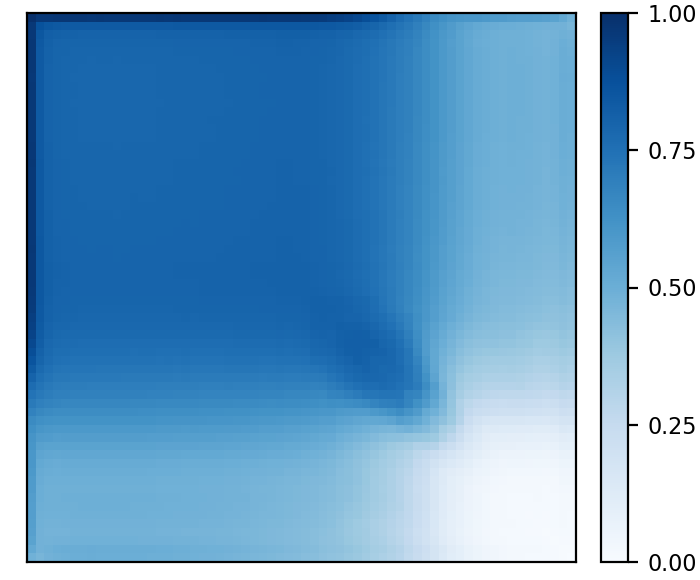}
\end{subfigure}
\begin{subfigure}[t]{0.24\textwidth}
  \centering\includegraphics[width=\linewidth]{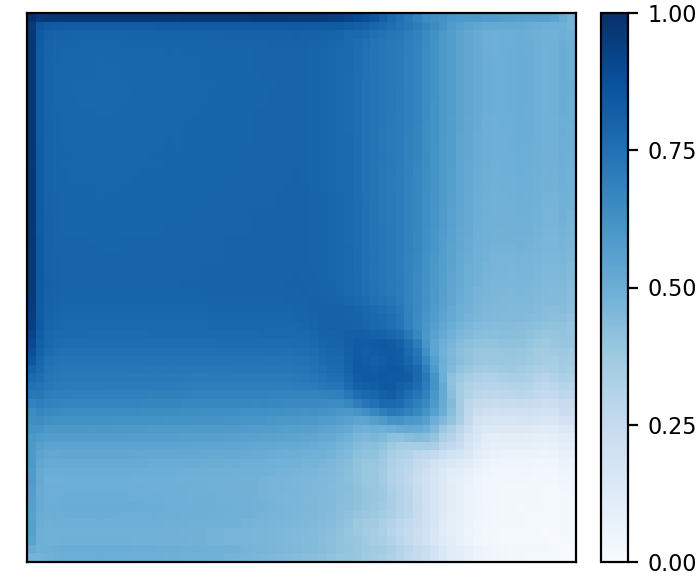}
\end{subfigure}

\caption{\textbf{Estimated attention graphons across datasets and graph sizes (Graphormer-GD, single-head).} Each row corresponds to a dataset (NCI1, NCI109, OGBmolhiv, PROTEINS, REDDIT-MULTI-5K), and each column shows the block-step estimate of the dataset-level attention kernel at increasing target sizes $N$. Kernel estimates stabilize within each dataset as $N$ grows.}
\label{fig:graphon_est_graphormer_h1_2}
\end{figure*}

\begin{figure*}[ht!]
\centering
% Row 1: collab
\begin{subfigure}[t]{0.24\textwidth}
  \centering\includegraphics[width=\linewidth]{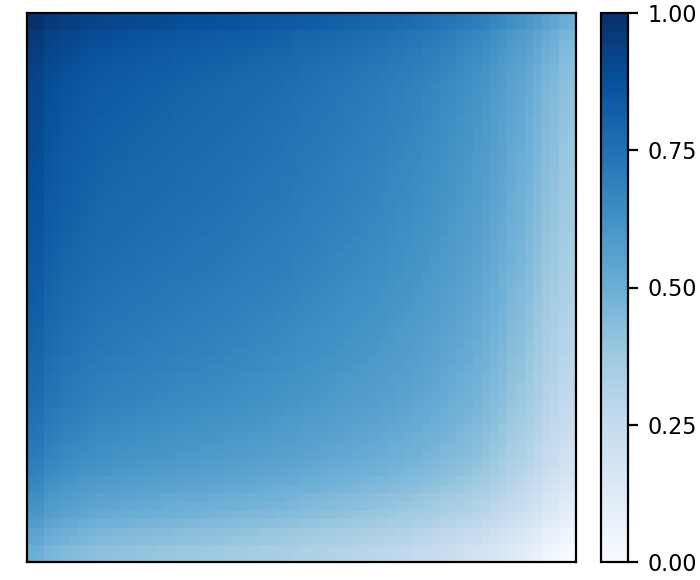}
\end{subfigure}
\begin{subfigure}[t]{0.24\textwidth}
  \centering\includegraphics[width=\linewidth]{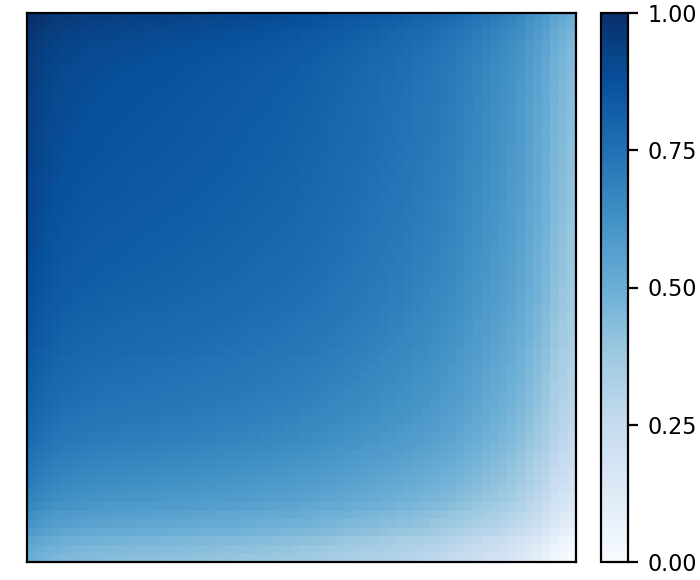}
\end{subfigure}
\begin{subfigure}[t]{0.24\textwidth}
  \centering\includegraphics[width=\linewidth]{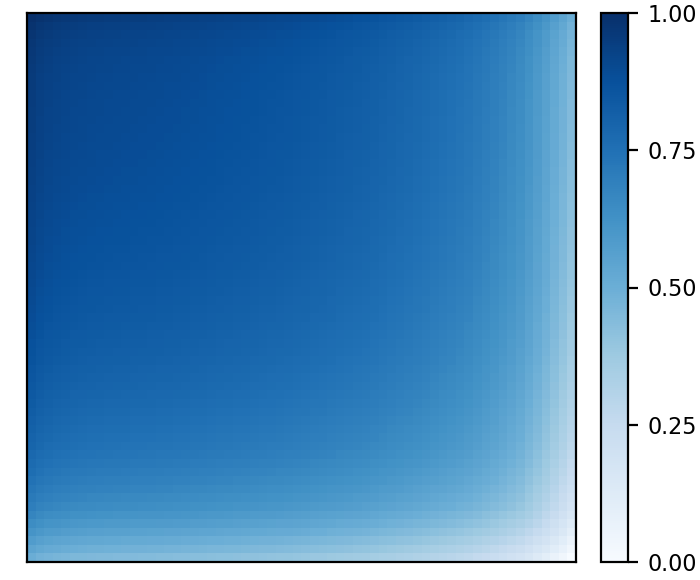}
\end{subfigure}
\begin{subfigure}[t]{0.24\textwidth}
  \centering\includegraphics[width=\linewidth]{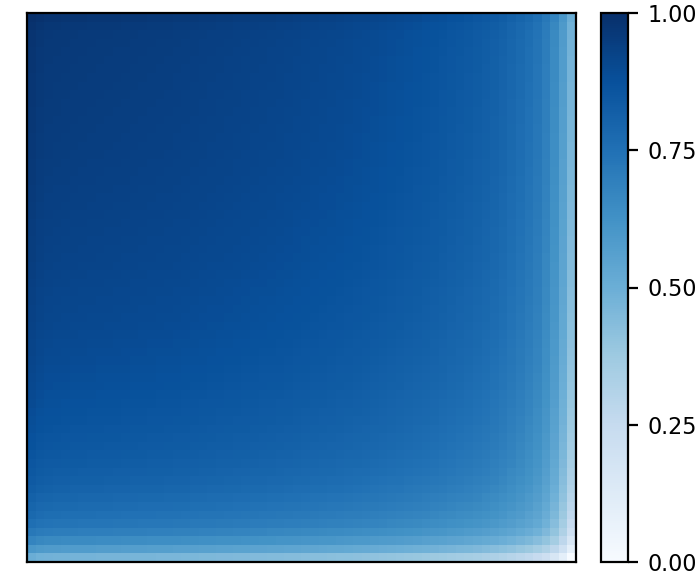}
\end{subfigure}

\vspace{0.3em}

% Row 2: imdbmulti
\begin{subfigure}[t]{0.24\textwidth}
  \centering\includegraphics[width=\linewidth]{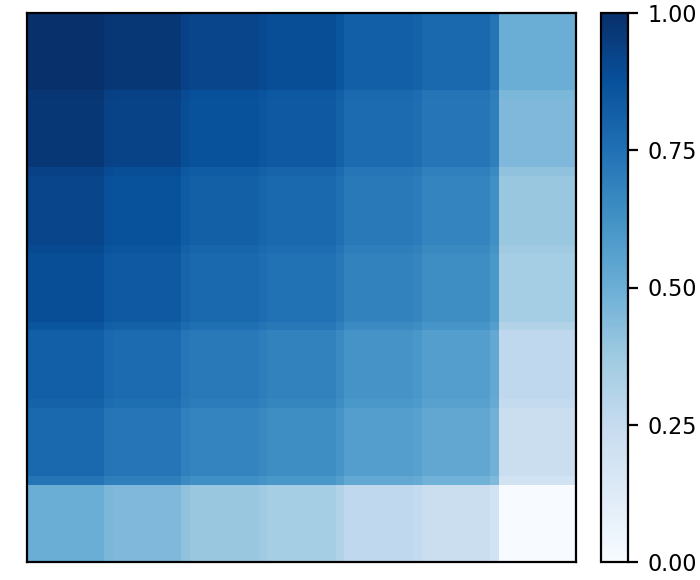}
\end{subfigure}
\begin{subfigure}[t]{0.24\textwidth}
  \centering\includegraphics[width=\linewidth]{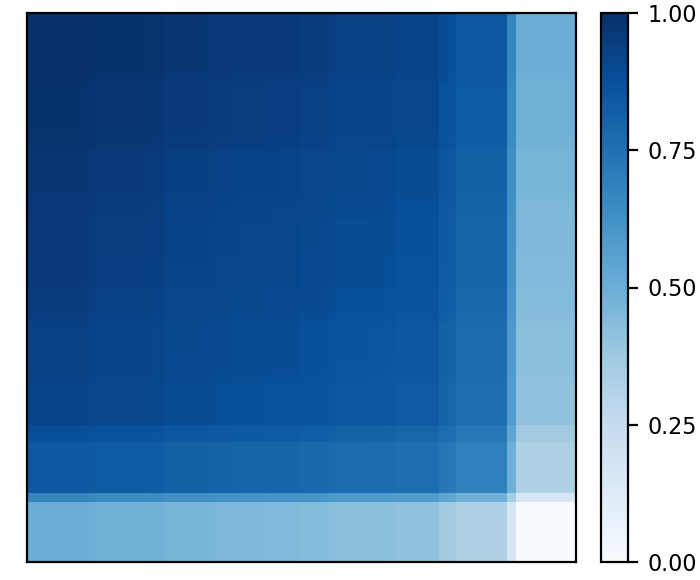}
\end{subfigure}
\begin{subfigure}[t]{0.24\textwidth}
  \centering\includegraphics[width=\linewidth]{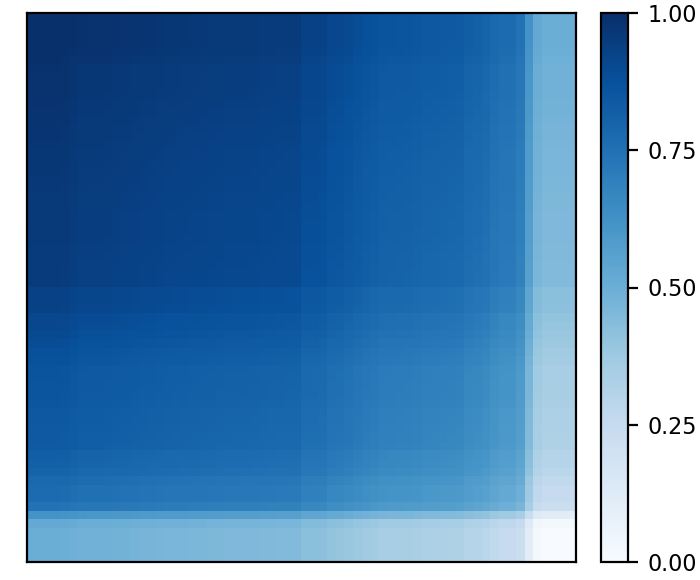}
\end{subfigure}
\begin{subfigure}[t]{0.24\textwidth}
  \centering\includegraphics[width=\linewidth]{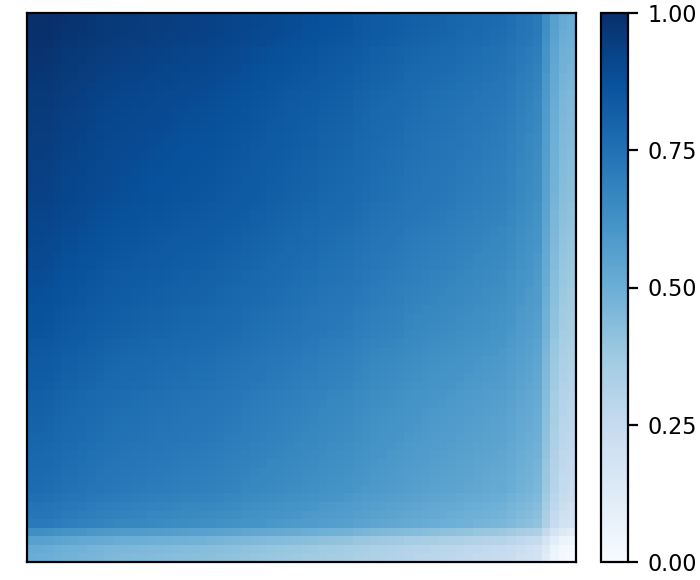}
\end{subfigure}

\vspace{0.3em}

% Row 3: lrgbpeptides
\begin{subfigure}[t]{0.24\textwidth}
  \centering\includegraphics[width=\linewidth]{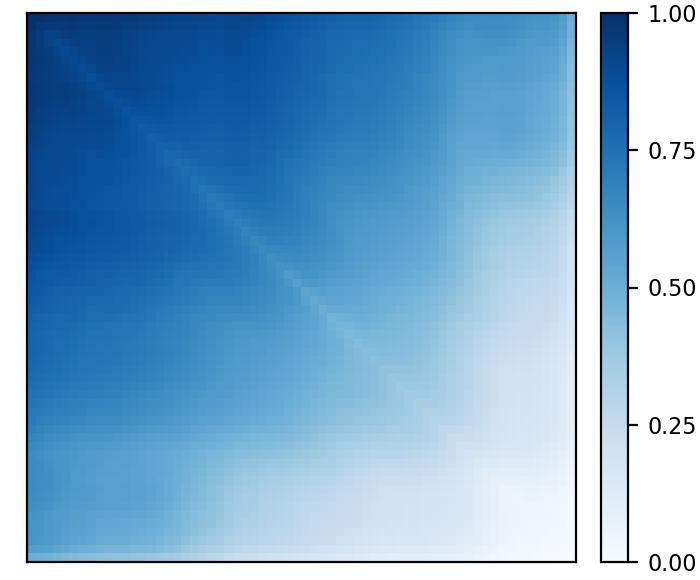}
\end{subfigure}
\begin{subfigure}[t]{0.24\textwidth}
  \centering\includegraphics[width=\linewidth]{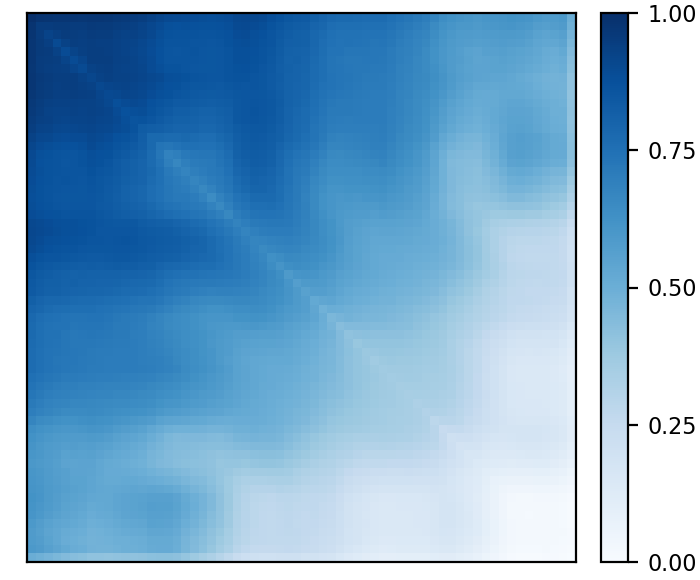}
\end{subfigure}
\begin{subfigure}[t]{0.24\textwidth}
  \centering\includegraphics[width=\linewidth]{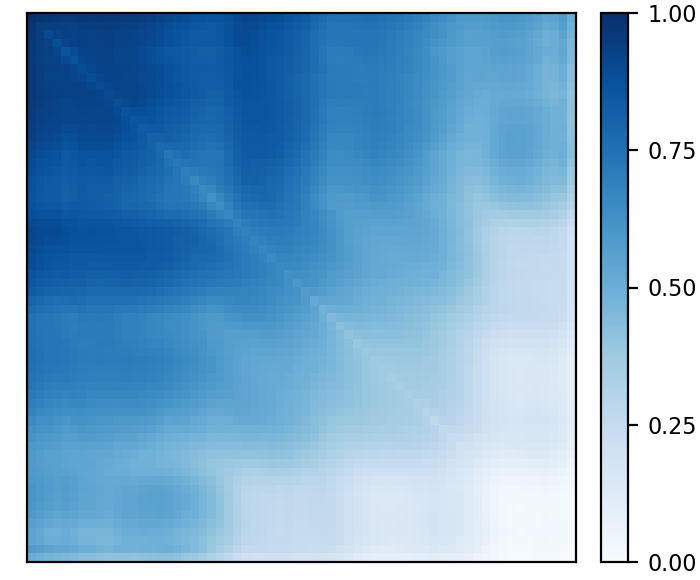}
\end{subfigure}
\begin{subfigure}[t]{0.24\textwidth}
  \centering\includegraphics[width=\linewidth]{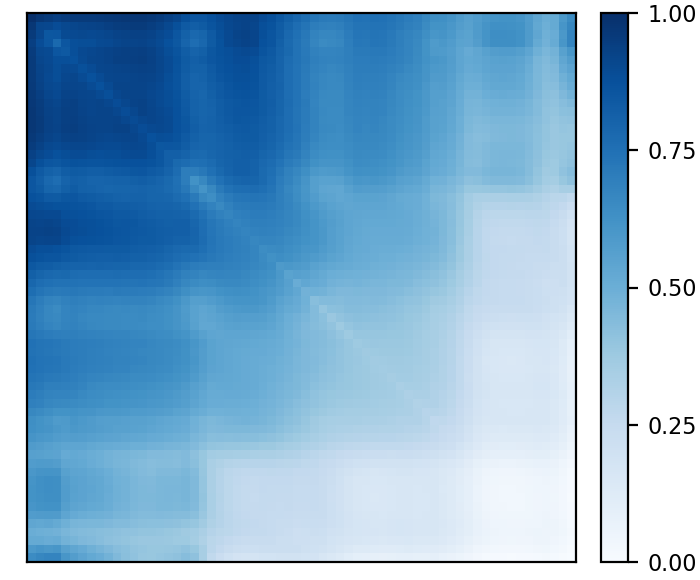}
\end{subfigure}

\vspace{0.3em}

% Row 4: modelnet
\begin{subfigure}[t]{0.24\textwidth}
  \centering\includegraphics[width=\linewidth]{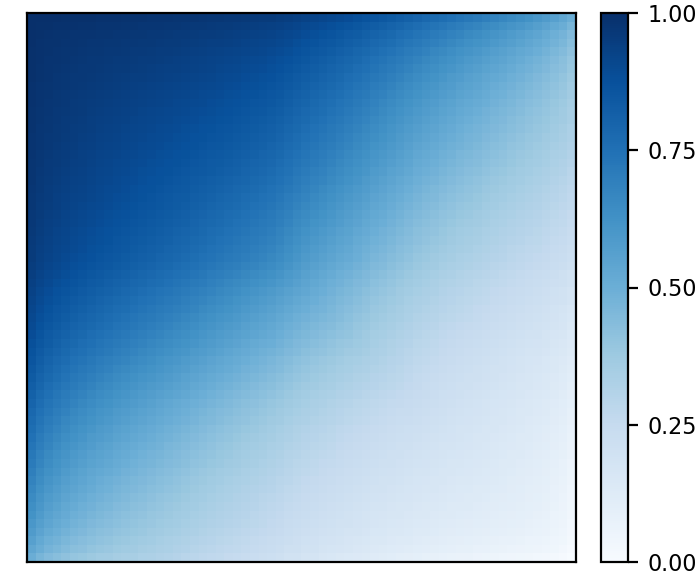}
\end{subfigure}
\begin{subfigure}[t]{0.24\textwidth}
  \centering\includegraphics[width=\linewidth]{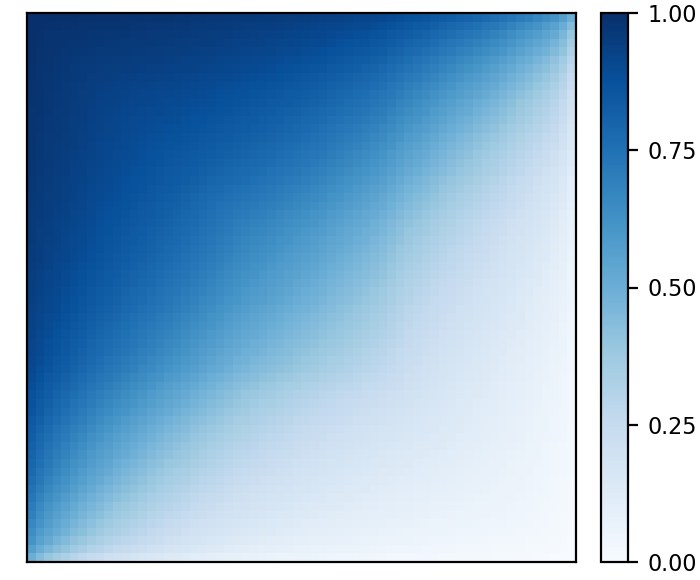}
\end{subfigure}
\begin{subfigure}[t]{0.24\textwidth}
  \centering\includegraphics[width=\linewidth]{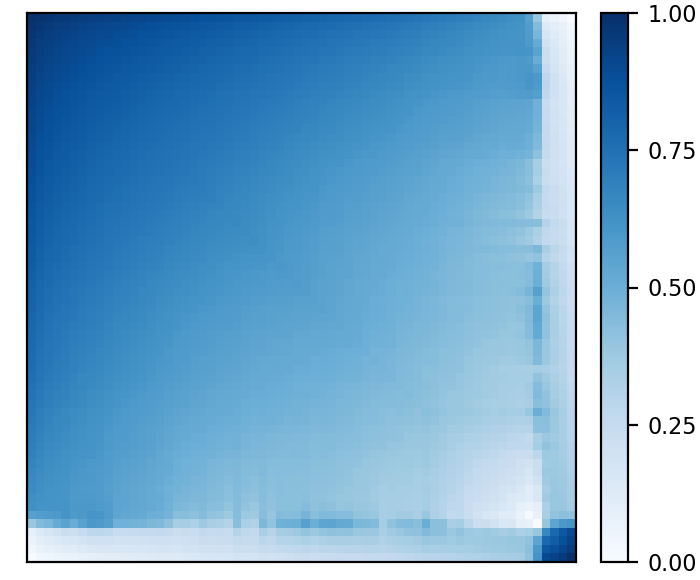}
\end{subfigure}
\begin{subfigure}[t]{0.24\textwidth}
  \centering\includegraphics[width=\linewidth]{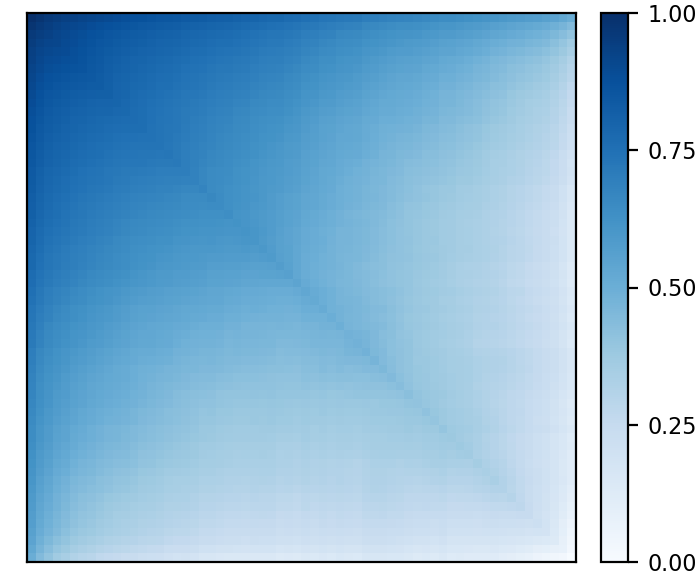}
\end{subfigure}

\vspace{0.3em}

% Row 5: mutag
\begin{subfigure}[t]{0.24\textwidth}
  \centering\includegraphics[width=\linewidth]{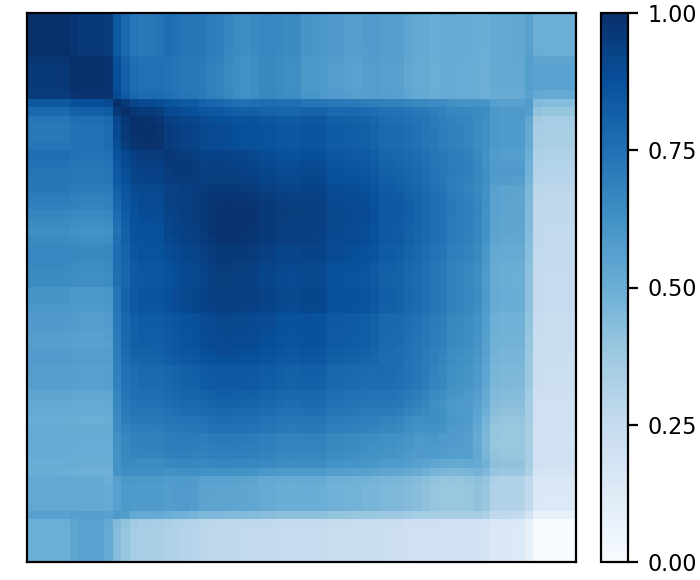}
\end{subfigure}
\begin{subfigure}[t]{0.24\textwidth}
  \centering\includegraphics[width=\linewidth]{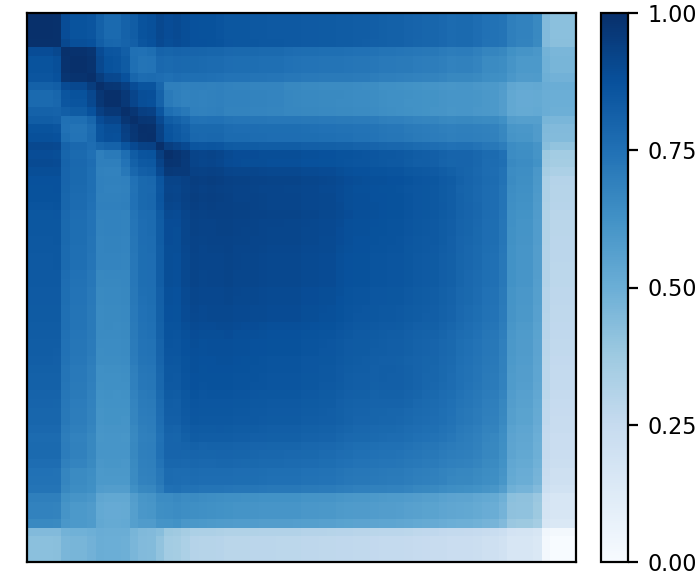}
\end{subfigure}
\begin{subfigure}[t]{0.24\textwidth}
  \centering\includegraphics[width=\linewidth]{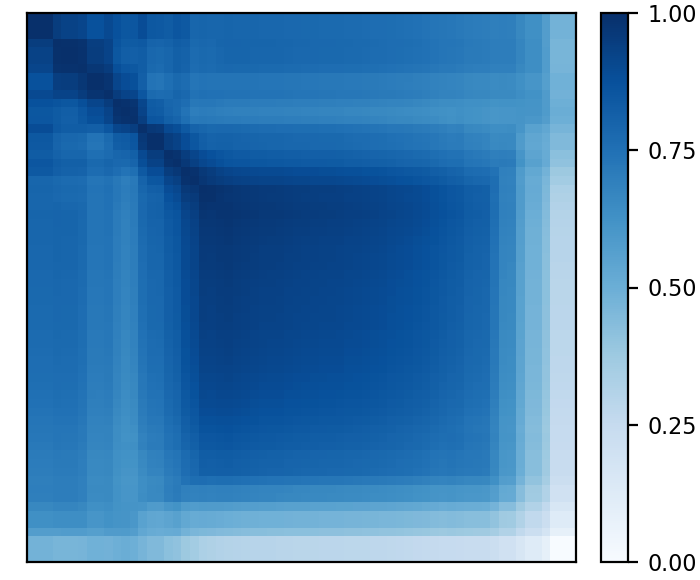}
\end{subfigure}
\begin{subfigure}[t]{0.24\textwidth}
  \centering\includegraphics[width=\linewidth]{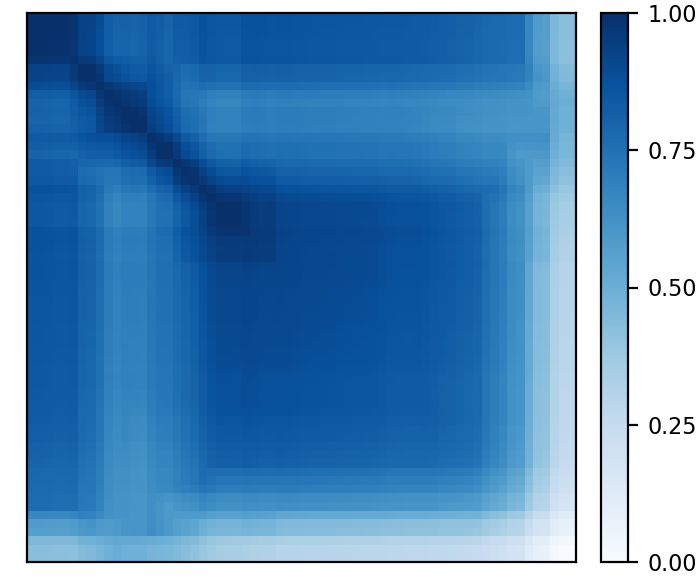}
\end{subfigure}

\caption{\textbf{Estimated attention graphons across datasets and graph sizes (GRIT, single-head).} Each row corresponds to a dataset (COLLAB, IMDB-MULTI, LRGBPeptides, ModelNet10, MUTAG), and each column shows the block-step estimate of the dataset-level attention kernel at increasing target sizes $N$. Kernel estimates stabilize within each dataset as $N$ grows.}
\label{fig:graphon_est_grit_h1_1}
\end{figure*}

\begin{figure*}[ht!]
\centering
% Row 1: nci1
\begin{subfigure}[t]{0.24\textwidth}
  \centering\includegraphics[width=\linewidth]{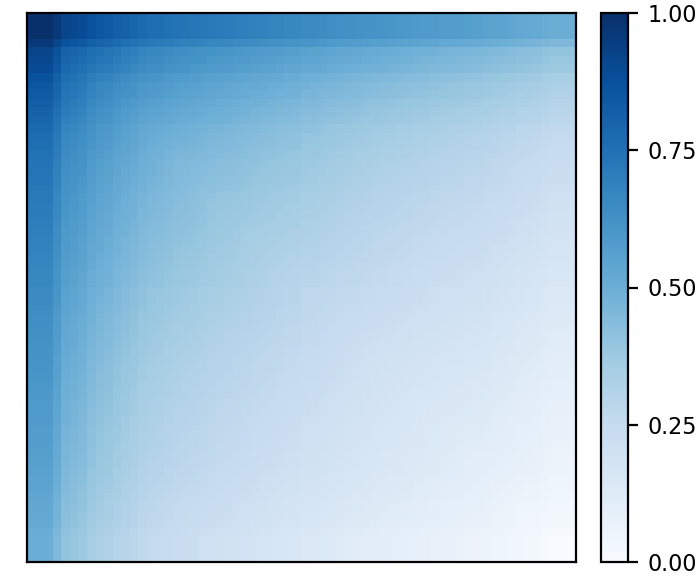}
\end{subfigure}
\begin{subfigure}[t]{0.24\textwidth}
  \centering\includegraphics[width=\linewidth]{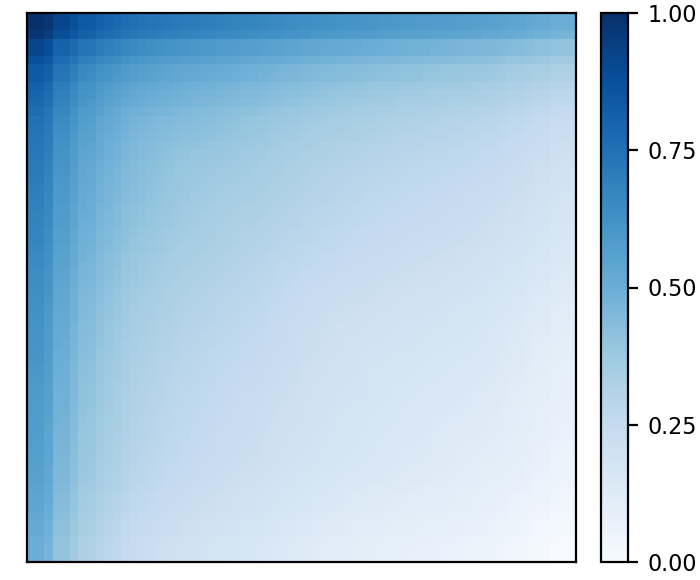}
\end{subfigure}
\begin{subfigure}[t]{0.24\textwidth}
  \centering\includegraphics[width=\linewidth]{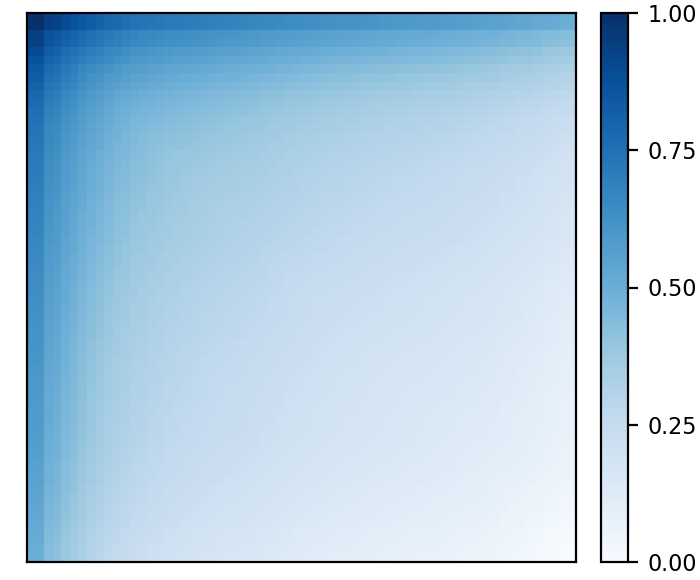}
\end{subfigure}
\begin{subfigure}[t]{0.24\textwidth}
  \centering\includegraphics[width=\linewidth]{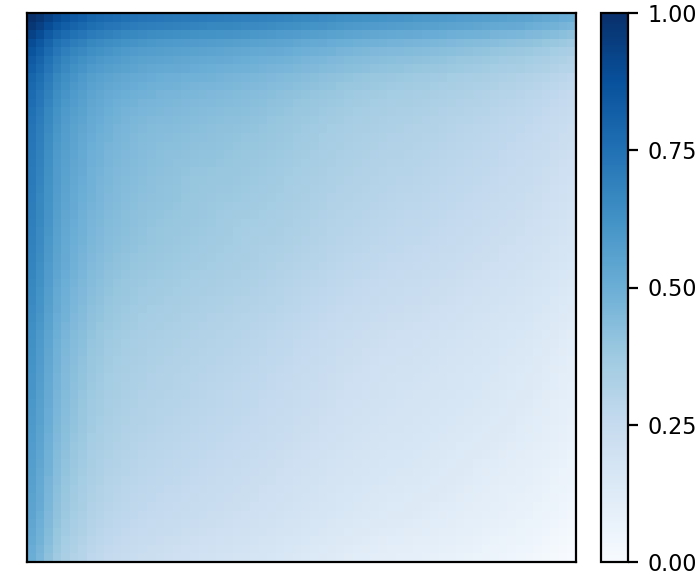}
\end{subfigure}

\vspace{0.3em}

% Row 2: nci109
\begin{subfigure}[t]{0.24\textwidth}
  \centering\includegraphics[width=\linewidth]{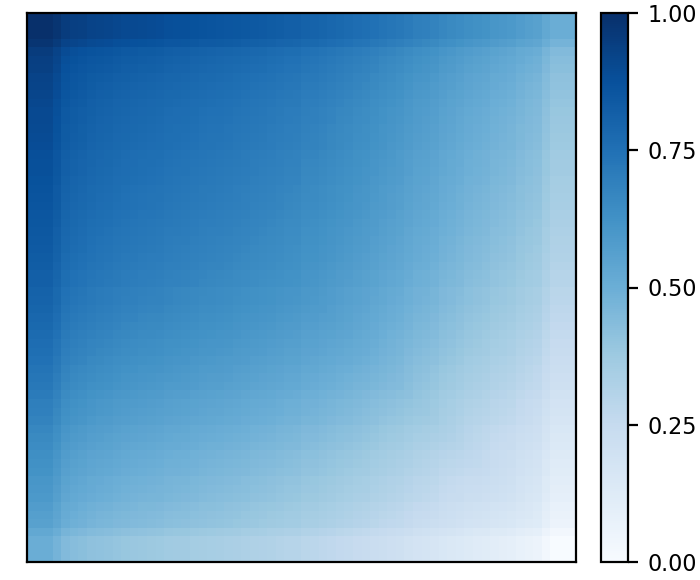}
\end{subfigure}
\begin{subfigure}[t]{0.24\textwidth}
  \centering\includegraphics[width=\linewidth]{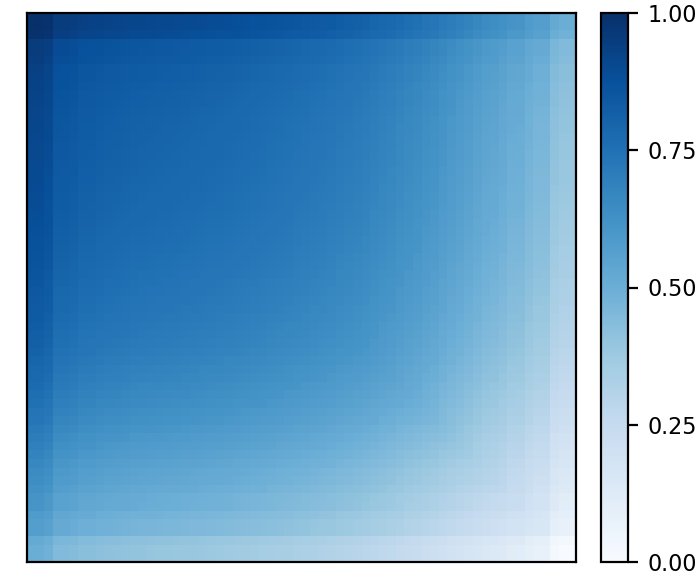}
\end{subfigure}
\begin{subfigure}[t]{0.24\textwidth}
  \centering\includegraphics[width=\linewidth]{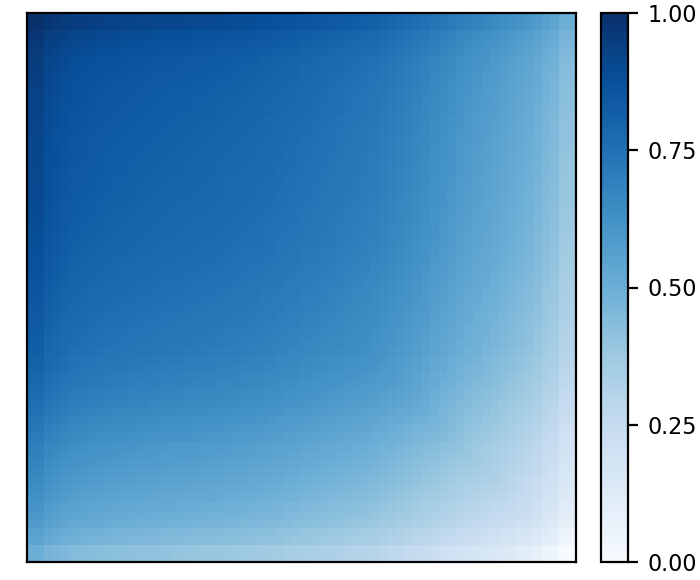}
\end{subfigure}
\begin{subfigure}[t]{0.24\textwidth}
  \centering\includegraphics[width=\linewidth]{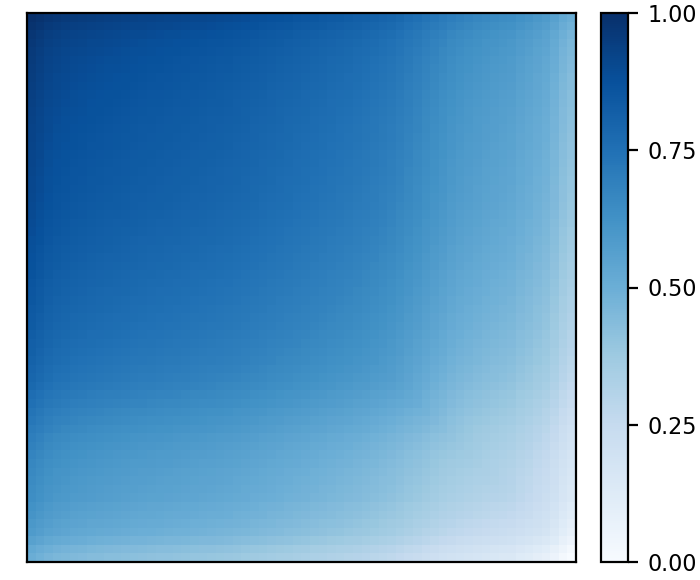}
\end{subfigure}

\vspace{0.3em}

% Row 3: ogbmolhiv
\begin{subfigure}[t]{0.24\textwidth}
  \centering\includegraphics[width=\linewidth]{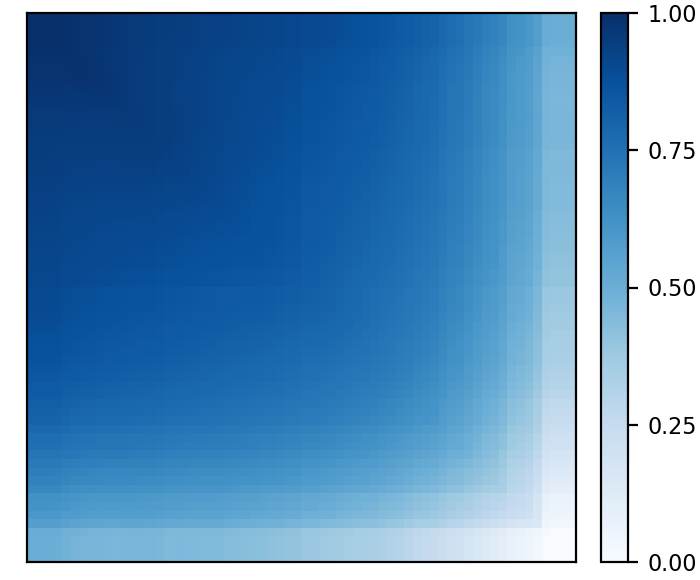}
\end{subfigure}
\begin{subfigure}[t]{0.24\textwidth}
  \centering\includegraphics[width=\linewidth]{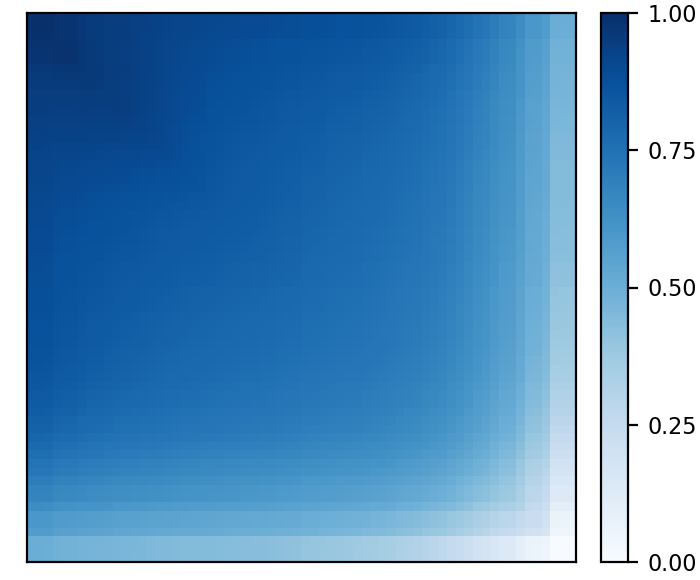}
\end{subfigure}
\begin{subfigure}[t]{0.24\textwidth}
  \centering\includegraphics[width=\linewidth]{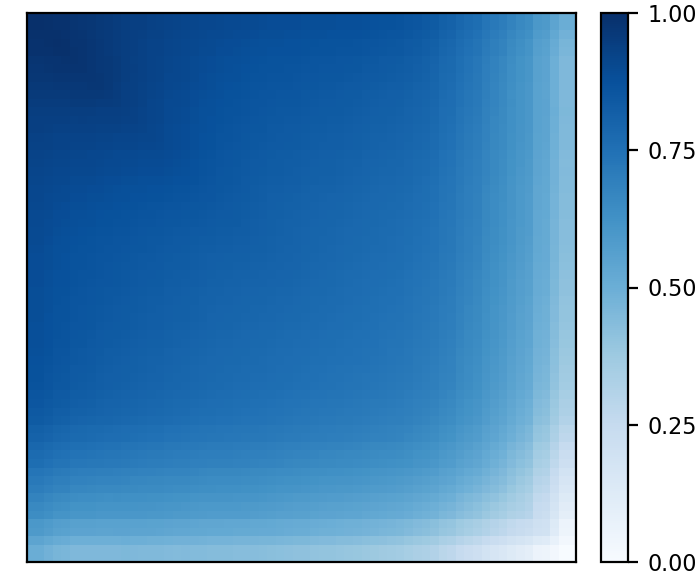}
\end{subfigure}
\begin{subfigure}[t]{0.24\textwidth}
  \centering\includegraphics[width=\linewidth]{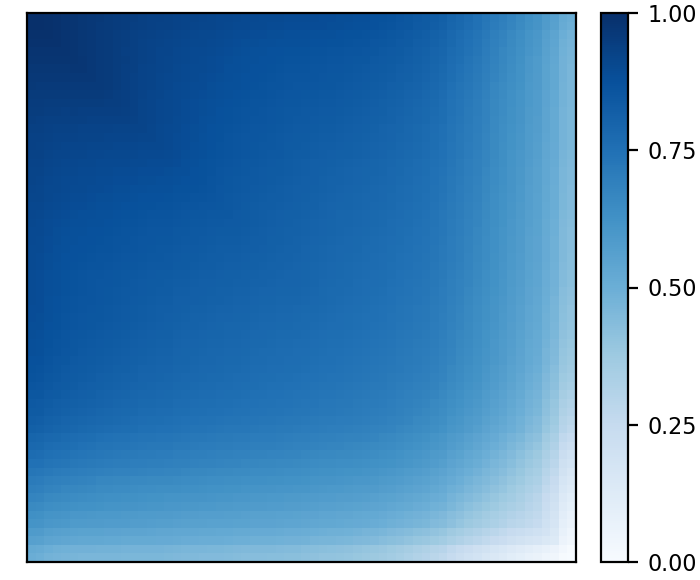}
\end{subfigure}

\vspace{0.3em}

% Row 4: proteins
\begin{subfigure}[t]{0.24\textwidth}
  \centering\includegraphics[width=\linewidth]{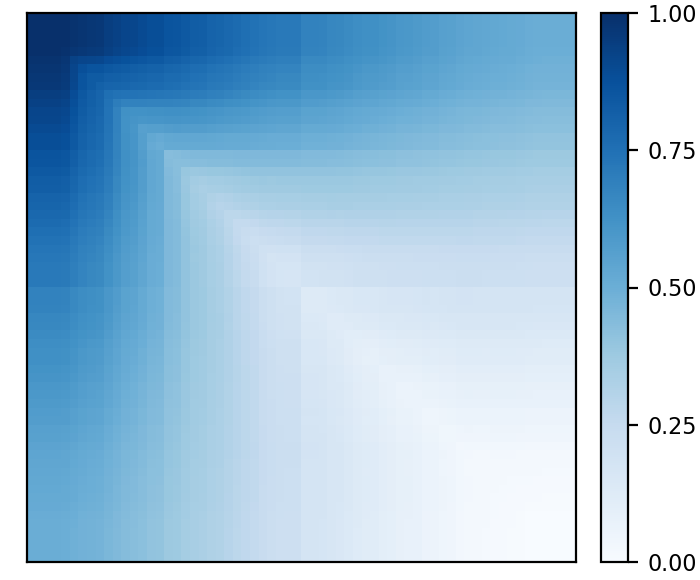}
\end{subfigure}
\begin{subfigure}[t]{0.24\textwidth}
  \centering\includegraphics[width=\linewidth]{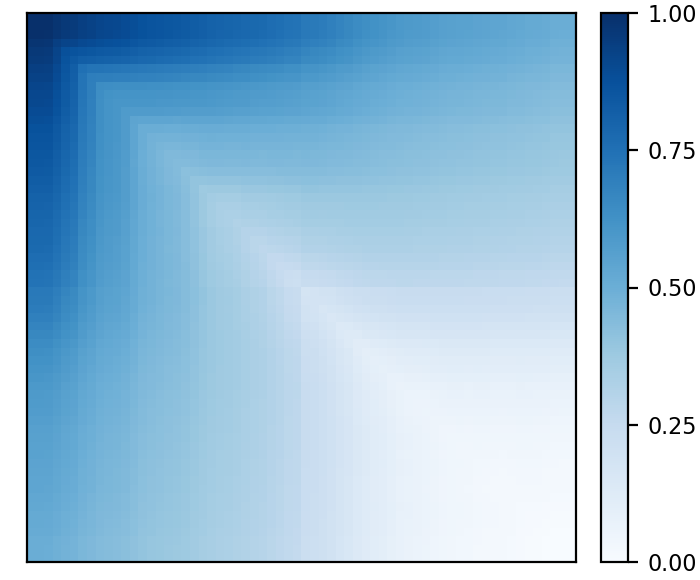}
\end{subfigure}
\begin{subfigure}[t]{0.24\textwidth}
  \centering\includegraphics[width=\linewidth]{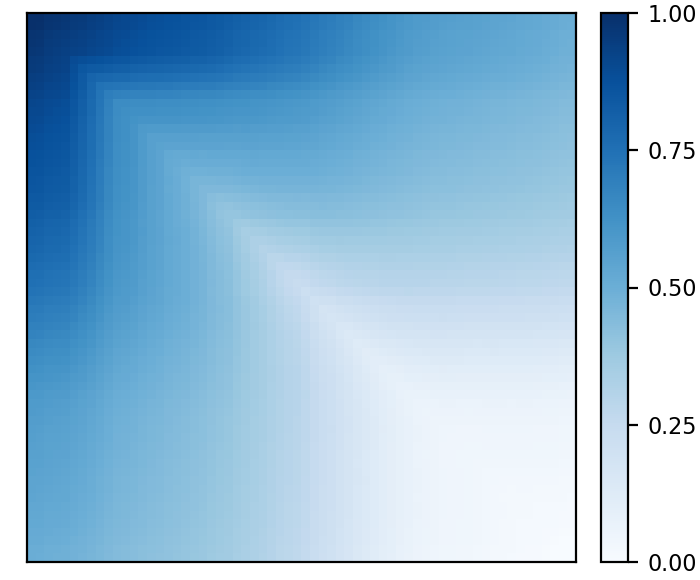}
\end{subfigure}
\begin{subfigure}[t]{0.24\textwidth}
  \centering\includegraphics[width=\linewidth]{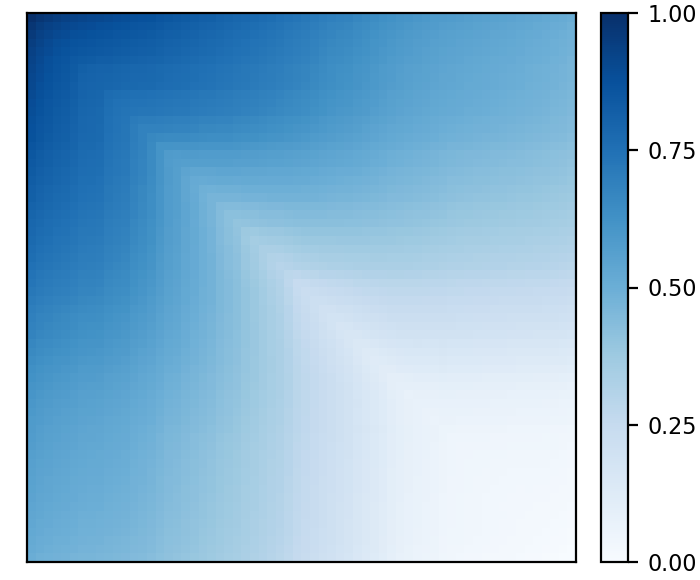}
\end{subfigure}

\vspace{0.3em}

% Row 5: redditmulti5k
\begin{subfigure}[t]{0.24\textwidth}
  \centering\includegraphics[width=\linewidth]{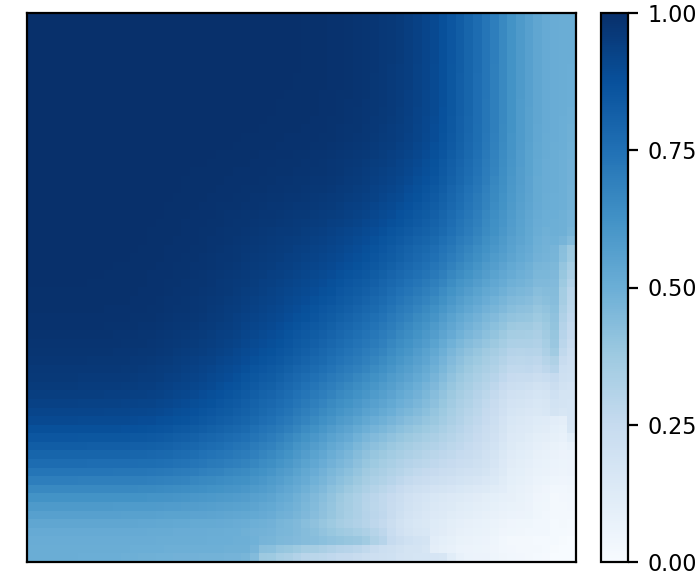}
\end{subfigure}
\begin{subfigure}[t]{0.24\textwidth}
  \centering\includegraphics[width=\linewidth]{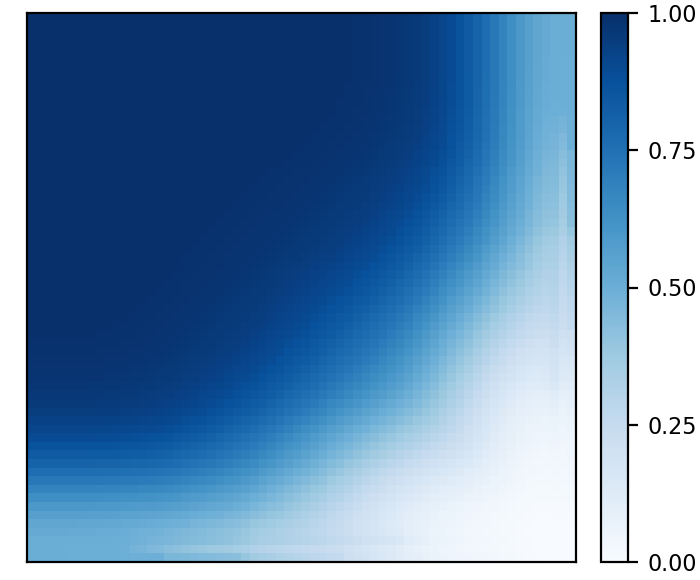}
\end{subfigure}
\begin{subfigure}[t]{0.24\textwidth}
  \centering\includegraphics[width=\linewidth]{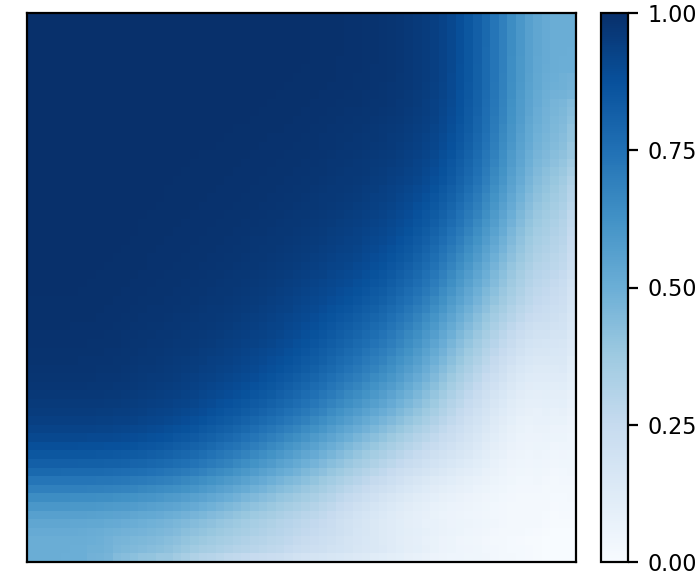}
\end{subfigure}
\begin{subfigure}[t]{0.24\textwidth}
  \centering\includegraphics[width=\linewidth]{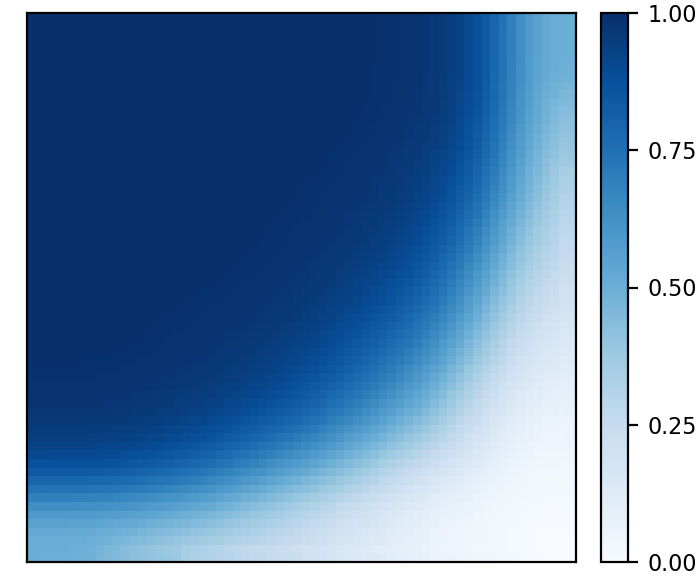}
\end{subfigure}

\caption{\textbf{Estimated attention graphons across datasets and graph sizes (GRIT, single-head).} Each row corresponds to a dataset (NCI1, NCI109, OGBmolhiv, PROTEINS, REDDIT-MULTI-5K), and each column shows the block-step estimate of the dataset-level attention kernel at increasing target sizes $N$. Kernel estimates stabilize within each dataset as $N$ grows.}
\label{fig:graphon_est_grit_h1_2}
\end{figure*}

\begin{figure*}[ht!]
\centering
% Row 1
\begin{subfigure}[t]{0.24\textwidth}
  \centering\includegraphics[width=\linewidth]{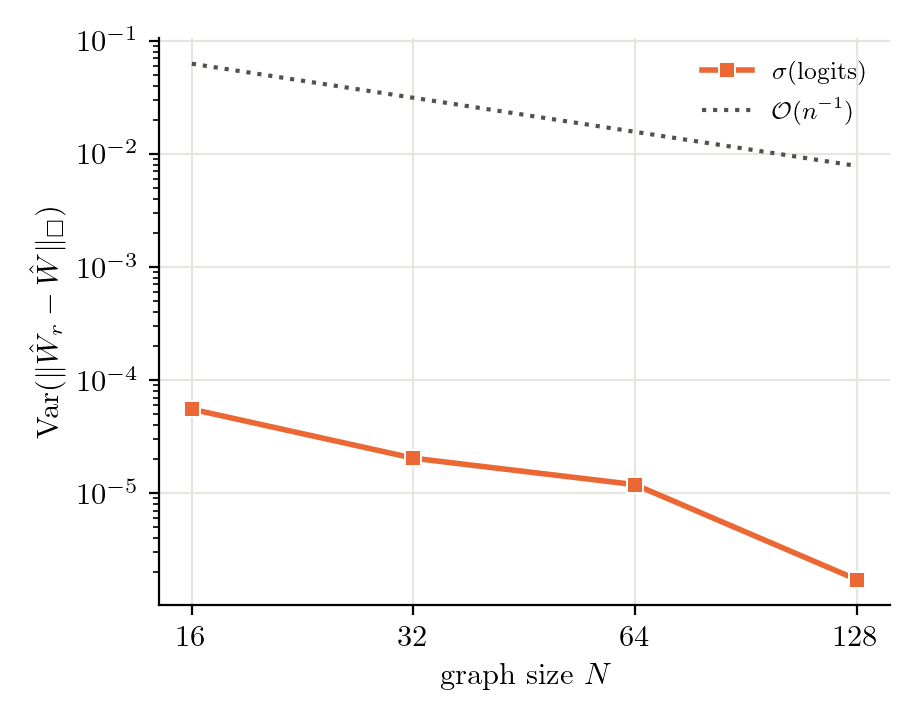}
  \caption{\centering{\textbf{BFS}}}
\end{subfigure}
\begin{subfigure}[t]{0.24\textwidth}
  \centering\includegraphics[width=\linewidth]{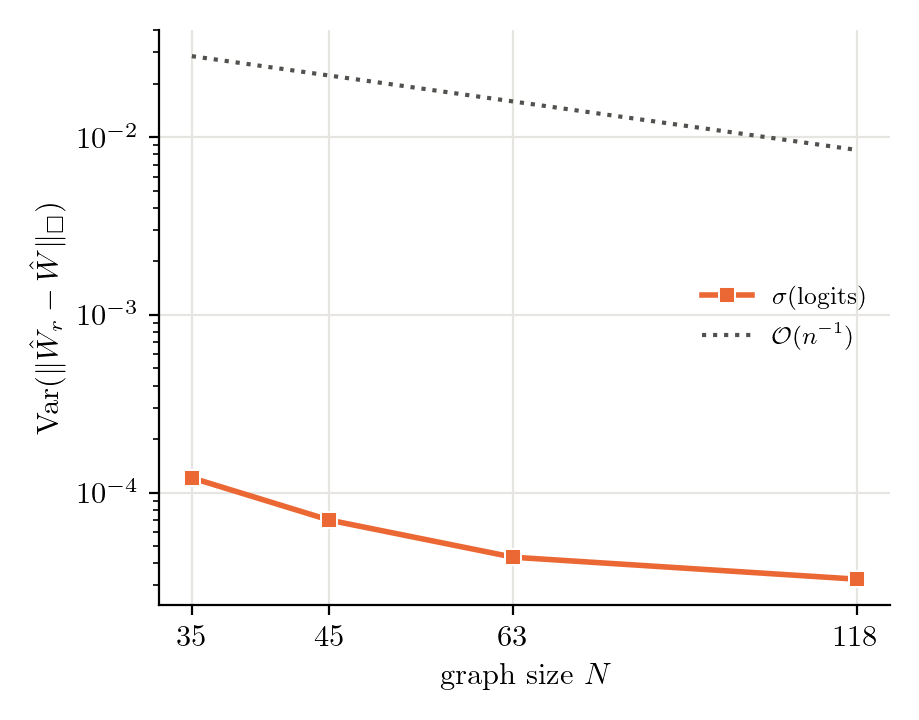}
  \caption{\centering{\textbf{COLLAB}}}
\end{subfigure}
\begin{subfigure}[t]{0.24\textwidth}
  \centering\includegraphics[width=\linewidth]{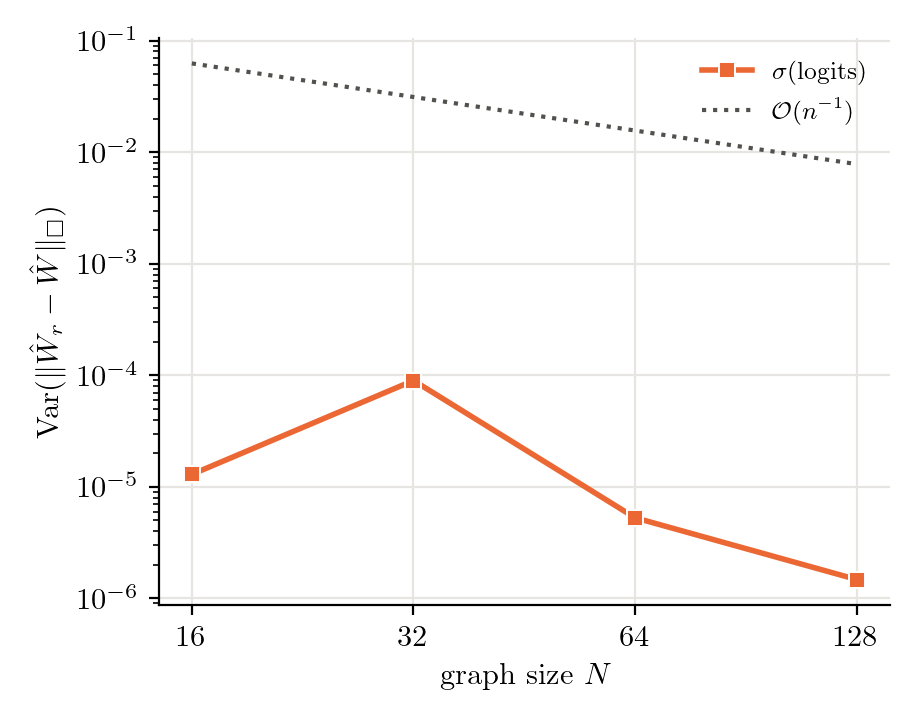}
  \caption{\centering{\textbf{Dijkstra}}}
\end{subfigure}
\begin{subfigure}[t]{0.24\textwidth}
  \centering\includegraphics[width=\linewidth]{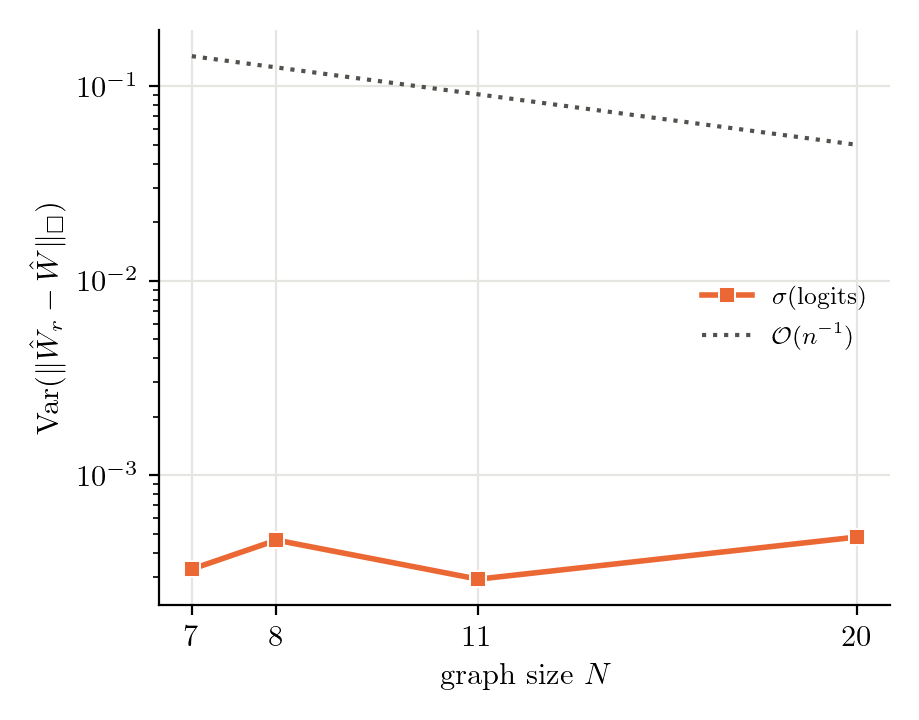}
  \caption{\centering{\textbf{IMDB-MULTI}}}
\end{subfigure}

\vspace{0.3em}

% Row 2
\begin{subfigure}[t]{0.24\textwidth}
  \centering\includegraphics[width=\linewidth]{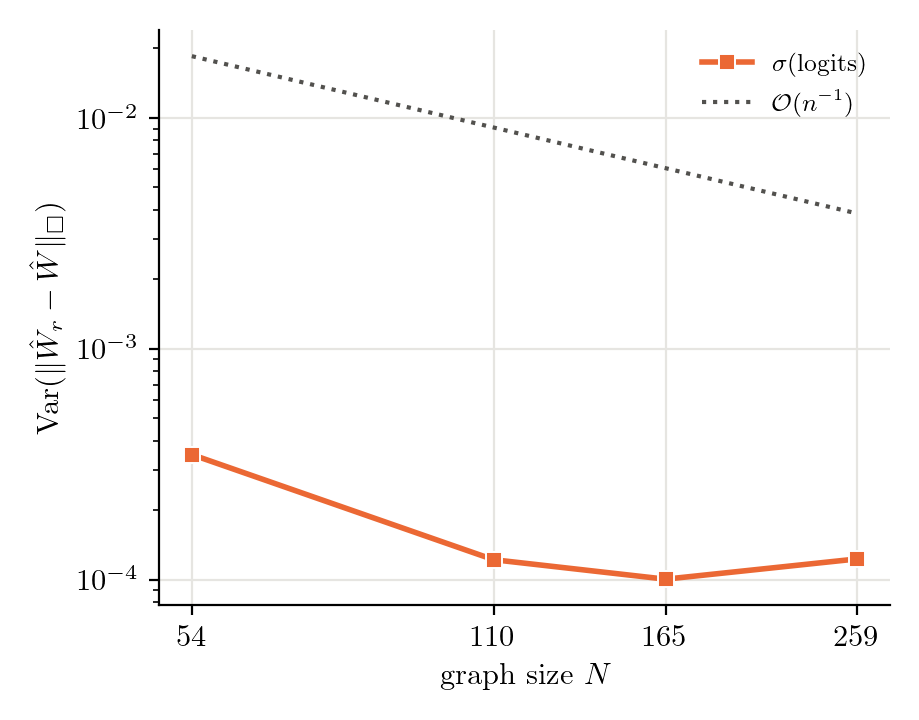}
  \caption{\centering{\textbf{LRGBPeptides}}}
\end{subfigure}
\begin{subfigure}[t]{0.24\textwidth}
  \centering\includegraphics[width=\linewidth]{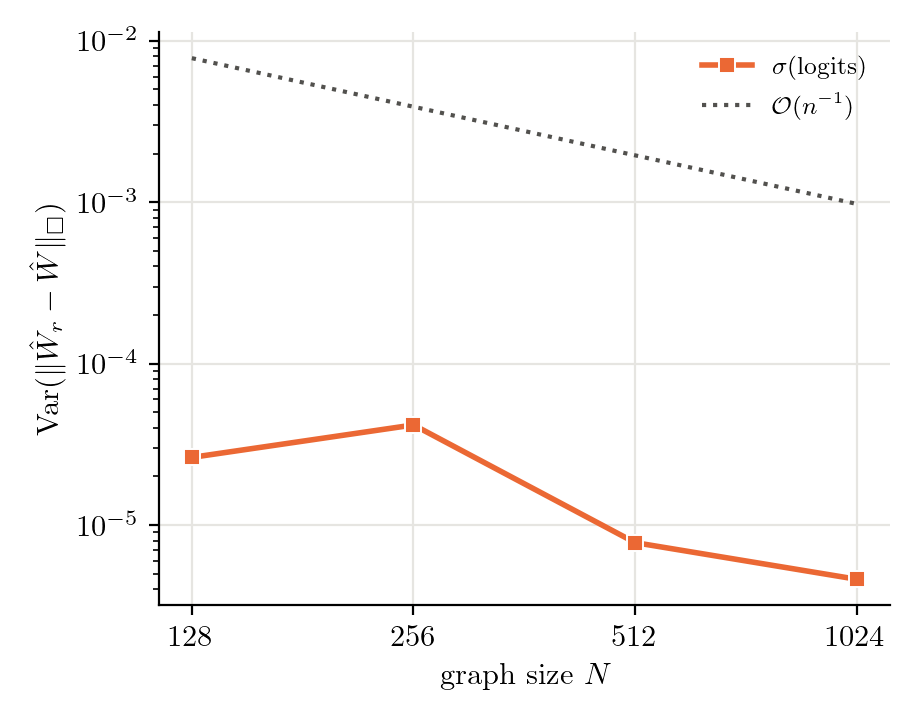}
  \caption{\centering{\textbf{ModelNet10}}}
\end{subfigure}
\begin{subfigure}[t]{0.24\textwidth}
  \centering\includegraphics[width=\linewidth]{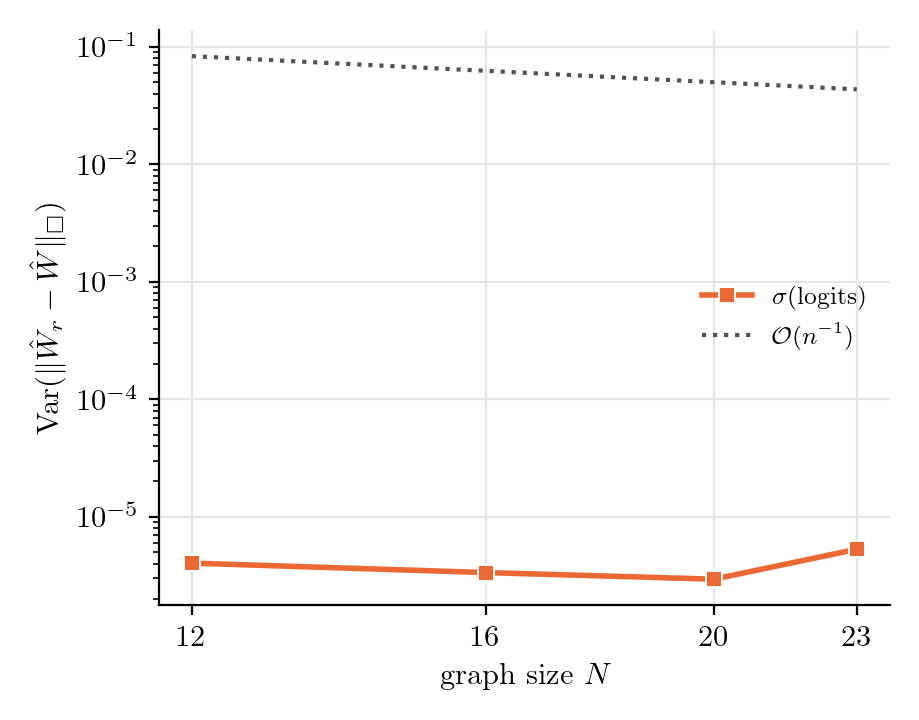}
  \caption{\centering{\textbf{MUTAG}}}
\end{subfigure}
\begin{subfigure}[t]{0.24\textwidth}
  \centering\includegraphics[width=\linewidth]{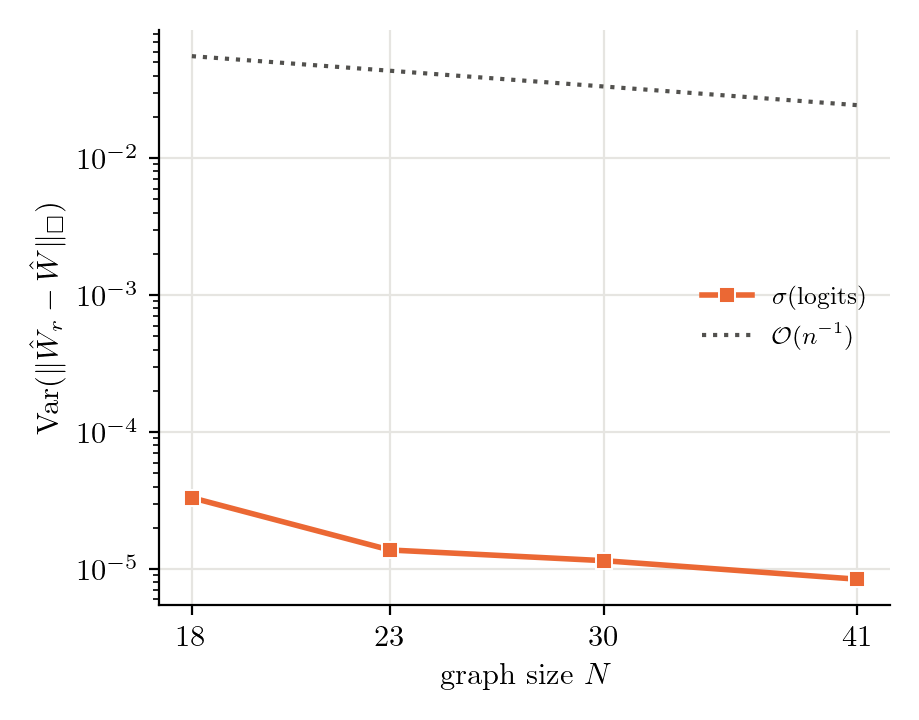}
  \caption{\centering{\textbf{NCI1}}}
\end{subfigure}

\vspace{0.3em}

% Row 3
\begin{subfigure}[t]{0.24\textwidth}
  \centering\includegraphics[width=\linewidth]{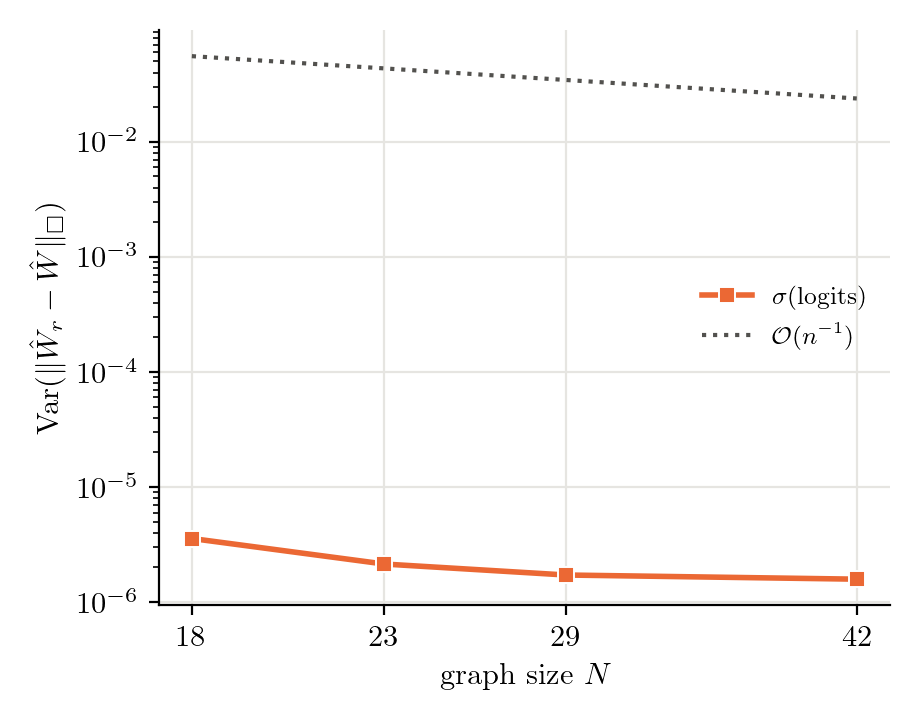}
  \caption{\centering{\textbf{NCI109}}}
\end{subfigure}
\begin{subfigure}[t]{0.24\textwidth}
  \centering\includegraphics[width=\linewidth]{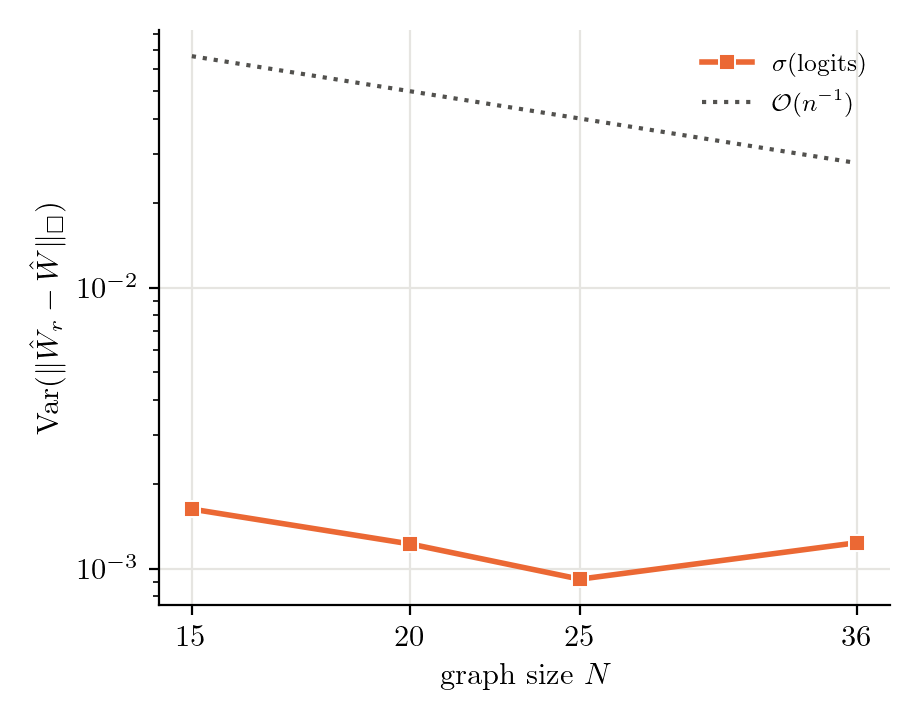}
  \caption{\centering{\textbf{OGBmolhiv}}}
\end{subfigure}
\begin{subfigure}[t]{0.24\textwidth}
  \centering\includegraphics[width=\linewidth]{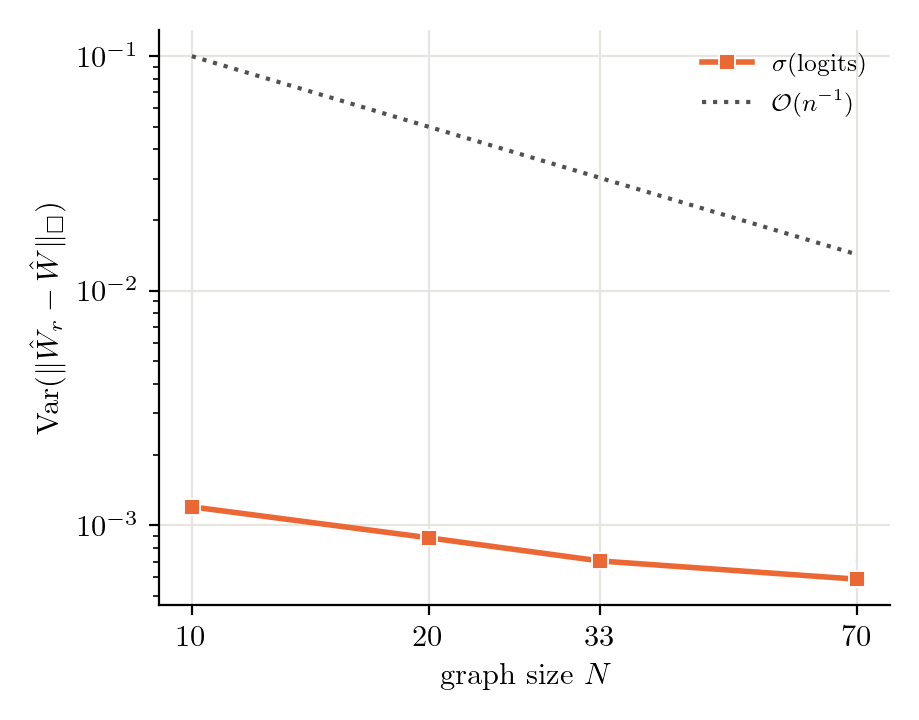}
  \caption{\centering{\textbf{PROTEINS}}}
\end{subfigure}
\begin{subfigure}[t]{0.24\textwidth}
  \centering\includegraphics[width=\linewidth]{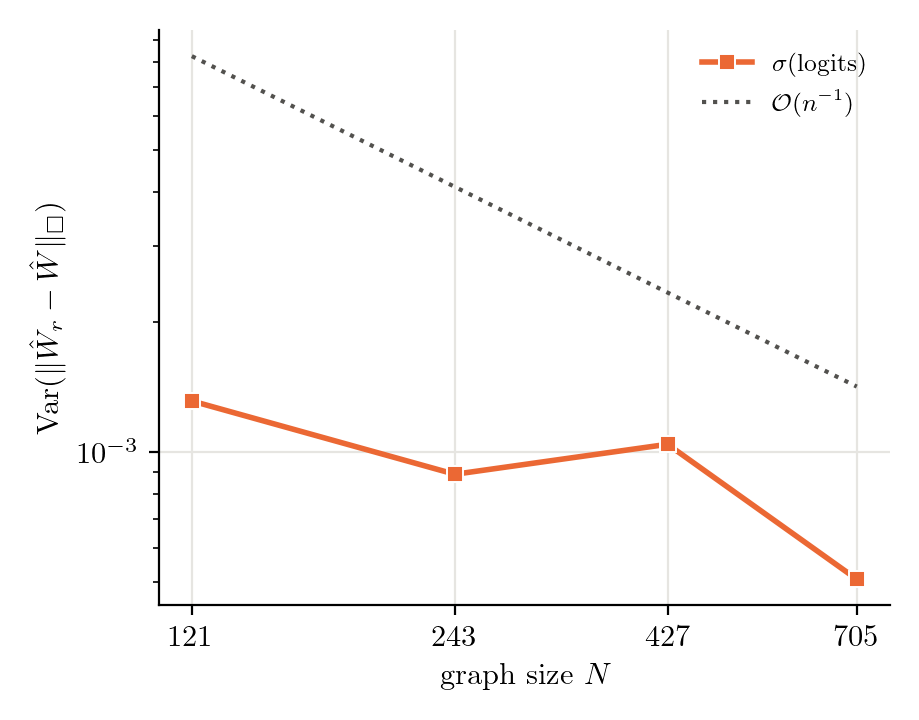}
  \caption{\centering{\textbf{REDDIT-MULTI-5K}}}
\end{subfigure}

% \caption{\textbf{Empirical cut-norm variance versus theoretical proxy (GPS, 4 heads).} Each panel compares the empirical variance of $\|\hat W_r - \hat W\|_\square$ against the size-dependent scaling proxy from Theorem~4.2. The target sizes $N$ are dataset-dependent.}
\caption{\textbf{Empirical cut-norm variance versus theoretical proxy (GPS, 4 heads).} Each panel compares the \orange{empirical variance} of $\|\hat W_r - \hat W\|_\square$ against the size-dependent \textcolor{darkgray}{scaling proxy} $\ccalO(n^{-1})$ from Theorem~4.2. The target sizes $N$ are dataset-dependent.}
\label{fig:ht_gps_h4}
\end{figure*}

\begin{figure*}[ht!]
\centering
% Row 1
\begin{subfigure}[t]{0.24\textwidth}
  \centering\includegraphics[width=\linewidth]{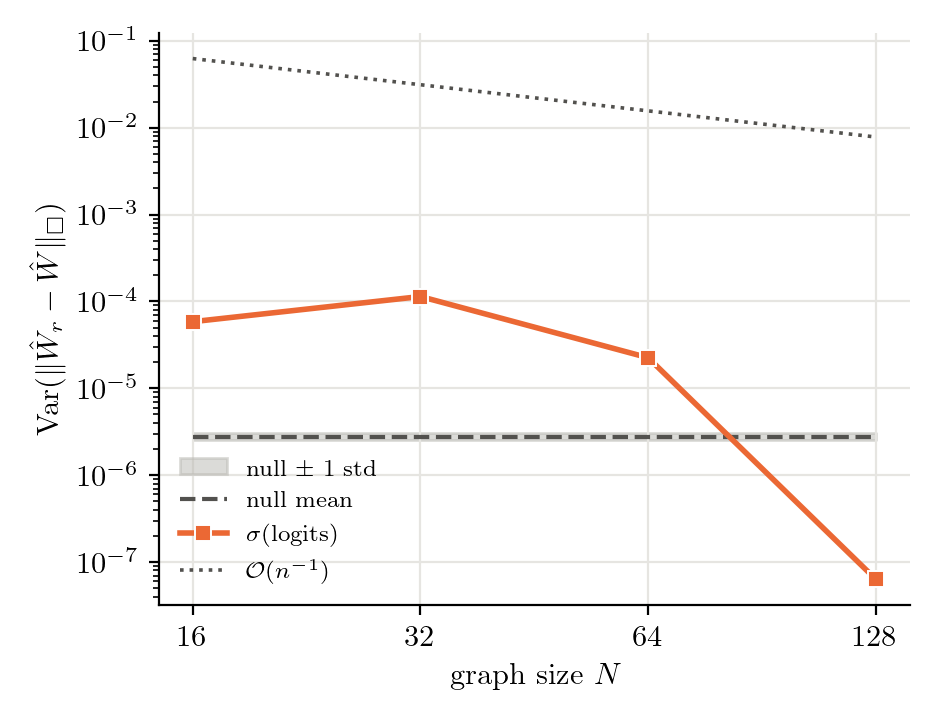}
  \caption{\centering{\textbf{BFS}}}
\end{subfigure}
\begin{subfigure}[t]{0.24\textwidth}
  \centering\includegraphics[width=\linewidth]{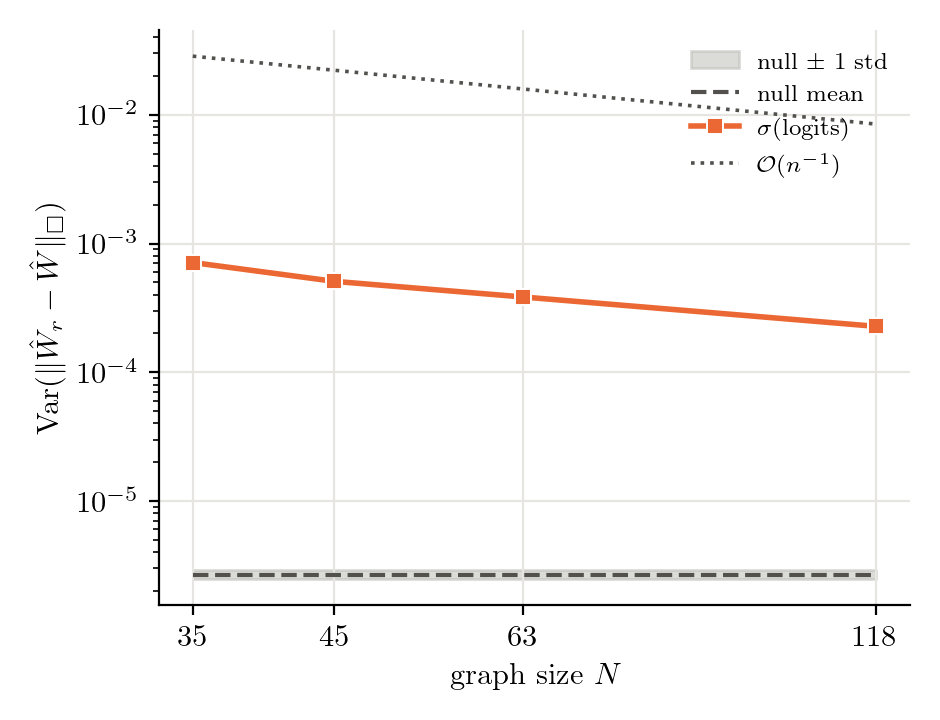}
  \caption{\centering{\textbf{COLLAB}}}
\end{subfigure}
\begin{subfigure}[t]{0.24\textwidth}
  \centering\includegraphics[width=\linewidth]{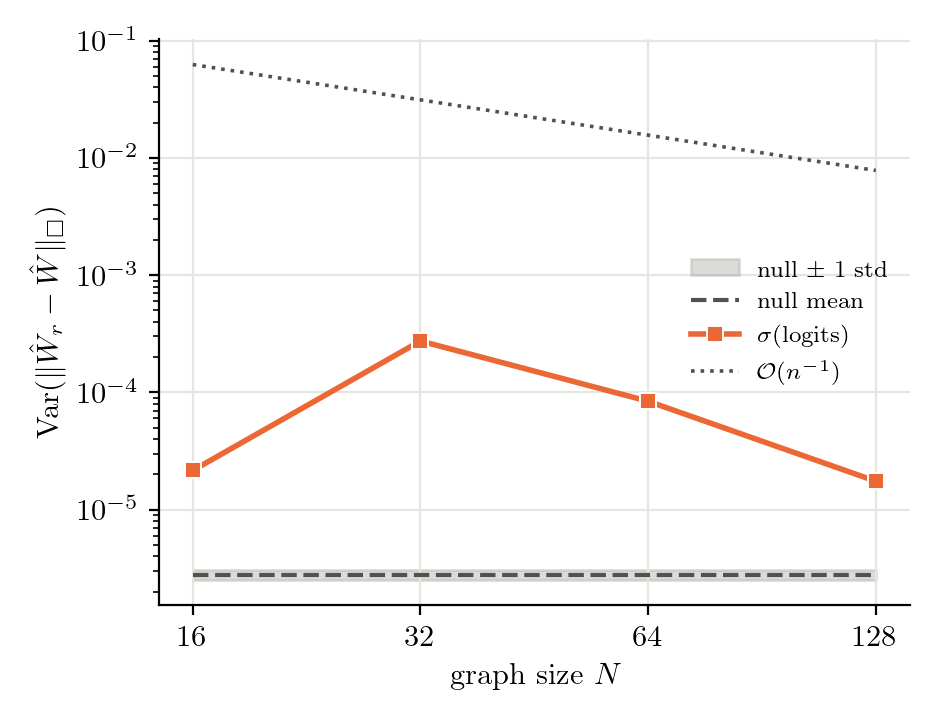}
  \caption{\centering{\textbf{Dijkstra}}}
\end{subfigure}
\begin{subfigure}[t]{0.24\textwidth}
  \centering\includegraphics[width=\linewidth]{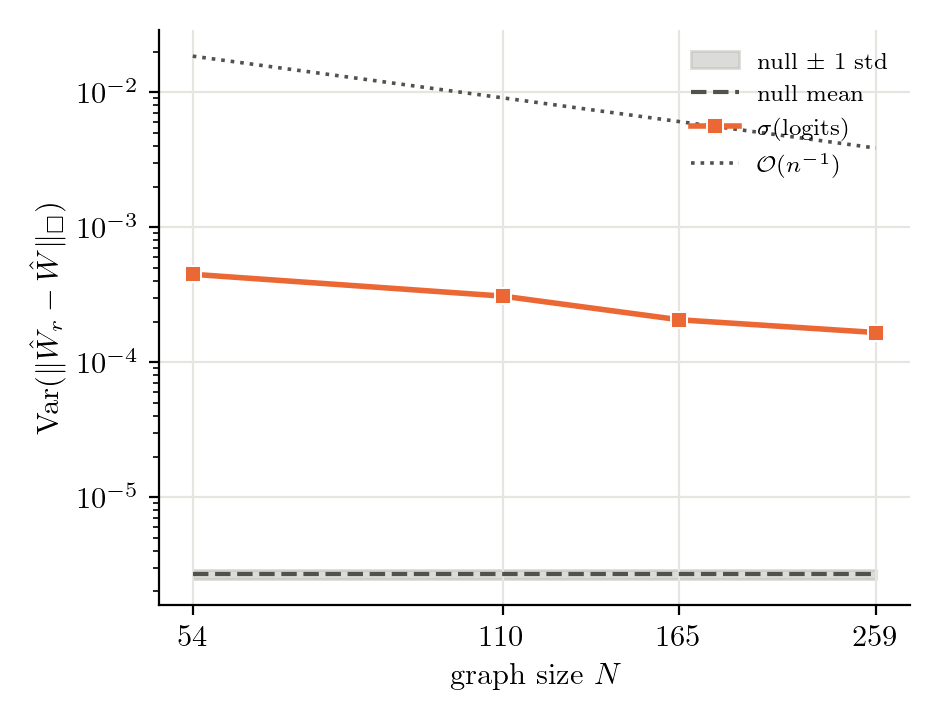}
  \caption{\centering{\textbf{LRGBPeptides}}}
\end{subfigure}

\vspace{0.3em}

% Row 2
\begin{subfigure}[t]{0.24\textwidth}
  \centering\includegraphics[width=\linewidth]{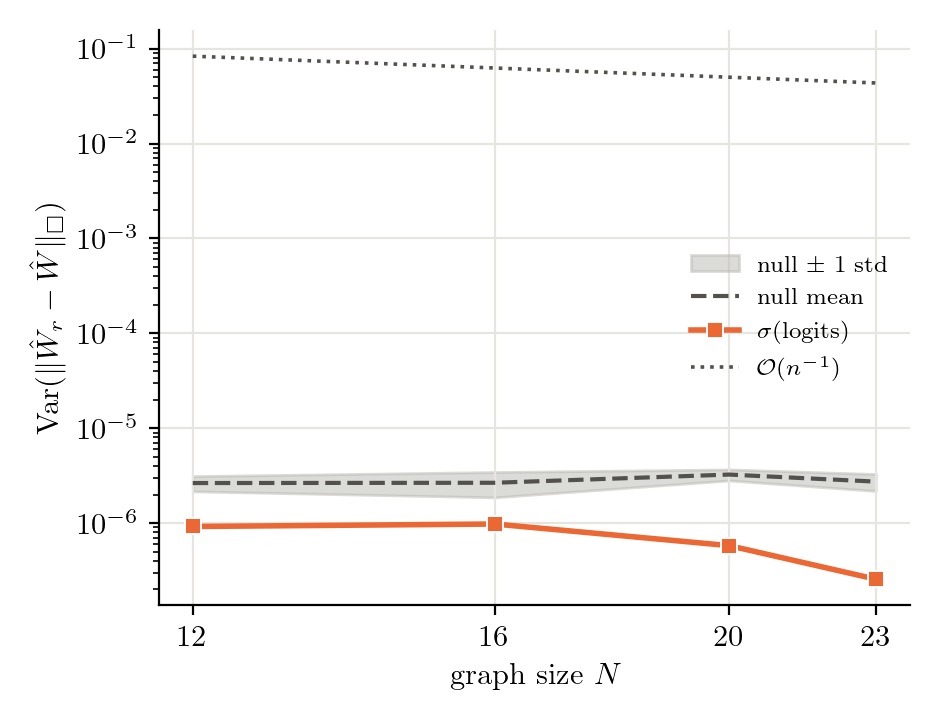}
  \caption{\centering{\textbf{MUTAG}}}
\end{subfigure}
\begin{subfigure}[t]{0.24\textwidth}
  \centering\includegraphics[width=\linewidth]{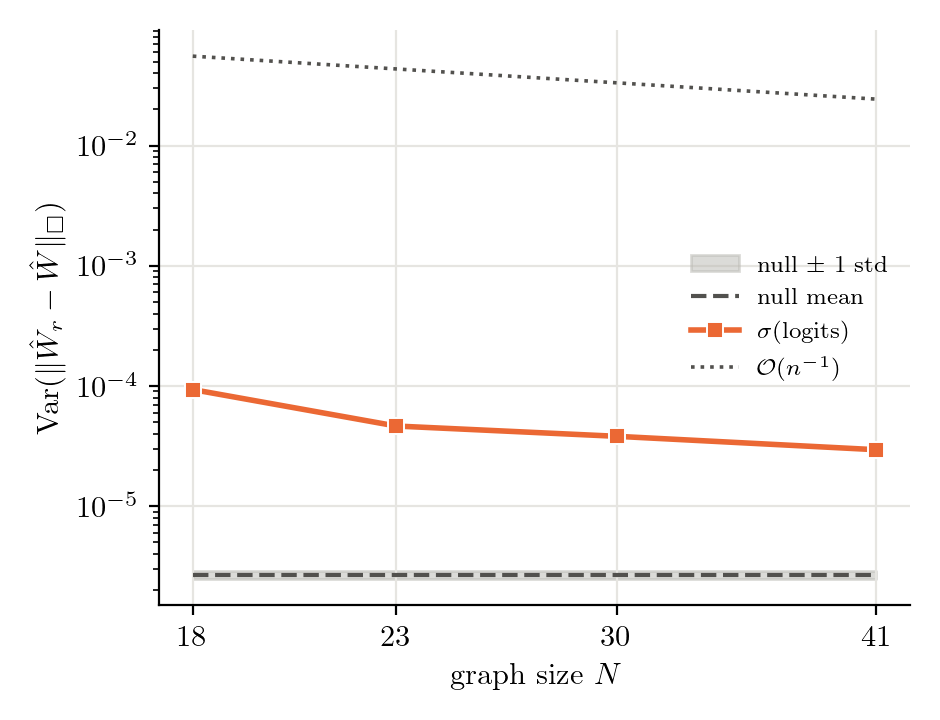}
  \caption{\centering{\textbf{NCI1}}}
\end{subfigure}
\begin{subfigure}[t]{0.24\textwidth}
  \centering\includegraphics[width=\linewidth]{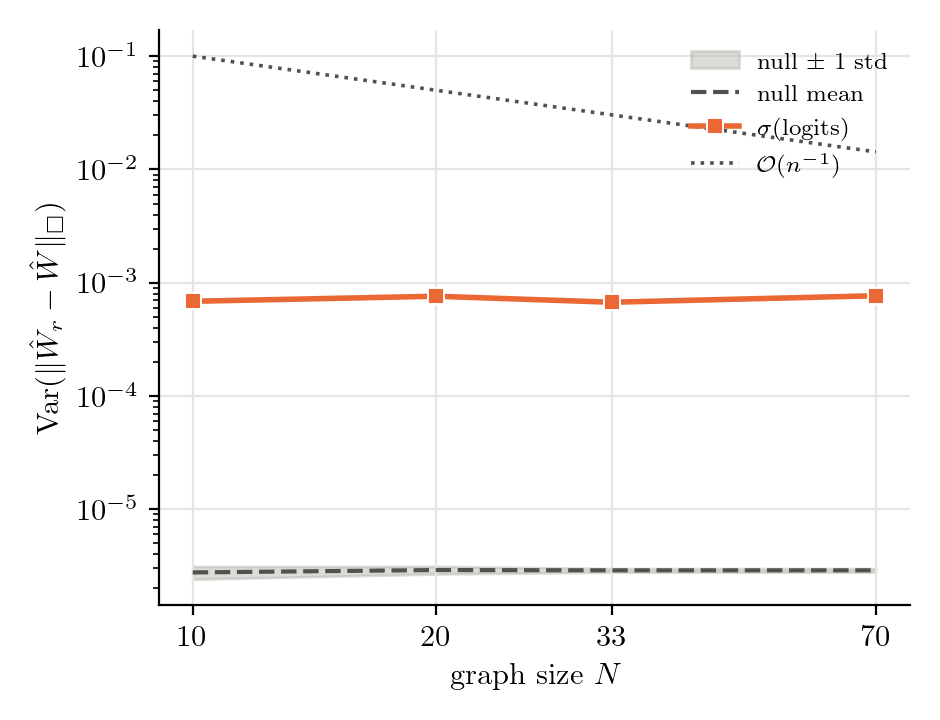}
  \caption{\centering{\textbf{PROTEINS}}}
\end{subfigure}
\begin{subfigure}[t]{0.24\textwidth}
  \centering\includegraphics[width=\linewidth]{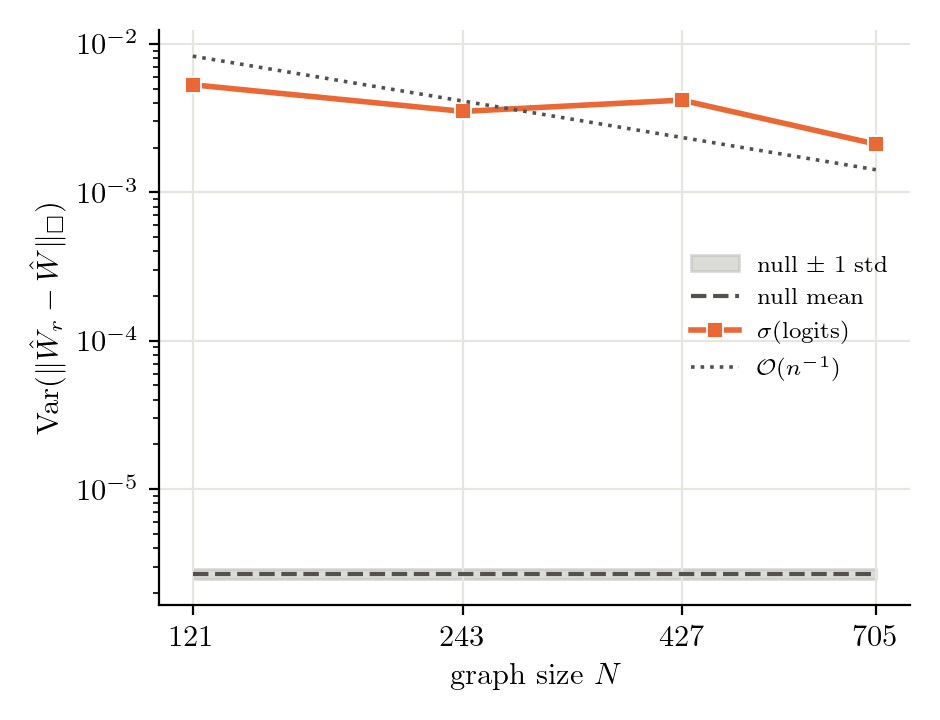}
  \caption{\centering{\textbf{REDDIT-MULTI-5K}}}
\end{subfigure}

% \caption{\textbf{Stronger hypothesis test results (GPS, single-head).} Each panel shows the results of the stronger hypothesis test for the corresponding dataset.}
\caption{\textbf{Stronger hypothesis test results (GPS, single-head).} Each panel shows the results of the stronger hypothesis test for the corresponding dataset. Each panel compares the \orange{empirical variance} against the synthetic null, shown as its \textcolor{darkgray}{mean} (dashed) with a \textcolor{gray}{band of one standard deviation}, and against the \textcolor{darkgray}{scaling proxy} $\ccalO(n^{-1})$ (dotted).}
\label{fig:ht_gps_stronger}
\end{figure*}

\begin{figure*}[ht!]
\centering
\begin{subfigure}[t]{0.19\textwidth}
  \centering\includegraphics[width=\linewidth]{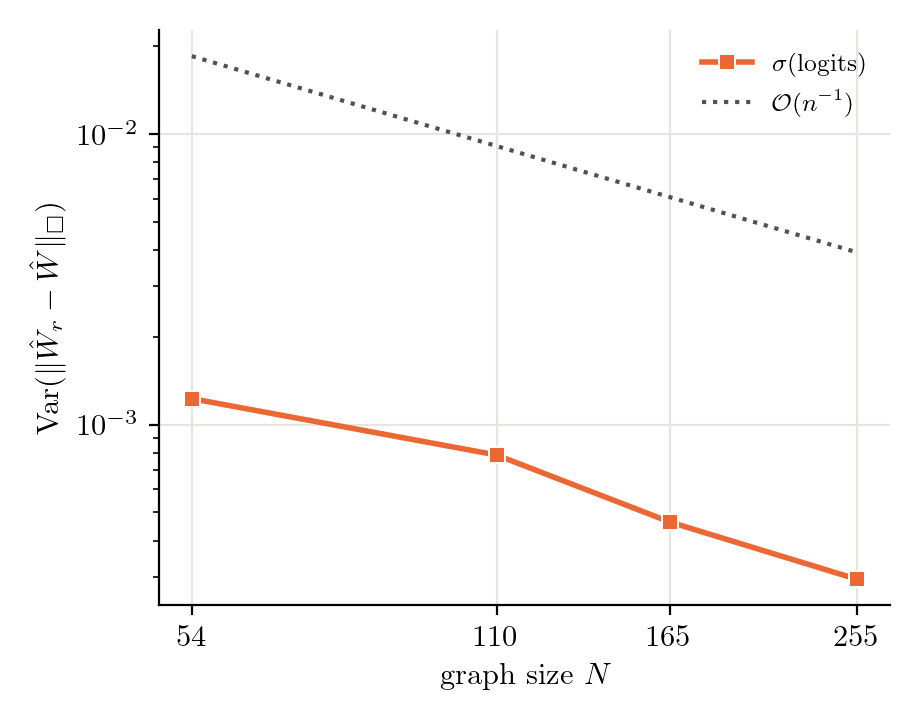}
  \caption{\centering{\textbf{LRGBPeptides}}}
\end{subfigure}
\begin{subfigure}[t]{0.19\textwidth}
  \centering\includegraphics[width=\linewidth]{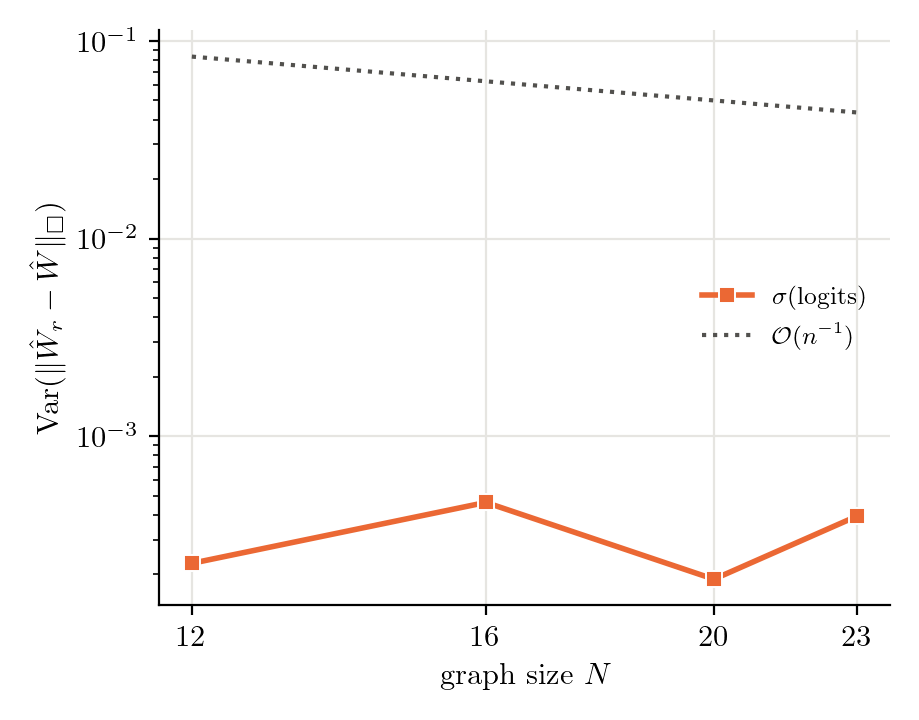}
  \caption{\centering{\textbf{MUTAG}}}
\end{subfigure}
\begin{subfigure}[t]{0.19\textwidth}
  \centering\includegraphics[width=\linewidth]{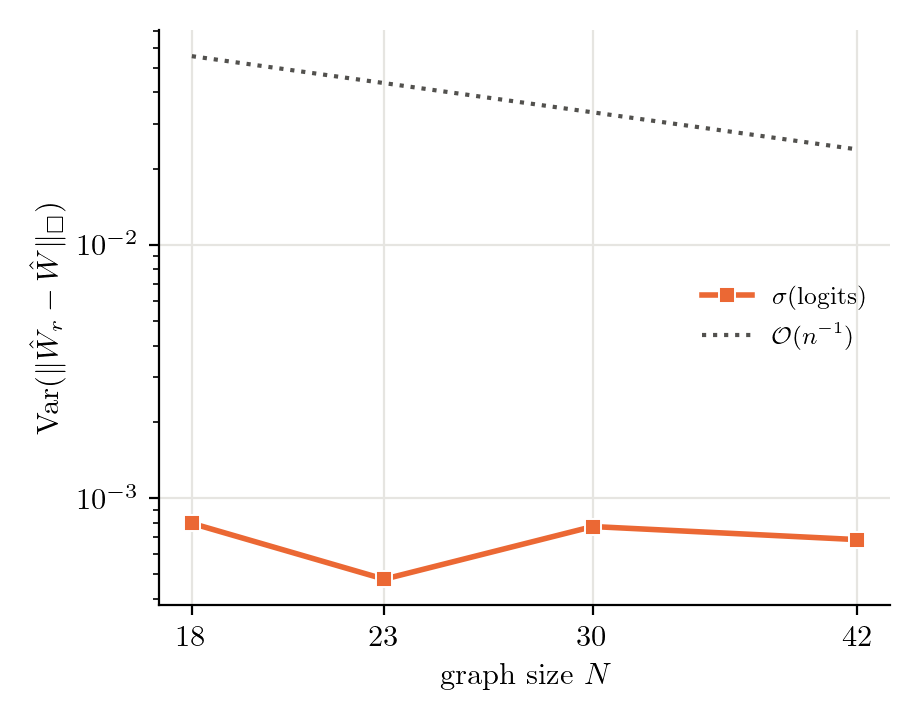}
  \caption{\centering{\textbf{NCI1}}}
\end{subfigure}
\begin{subfigure}[t]{0.19\textwidth}
  \centering\includegraphics[width=\linewidth]{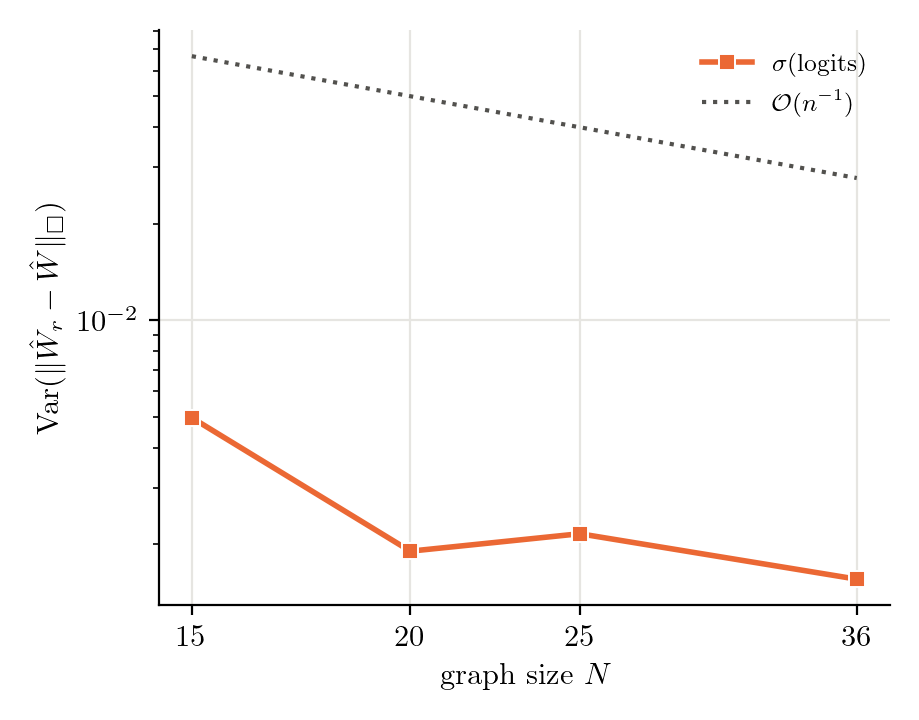}
  \caption{\centering{\textbf{OGBmolhiv}}}
\end{subfigure}
\begin{subfigure}[t]{0.19\textwidth}
  \centering\includegraphics[width=\linewidth]{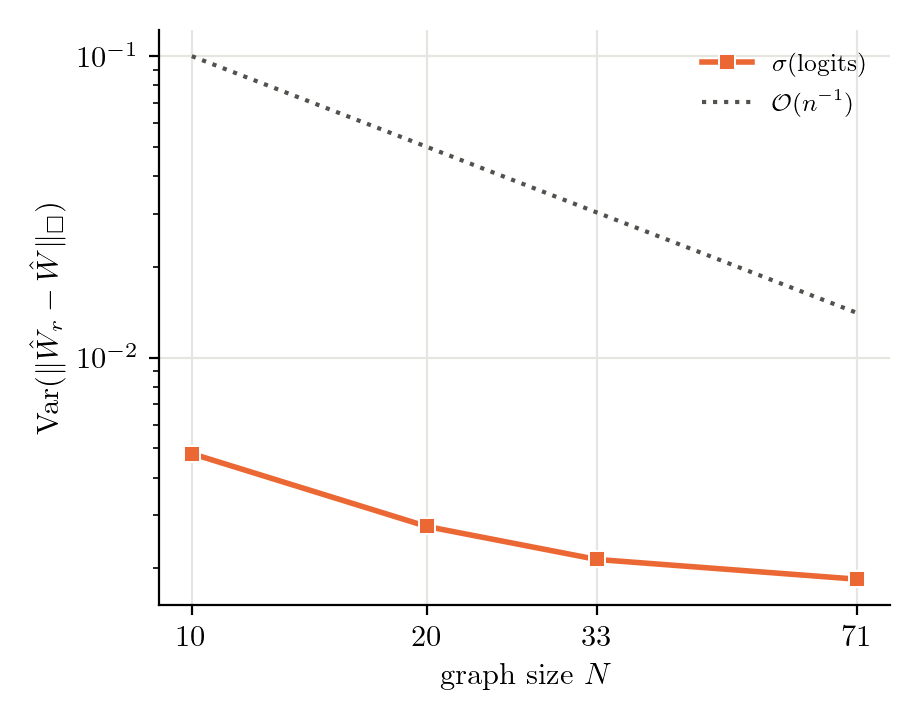}
  \caption{\centering{\textbf{PROTEINS}}}
\end{subfigure}

% \caption{\textbf{Empirical cut-norm variance versus theoretical proxy (Graphormer-GD, single-head).} Each panel compares the empirical variance of $\|\hat W_r - \hat W\|_\square$ against the size-dependent scaling proxy from Theorem~4.2. The target sizes $N$ are dataset-dependent.}
\caption{\textbf{Empirical cut-norm variance versus theoretical proxy (Graphormer-GD, single-head).} Each panel compares the \orange{empirical variance} of $\|\hat W_r - \hat W\|_\square$ against the size-dependent \textcolor{darkgray}{scaling proxy} $\ccalO(n^{-1})$ from Theorem~4.2. The target sizes $N$ are dataset-dependent.}
\label{fig:ht_graphormer_h1}
\end{figure*}

\begin{figure*}[ht!]
\centering
\begin{subfigure}[t]{0.19\textwidth}
  \centering\includegraphics[width=\linewidth]{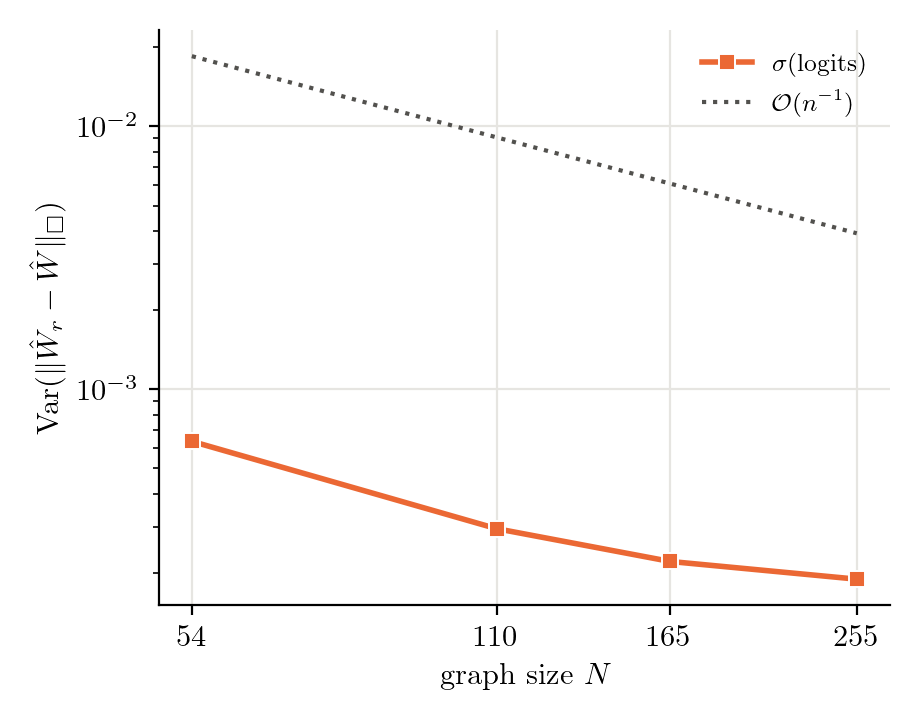}
  \caption{\centering{\textbf{LRGBPeptides}}}
\end{subfigure}
\begin{subfigure}[t]{0.19\textwidth}
  \centering\includegraphics[width=\linewidth]{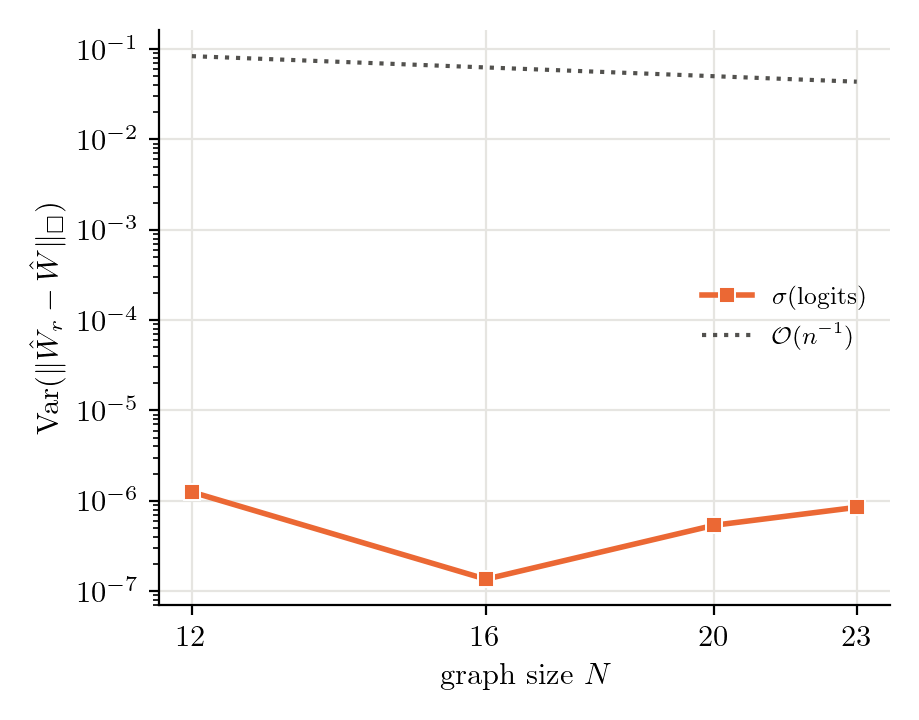}
  \caption{\centering{\textbf{MUTAG}}}
\end{subfigure}
\begin{subfigure}[t]{0.19\textwidth}
  \centering\includegraphics[width=\linewidth]{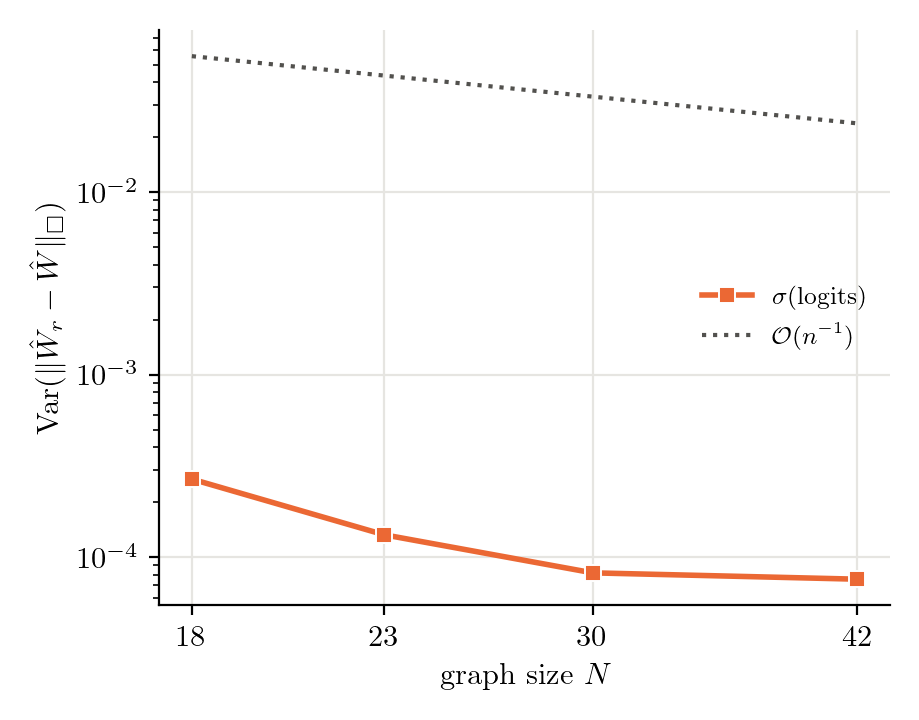}
  \caption{\centering{\textbf{NCI1}}}
\end{subfigure}
\begin{subfigure}[t]{0.19\textwidth}
  \centering\includegraphics[width=\linewidth]{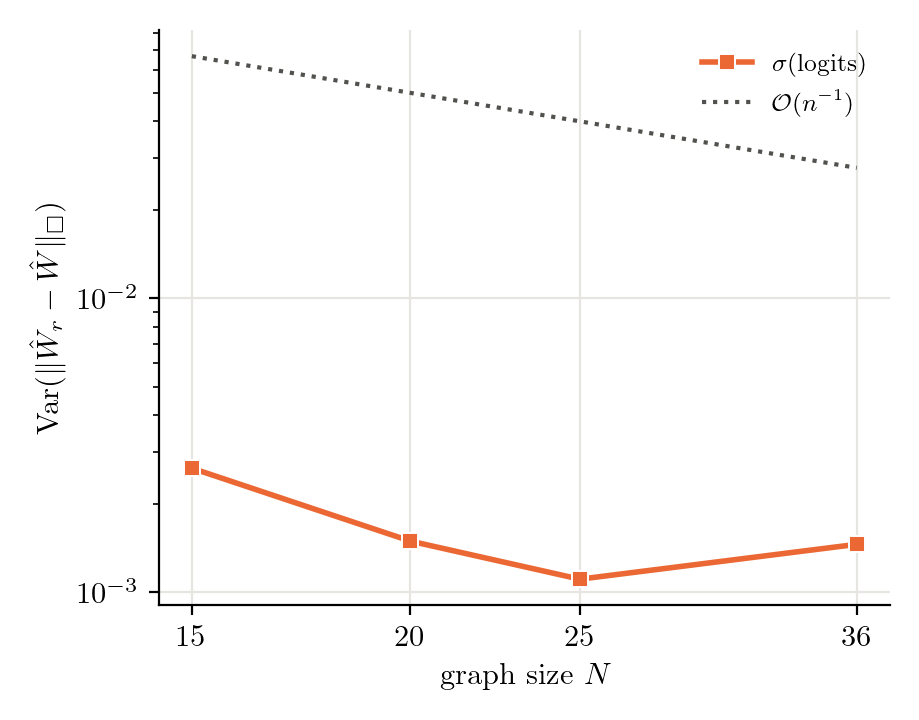}
  \caption{\centering{\textbf{OGBmolhiv}}}
\end{subfigure}
\begin{subfigure}[t]{0.19\textwidth}
  \centering\includegraphics[width=\linewidth]{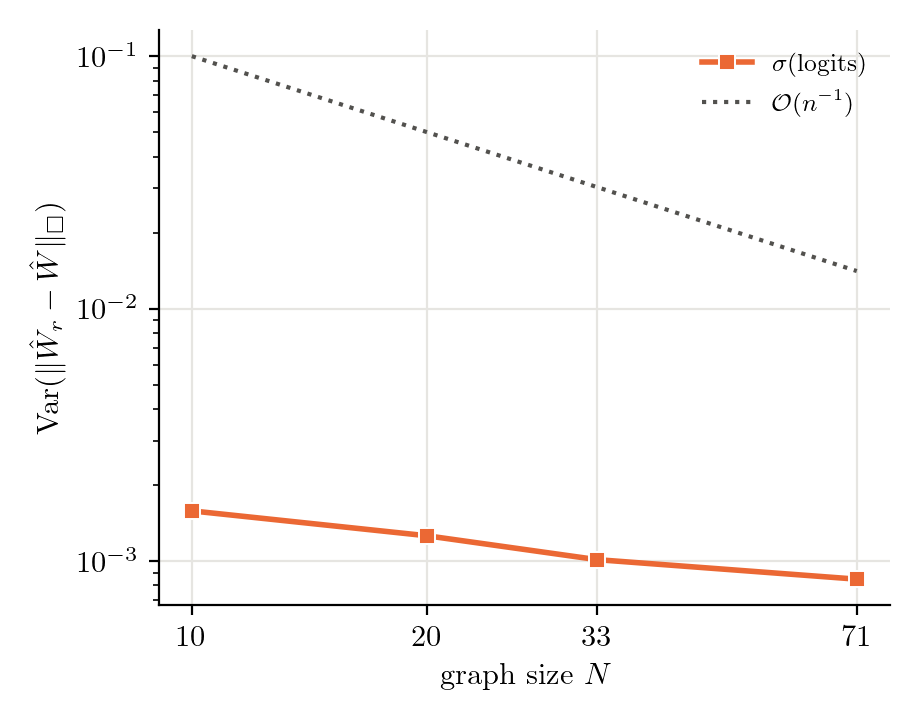}
  \caption{\centering{\textbf{PROTEINS}}}
\end{subfigure}

% \caption{\textbf{Empirical cut-norm variance versus theoretical proxy (GRIT, single-head).} Each panel compares the empirical variance of $\|\hat W_r - \hat W\|_\square$ against the size-dependent scaling proxy from Theorem~4.2. The target sizes $N$ are dataset-dependent.}
\caption{\textbf{Empirical cut-norm variance versus theoretical proxy (GRIT, single-head).} Each panel compares the \orange{empirical variance} of $\|\hat W_r - \hat W\|_\square$ against the size-dependent \textcolor{darkgray}{scaling proxy} $\ccalO(n^{-1})$ from Theorem~4.2. The target sizes $N$ are dataset-dependent.}
\label{fig:ht_grit_h1}
\end{figure*}

\begin{figure*}[ht!]
\centering
% Row 1
\begin{subfigure}[t]{0.32\textwidth}
  \centering\includegraphics[width=\linewidth]{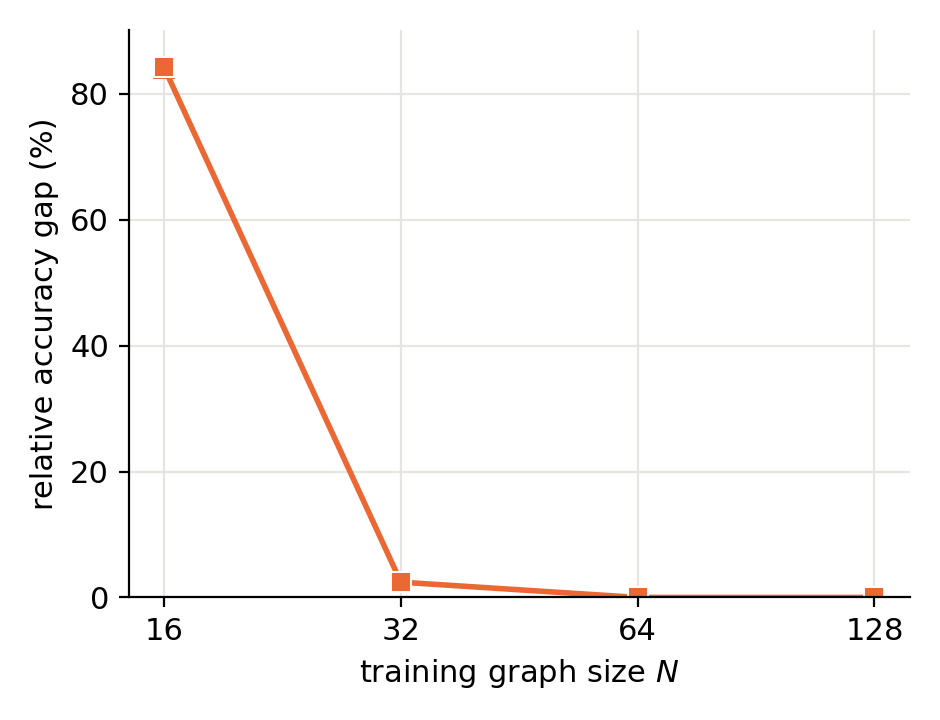}
  \caption{\centering{\textbf{BFS}}}
\end{subfigure}
\begin{subfigure}[t]{0.32\textwidth}
  \centering\includegraphics[width=\linewidth]{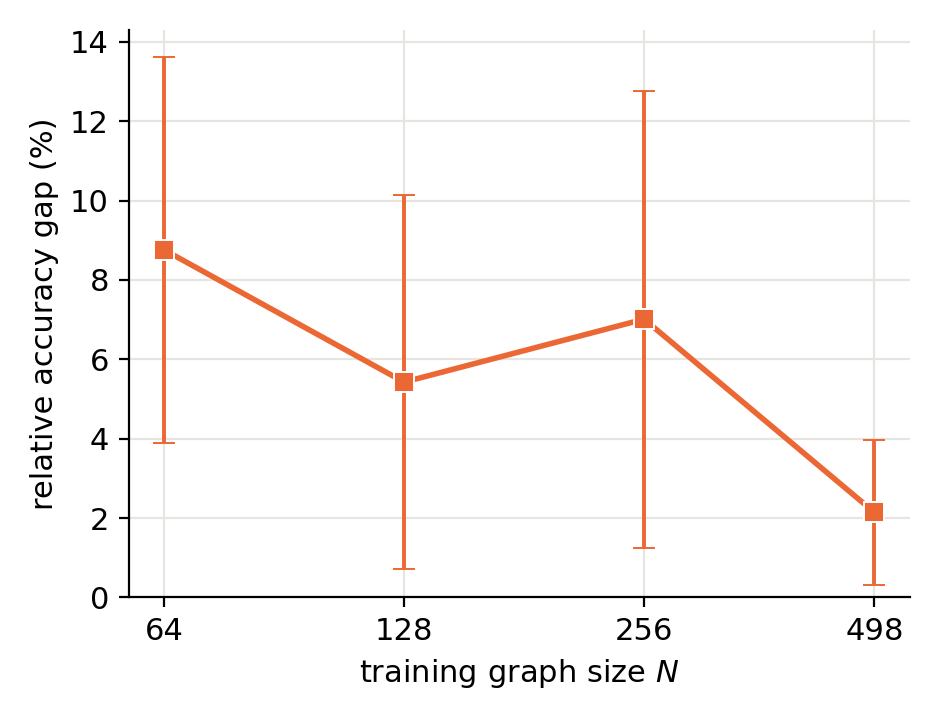}
  \caption{\centering{\textbf{COLLAB}}}
\end{subfigure}
\begin{subfigure}[t]{0.32\textwidth}
  \centering\includegraphics[width=\linewidth]{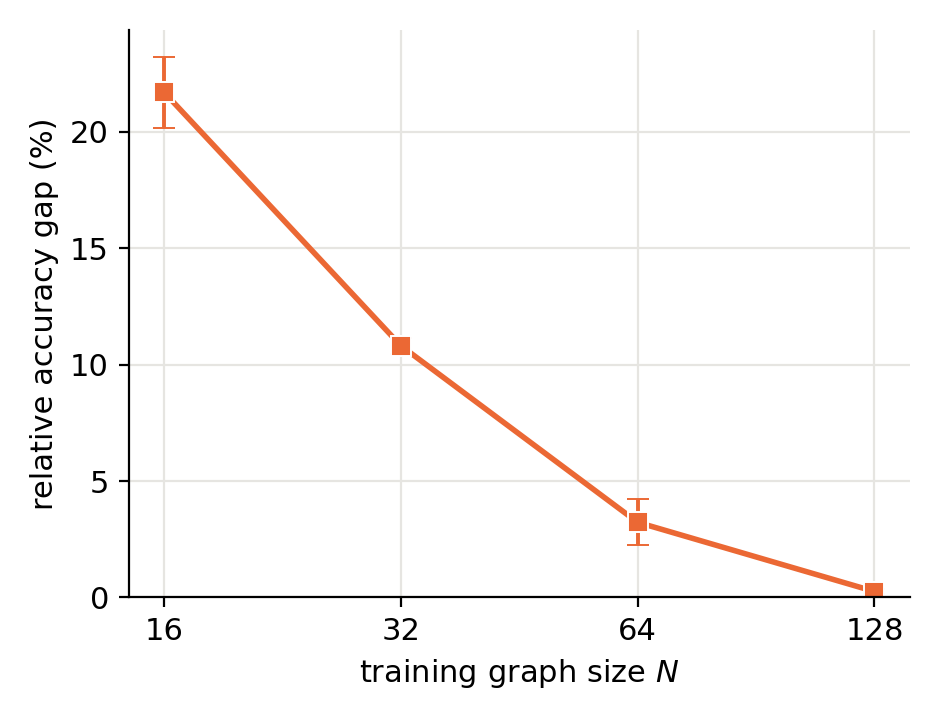}
  \caption{\centering{\textbf{Dijkstra}}}
\end{subfigure}

\vspace{0.3em}

% Row 2
\begin{subfigure}[t]{0.32\textwidth}
  \centering\includegraphics[width=\linewidth]{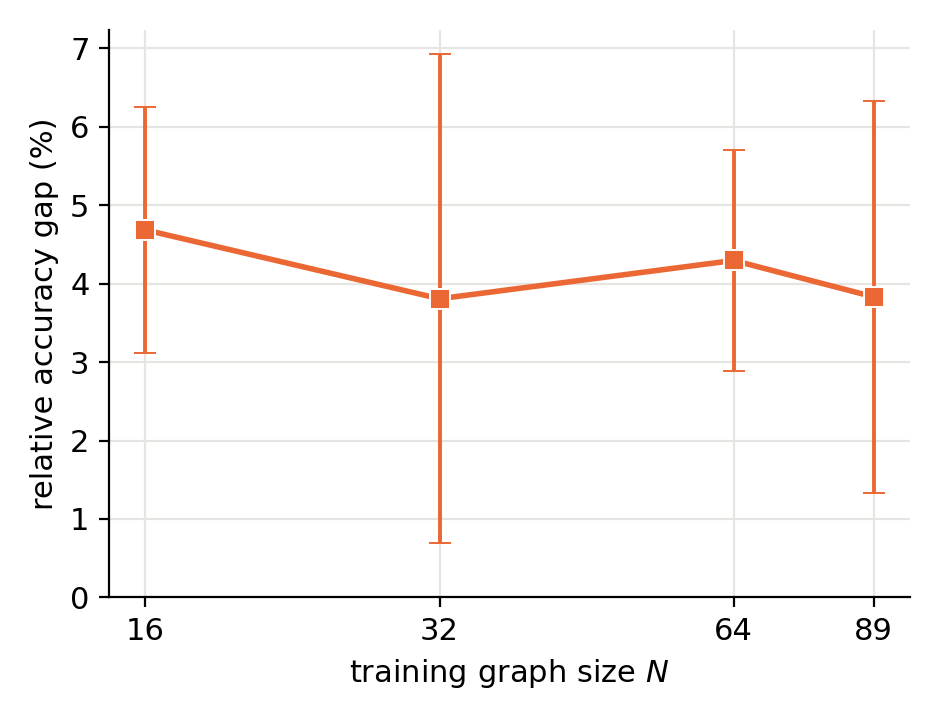}
  \caption{\centering{\textbf{IMDB-MULTI}}}
\end{subfigure}
\begin{subfigure}[t]{0.32\textwidth}
  \centering\includegraphics[width=\linewidth]{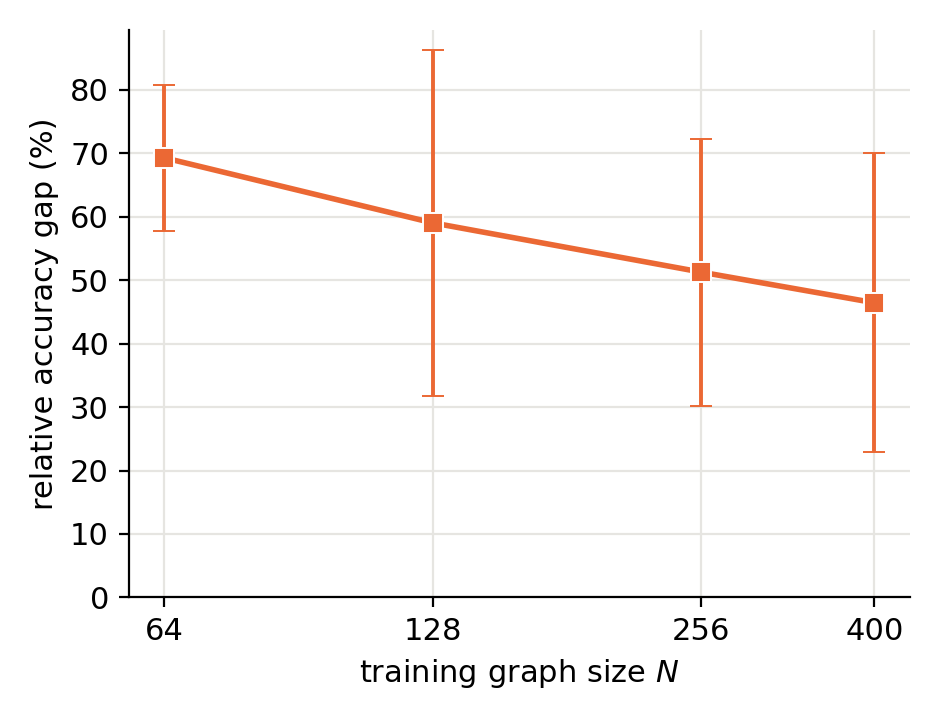}
  \caption{\centering{\textbf{LRGBPeptides}}}
\end{subfigure}
\begin{subfigure}[t]{0.32\textwidth}
  \centering\includegraphics[width=\linewidth]{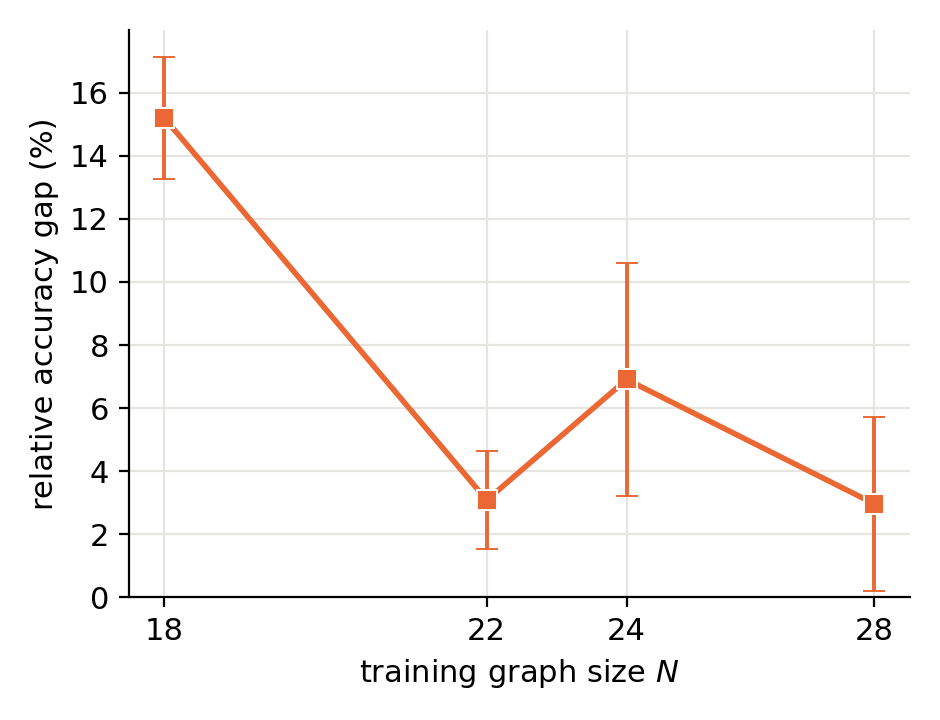}
  \caption{\centering{\textbf{MUTAG}}}
\end{subfigure}

\vspace{0.3em}

% Row 3
\begin{subfigure}[t]{0.32\textwidth}
  \centering\includegraphics[width=\linewidth]{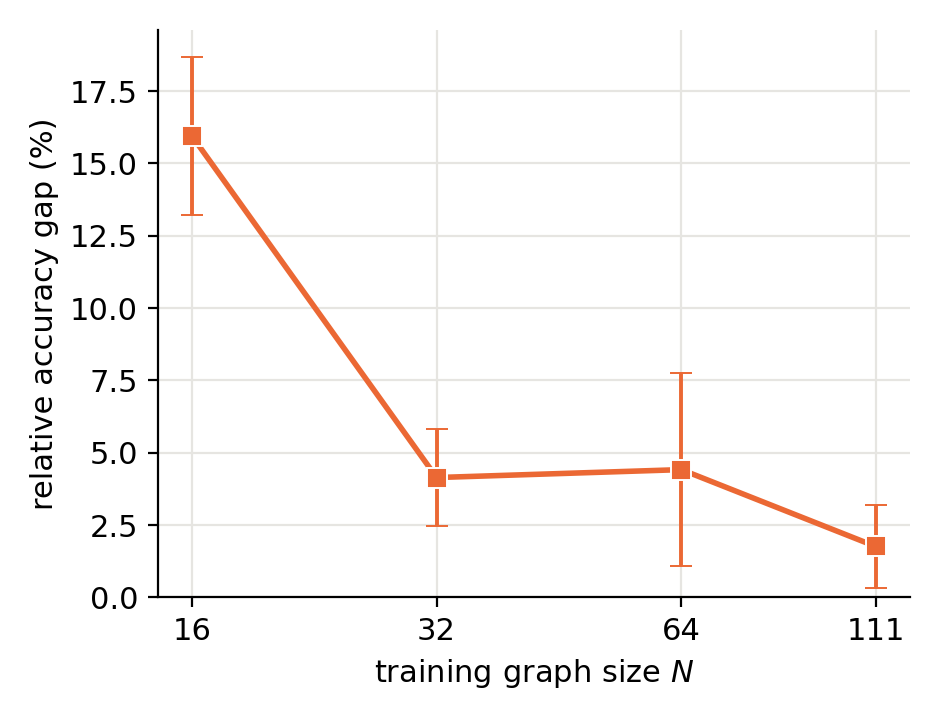}
  \caption{\centering{\textbf{NCI1}}}
\end{subfigure}
\begin{subfigure}[t]{0.32\textwidth}
  \centering\includegraphics[width=\linewidth]{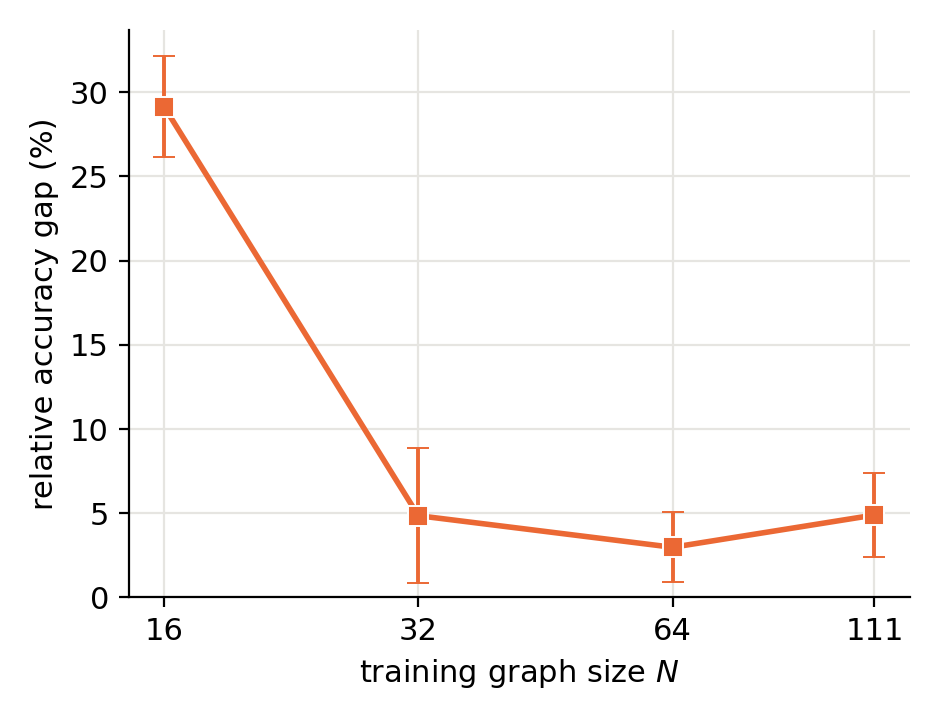}
  \caption{\centering{\textbf{NCI109}}}
\end{subfigure}
\begin{subfigure}[t]{0.32\textwidth}
  \centering\includegraphics[width=\linewidth]{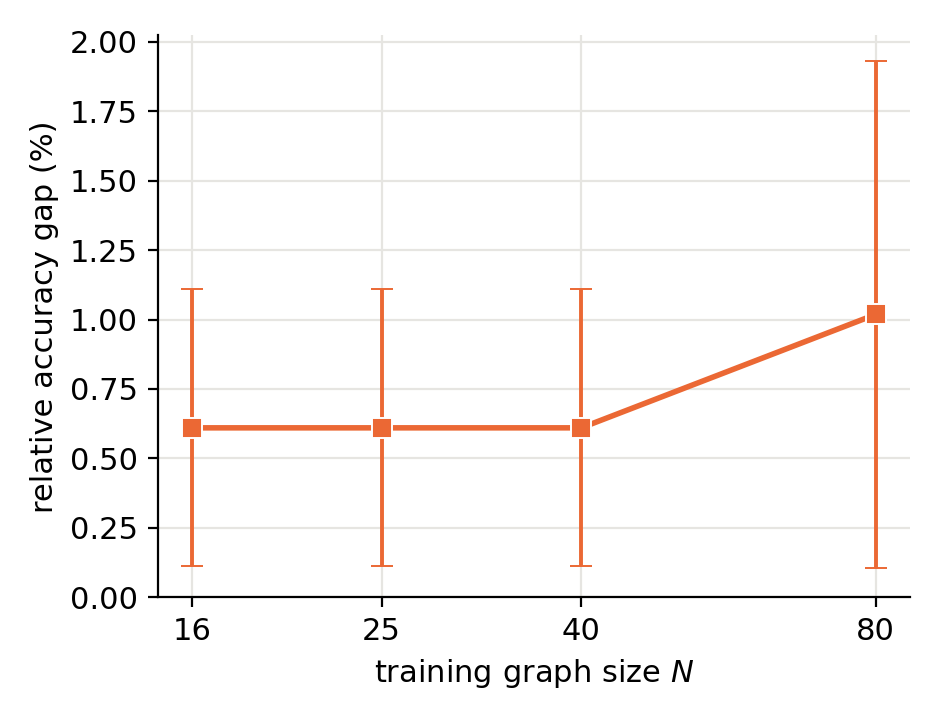}
  \caption{\centering{\textbf{OGBmolhiv}}}
\end{subfigure}

\caption{\textbf{Transferability gap across graph sizes (GPS, single-head).} Each panel shows the transferability gap as a function of graph size for the corresponding dataset.}
\label{fig:transferability_gps_h1_1}
\end{figure*}

\begin{figure*}[ht!]
\centering
\begin{subfigure}[t]{0.32\textwidth}
  \centering\includegraphics[width=\linewidth]{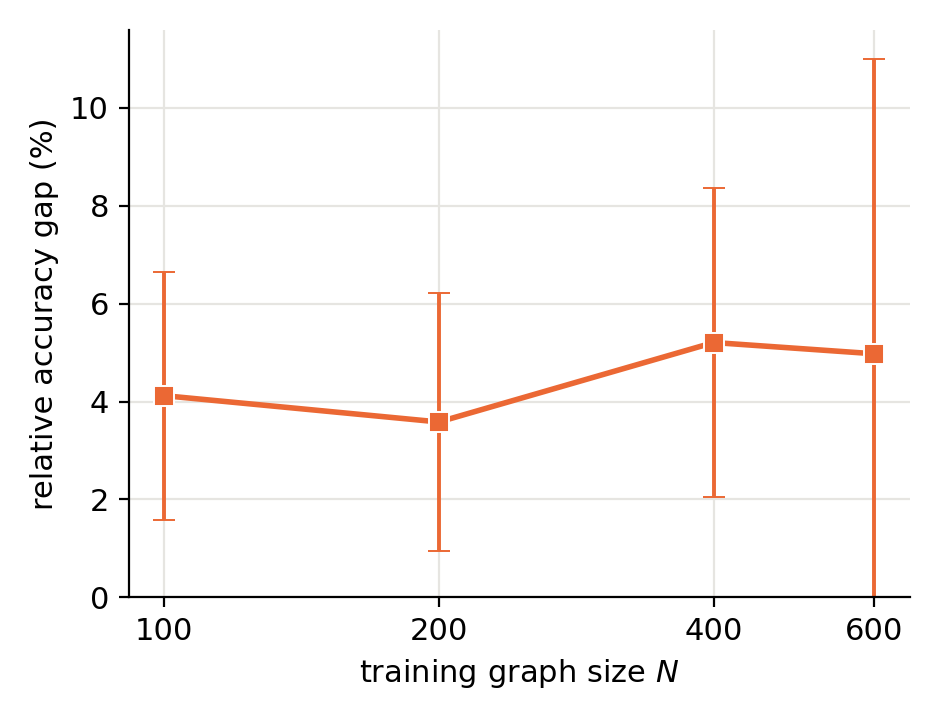}
  \caption{\centering{\textbf{PROTEINS}}}
\end{subfigure}
\begin{subfigure}[t]{0.32\textwidth}
  \centering\includegraphics[width=\linewidth]{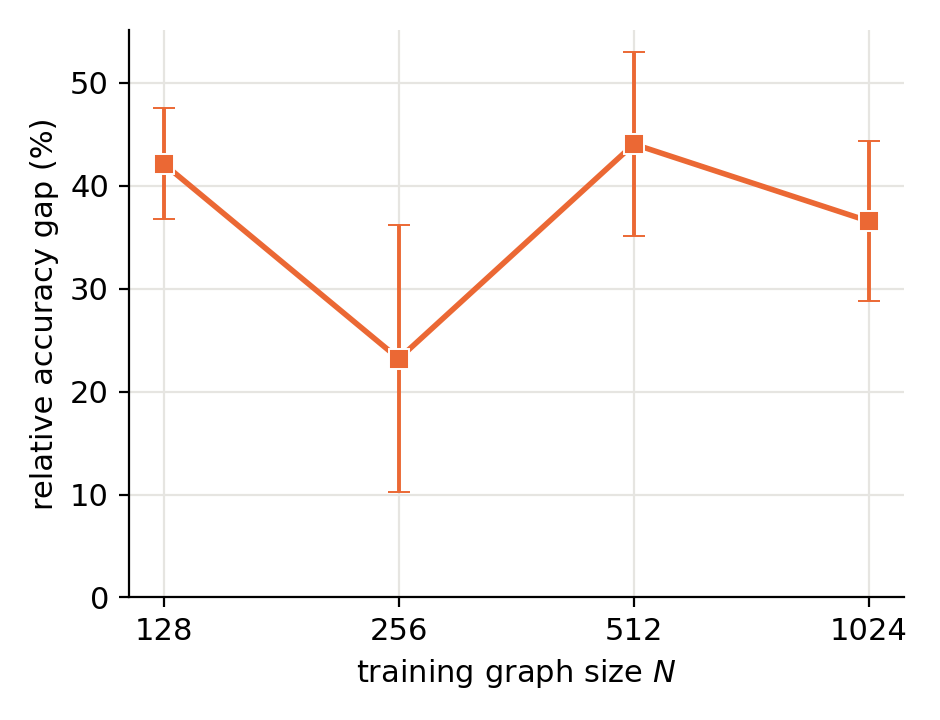}
  \caption{\centering{\textbf{REDDIT-MULTI-5K}}}
\end{subfigure}

\caption{\textbf{Transferability gap across graph sizes (GPS, single-head).} Each panel shows the transferability gap as a function of graph size for the corresponding dataset.}
\label{fig:transferability_gps_h1_2}
\end{figure*}

\begin{figure*}[ht!]
\centering
% Row 1
\begin{subfigure}[t]{0.32\textwidth}
  \centering\includegraphics[width=\linewidth]{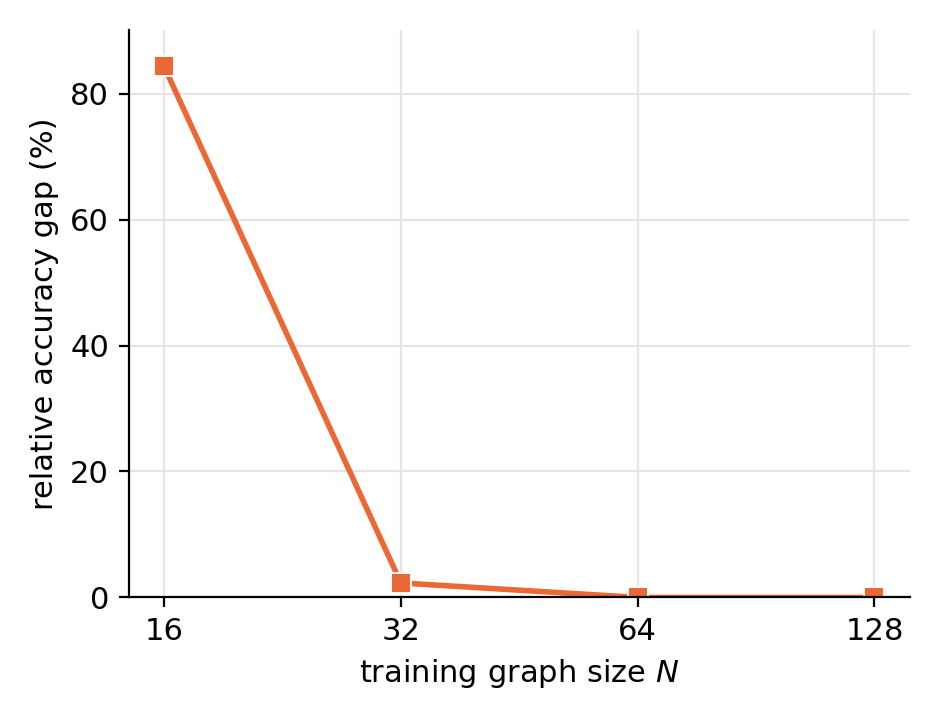}
  \caption{\centering{\textbf{BFS}}}
\end{subfigure}
\begin{subfigure}[t]{0.32\textwidth}
  \centering\includegraphics[width=\linewidth]{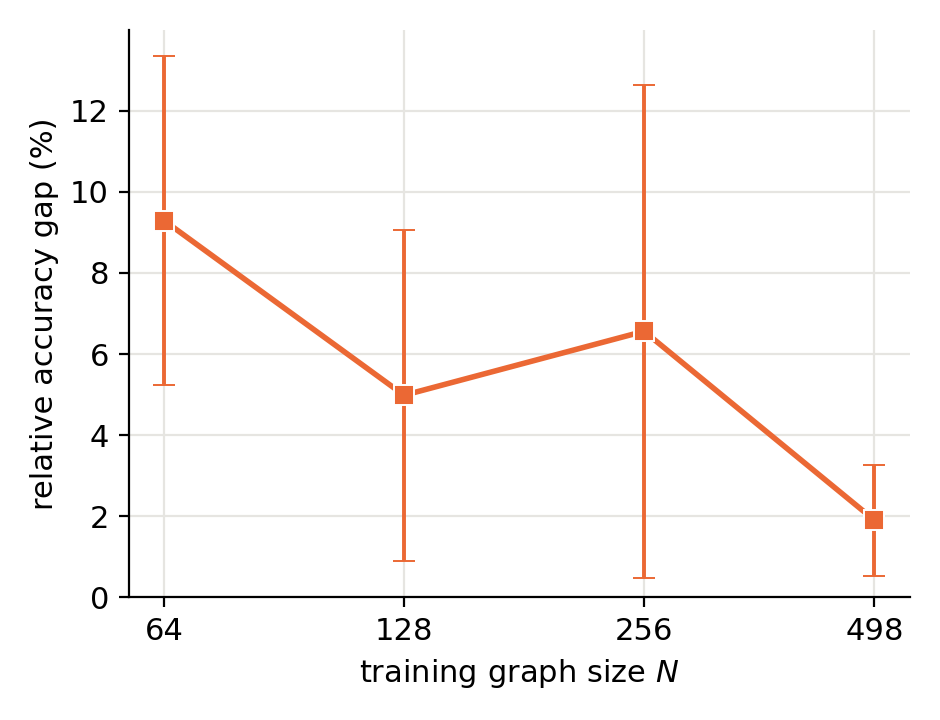}
  \caption{\centering{\textbf{COLLAB}}}
\end{subfigure}
\begin{subfigure}[t]{0.32\textwidth}
  \centering\includegraphics[width=\linewidth]{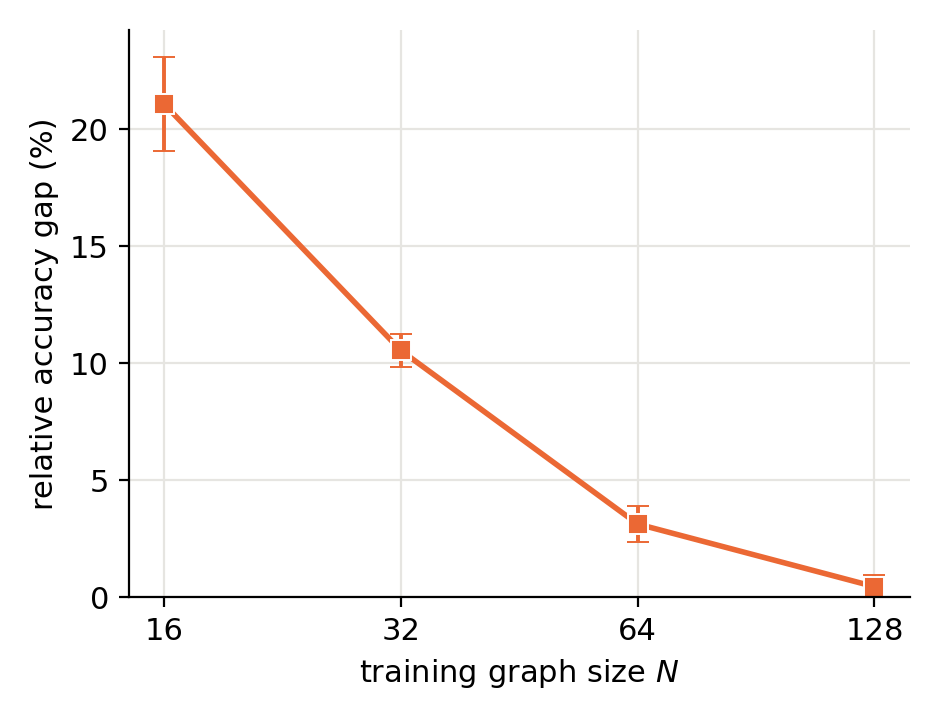}
  \caption{\centering{\textbf{Dijkstra}}}
\end{subfigure}

\vspace{0.3em}

% Row 2
\begin{subfigure}[t]{0.32\textwidth}
  \centering\includegraphics[width=\linewidth]{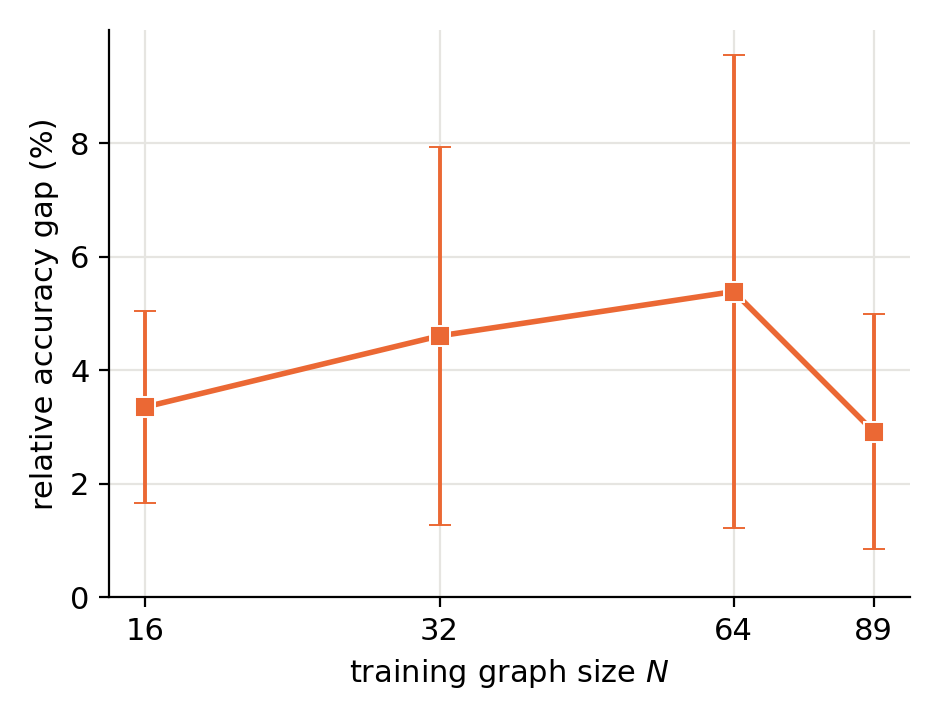}
  \caption{\centering{\textbf{IMDB-MULTI}}}
\end{subfigure}
\begin{subfigure}[t]{0.32\textwidth}
  \centering\includegraphics[width=\linewidth]{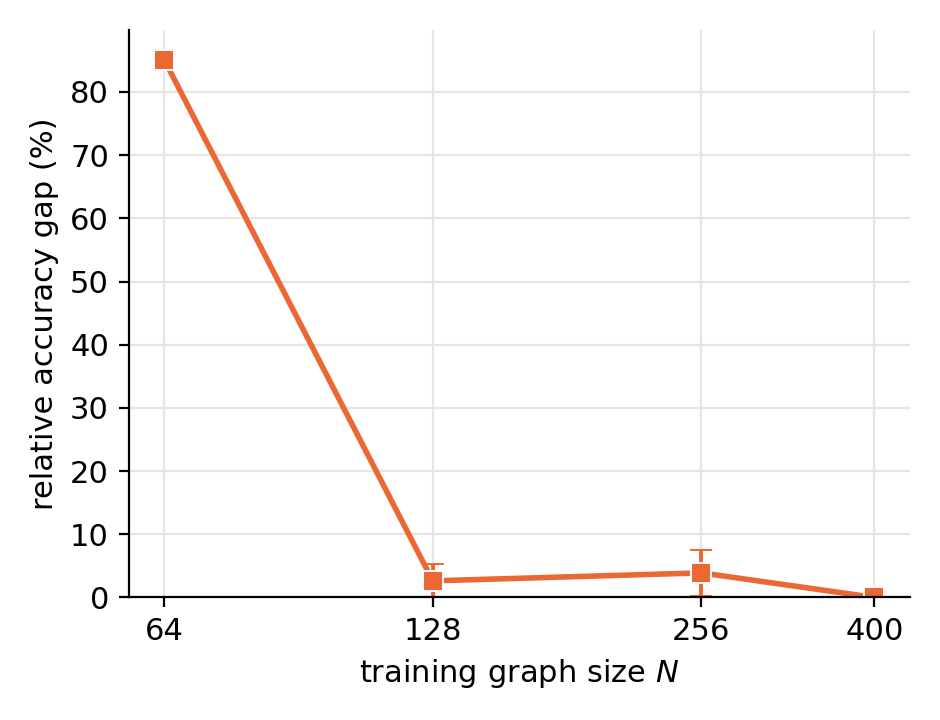}
  \caption{\centering{\textbf{LRGBPeptides}}}
\end{subfigure}
\begin{subfigure}[t]{0.32\textwidth}
  \centering\includegraphics[width=\linewidth]{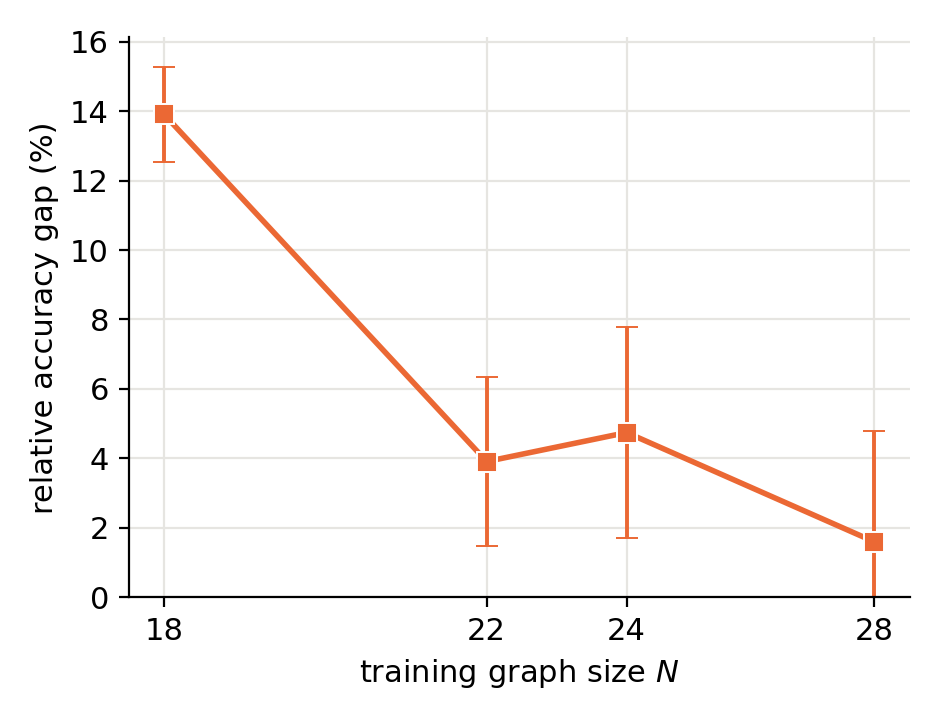}
  \caption{\centering{\textbf{MUTAG}}}
\end{subfigure}

\vspace{0.3em}

% Row 3
\begin{subfigure}[t]{0.32\textwidth}
  \centering\includegraphics[width=\linewidth]{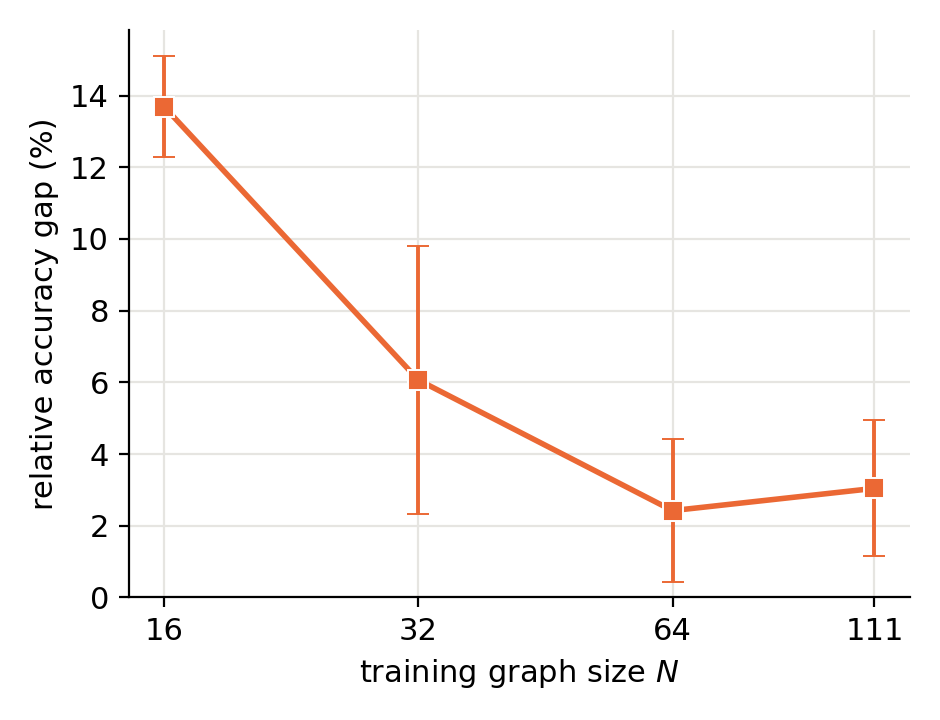}
  \caption{\centering{\textbf{NCI1}}}
\end{subfigure}
\begin{subfigure}[t]{0.32\textwidth}
  \centering\includegraphics[width=\linewidth]{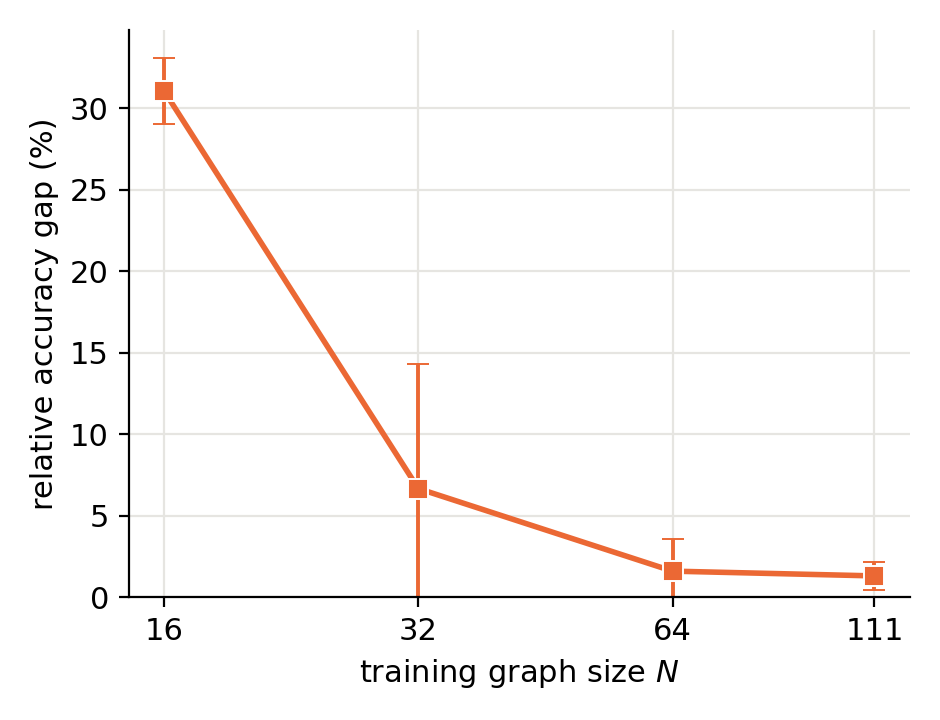}
  \caption{\centering{\textbf{NCI109}}}
\end{subfigure}
\begin{subfigure}[t]{0.32\textwidth}
  \centering\includegraphics[width=\linewidth]{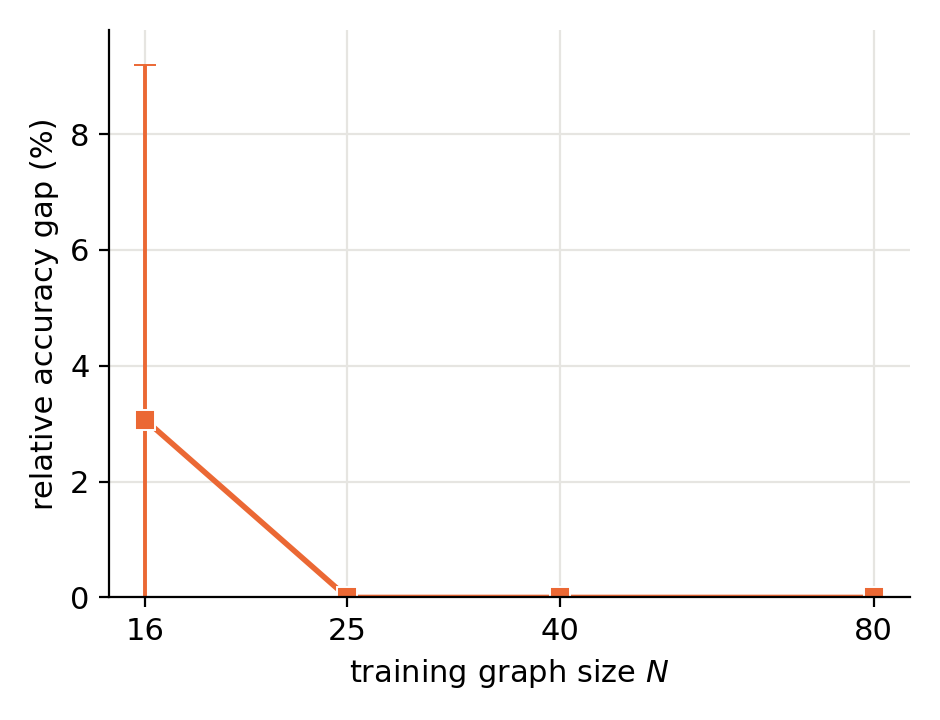}
  \caption{\centering{\textbf{OGBmolhiv}}}
\end{subfigure}

\caption{\textbf{Transferability gap across graph sizes (GPS, 4 heads).} Each panel shows the transferability gap as a function of graph size for the corresponding dataset.}
\label{fig:transferability_gps_h4_1}
\end{figure*}

\begin{figure*}[ht!]
\centering
\begin{subfigure}[t]{0.32\textwidth}
  \centering\includegraphics[width=\linewidth]{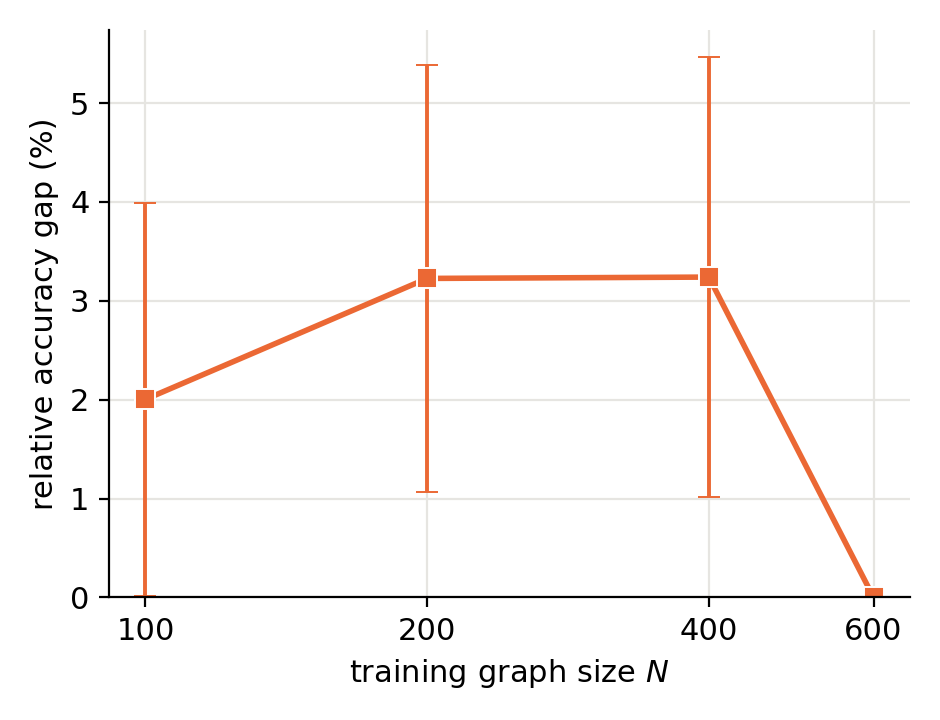}
  \caption{\centering{\textbf{PROTEINS}}}
\end{subfigure}
\begin{subfigure}[t]{0.32\textwidth}
  \centering\includegraphics[width=\linewidth]{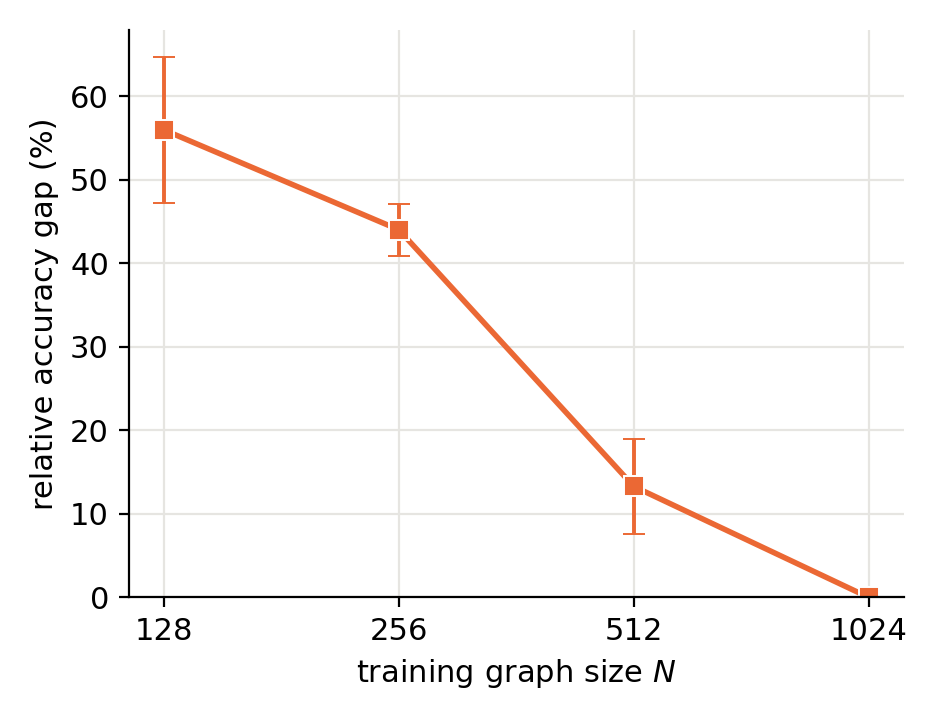}
  \caption{\centering{\textbf{REDDIT-MULTI-5K}}}
\end{subfigure}

\caption{\textbf{Transferability gap across graph sizes (GPS, 4 heads).} Each panel shows the transferability gap as a function of graph size for the corresponding dataset.}
\label{fig:transferability_gps_h4_2}
\end{figure*}

\begin{figure*}[ht!]
\centering
% Row 1
\begin{subfigure}[t]{0.32\textwidth}
  \centering\includegraphics[width=\linewidth]{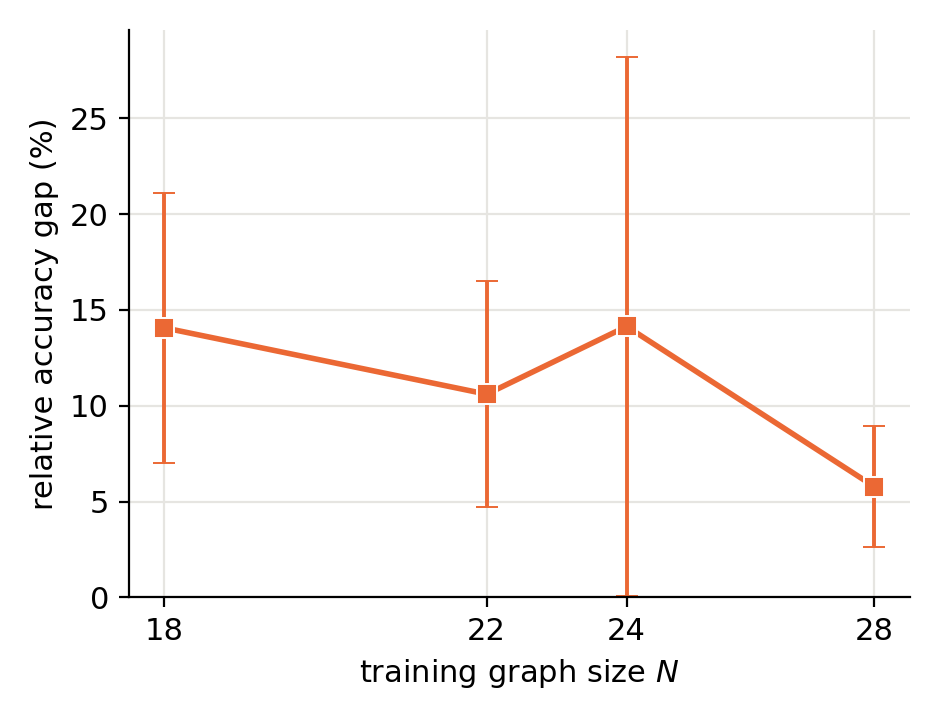}
  \caption{\centering{\textbf{MUTAG}}}
\end{subfigure}
\begin{subfigure}[t]{0.32\textwidth}
  \centering\includegraphics[width=\linewidth]{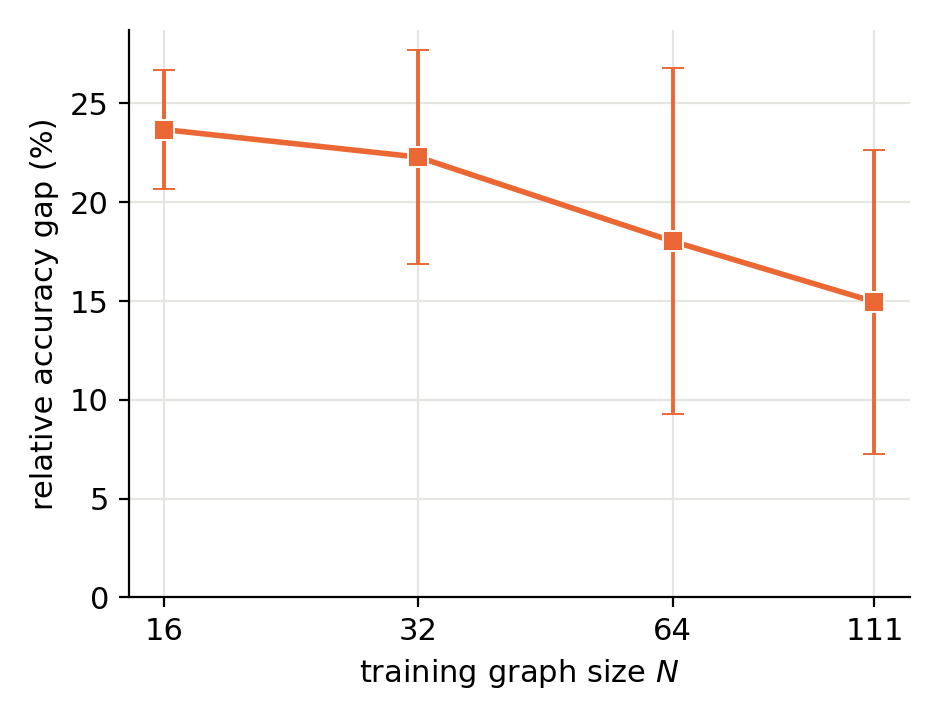}
  \caption{\centering{\textbf{NCI1}}}
\end{subfigure}

\vspace{0.3em}

% Row 2
\begin{subfigure}[t]{0.32\textwidth}
  \centering\includegraphics[width=\linewidth]{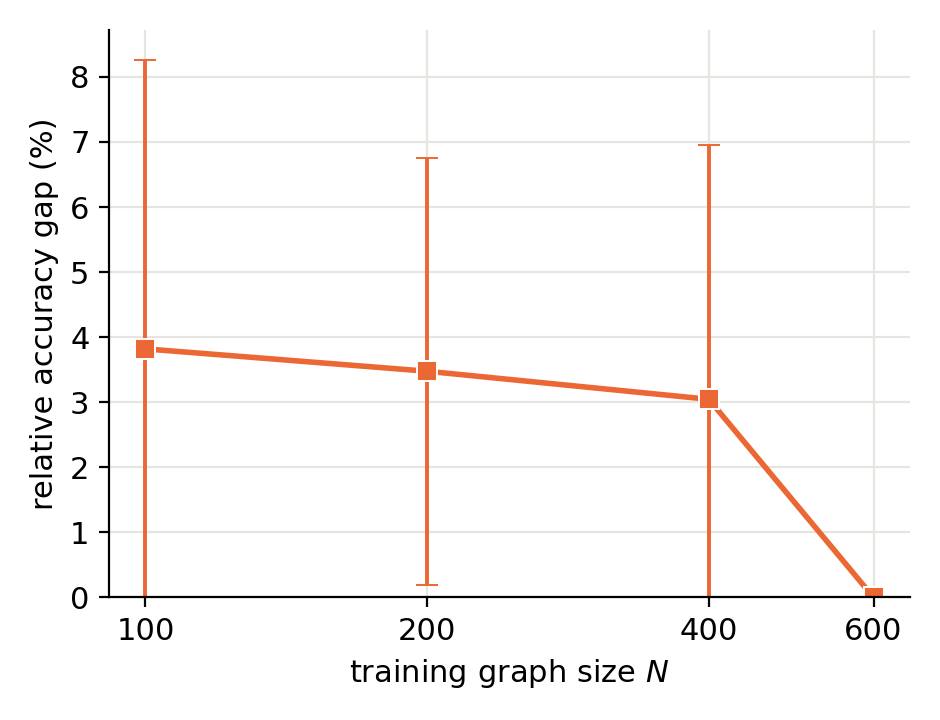}
  \caption{\centering{\textbf{PROTEINS}}}
\end{subfigure}
\begin{subfigure}[t]{0.32\textwidth}
  \centering\includegraphics[width=\linewidth]{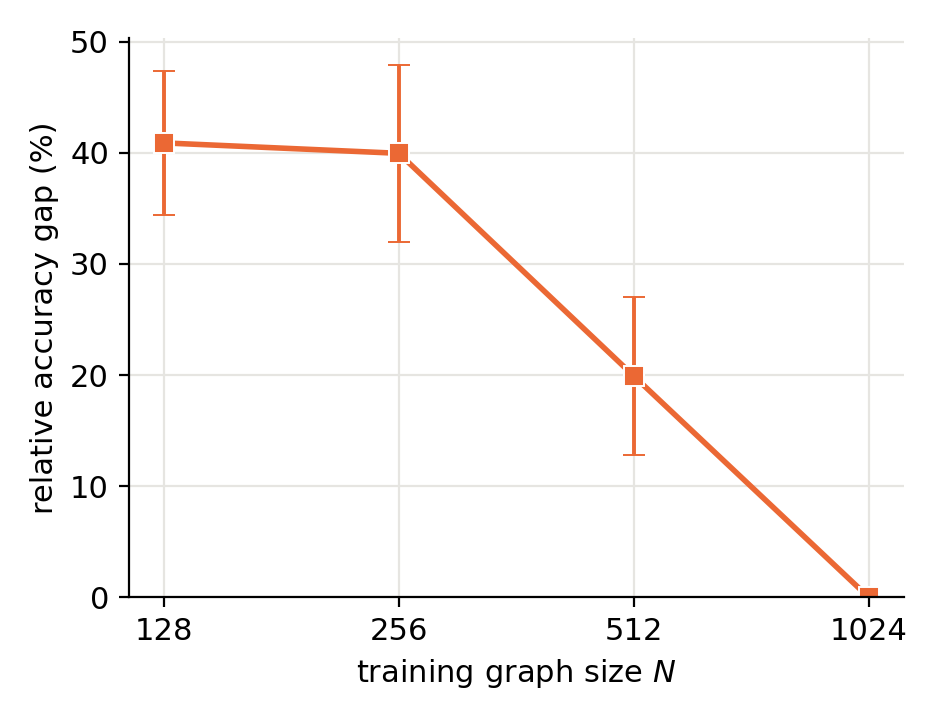}
  \caption{\centering{\textbf{REDDIT-MULTI-5K}}}
\end{subfigure}

\caption{\textbf{Transferability gap across graph sizes (Graphormer-GD, single-head).} Each panel shows the transferability gap as a function of graph size for the corresponding dataset.}
\label{fig:transferability_graphormer_h1}
\end{figure*}

\begin{figure*}[ht!]
\centering
% Row 1
\begin{subfigure}[t]{0.32\textwidth}
  \centering\includegraphics[width=\linewidth]{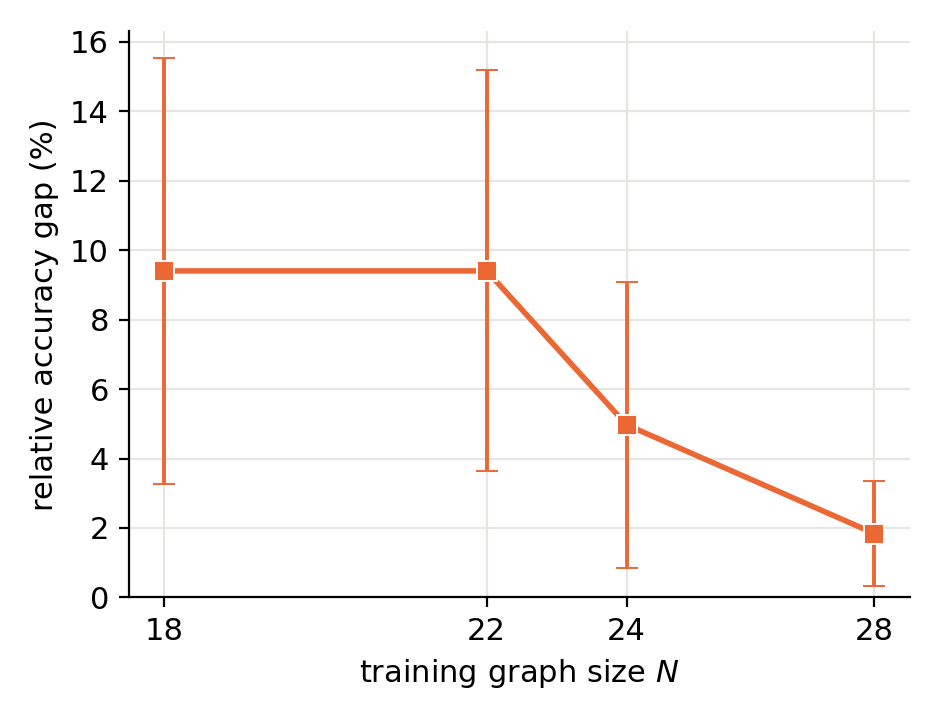}
  \caption{\centering{\textbf{MUTAG}}}
\end{subfigure}
\begin{subfigure}[t]{0.32\textwidth}
  \centering\includegraphics[width=\linewidth]{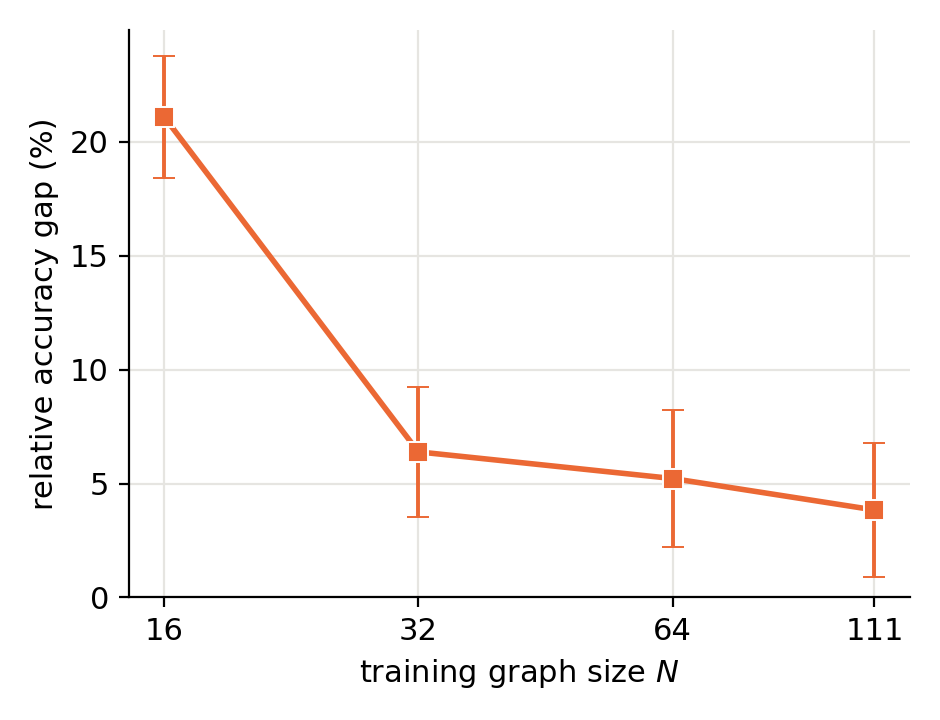}
  \caption{\centering{\textbf{NCI1}}}
\end{subfigure}

\vspace{0.3em}

% Row 2
\begin{subfigure}[t]{0.32\textwidth}
  \centering\includegraphics[width=\linewidth]{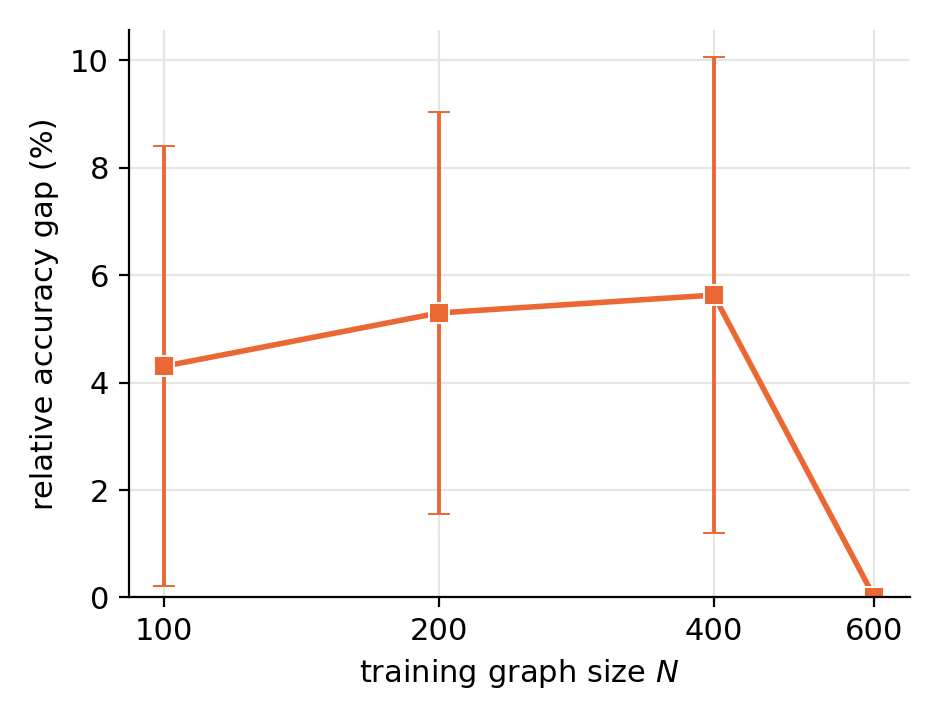}
  \caption{\centering{\textbf{PROTEINS}}}
\end{subfigure}
\begin{subfigure}[t]{0.32\textwidth}
  \centering\includegraphics[width=\linewidth]{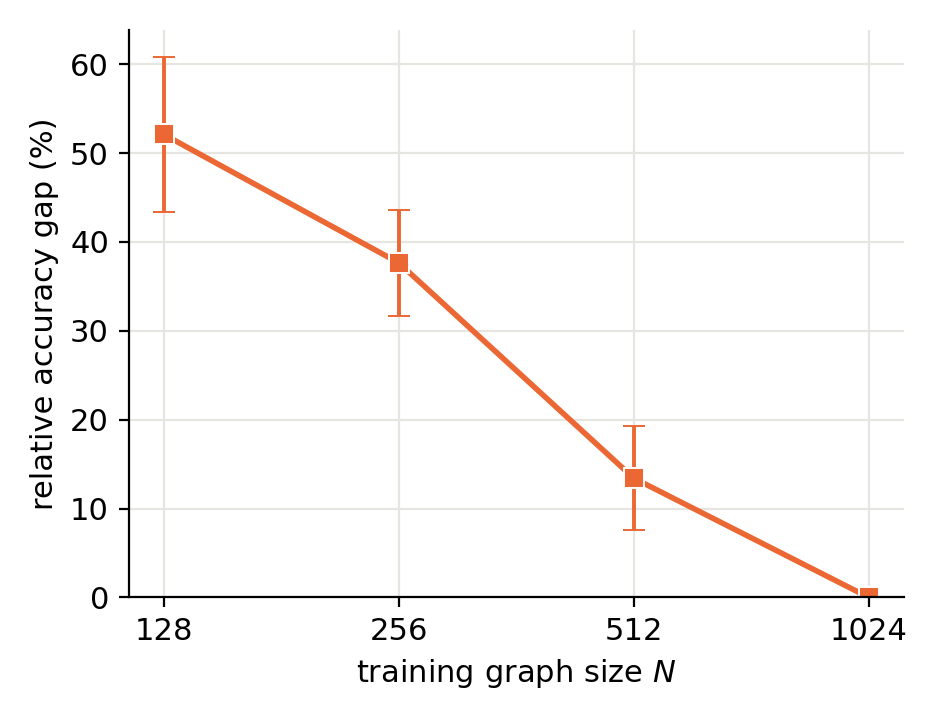}
  \caption{\centering{\textbf{REDDIT-MULTI-5K}}}
\end{subfigure}

\caption{\textbf{Transferability gap across graph sizes (GRIT, single-head).} Each panel shows the transferability gap as a function of graph size for the corresponding dataset.}
\label{fig:transferability_grit_h1}
\end{figure*}

\begin{figure*}[ht!]
\centering
% Row 1
\begin{subfigure}[t]{0.24\textwidth}
  \centering\includegraphics[width=\linewidth]{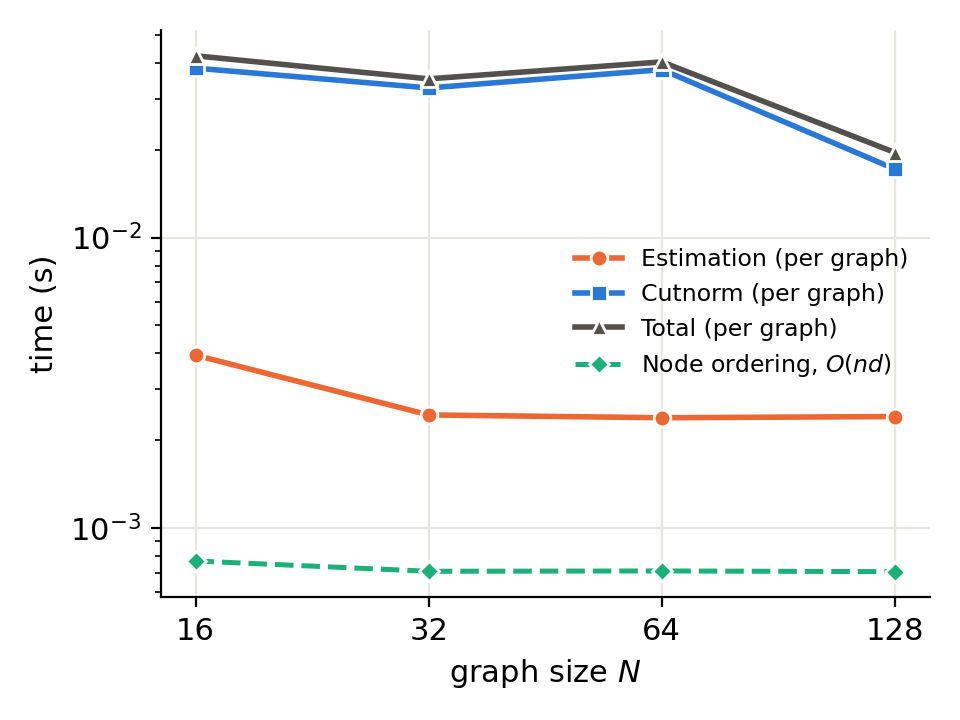}
  \caption{\centering{\textbf{BFS}}}
\end{subfigure}
\begin{subfigure}[t]{0.24\textwidth}
  \centering\includegraphics[width=\linewidth]{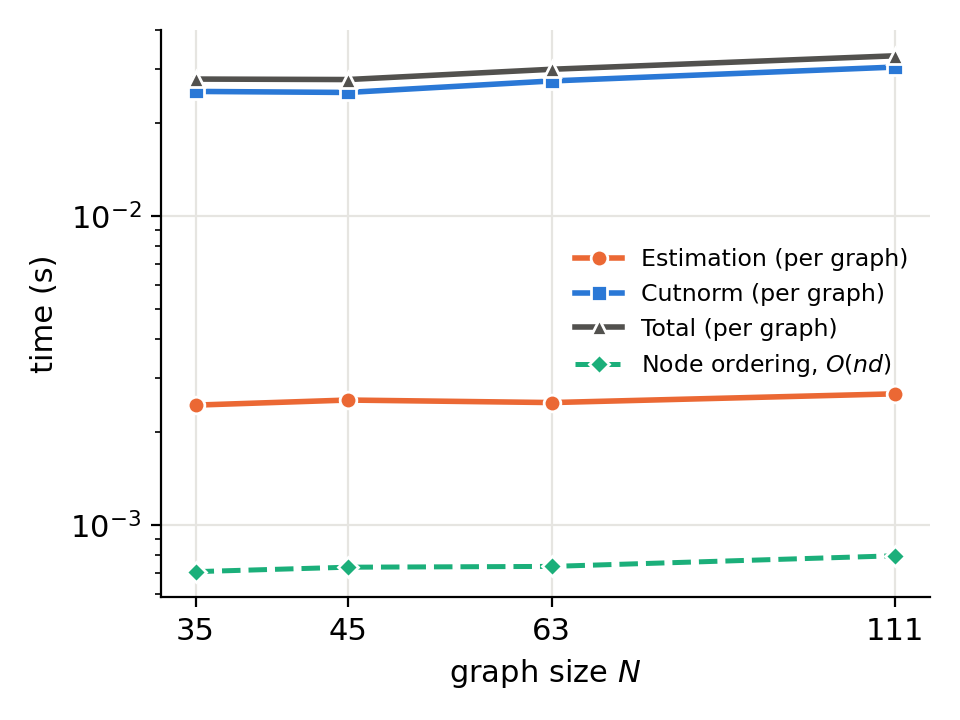}
  \caption{\centering{\textbf{COLLAB}}}
\end{subfigure}
\begin{subfigure}[t]{0.24\textwidth}
  \centering\includegraphics[width=\linewidth]{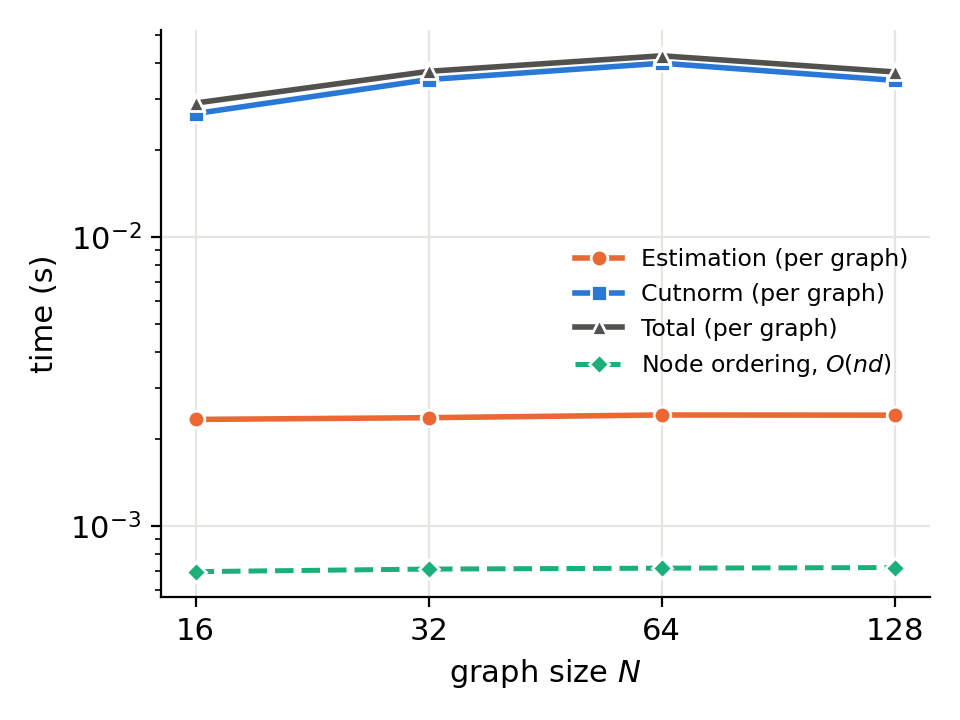}
  \caption{\centering{\textbf{Dijkstra}}}
\end{subfigure}
\begin{subfigure}[t]{0.24\textwidth}
  \centering\includegraphics[width=\linewidth]{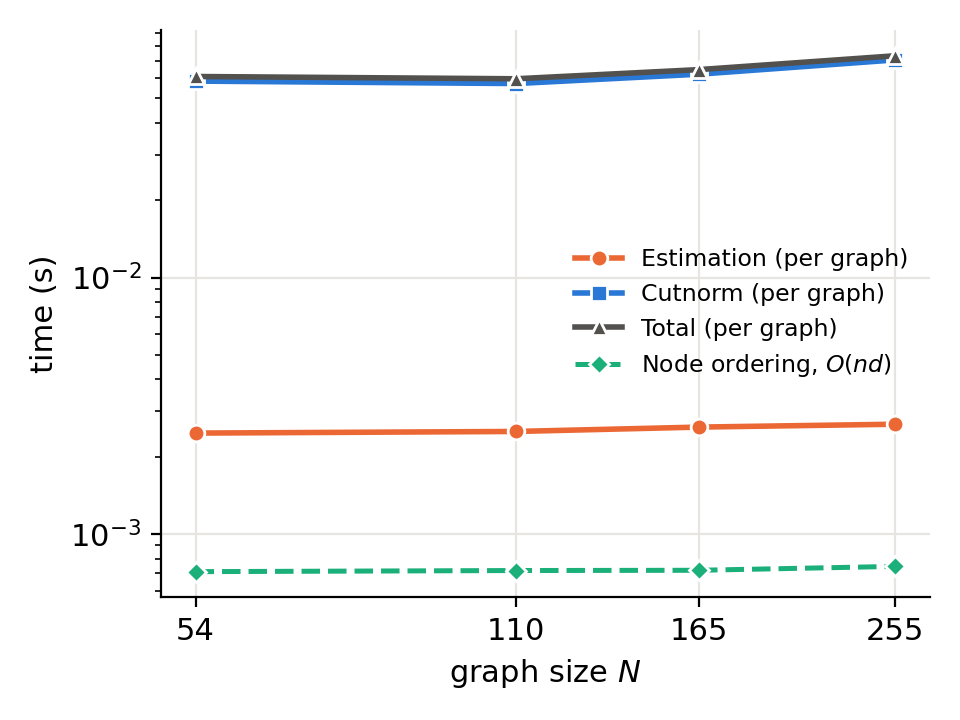}
  \caption{\centering{\textbf{LRGBPeptides}}}
\end{subfigure}

\vspace{0.3em}

% Row 2
\begin{subfigure}[t]{0.24\textwidth}
  \centering\includegraphics[width=\linewidth]{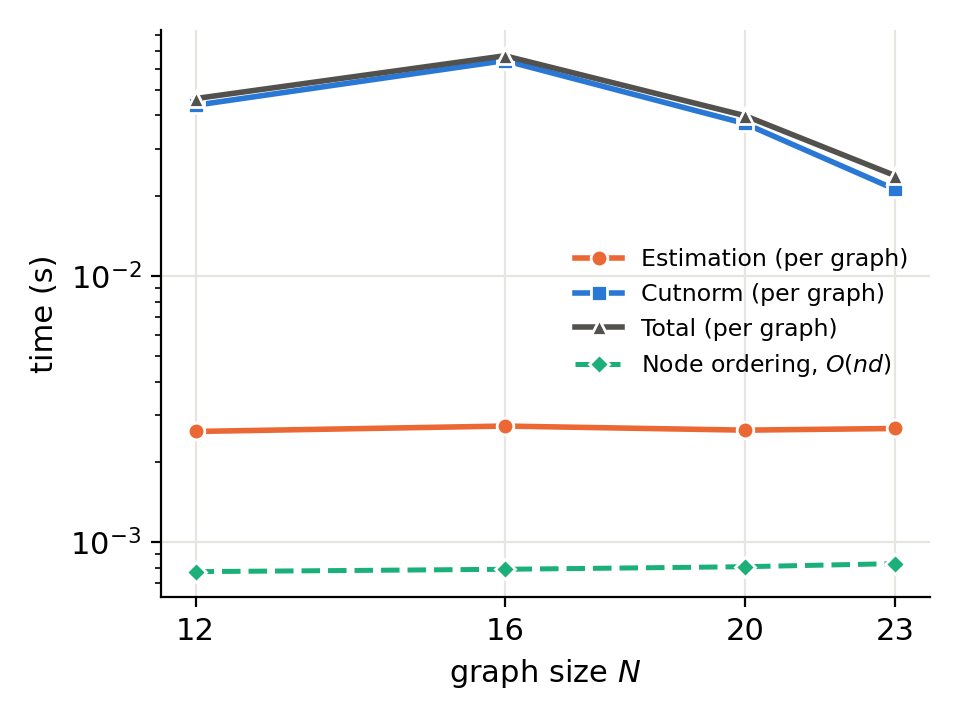}
  \caption{\centering{\textbf{MUTAG}}}
\end{subfigure}
\begin{subfigure}[t]{0.24\textwidth}
  \centering\includegraphics[width=\linewidth]{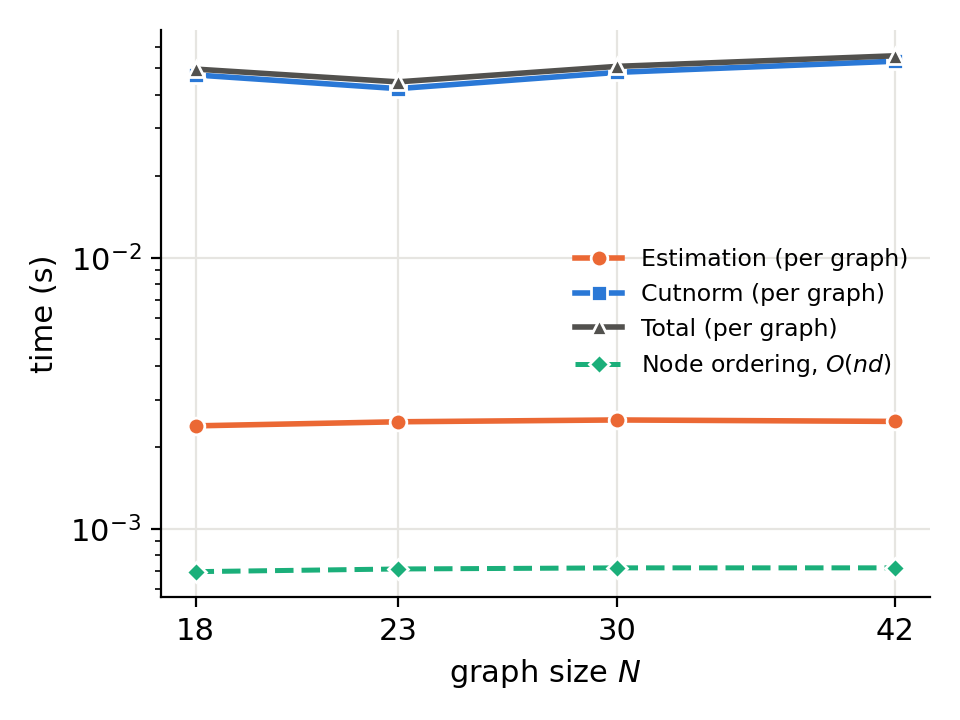}
  \caption{\centering{\textbf{NCI1}}}
\end{subfigure}
\begin{subfigure}[t]{0.24\textwidth}
  \centering\includegraphics[width=\linewidth]{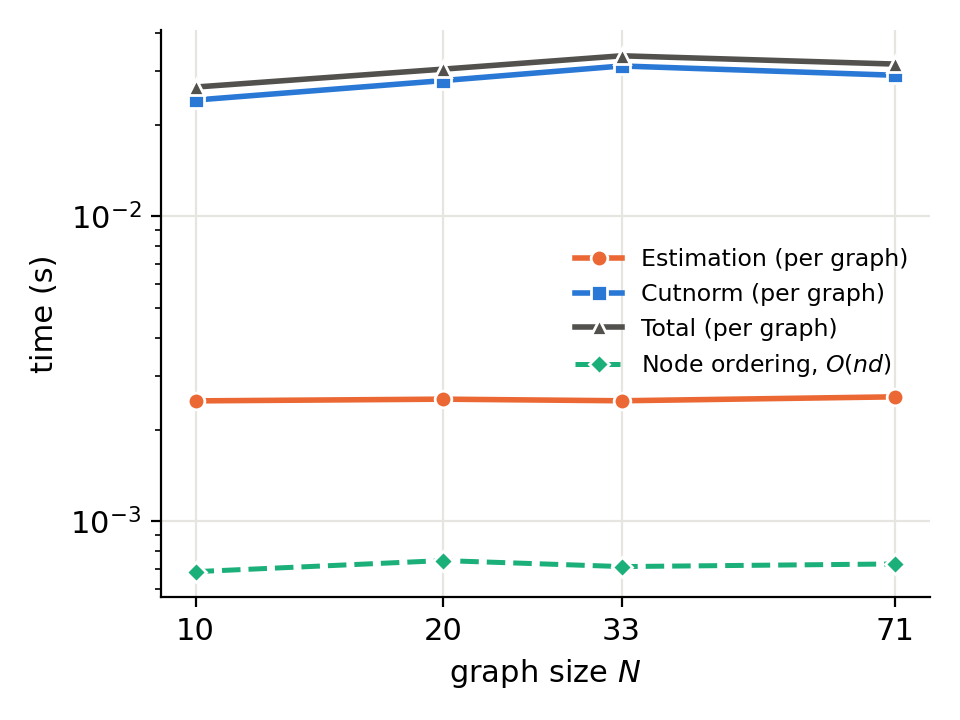}
  \caption{\centering{\textbf{PROTEINS}}}
\end{subfigure}

% \caption{\textbf{Runtime analysis of the graphon estimation pipeline.} Each panel shows the per-graph runtime (in seconds, log scale) as a function of graph size $n$ for the corresponding dataset. The runtime is decomposed into: \blue{Sort+smooth}, which includes degree-based canonicalization and block averaging; {\color{orange}Cutnorm}, the approximate cut-norm computation; and \green{Total}. The Sort+smooth step remains nearly constant across graph sizes, while the cut-norm computation dominates the total runtime. Overall, the pipeline processes each graph in milliseconds, demonstrating computational efficiency even for larger graphs.}
\caption{\textbf{Runtime analysis of the graphon estimation pipeline.} Each panel shows the per-graph runtime (in seconds, log scale) as a function of graph size $n$ for the corresponding dataset. The runtime is decomposed into \orange{Estimation}, which computes the attention graph, orders its nodes, and computes the block averages, and \textcolor[HTML]{2A78D6}{Cutnorm}, the approximate cut-norm computation, which add up to the \textcolor{darkgray}{Total}. The \textcolor[HTML]{1BAF7A}{Node ordering} curve times the $\ccalO(nd)$ computation of the ordering statistic from queries and keys, without forming the $n\times n$ matrix. Estimation remains nearly constant across graph sizes, while the cut-norm computation dominates the total runtime. Overall, the pipeline processes each graph in tens of milliseconds, even for larger graphs.}
\label{fig:runtime}
\end{figure*}

\begin{figure*}[ht!]
\centering
% Row 1: BFS
\begin{subfigure}[t]{0.24\textwidth}
  \centering\includegraphics[width=\linewidth]{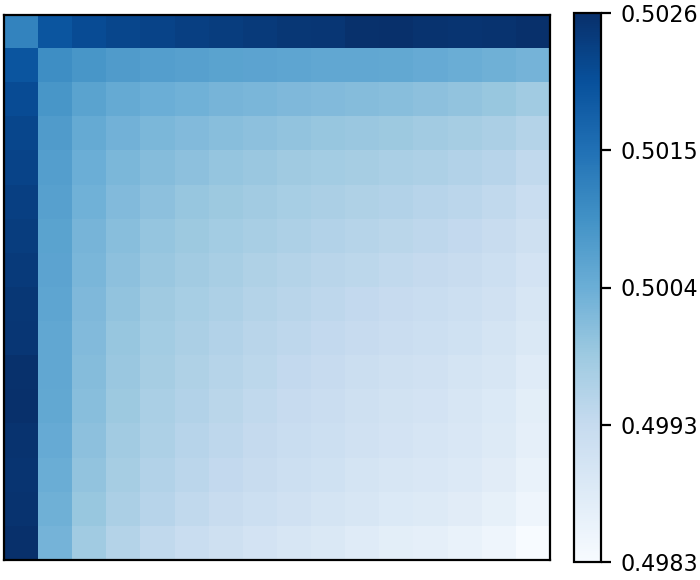}
\end{subfigure}
\begin{subfigure}[t]{0.24\textwidth}
  \centering\includegraphics[width=\linewidth]{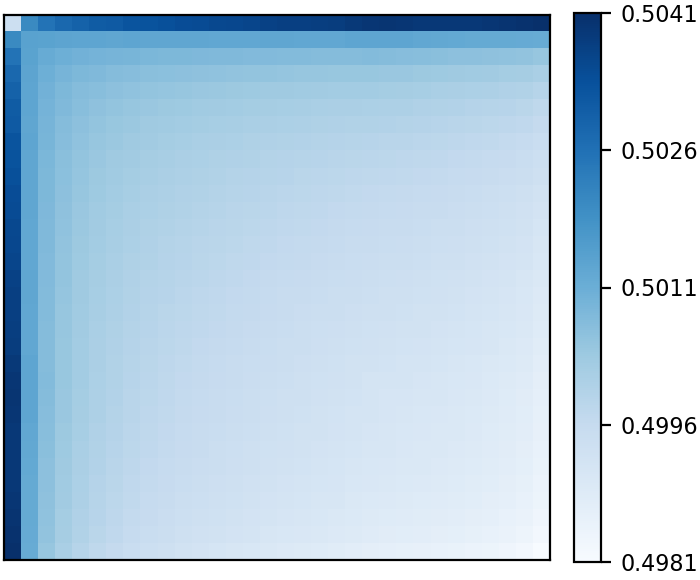}
\end{subfigure}
\begin{subfigure}[t]{0.24\textwidth}
  \centering\includegraphics[width=\linewidth]{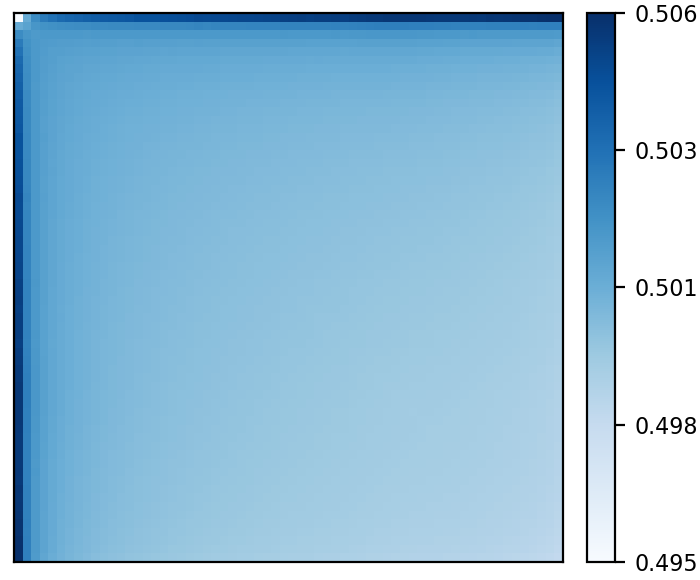}
\end{subfigure}
\begin{subfigure}[t]{0.24\textwidth}
  \centering\includegraphics[width=\linewidth]{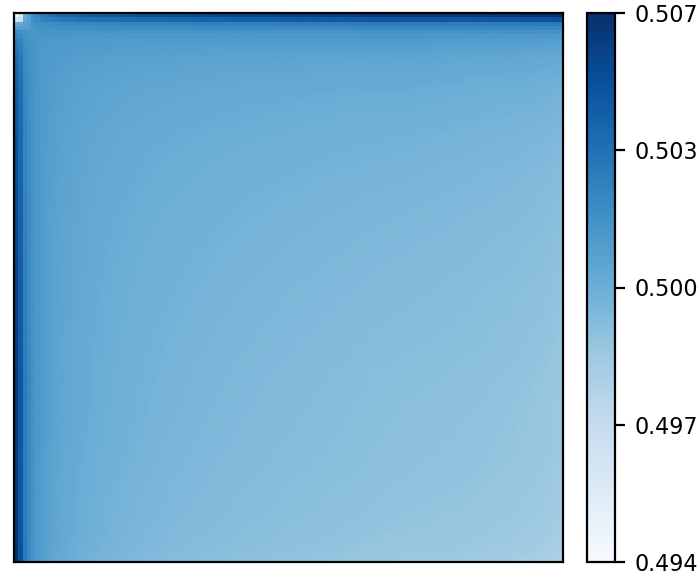}
\end{subfigure}

\vspace{0.3em}

% Row 2: COLLAB
\begin{subfigure}[t]{0.24\textwidth}
  \centering\includegraphics[width=\linewidth]{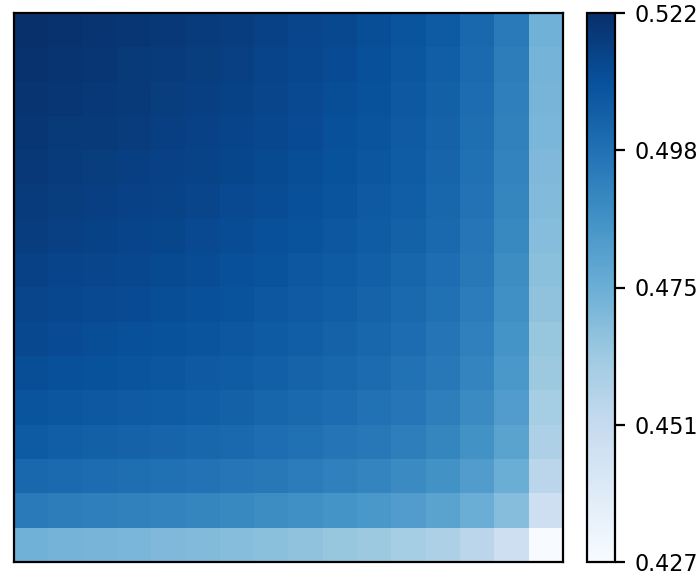}
\end{subfigure}
\begin{subfigure}[t]{0.24\textwidth}
  \centering\includegraphics[width=\linewidth]{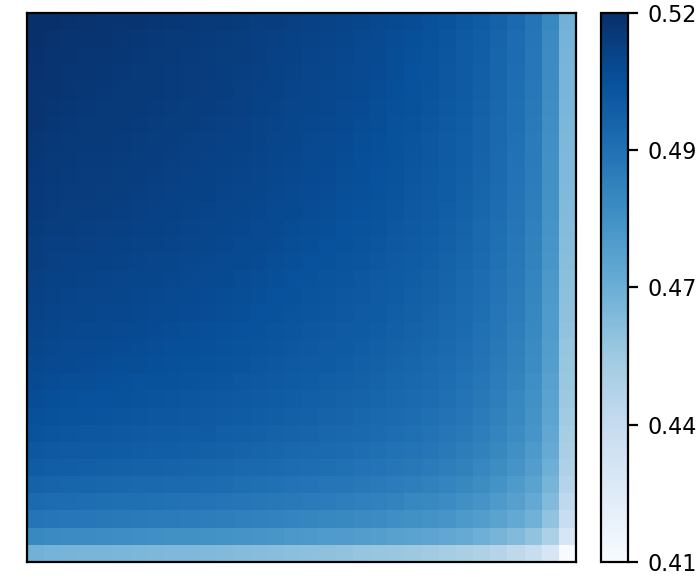}
\end{subfigure}
\begin{subfigure}[t]{0.24\textwidth}
  \centering\includegraphics[width=\linewidth]{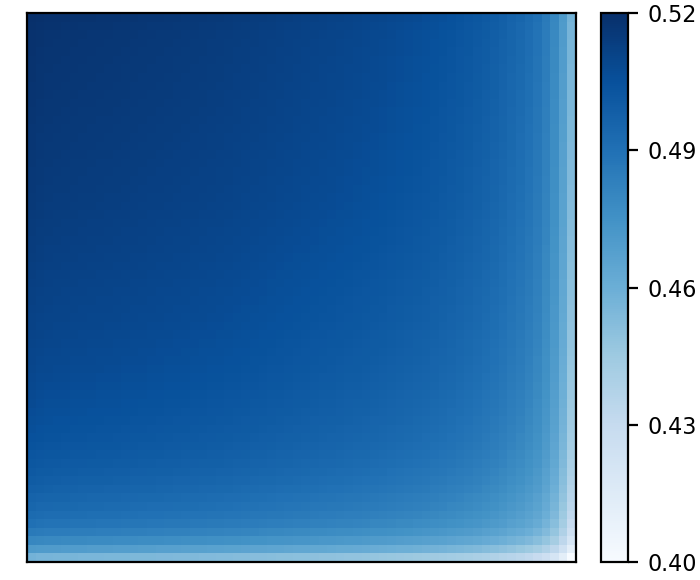}
\end{subfigure}
\begin{subfigure}[t]{0.24\textwidth}
  \centering\includegraphics[width=\linewidth]{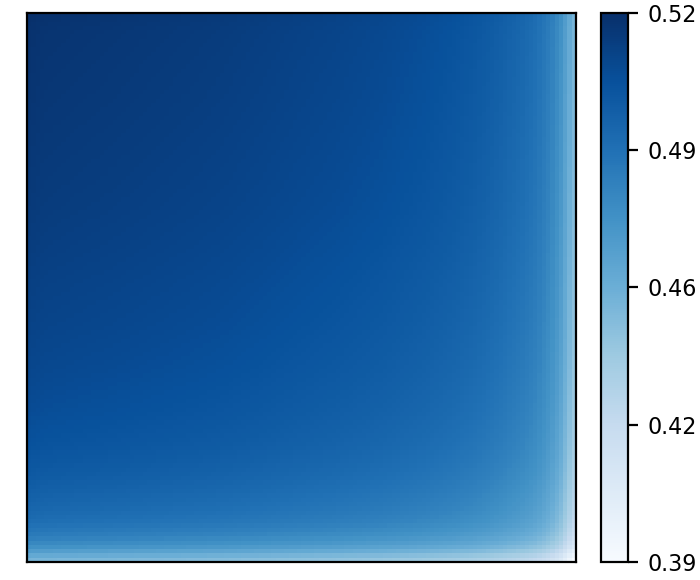}
\end{subfigure}

\vspace{0.3em}

% Row 3: Dijkstra
\begin{subfigure}[t]{0.24\textwidth}
  \centering\includegraphics[width=\linewidth]{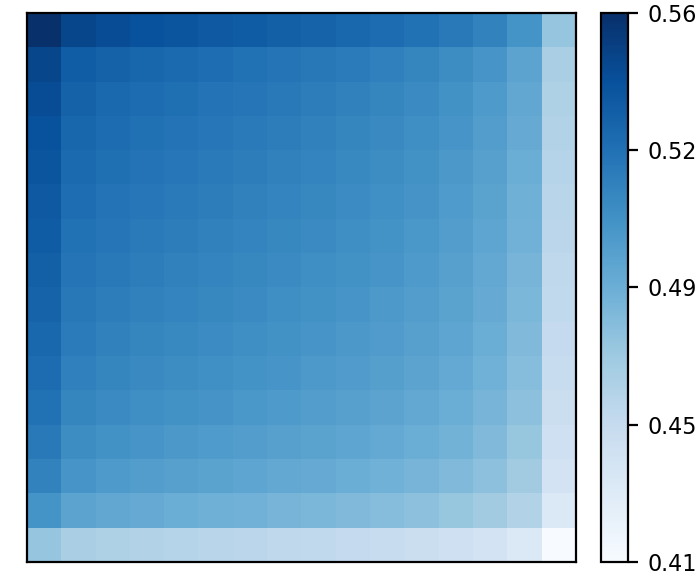}
\end{subfigure}
\begin{subfigure}[t]{0.24\textwidth}
  \centering\includegraphics[width=\linewidth]{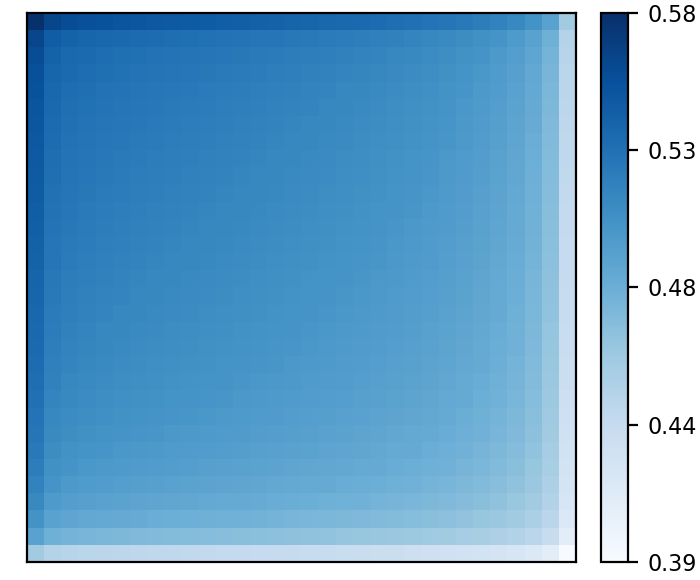}
\end{subfigure}
\begin{subfigure}[t]{0.24\textwidth}
  \centering\includegraphics[width=\linewidth]{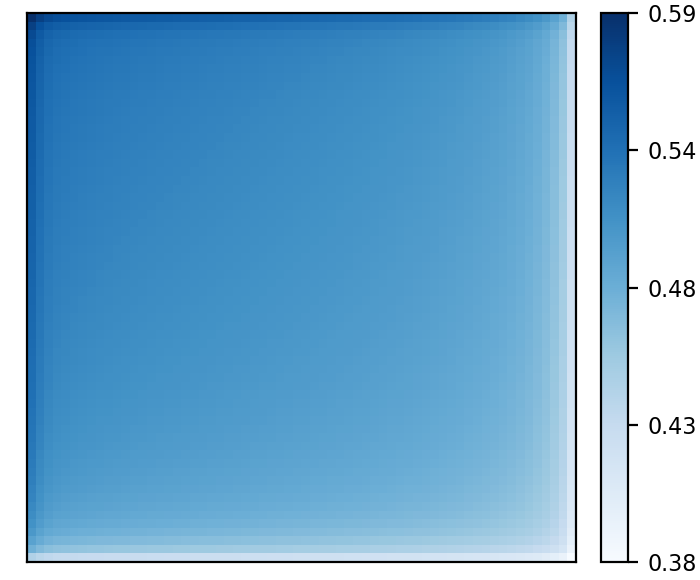}
\end{subfigure}
\begin{subfigure}[t]{0.24\textwidth}
  \centering\includegraphics[width=\linewidth]{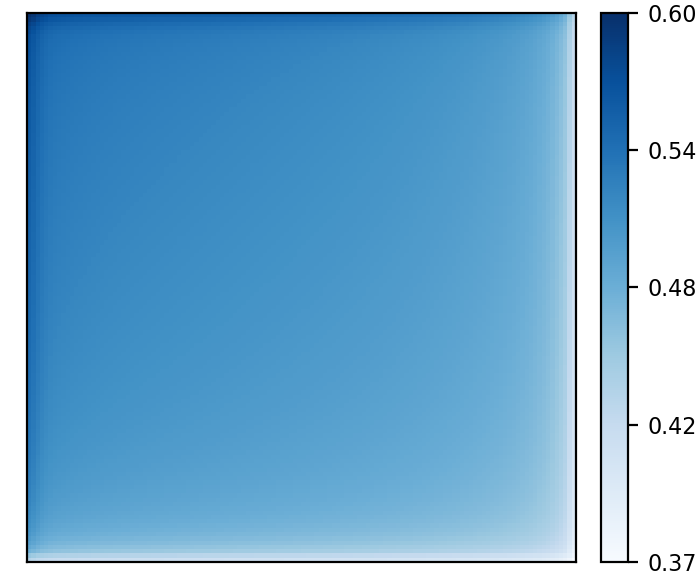}
\end{subfigure}

\vspace{0.3em}

% Row 4: LRGBPeptides
\begin{subfigure}[t]{0.24\textwidth}
  \centering\includegraphics[width=\linewidth]{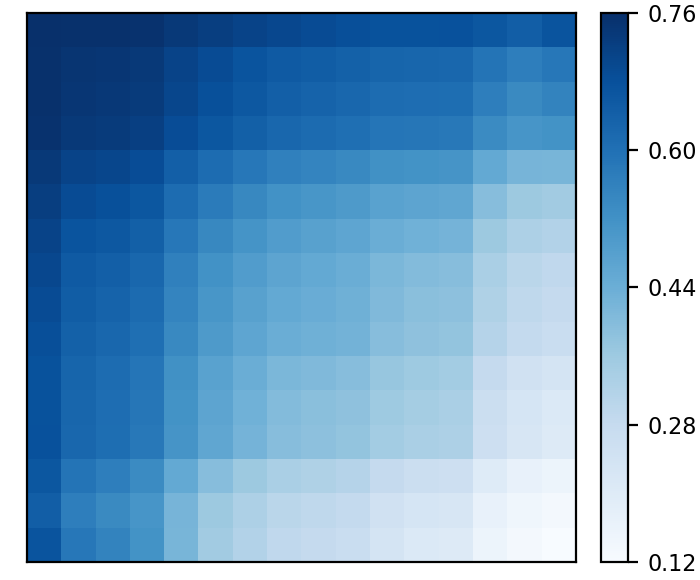}
\end{subfigure}
\begin{subfigure}[t]{0.24\textwidth}
  \centering\includegraphics[width=\linewidth]{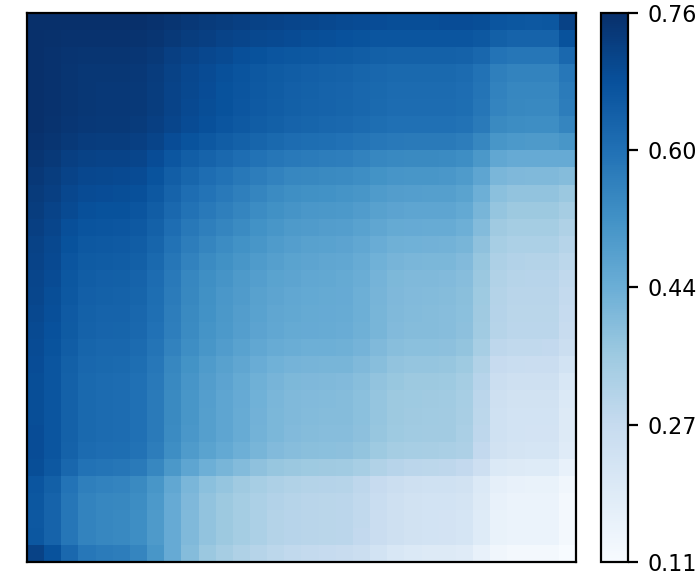}
\end{subfigure}
\begin{subfigure}[t]{0.24\textwidth}
  \centering\includegraphics[width=\linewidth]{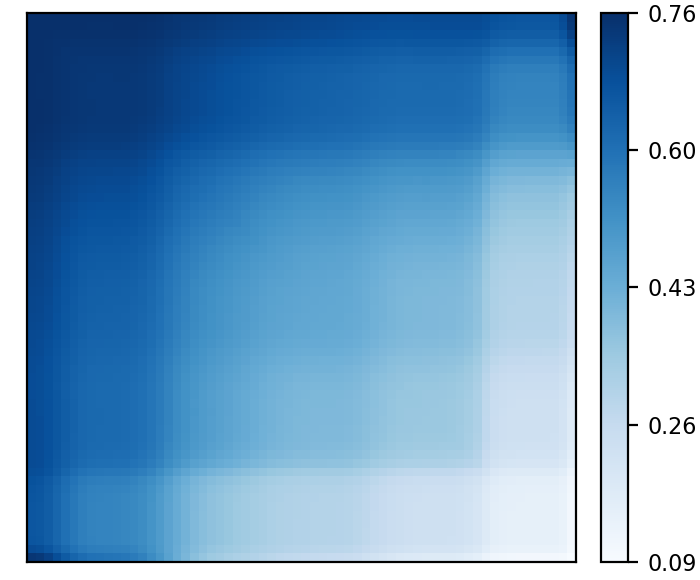}
\end{subfigure}
\begin{subfigure}[t]{0.24\textwidth}
  \centering\includegraphics[width=\linewidth]{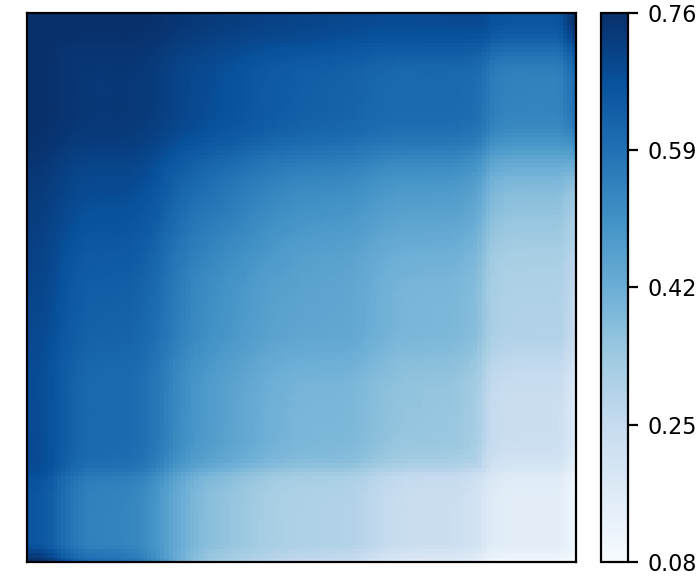}
\end{subfigure}

\vspace{0.3em}

% Row 5: PROTEINS
\begin{subfigure}[t]{0.24\textwidth}
  \centering\includegraphics[width=\linewidth]{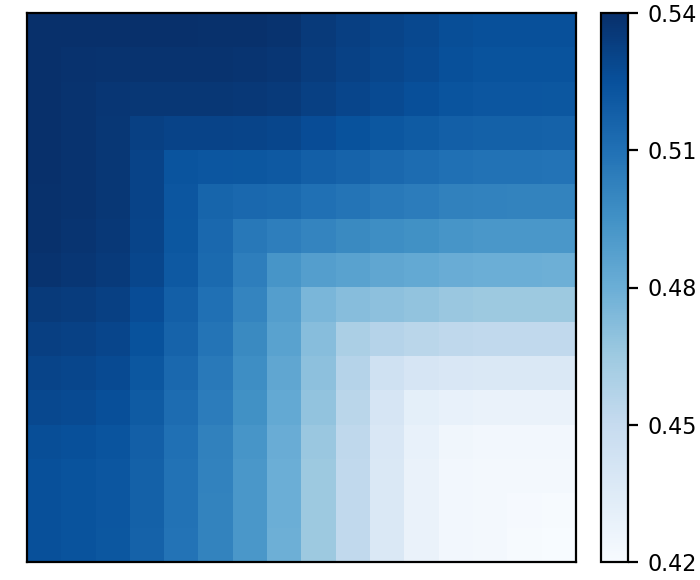}
\end{subfigure}
\begin{subfigure}[t]{0.24\textwidth}
  \centering\includegraphics[width=\linewidth]{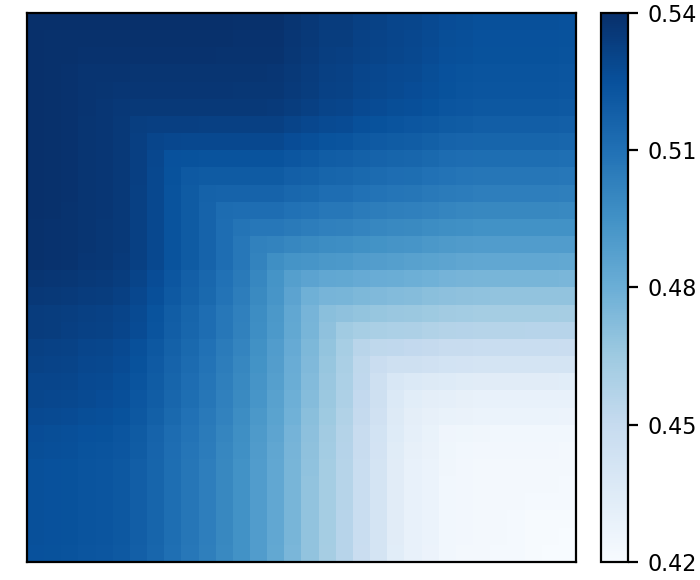}
\end{subfigure}
\begin{subfigure}[t]{0.24\textwidth}
  \centering\includegraphics[width=\linewidth]{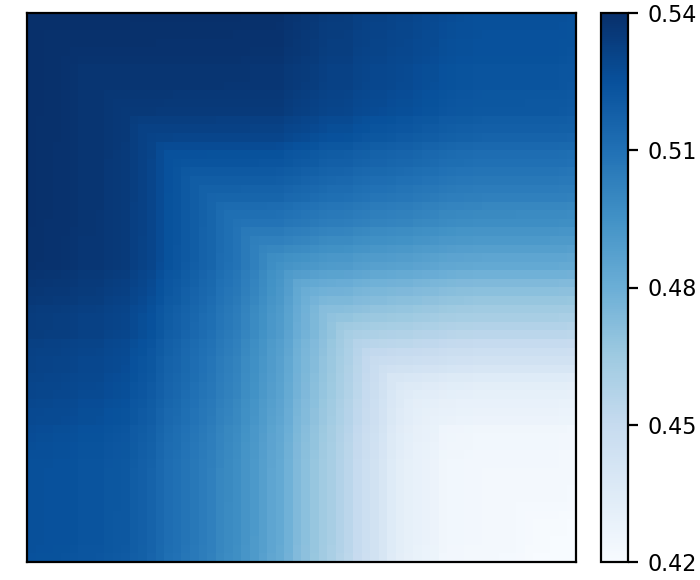}
\end{subfigure}
\begin{subfigure}[t]{0.24\textwidth}
  \centering\includegraphics[width=\linewidth]{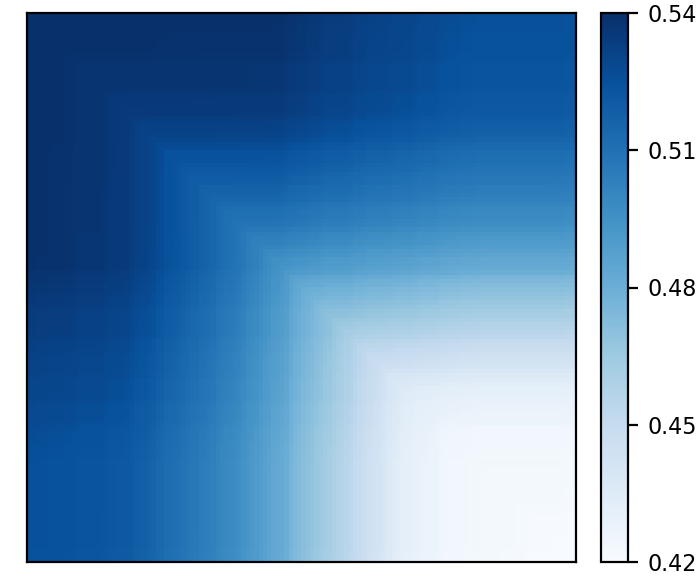}
\end{subfigure}

\caption{\textbf{Sensitivity of graphon estimation to block size $K$.} Each row corresponds to a dataset (BFS, COLLAB, Dijkstra, LRGBPeptides, PROTEINS), and each column shows the estimated graphon for increasing block sizes ($K \in \{16, 32, 64, 128\}$). Larger values of $K$ yield finer-grained estimates. The estimated graphons remain qualitatively stable across different values of $K$, demonstrating robustness of the estimation procedure to this hyperparameter choice.}
\label{fig:k_sensitivity_graphon}
\end{figure*}

\begin{figure*}[ht!]
\centering
\begin{subfigure}[t]{0.45\textwidth}
  \centering\includegraphics[width=\linewidth]{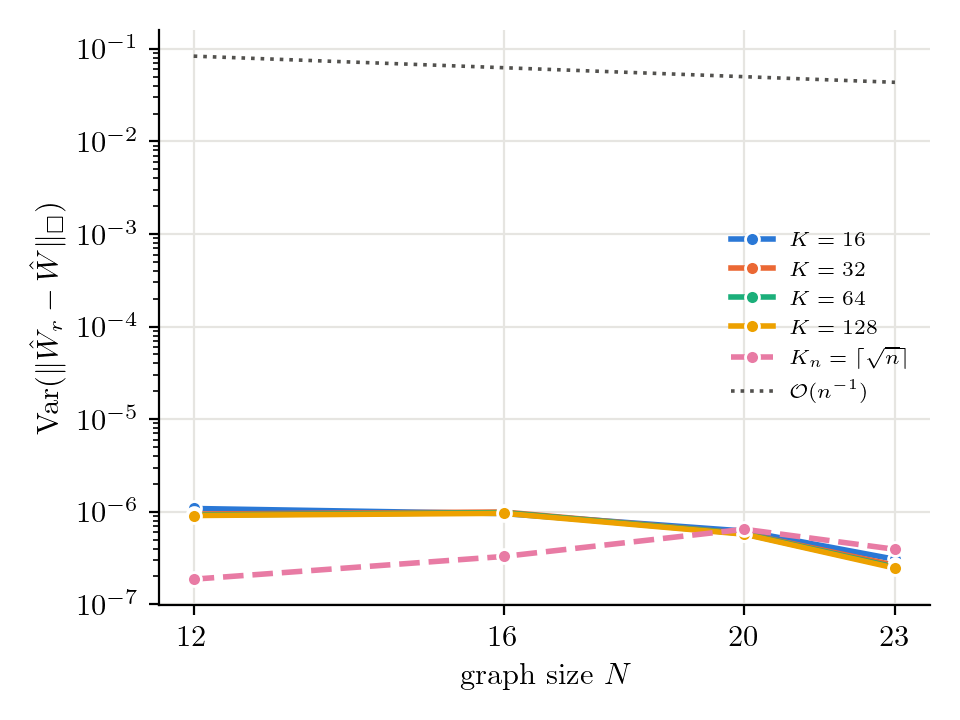}
  \caption{\centering{\textbf{MUTAG}}}
\end{subfigure}
\begin{subfigure}[t]{0.45\textwidth}
  \centering\includegraphics[width=\linewidth]{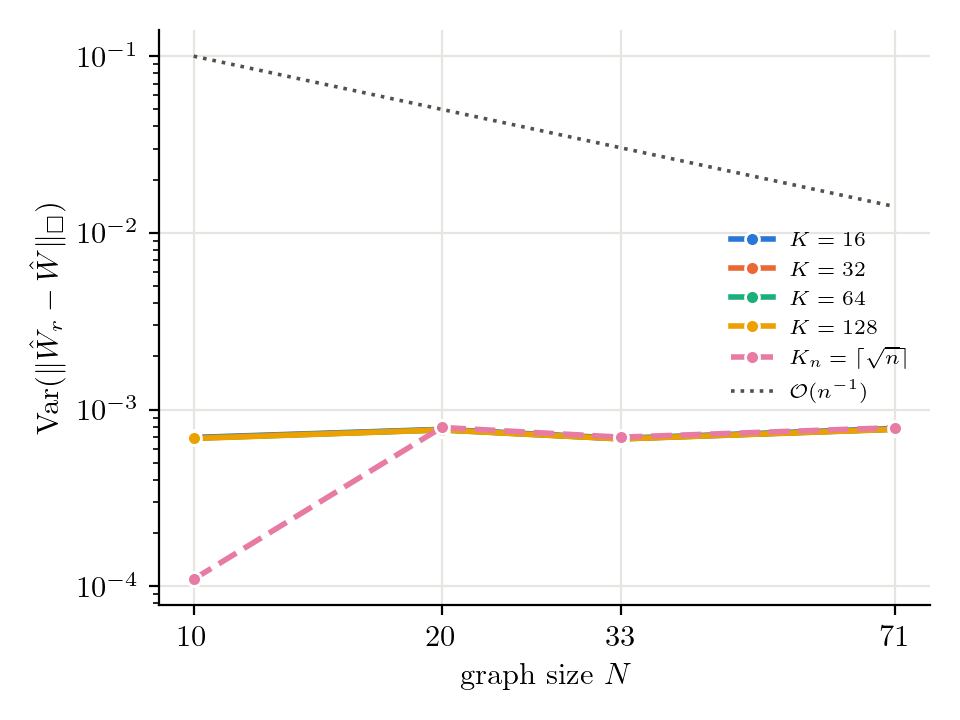}
  \caption{\centering{\textbf{PROTEINS}}}
\end{subfigure}

% \caption{\textbf{Sensitivity of hypothesis test to block size $K$.} Each panel shows the empirical variance of cut-distance as a function of graph size $n$ for different values of $K \in \{16, 32, 64, 128\}$, compared against the \green{theoretical bound}. Larger values of $K$ tend to produce lower variance estimates, particularly for datasets with smaller graphs (e.g., MUTAG), as the block-averaging procedure introduces zero-valued regions when $K$ exceeds the graph size. For MUTAG, even with $K=16$, the empirical variance remains below the theoretical bound. Our choice of $K=64$ balances resolution and robustness, since most graphs in our experiments have sizes exceeding this threshold.}
\caption{\textbf{Sensitivity of hypothesis test to block size $K$.} Each panel shows the empirical variance of the cut-norm as a function of graph size $n$ for $K\in\{16,32,64,128\}$ and for the growing resolution $K_n=\lceil\sqrt n\rceil$ (dashed), compared against the \textcolor{darkgray}{scaling proxy} $\ccalO(n^{-1})$ (dotted). The curves for fixed $K$ nearly coincide and lie well below the proxy, and the growing resolution gives a lower variance only at the smallest sizes, where it uses few blocks. The results therefore do not depend on the default $K=64$.}
\label{fig:k_sensitivity_hypothesis}
\end{figure*}

\begin{comment}
\begin{figure*}[ht!]
\centering
\includegraphics[width=\linewidth]{Figures/Rebuttal/framework.png}
\caption{\textbf{Visual explanation of the cut distance $\delta_\square$ and graphon computation.} \textit{Top (cut distance)}: The cut distance roughly measures the largest discrepancy over node subsets $S, T$. \textit{Bottom (graphon computation)}: From a dense, weighted attention graph with $n = 8$, the matrix is symmetrized, canonicalized via degree sorting, block-averaged, and finally used to estimate a dataset-level attention graphon.}
\label{fig:pipeline}
\end{figure*}
\end{comment}

\begin{figure*}[t]
\centering
\begin{subfigure}[t]{0.24\textwidth}
  \centering\includegraphics[width=\linewidth]{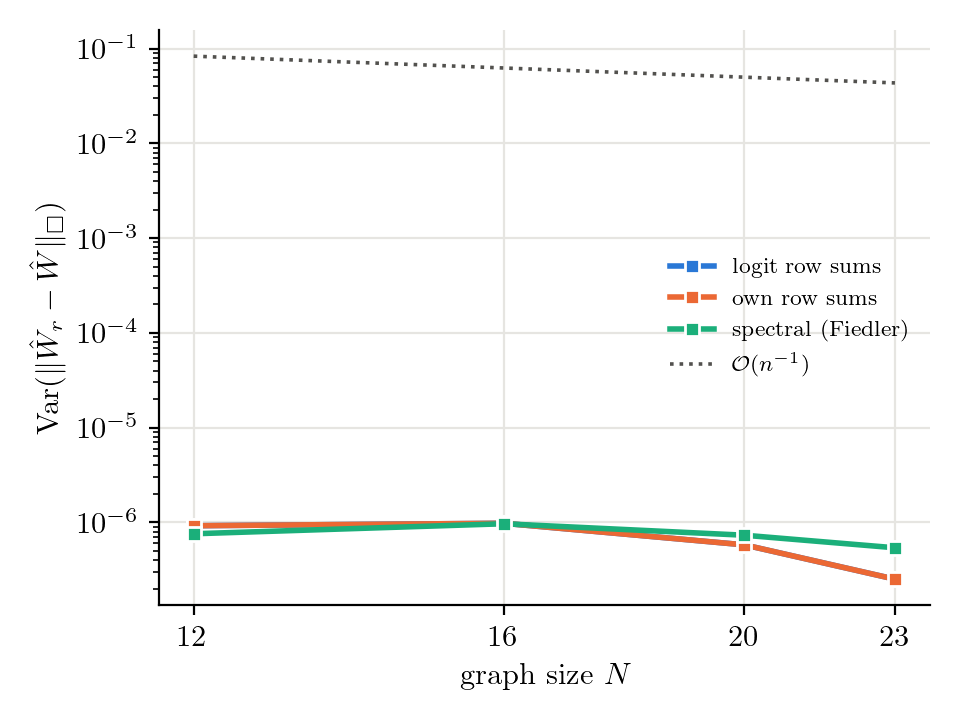}
  \caption{\centering{\textbf{MUTAG} %\\ ($N\in\{40,128,256,650\}$)
  }
  }
\end{subfigure}
\begin{subfigure}[t]{0.24\textwidth}
  \centering\includegraphics[width=\linewidth]{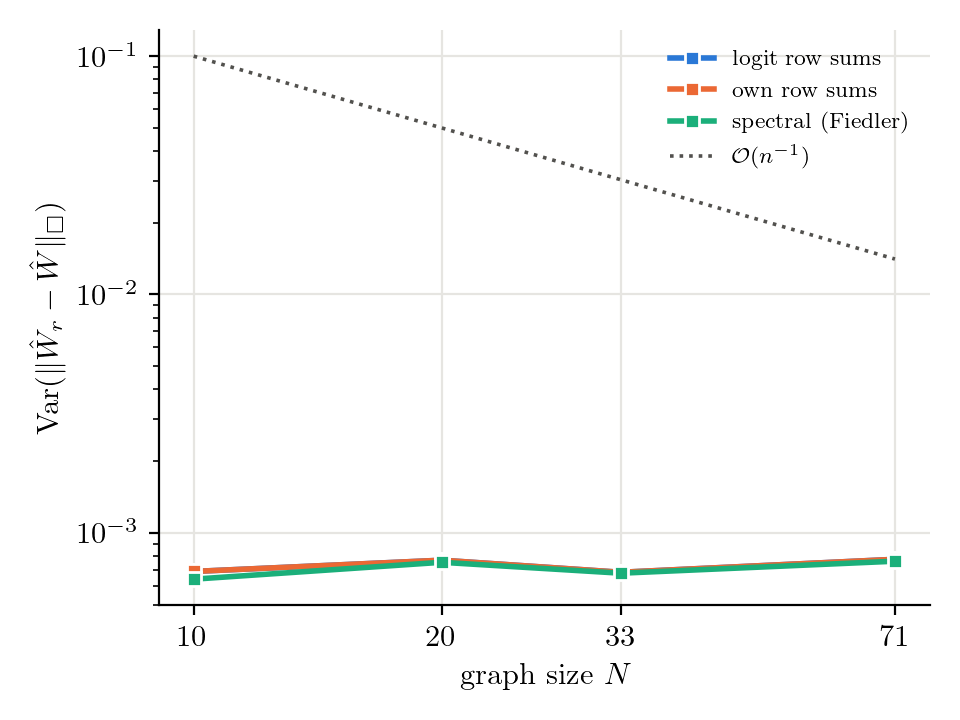}
  \caption{\centering{\textbf{PROTEINS}% \\ ($N\in\{128,256,512,1024\}$)
  }}
\end{subfigure}
\begin{subfigure}[t]{0.24\textwidth}
  \centering\includegraphics[width=\linewidth]{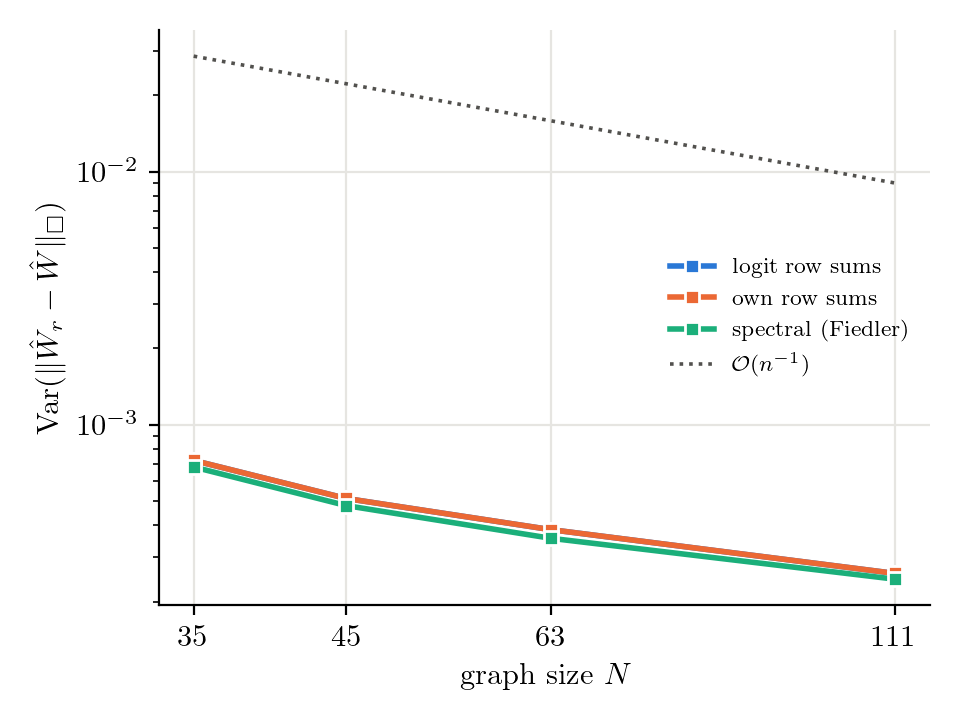}
  \caption{\centering{\textbf{COLLAB} %\\ ($N\in\{128,256,512,1024\}$)
  }}
\end{subfigure}
\begin{subfigure}[t]{0.24\textwidth}
  \centering\includegraphics[width=\linewidth]{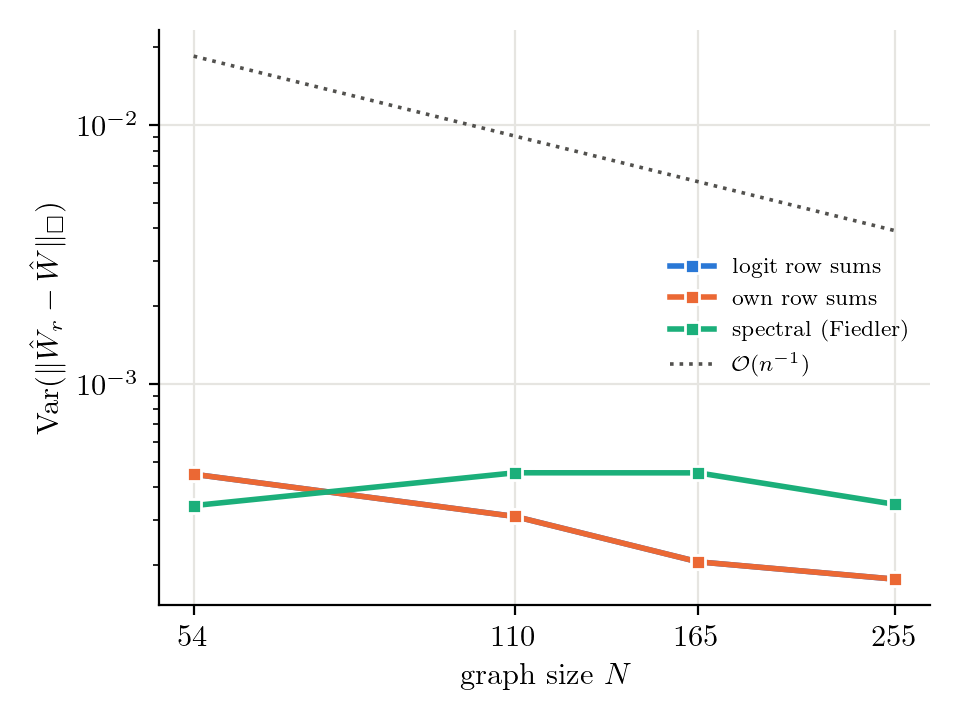}
  \caption{\centering{\textbf{LRGBPeptides} %\\ ($N\in\{128,256,512,1024\}$)
  }}
\end{subfigure}

\caption{\textbf{Sensitivity of hypothesis test to canonicalization method.} Each panel compares the empirical variance of cut-distance using two canonicalization methods: \blue{degree sorting} and \red{spectral sorting} (via Fiedler vector), against the \green{theoretical bound}. Despite differences in absolute values, the two methods exhibit nearly identical decay patterns across all datasets, with curves closely aligned in log-log scale. This demonstrates that our findings are robust to the choice of canonicalization method, and the observed variance decay reflects genuine properties of the learned attention rather than artifacts of a particular node ordering strategy.
}
\label{fig:sensitivity_canon_hypothesis}
%\vskip -0.1in
\end{figure*}

\begin{figure*}[ht!]
\centering
% Row 1
\begin{subfigure}[t]{0.45\textwidth}
  \centering\includegraphics[width=\linewidth]{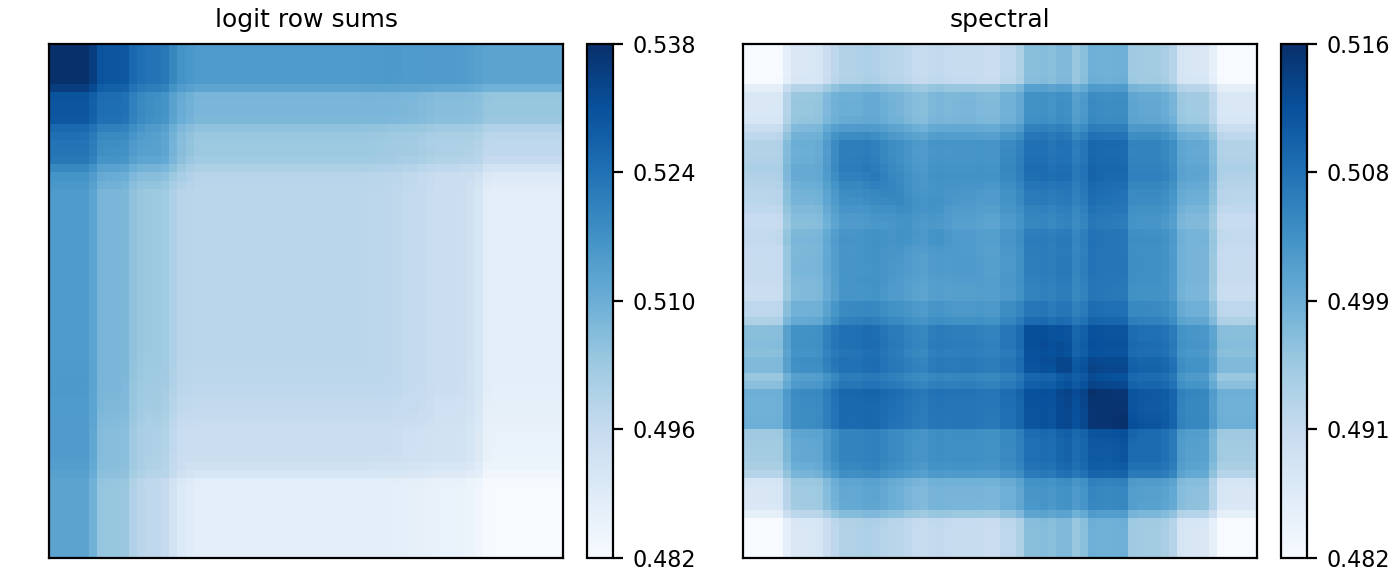}
  % \caption{\centering{\textbf{$N=18$}}}
  \caption{\centering{\textbf{$N=12$}}}
\end{subfigure}
\begin{subfigure}[t]{0.45\textwidth}
  \centering\includegraphics[width=\linewidth]{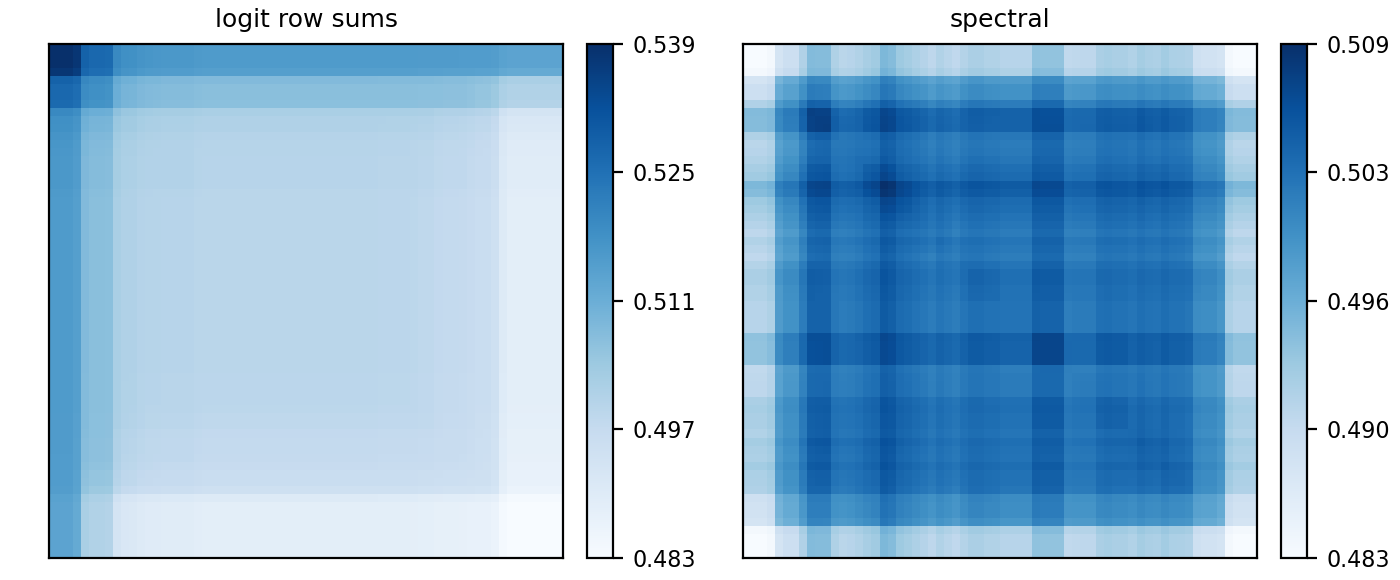}
  % \caption{\centering{\textbf{$N=22$}}}
  \caption{\centering{\textbf{$N=16$}}}
\end{subfigure}

\vspace{0.3em}

% Row 2
\begin{subfigure}[t]{0.45\textwidth}
  \centering\includegraphics[width=\linewidth]{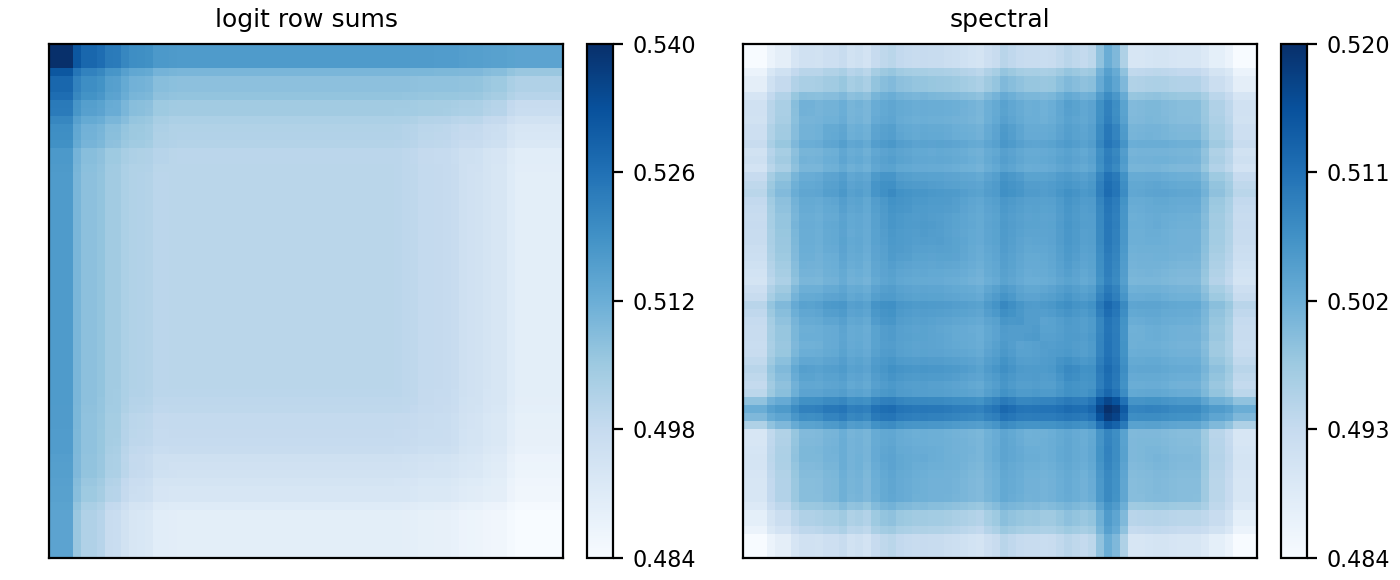}
  % \caption{\centering{\textbf{$N=24$}}}
  \caption{\centering{\textbf{$N=20$}}}
\end{subfigure}
\begin{subfigure}[t]{0.45\textwidth}
  \centering\includegraphics[width=\linewidth]{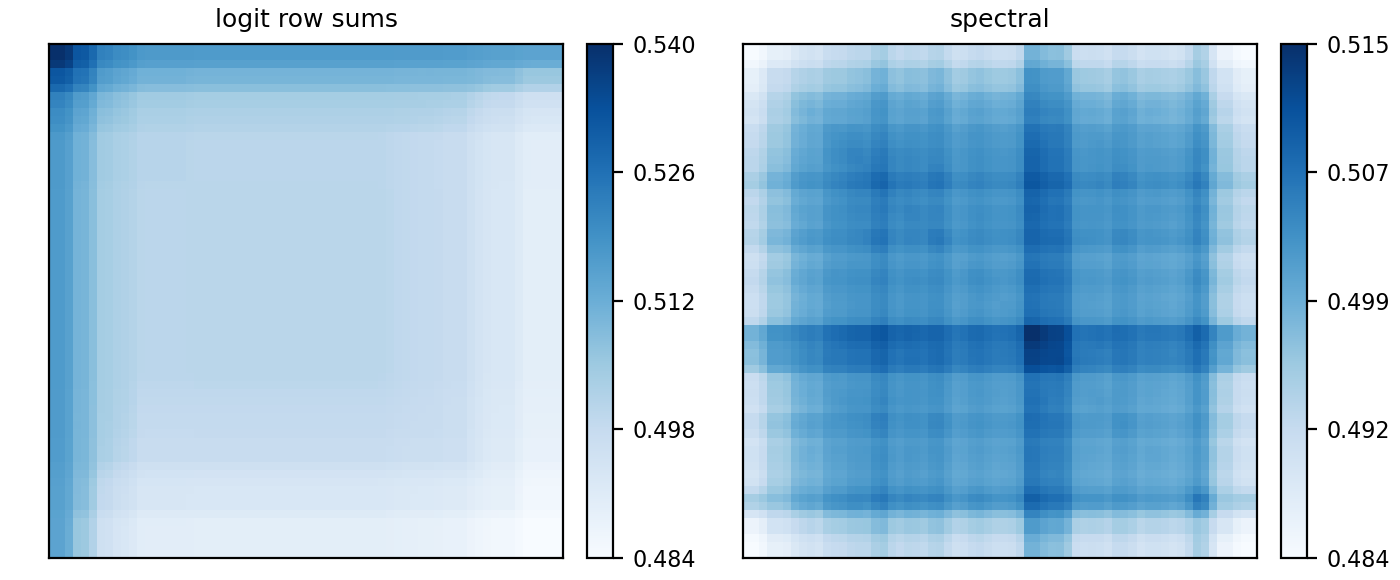}
  % \caption{\centering{\textbf{$N=28$}}}
  \caption{\centering{\textbf{$N=23$}}}
\end{subfigure}

\caption{\textbf{Sensitivity of hypothesis test to canonicalization method (MUTAG).} On each panel: on the \textbf{left} is the estimated graphon via degree sorting, while on the \textbf{right} via Fiedler vector sorting.}
\label{fig:sensitivity_canon_graphon_mutag}
\end{figure*}

\begin{figure*}[ht!]
\centering
% Row 1
\begin{subfigure}[t]{0.45\textwidth}
  \centering\includegraphics[width=\linewidth]{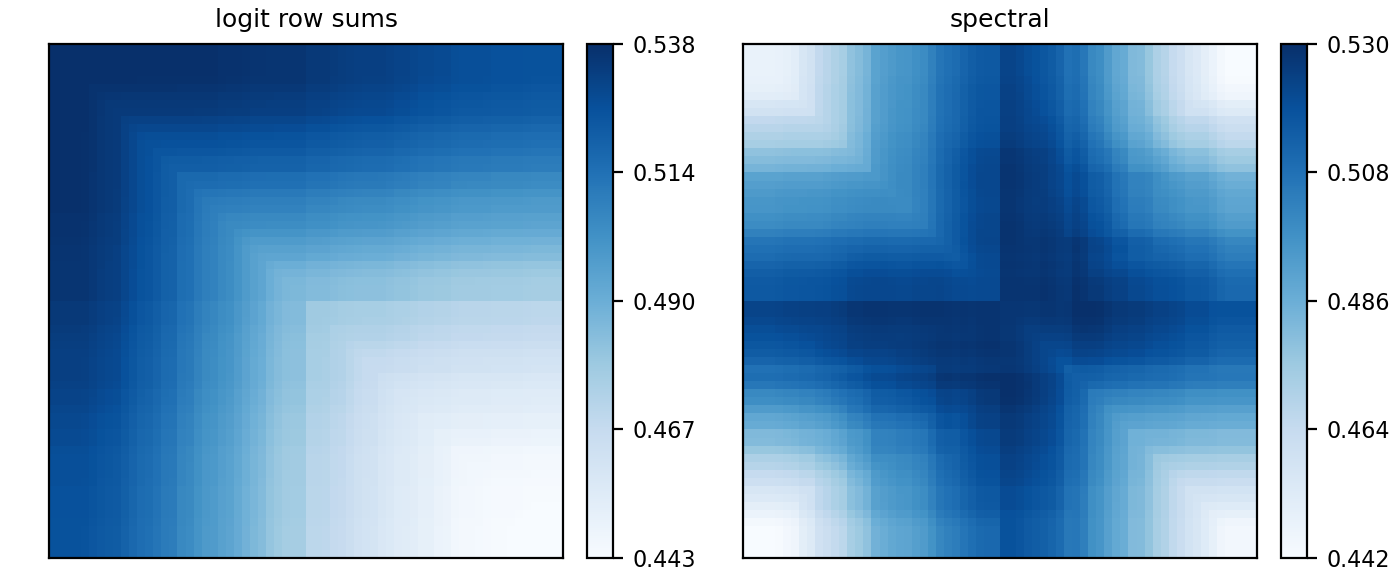}
  % \caption{\centering{\textbf{$N=100$}}}
  \caption{\centering{\textbf{$N=10$}}}
\end{subfigure}
\begin{subfigure}[t]{0.45\textwidth}
  \centering\includegraphics[width=\linewidth]{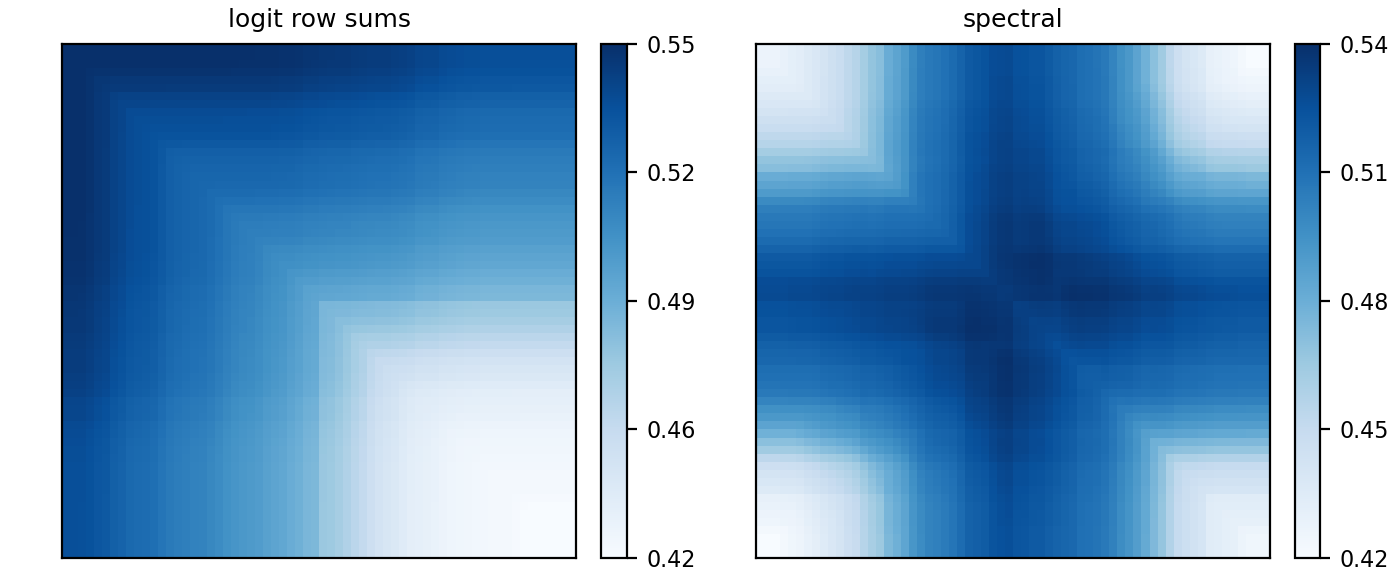}
  % \caption{\centering{\textbf{$N=200$}}}
  \caption{\centering{\textbf{$N=20$}}}
\end{subfigure}

\vspace{0.3em}

% Row 2
\begin{subfigure}[t]{0.45\textwidth}
  \centering\includegraphics[width=\linewidth]{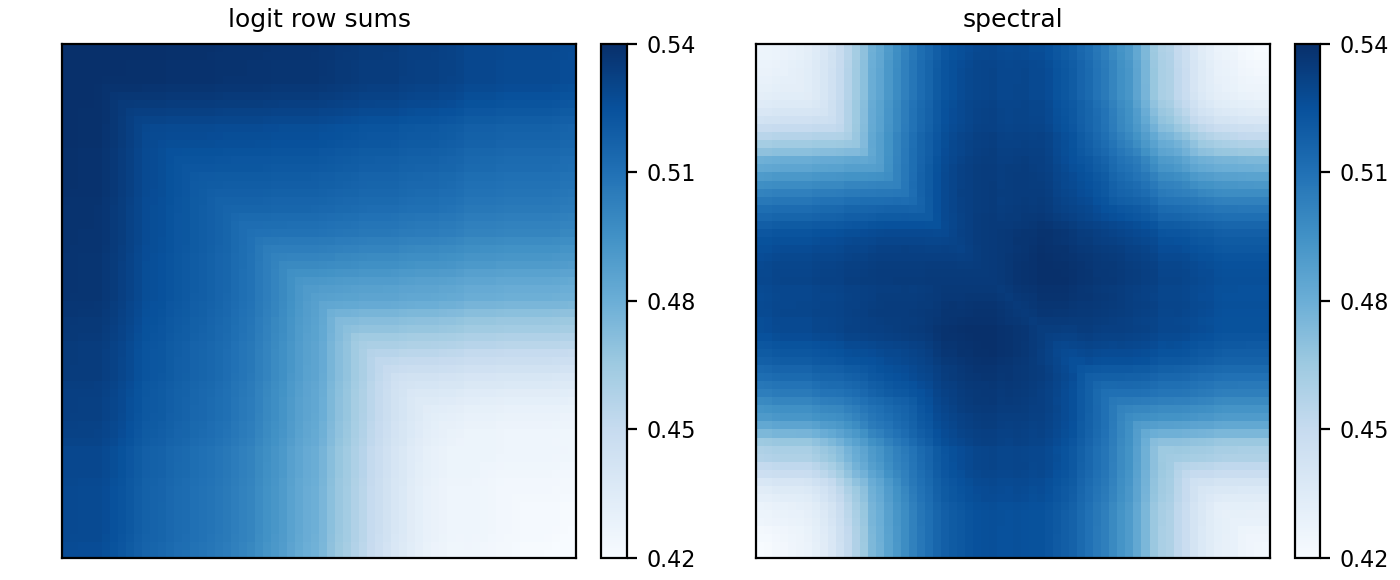}
  % \caption{\centering{\textbf{$N=400$}}}
  \caption{\centering{\textbf{$N=33$}}}
\end{subfigure}
\begin{subfigure}[t]{0.45\textwidth}
  \centering\includegraphics[width=\linewidth]{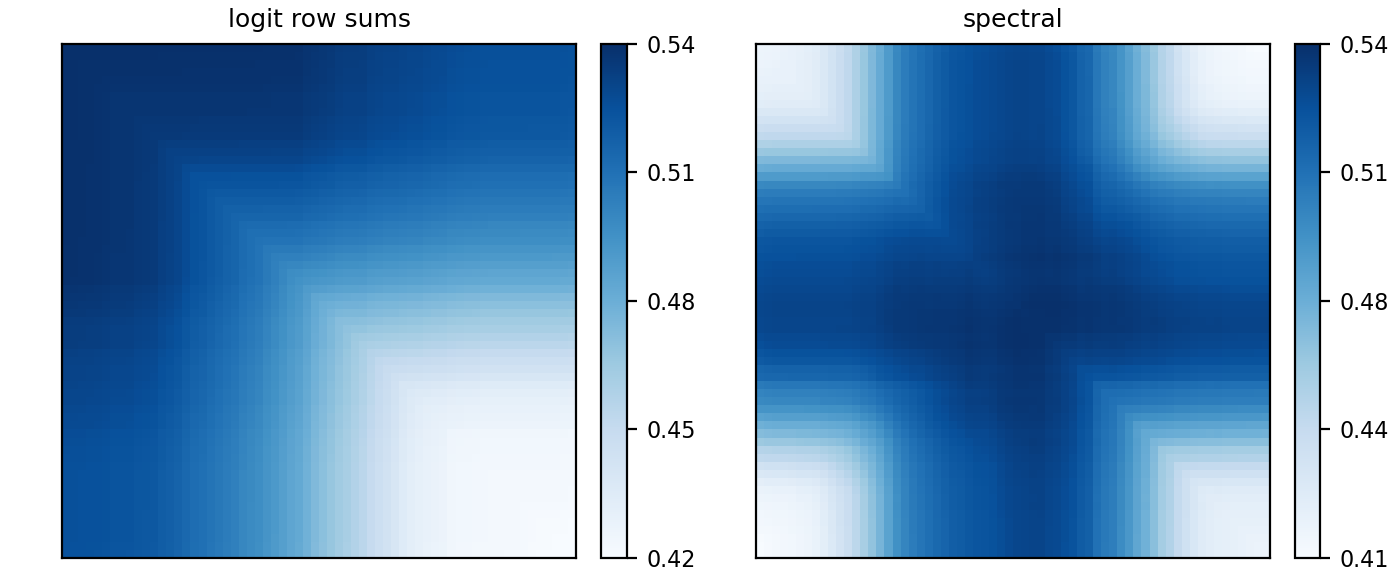}
  % \caption{\centering{\textbf{$N=600$}}}
  \caption{\centering{\textbf{$N=71$}}}
\end{subfigure}

\caption{\textbf{Sensitivity of hypothesis test to canonicalization method (PROTEINS).} On each panel: on the \textbf{left} is the estimated graphon via degree sorting, while on the \textbf{right} via Fiedler vector sorting.}
\label{fig:sensitivity_canon_graphon_proteins}
\end{figure*}

\begin{figure*}[ht!]
\centering
% Row 1
\begin{subfigure}[t]{0.45\textwidth}
  \centering\includegraphics[width=\linewidth]{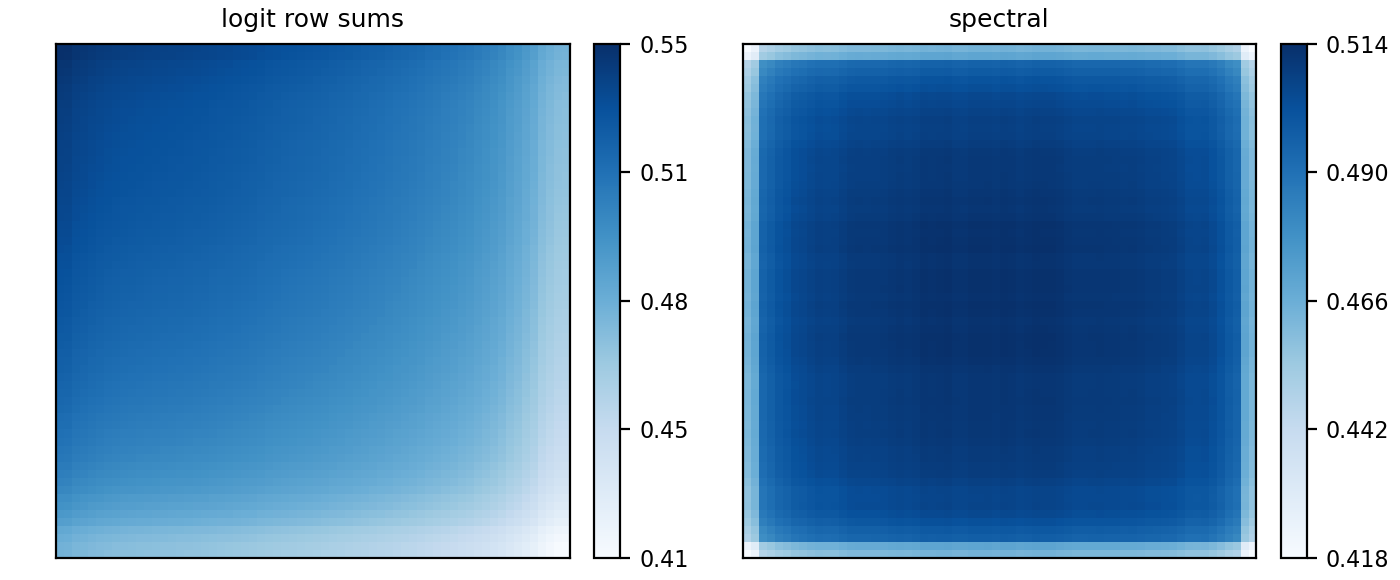}
  % \caption{\centering{\textbf{$N=64$}}}
  \caption{\centering{\textbf{$N=35$}}}
\end{subfigure}
\begin{subfigure}[t]{0.45\textwidth}
  \centering\includegraphics[width=\linewidth]{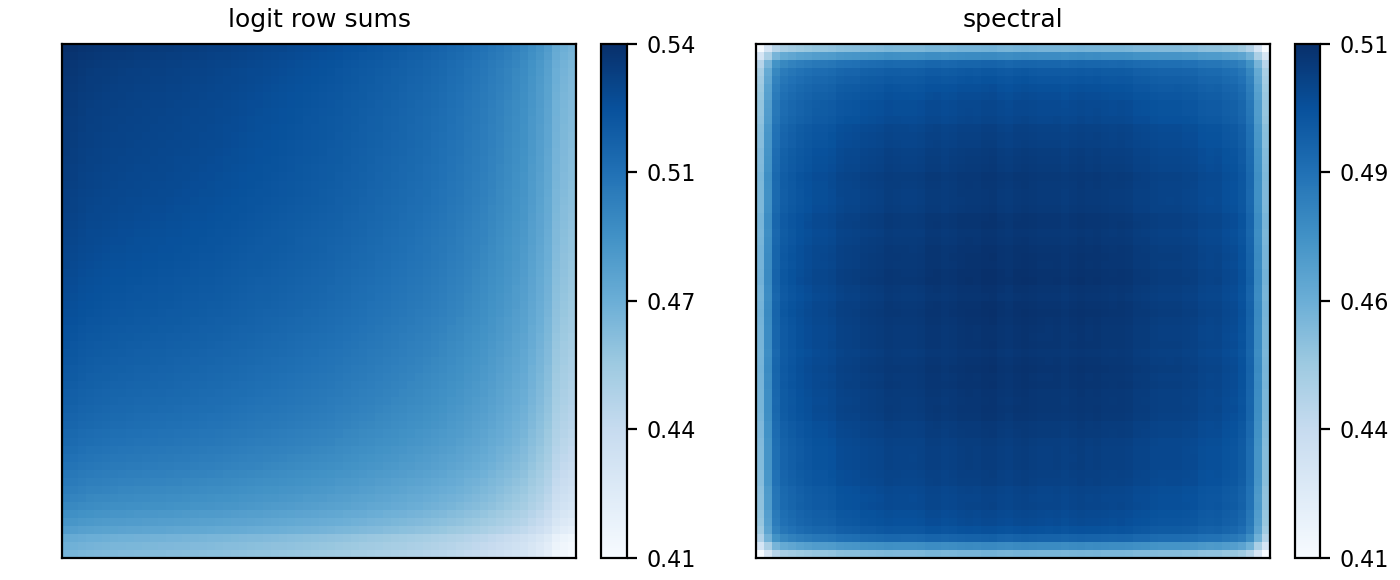}
  % \caption{\centering{\textbf{$N=128$}}}
  \caption{\centering{\textbf{$N=45$}}}
\end{subfigure}

\vspace{0.3em}

% Row 2
\begin{subfigure}[t]{0.45\textwidth}
  \centering\includegraphics[width=\linewidth]{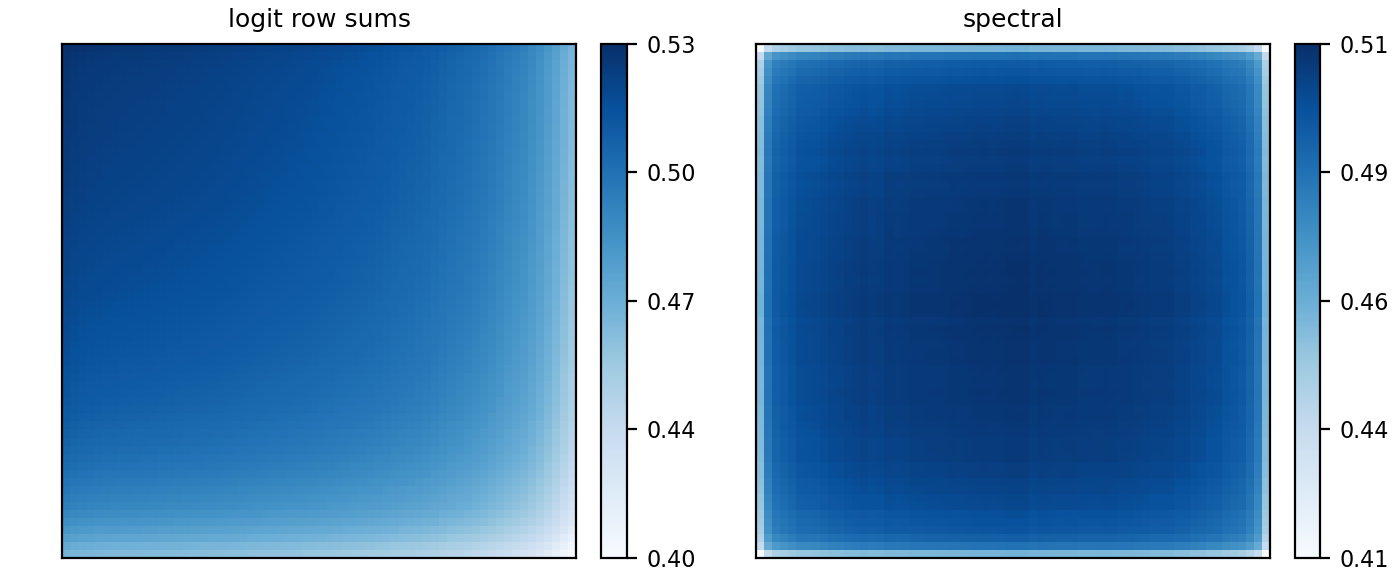}
  % \caption{\centering{\textbf{$N=256$}}}
  \caption{\centering{\textbf{$N=63$}}}
\end{subfigure}
\begin{subfigure}[t]{0.45\textwidth}
  \centering\includegraphics[width=\linewidth]{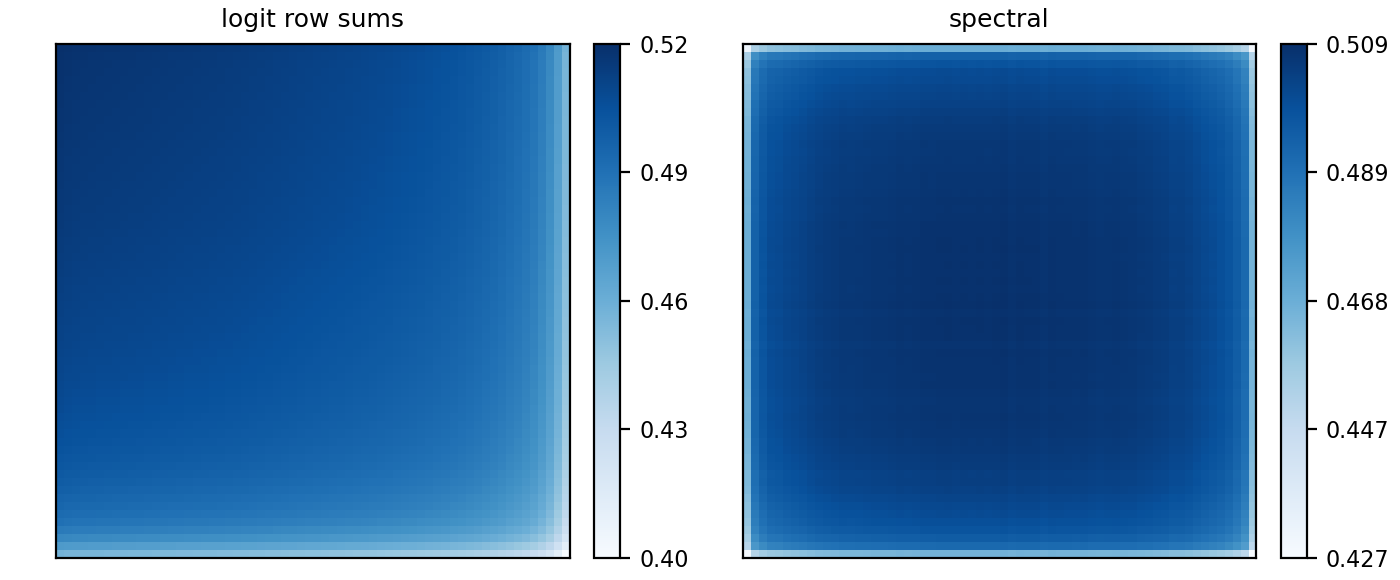}
  % \caption{\centering{\textbf{$N=498$}}}
  \caption{\centering{\textbf{$N=111$}}}
\end{subfigure}

\caption{\textbf{Sensitivity of hypothesis test to canonicalization method (COLLAB).} On each panel: on the \textbf{left} is the estimated graphon via degree sorting, while on the \textbf{right} via Fiedler vector sorting.}
\label{fig:sensitivity_canon_graphon_collab}
\end{figure*}

\begin{figure*}[ht!]
\centering
% Row 1
\begin{subfigure}[t]{0.45\textwidth}
  \centering\includegraphics[width=\linewidth]{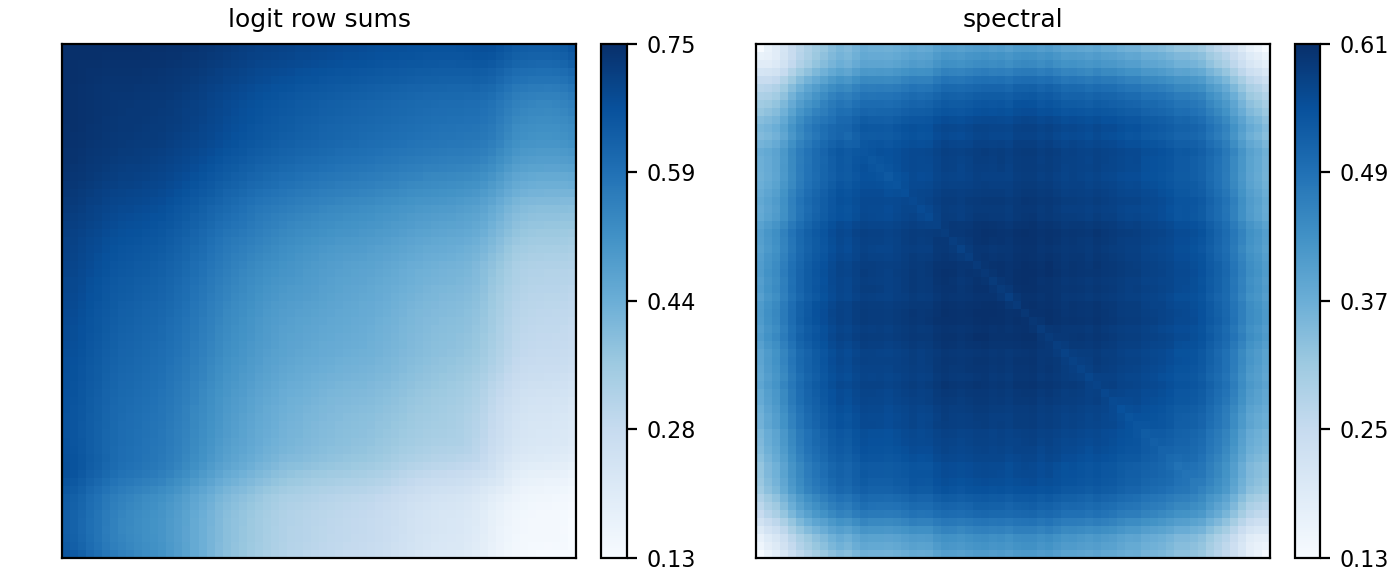}
  % \caption{\centering{\textbf{$N=64$}}}
  \caption{\centering{\textbf{$N=54$}}}
\end{subfigure}
\begin{subfigure}[t]{0.45\textwidth}
  \centering\includegraphics[width=\linewidth]{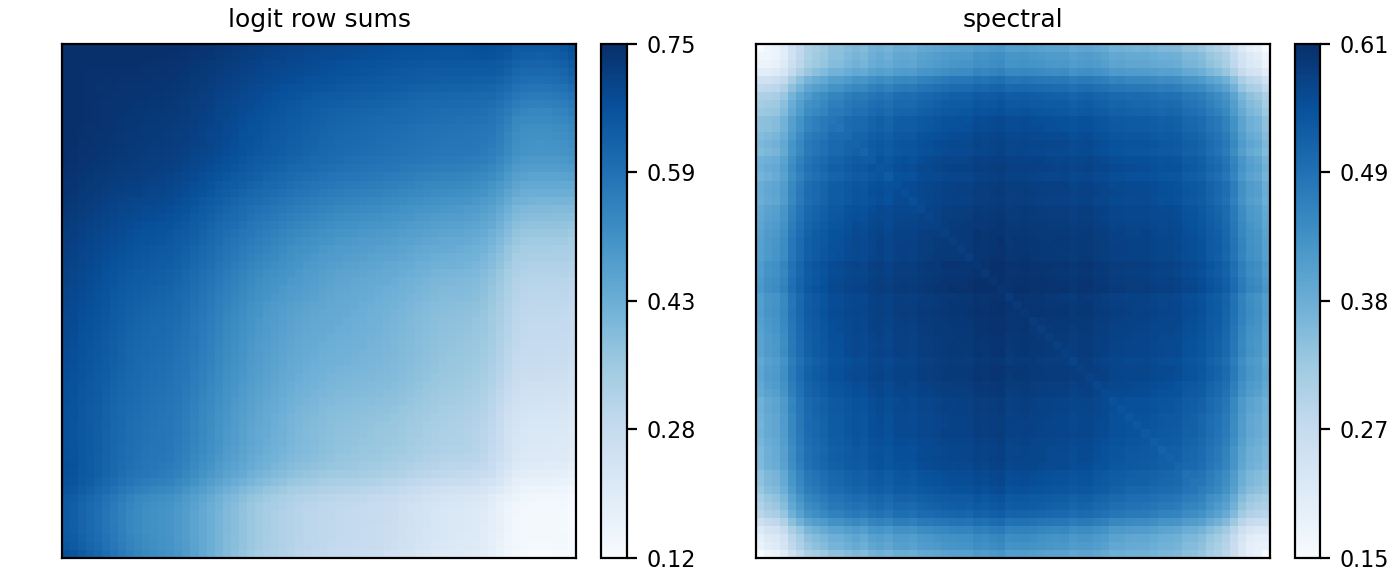}
  % \caption{\centering{\textbf{$N=128$}}}
  \caption{\centering{\textbf{$N=110$}}}
\end{subfigure}

\vspace{0.3em}

% Row 2
\begin{subfigure}[t]{0.45\textwidth}
  \centering\includegraphics[width=\linewidth]{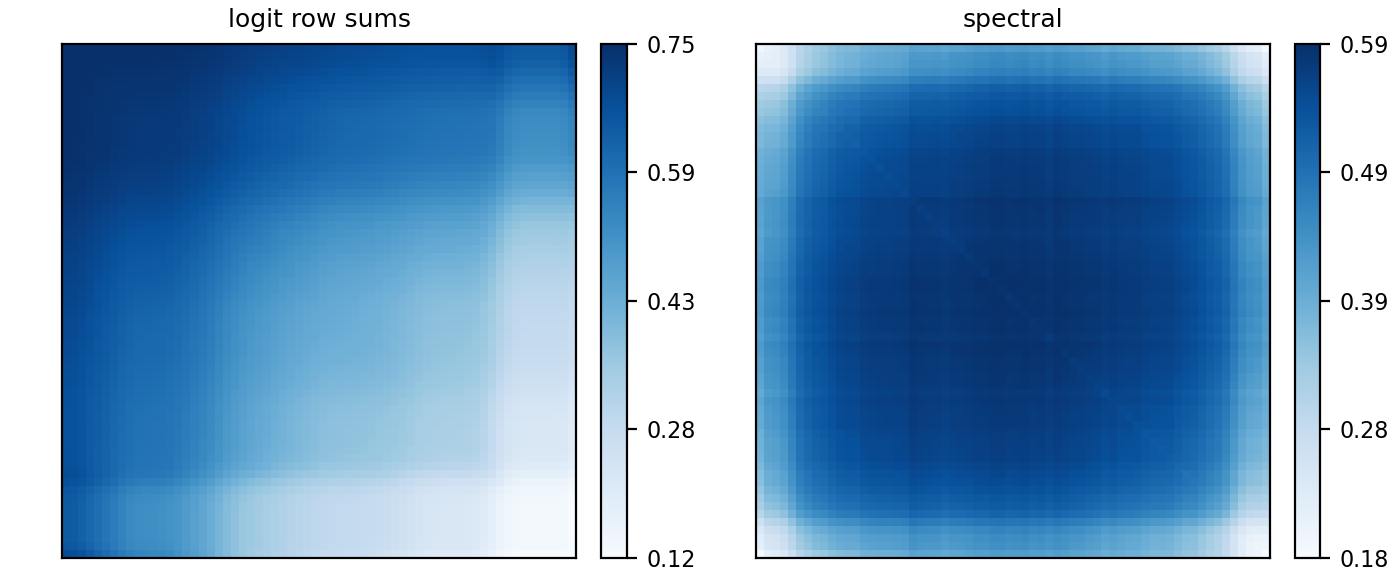}
  % \caption{\centering{\textbf{$N=256$}}}
  \caption{\centering{\textbf{$N=165$}}}
\end{subfigure}
\begin{subfigure}[t]{0.45\textwidth}
  \centering\includegraphics[width=\linewidth]{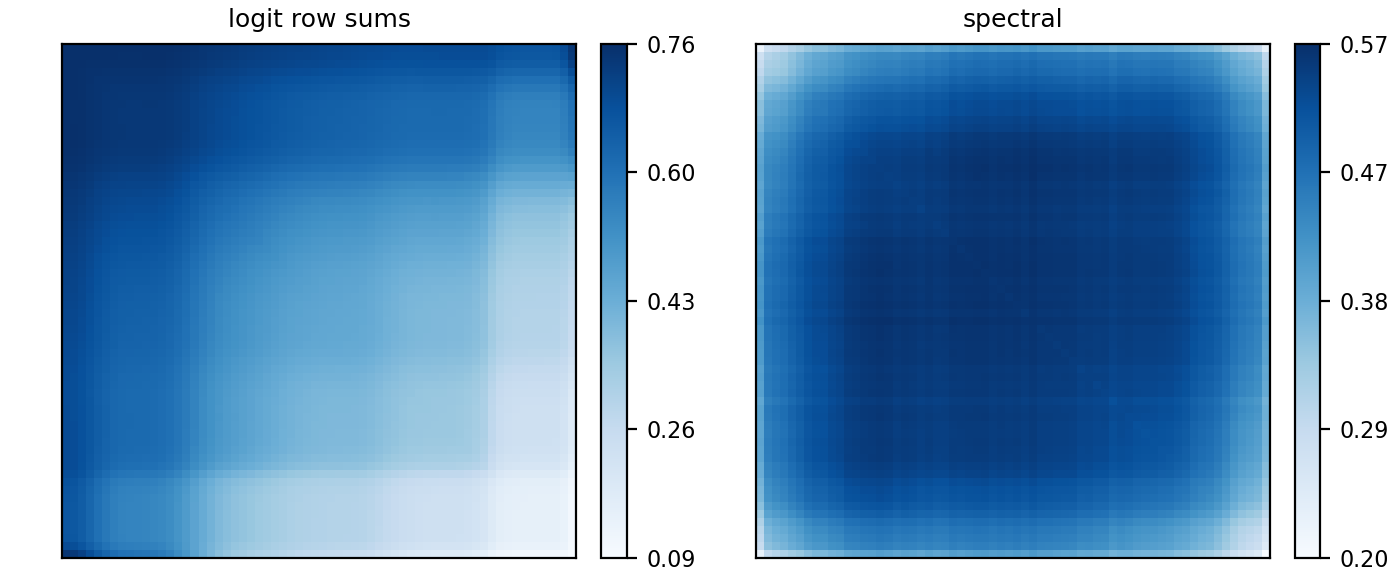}
  % \caption{\centering{\textbf{$N=400$}}}
  \caption{\centering{\textbf{$N=255$}}}
\end{subfigure}

\caption{\textbf{Sensitivity of hypothesis test to canonicalization method (LRGBPeptides).} On each panel: on the \textbf{left} is the estimated graphon via degree sorting, while on the \textbf{right} via Fiedler vector sorting.}
\label{fig:sensitivity_canon_graphon_lrbbpep}
\end{figure*}

\clearpage
% Tables 4-6 and the masked-cSBM figure, moved here from Appendix G so they appear after the other additional experiments.
\begin{table}[ht]
\centering
\small
% \caption{Runtime of the dense product $AV$ and of the factorized product $S(\hat\Theta(S^\top V))$ for $d_v=32$.}
\caption{Runtime of the dense product $\tilde PV$ and of the factorized product $Z(\tilde\Theta(Z^\top V))$ for $d_v=32$.}
\label{tab:factorized}
\begin{tabular}{rccc}
\toprule
% $n$ & Dense $AV$ & Factorized & Speed-up \\
$n$ & Dense $\tilde PV$ & Factorized & Speed-up \\
\midrule
512 & 13.9 ms & 0.6 ms & $22.7\times$ \\
2048 & 8.9 ms & 1.5 ms & $5.8\times$ \\
8192 & 19.2 ms & 4.7 ms & $4.1\times$ \\
\bottomrule
\end{tabular}
\end{table}

\begin{table}[ht]
\centering
\small
% \caption{Ratio of cut-norm variances after and before replacing node features with i.i.d.\ noise, over the four sizes of each sweep in increasing order.}
\caption{Ratio of cut-norm variances after and before replacing node features with i.i.d.\ noise, over the four sizes of each sweep in increasing order.}
\label{tab:featabl}
\begin{tabular}{llcccc}
\toprule
Dataset & Features & \multicolumn{4}{c}{Variance ratio (randomized / original)} \\
\midrule
SmoothW & Random-walk PE & 380 & 931 & $1.9\times10^{4}$ & $2.7\times10^{5}$ \\
SharpW & Random-walk PE & $1.8\times10^{8}$ & $2.8\times10^{9}$ & $1.2\times10^{10}$ & $1.4\times10^{9}$ \\
PROTEINS & Node attributes & 1.00 & 1.01 & 1.00 & 1.00 \\
\bottomrule
\end{tabular}
\end{table}

\begin{table}[ht]
\centering
\small
% \caption{Relative symmetrization residual $\|P-A\|_F/\|P\|_F$, as a range over the sizes of each sweep.}
\caption{Relative symmetrization residual $\|P-\bar P\|_F/\|P\|_F$, as a range over the sizes of each sweep.}
\label{tab:asym}
\begin{tabular}{lclc}
\toprule
Dataset & Range & Dataset & Range \\
\midrule
BFS & 0.006--0.070 & MUTAG & 0.064--0.247 \\
NCI1 & 0.009--0.114 & REDDIT-MULTI-5K & 0.047--0.217 \\
IMDB-MULTI & 0.041--0.074 & PROTEINS & 0.113--0.260 \\
ModelNet10 & 0.051--0.333 & COLLAB & 0.082--0.284 \\
\bottomrule
\end{tabular}
\end{table}

\begin{figure}[!b]
\centering
\begin{subfigure}[t]{0.45\linewidth}\centering\includegraphics[width=\linewidth]{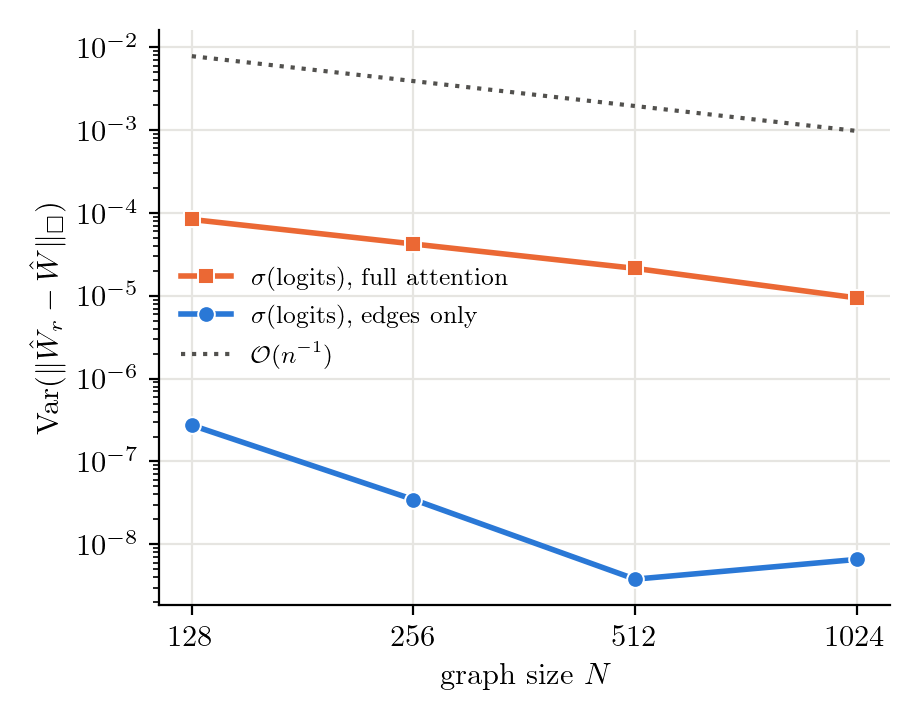}\caption{Empirical variance}\end{subfigure}
\hfill
\begin{subfigure}[t]{0.45\linewidth}\centering\includegraphics[width=\linewidth]{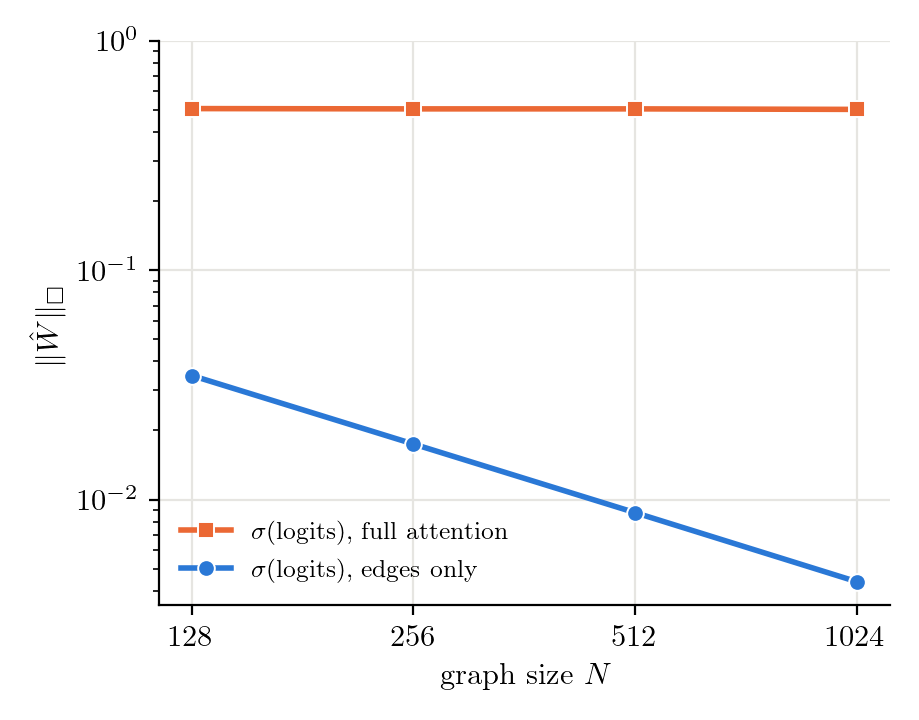}\caption{Kernel cut norm}\end{subfigure}
\caption{Attention restricted to input edges on NoisyCSBM (GPS, single-head), against full attention. \textit{Left}: empirical variance and the scaling proxy $\ccalO(n^{-1})$. \textit{Right}: cut norm $\|\hat W\|_\square$ of the estimated kernel, which vanishes when attention is masked.}
\label{fig:masked_csbm}
\end{figure}

\end{document}